\documentclass[preprint,12pt]{elsarticle}
\journal{Artificial Intelligence}

\usepackage[utf8]{inputenc}

\usepackage{fancyhdr}
\usepackage{graphicx}
\usepackage{amsmath}
\usepackage{amssymb}
\usepackage{amsthm}
\usepackage{xcolor}
\usepackage{todonotes}

\usepackage{tikz}
\usetikzlibrary{arrows.meta, arrows}

\newcommand{\mt}[1]{\textcolor{red}{MT: #1}}
\newcommand{\ks}[1]{\textcolor{violet}{KS: #1}}

\theoremstyle{definition}
\newtheorem{definition}{Definition}
\newtheorem{example}{Example}
\newtheorem{proposition}{Proposition}

\newtheorem{lemma}[proposition]{Lemma}

\newcommand{\co}{\ensuremath{\mathsf{co}}}
\newcommand{\gr}{\ensuremath{\mathsf{gr}}}
\newcommand{\pr}{\ensuremath{\mathsf{pr}}}
\newcommand{\st}{\ensuremath{\mathsf{st}}}
\newcommand{\cf}{\ensuremath{\mathsf{cf}}}
\newcommand{\ad}{\ensuremath{\mathsf{ad}}}
\newcommand{\sst}{\ensuremath{\mathsf{sst}}}

\newcommand{\maxpl}{\ensuremath{\mathtt{max}}}
\newcommand{\LD}{\ensuremath{\mathsf{LD}}}

\newcommand{\CF}{\ensuremath{\mathsf{Conflicts}}}
\newcommand{\DN}{\ensuremath{\mathsf{Condef}}}		
\newcommand{\UD}{\ensuremath{\mathsf{Undef}}}	
\newcommand{\UA}{\ensuremath{\mathsf{Unatt}}}

\newcommand{\Min}{\ensuremath{\mathsf{mini}}}
\newcommand{\Max}{\ensuremath{\mathsf{maxi}}}

\newcommand{\nonatt}{\ensuremath{\mathsf{nonatt}}}
\newcommand{\strdef}{\ensuremath{\mathsf{strdef}}}

\newcommand{\rAd}{\ensuremath{\mathsf{r\text{-}ad}}}
\newcommand{\rCo}{\ensuremath{\mathsf{r\text{-}co}}}
\newcommand{\rGr}{\ensuremath{\mathsf{r\text{-}gr}}}
\newcommand{\rPr}{\ensuremath{\mathsf{r\text{-}pr}}}
\newcommand{\rSst}{\ensuremath{\mathsf{r\text{-}sst}}}
\newcommand{\rCoPr}{\ensuremath{\mathsf{r\text{-}co\text{-}pr}}}

\newcommand{\cope}{\ensuremath{\mathsf{cope}}}
\newcommand{\lex}{\ensuremath{\mathsf{lex}}}
\newcommand{\OBE}{\ensuremath{\mathsf{OBE}}}

\newcommand{\INMAX}{In-Maximality}

\begin{document}
\begin{frontmatter}

\title{Extension-ranking Semantics for Abstract Argumentation \\ Preprint}

\author[ha]{Kenneth Skiba}
\author[tr]{Tjitze Rienstra}
\author[ha]{Matthias Thimm}
\author[jh,sa]{Jesse Heyninck}
\author[gki]{Gabriele Kern-Isberner}

 \address[ha]{Artificial Intelligence Group, University of Hagen, Hagen, Germany}
\address[tr]{Department of Advanced Computing Sciences, Maastricht University, The Netherlands}
\address[jh]{Department of Computer Science, Open Universiteit, The Netherlands}
\address[sa]{University of Cape Town, South Africa}
\address[gki]{Department of Computer Science, TU Dortmund, Dortmund, Germany}

\begin{abstract}
In this paper, we present a general framework for ranking sets of arguments in abstract argumentation based on their plausibility of acceptance.
We present a generalisation of Dung's extension semantics as \emph{extension-ranking semantics}, which induce a preorder over the power set of all arguments, allowing us to state that one set is ``closer'' to being acceptable than another. To evaluate the extension-ranking semantics, we introduce a number of principles that a well-behaved extension-ranking semantics should satisfy.
We consider several simple base relations, each of which models a single central aspect of argumentative reasoning. The combination of these base relations provides us with a family of extension-ranking semantics.
We also adapt a number of approaches from the literature for ranking extensions to be usable in the context of extension-ranking semantics, and evaluate their behaviour. 
\end{abstract}

%\maketitle

\begin{keyword}
Abstract Argumentation 
\sep Ranking Sets of Objects
\sep Extension-ranking semantics
\end{keyword}

\end{frontmatter}

%\listoftodos

\section{Introduction}
Formal argumentation \cite{DBLP:journals/aim/AtkinsonBGHPRST17} is concerned with models of rational decision-making based on representations of arguments and their relations. A particularly important approach is that of abstract argumentation frameworks (AF) \cite{DBLP:journals/ai/Dung95}, which represent argumentative scenarios as directed graphs. Here, \emph{arguments} are identified by vertices, and an \emph{attack} from one argument to another is represented as a directed edge.
Reasoning is usually performed in abstract argumentation by considering \emph{extensions}, i.\,e., sets of arguments that are jointly acceptable, given some formal account of ``acceptability''. One way to determine the acceptability of a set is to look at the internal conflicts of this set. A set of arguments is considered ``acceptable'' if it is internally consistent, so that no two arguments within a set attack each other. Another concept to establish the acceptability of a set of arguments is \emph{admissibility}, which states that a set is only acceptable if it can defend itself against any threat. Acceptability concepts such as admissibility give rise to \emph{extension semantics}.
Extension semantics distinguishes between ``acceptable'' sets of arguments and not ``acceptable'' sets of arguments.%\tr{Two points: (1) whereas extensions are indeed called acceptable, it is not the case that non-extensions are referred to as ``rejected'' sets. (2) in this context it might make sense to mention skeptical/credulous acceptance, which are the two distinctions that you can make if you use a traditional semantics.}

The binary classification of extension semantics can be too limiting, especially if we want to compare sets of arguments based on their acceptability; extension semantics consider a set to be either fully accepted or not, there is no in-between. Two sets with the same classification cannot be distinguished.
Let us illustrate this behaviour in the following example (for the formal definitions see Section~\ref{sec:preliminaries}). 
\begin{example}\label{ex:intro_legal}
    Suppose we have a murder case in which Alex is the prime suspect. However, Alex has an alibi for the time of the murder, but a witness claims to have seen Alex at the scene. Later it turns out that the witness was unreliable and their testimony should be questioned. 
    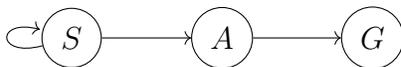
\begin{figure}
    \centering
 \scalebox{1}{
\begin{tikzpicture}

\node (g) at (4,0) [circle, draw,minimum size= 0.65cm] {$G$};
\node (a) at (2,0) [circle, draw,minimum size= 0.65cm] {$A$};
\node (s) at (0,0) [circle, draw,minimum size= 0.65cm] {$S$};

\path[->] (a) edge  (g);
\path[->] (s) edge  (a);
\path[->] (s) edge  [loop left] node {} ();

\end{tikzpicture}}
   \caption{Abstract argumentation framework $F_1$ from Example \ref{ex:intro_legal}.}
    \label{tikz:intro_legal}
\end{figure}

    We can model this case as an argumentation framework $F_1= (\{G, A, S\}, \\\{(A,G), (S,A), (S,S)\})$, where $G$ stands for \emph{guilty}, $A$ for \emph{has alibi}, and $S$ for \emph{was seen at the crime scene}. $G$ is attacked by $A$, and $A$ is attacked by $S$, but $S$ is unreliable and therefore attacks itself (as depicted in Figure \ref{tikz:intro_legal}).
    The judge must decide which position is the most plausible one. Alex is either \emph{guilty} or not, so $G$ is either part of an acceptable set or not. When the judge compares the two positions of $\{A\}$ and $\{G,S\}$, they conclude that neither of these sets is acceptable with respect to admissibility. However, while the set $\{A\}$ does not defend itself against the attack of $S$, at least this set has no internal conflicts, while $\{G,S\}$ has an internal conflict. So $\{A\}$ is ``closer'' to being acceptable than $\{G,S\}$. 
\end{example}

The example above shows that in the area of abstract argumentation, a notion is needed to compare two sets of arguments on the basis of their plausibility of acceptance. 

Approaches to determining the relative degree of plausibility of acceptance of an individual argument have already been developed. These approaches are called \emph{argument-ranking}, \emph{graded} or \emph{gradual} semantics and associate each argumentation framework with an ordering of arguments according to their relative degrees of justification or plausibility~\cite{DBLP:conf/sum/AmgoudB13a,DBLP:conf/kr/AmgoudBDV16,DBLP:conf/aaai/BonzonDKM16,DBLP:conf/comma/YunVCB18,DBLP:journals/ai/GrossiM19}.  
These semantics provide a finer-grained interpretation of the acceptability of each argument than the binary classification of an extension semantics. However, they cannot be used to directly compare two sets of arguments based on their plausibility of being accepted, since with these approaches it is not possible to represent a set of arguments as being ``jointly'' acceptable. In particular, two arguments with a high degree of plausibility of acceptance may not be allowed to be accepted at the same time because they are in conflict with each other. Therefore, the plausibility of acceptance of each argument cannot be used directly to infer the plausibility of acceptance for a set of arguments.

We can use comparisons of sets of arguments to rank these sets on the basis of their plausibility of acceptance. Later, we will introduce the notion of \emph{extension rankings} for this kind of ordering of sets of arguments, but first let us further motivate the use of such a ranking.
\begin{example}\label{ex:travel}
   Alice and Bob are planning their holiday together for next year. They start in Germany. Alice suggests going to Berlin, but Bob visited Berlin last year on a business trip and does not want to go back. Bob suggests flying to Argentina, but as Alice is afraid of flying and flights to Argentina are expensive, she does not like the idea. Although Alice and Bob cannot agree on a destination, they are sure that they want to travel next year. Berlin and Argentina are their only options and they do not have the time or the money to visit both places.
    
    To help Alice and Bob with their decision, we construct their problem as an argumentation framework. The two destinations \textbf{A}rgentina and \textbf{B}erlin are arguments that are attacked by the reasons for their rejection, i.e. the arguments \textbf{E}xpensive and \textbf{F}ear  are attacking \textbf{A}rgentina and \textbf{V}isited is attacking \textbf{B}erlin (as depicted as $F_2$ in Figure \ref{tikz:travel}).
    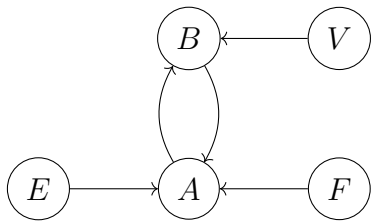
\begin{figure}
    \centering
 
 \scalebox{1}{
\begin{tikzpicture}

\node (A) at (0,0) [circle, draw,minimum size= 0.65cm] {$A$};
\node (B) at (0,2) [circle, draw,minimum size= 0.65cm] {$B$};
\node (F) at (2,0) [circle, draw,minimum size= 0.65cm] {$F$};
\node (E) at (-2,0) [circle, draw,minimum size= 0.65cm] {$E$};
\node (V) at (2,2) [circle, draw,minimum size= 0.65cm] {$V$};

\path[->, bend left] (A) edge (B);
\path[->, bend left] (B) edge (A); 

\path[->] (F) edge  (A);
\path[->] (E) edge  (A);
\path[->] (V) edge  (B);
\end{tikzpicture}}
   \caption{Argumentation framework $F_2$ for Example \ref{ex:travel}}
    \label{tikz:travel}
\end{figure}
    
    For any extension semantics defined by Dung \cite{DBLP:journals/ai/Dung95} arguments, \textbf{A} and \textbf{B} are never jointly part of an acceptable set. However, Alice and Bob have the constraint that they want to travel to either \textbf{A} or \textbf{B}, so the question is which is the most plausible set that contains either \textbf{A} or \textbf{B}? For any set containing \textbf{A}, the attacks from \textbf{E} and \textbf{F} have to be disregarded, while for a set containing \textbf{B}, only the attack from \textbf{V} has to be disregarded, so at first glance \textbf{B} should be easier to accept. If we start with $\{\textbf{B}\}$, then arguments \textbf{V}, \textbf{E} and \textbf{F} are not attacked, so there is no reason to reject them. However, if we add \textbf{V}, the resulting set has an internal conflict. Thus, \textbf{V} should not be added. For \textbf{E} and \textbf{F} we find no reason to reject them, so these arguments should be added, resulting in $\{\textbf{B,E,F}\}$ as our final set. This set is the most plausible set containing either \textbf{A} or \textbf{B}, since no additional argument can be added and this set is more plausible to be accepted than any of its subsets containing \textbf{B}. 
\end{example}
In the example above, we were able to find a conflict-free set of arguments, however
sometimes constraints or requirements are non-sensical and we have to accept conflicting arguments to satisfy them.
\begin{example}\label{ex:burger}
    Suppose a burger restaurant receives the following order: \emph{``I want a vegan burger with meat, but please no pork or beef, and please add hot sauce, but do not make the burger spicy.''} Note that the restaurant only serves pork, beef or vegan patties. We can write these requests as an AF 
    \begin{align*}
        F_3= &(\{V, M, \Bar{P}, \Bar{B}, S, \Bar{S}\}, \{(V,M), (M,V), (M, \Bar{P}), (\Bar{P}, M), (M, \Bar{B}), (\Bar{B}, M),\\ &(H,\Bar{S}), (\Bar{S}, H)\})    
    \end{align*}
    where $V$ stands for \emph{vegan patty}, $M$ for \emph{extra meat}, $\Bar{P}$ for \emph{no pork}, $\Bar{B}$ for \emph{no beef}, $H$ for \emph{hot sauce}, and $\Bar{S}$ for \emph{not spicy}, as depicted in Figure \ref{tikz:burger}. 
    \begin{figure}
    \centering
 
 \scalebox{1}{
\begin{tikzpicture}

\node (V) at (0,0) [circle, draw,minimum size= 0.65cm] {$V$};
\node (EM) at (0,2) [circle, draw,minimum size= 0.65cm] {$M$};
\node (P) at (-1.2,4)[circle, draw,minimum size= 0.65cm] {$\Bar{P}$};
\node (B) at (1.2,4) [circle, draw,minimum size= 0.65cm] {$\Bar{B}$};
\node (HS) at (2,0) [circle, draw,minimum size= 0.65cm] {$H$};
\node (S) at (2,2) [circle, draw,minimum size= 0.65cm] {$\Bar{S}$};

\path[->, bend left] (V) edge  (EM);
\path[->, bend left] (EM) edge (V); 

\path[->, bend left] (P) edge  (EM);
\path[->, bend left] (EM) edge (P); 

\path[->, bend left] (B) edge  (EM);
\path[->, bend left] (EM) edge (B); 

\path[->, bend left] (S) edge  (HS);
\path[->, bend left] (HS) edge (S); 
\end{tikzpicture}}
   \caption{AF $F_3$ for Example \ref{ex:burger}}
    \label{tikz:burger}
\end{figure}
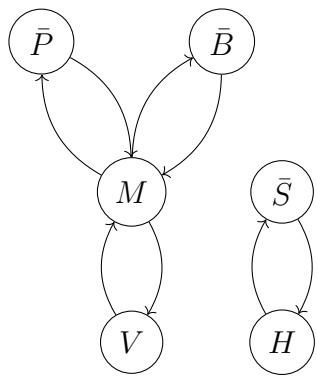
We can see that all the requirements from the order above cannot be met at the same time without introducing conflicts. If the restaurant adds hot sauce to the burger, it will be spicy, and if a meat patty is used, the burger will no longer be vegan. In addition, the restaurant only offers two types of meat patties (pork and beef), but the customer dislikes both of these choices. So the only sensible option would be to deny the whole order; but, let us assume the restaurant does not want to do this. Then some constraints have to be broken. Since adding a meat patty would break two constraints (vegan and either no pork or no beef), the restaurant decides to use a vegan patty for the burger. For the hot sauce, the restaurant can choose either option, as they both break one constraint. To be on the safe side, the restaurant does not add any hot sauce to the burger. So the customer gets a burger with a vegan patty and no hot sauce, as this is the closed option that satisfies the costumer's constraints. 
\end{example}

The examples above show that extension rankings are useful for decision-making under constraints, where we need solutions (typically represented by extensions of an argumentation framework) that satisfy constraints that may not be satisfiable by the set of arguments under an extension semantics. In such a situation, we can select the most plausible set of arguments among those that satisfy the constraints. 
Another application is belief dynamics ~\cite{DBLP:conf/sum/BoothKRT13,DBLP:conf/kr/Coste-MarquisKMM14,DBLP:journals/ijar/DillerHLRW18}, where rankings over sets of arguments can be used as a relation of epistemic entrenchment. 
In the \emph{extension enforcement} problem \cite{DBLP:journals/flap/BaumannDMW21}, an agent modifies a given AF to accept selected extensions that were previously rejected. An extension ranking can be used to streamline this process. It may be easier to enforce the acceptance of a set of arguments that is closer to being accepted, i.\,e., ranked higher, than any weaker ranked set. A ranking over sets of arguments helps us to estimate an upper bound on the number of changes that need to be made to accept a set of arguments. 
Other problems where the ability to compare sets of arguments is useful are belief merging and judgement aggregation \cite{DBLP:conf/kr/DelobelleHKMRW16,DBLP:journals/aamas/CaminadaP11}.
In judgement aggregation, the goal is to find a common judgement or decision among different opinions. Some of these opinions may be incompatible. %, so the area of judgement aggregation try to find the most preferred decision for all agents \mt{The ``area'' does not try to find this, but the approaches developed in that area; also, the latter part of this sentence is the same as the sentence before}.
By modelling these decisions using an argumentation framework and ranking over sets of arguments, this problem can be solved.
Furthermore, in belief merging, incompatible knowledge bases are combined into a single new knowledge base. An extension ranking can be helpful, since such a ranking gives us insight into the most plausible sets from which we to start modifying the knowledge base to resolve the incompatibility. 

We contribute to the research question of ranking sets of arguments by generalising extension semantics, which allows us to state whether one set is, e.\,g., ``more admissible'' than another. 
More precisely, the main contributions of this paper are as follows.
\begin{itemize}
    \item We present \emph{extension-ranking semantics} $\tau$, which maps an argumentation framework $AF$ to a preorder $\sqsupseteq_{AF}^\tau$ over all sets of arguments. Intuitively, for two sets $E$ and $E'$ of $AF$ with $E \sqsupseteq^\tau_{AF} E'$ we say that $E$ is at least as plausible to be accepted as $E'$ wrt.\ $\tau$. We will show later that any extension semantics can be defined in terms of an extension-ranking semantics, however, the general framework of extension-ranking semantics is more expressive than that. For instance, an extension-ranking semantics can be used to compare sets that are not, e.\,g., admissible. Moreover, unlike argument-ranking semantics, we properly generalise extension semantics, in the sense that the set of most plausible sets is the set of extensions under a traditional semantics.
    \item We define a set of principles for evaluating and developing new approaches to extension-ranking semantics. These principles are all modelling aspects that an intuitively well-behaved extension-ranking semantics should satisfy. 
    \item We develop very general approaches to rank sets of arguments by splitting central aspects of argumentative reasoning into conceptually simple base relations and aggregating these relations to establish the plausibility of acceptance of each set, in turn defining generalisations of the classical Dung \emph{extension semantics}.
    \item Additionally, we discuss approaches to rank sets of arguments based on the relative degree of plausibility of acceptance of the contained arguments. We determine the plausibility degree of each argument and aggregate these degrees to estimate the plausibility of a set being accepted.
   % \item We investigate the computational complexity of the decision problem stating that a set $E$ is at least as plausible to be accepted as another set $E'$ with respect to an extension-ranking semantics. It turns out that this problem is easy for a good number of semantics, while for other semantics we have to construct the whole ranking, which leads to an exponential amount of space needed.
\end{itemize}

The remainder of this paper is structured as follows. In Section~\ref{sec:preliminaries} necessary background information about abstract argumentation is recalled. Extension-ranking semantics are presented and put into context with extension semantics in Section~\ref{sec:extension-ranking}. In Section~\ref{sec:principles} principles for extension-ranking semantics are defined.
Simple base relations modelling central notions of argumentative reasoning are introduced in Section~\ref{sec:Base functions} and these base relations are combined to extension-ranking semantics in Section~\ref{sec:combination}. 
%The analysis of the computational complexity of the underlying problem can be found in Section~\ref{sec:complexity}.
Additional extension-ranking semantics are presented in Section~\ref{sec:additional combinations}.
 Related work will be discussed in Section~\ref{sec:related work} and Section~\ref{sec:conclusion} concludes this work. 

This paper is a significantly extended version of a paper published in the proceedings of the Thirtieth International Joint Conference on Artificial Intelligence, {IJCAI} 2021 \cite{DBLP:conf/ijcai/SkibaRTHK21}. This version contains, in addition to the proofs, additional extension-ranking semantics and their analyses discussed in Sections \ref{subsubsec:cardinality_ext_ranking}, \ref{sec:voting}, and  \ref{sec:additional combinations}.

%-------------------------------------------------------------------------------------------------------
\section{Preliminaries} \label{sec:preliminaries}
In this section we recall all necessary notations in the area of formal argumentation focusing on abstract argumentation frameworks and the two reasoning approaches extension semantics (Subsection \ref{subsec:extension_semantics}) and argument-ranking semantics (Subsection \ref{subsec:argument-ranking_semantics}). 

\emph{Abstract argumentation frameworks} \cite{DBLP:journals/ai/Dung95} are a formalism that allows the representation of conflicts between pieces of information using arguments and attacks between arguments.
\begin{definition}
An \emph{abstract argumentation framework} ($AF$) is a directed graph $F=(A,R)$ where $A$ is a finite set of \emph{arguments} and $R$ is an \emph{attack relation} $R \subseteq A \times A$. 
\end{definition}
For an AF $F=(A,R)$,
an argument $a$ is said to \emph{attack} an argument $b$ if $(a,b) \in R $. We say that, a set $E \subseteq A$ \emph{defends} an argument $a$ if every argument $b \in A$ that attacks $a$ is attacked by some $c \in E$. For $a\in A$ we define
\begin{align*}
    a^{-}_{F} & =\{b\mid (b,a) \in R \} \quad\text{and}\quad a^{+}_{F}=\{b\mid (a, b) \in R\}.
\end{align*}
In other words, $a^{-}_{F}$ is the set of attackers of $a$ and $a^{+}_{F}$ is the set of arguments attacked by $a$. For a set of arguments $E \subseteq A$ we extend these definitions to $E^{+}_{F}$ and $E^{-}_{F}$ via $E^{+}_{F} = \bigcup_{a \in E} a^{+}_{F}$ and $E^{-}_{F} = \bigcup_{a \in E} a^{-}_{F}$, respectively. If the AF is clear from context, we sometimes omit the index. Additionally, we denote with $(a,E) \in R$ an attack from argument $a$ to an argument $b \in E$, extending this notation to $(E,E')$ to state that one argument $a \in E$ attacks an argument $b \in E'$. 
For two AFs $F = (A, R)$ and $F'= (A', R')$, we define $F \cup F'= (A \cup A', R \cup R')$.

\subsection{Extension Semantics}\label{subsec:extension_semantics}

To reason with AFs a number of different semantical notions have been developed, like the \emph{extension-based} or the \emph{labelling-based} approaches, for an overview see \cite{baroni2018handbook}. Both these approaches are handling sets of arguments, which can be considered jointly acceptable. The extension semantics are relying on two basic concepts: \emph{conflict-freeness} and \emph{admissibility}.
\begin{definition}\label{def:adm}
Given $F = (A, R)$, a set $E \subseteq A$ is
\begin{itemize}
    \item a \emph{conflict-free} set iff $\forall a,b \in E$, $(a,b) \not \in R$;
    \item an \emph{admissible} set iff it is \emph{conflict-free} and it defends its elements, i.\,e., $\forall a \in E$, $\forall b \in A$ s.t. $(b,a) \in R$, $\exists c \in E$ with $(c,b) \in R$. 
\end{itemize}
\end{definition}
We use $\cf(F)$ and $\ad(F)$ for denoting the sets of conflict-free and admissible sets of an argumentation framework $F$, respectively. 
%The intuition behind these principles is that a set of arguments may be accepted only if it is internally consistent (conflict-freeness) and able to defend itself against potential threats (admissibility).
A central characterisation of admissibility is given through the  \emph{Fundamental Lemma} \cite{DBLP:journals/ai/Dung95}.
\begin{lemma}[\cite{DBLP:journals/ai/Dung95}]
    Let $F=(A,R)$ be an AF and $E \in \ad(F)$. If $a \in A$ is defended by $E$, then $E \cup \{a\} \in \ad(F)$. 
\end{lemma}
Now extension semantics can be defined by making use of the \emph{characteristic function} defined as follows.
\begin{definition}\label{def:semantics}
For $F = (A,R)$, the characteristic function $\mathcal{F}_{F}: 2^A \rightarrow 2^A$ is defined for $E \subseteq A$ via
$$\mathcal{F}_{F}(E)=\{a \in A \mid E \text{ defends } a\}$$
An admissible set $E \subseteq A$ is 
 \begin{itemize}
 	\item 
 	a \emph{complete} extension (\co) iff $E= \mathcal{F}_{F}(E)$
%  it contains every argument that it defends;
 	\item 
 	a \emph{preferred} extension (\pr) iff it is a $\subseteq$-maximal admissible extension;
 	\item 
 	a \emph{grounded} extension (\gr) iff it is a $\subseteq$-minimal complete extension;
 	\item 
 	a \emph{stable} extension (\st) iff $E^+_{F} = A \setminus E$. %it attacks every argument in $\A \setminus E$.
 \end{itemize}
\end{definition}
Note that the grounded extension is unique, and for every AF there always exists an extension for every semantics, except possibly the stable semantics \cite{DBLP:journals/ai/Dung95}.
 Caminada et al.~\cite{DBLP:journals/logcom/CaminadaCD12} discussed the non-existence of stable extensions in some AFs. Their conclusion was to define a new extension semantics called \emph{semi-stable} semantics, a semantics that maximises the attacked arguments and coincides with the stable extensions if it exists.  

\begin{definition}
For $F=(A,R)$, a set $E \subseteq A$ is a \emph{semi-stable} extension if and only if
	$E$ is a complete extension where $E \cup E^{+}_F$ is maximal wrt.\ set inclusion.
\end{definition}

The sets of extensions of an argumentation framework $F$, for the six semantics from above, are
denoted (respectively) $\co(F)$, $\pr(F)$, $\gr(F)$,
$\st(F)$, and $\sst(F)$.

\begin{example} \label{ex:af_example}
Consider the abstract argumentation framework $F_4$ depicted as a directed graph in Figure \ref{tikz:af1}. $F_4$ has four complete extensions %$E_1, E_2$, $E_3$ and $E_4$ defined via
%\begin{align*}
    $E_1= \{a\}$, 
    $E_2= \{a, g\}$,
    $E_3= \{a, c, g\}$, and
    $E_4= \{a, d, g\}$.
%\end{align*}
$E_1$ is the grounded extension, while $E_3$ and $E_4$ are both preferred extensions, but only $E_3$ is a stable extension and therefore also semi-stable extension.

\begin{figure}
    \centering
 
 \scalebox{1}{
\begin{tikzpicture}

\node (a1) at (0,0) [circle, draw,minimum size= 0.65cm] {$a$};
\node (a2) at (2,0) [circle, draw,minimum size= 0.65cm] {$b$};
\node (a3) at (4,0) [circle, draw,minimum size= 0.65cm] {$c$};
\node (a4) at (6,0) [circle, draw,minimum size= 0.65cm] {$d$};
\node (a5) at (3,-2) [circle, draw,minimum size= 0.65cm] {$e$};
\node (a6) at (5,-2) [circle, draw,minimum size= 0.65cm] {$f$};
\node (a7) at (7,-2) [circle, draw,minimum size= 0.65cm] {$g$};

\path[<-] (a2) edge  (a1);
\path[<-] (a3) edge  (a2);
\path[->, bend left] (a3) edge  (a4);
\path[->, bend left] (a4) edge (a3); 

\path[->] (a3) edge (a5);
\path[->] (a3) edge (a6);
\path[->] (a4) edge (a6);
\path[->, bend left] (a6) edge  (a7);
\path[->, bend left] (a7) edge (a6); 

\path[->] (a5) edge  [loop left] node {} ();

\end{tikzpicture}}
   \caption{Abstract argumentation framework $F_4$ from Example \ref{ex:af_example}.}
    \label{tikz:af1}
\end{figure}
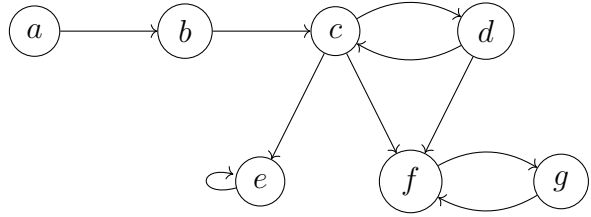
\end{example}

For more details about these semantics (and other ones), we refer the interested reader to \cite{DBLP:journals/ai/Dung95,baroni2018handbook}.
Furthermore an argument $a$ is \emph{credulously accepted} with respect to extension semantics $\sigma$ if there is a set $a \in E$ such that $E$ satisfies $\sigma$, argument $a$ is \emph{skeptically accepted} with respect to extension semantics $\sigma$ if $a$ is part of every $\sigma$-extension.   

\subsection{Argument-ranking Semantics}\label{subsec:argument-ranking_semantics}
\emph{Argument-ranking semantics} (also known as \emph{ranking-based semantics}) \cite{DBLP:conf/sum/AmgoudB13a} were introduced to focus on the strength of a single argument with respect to the other arguments. These approaches allow us to state whether an argument is ``stronger'' than another argument. 
Note that the order returned by an argument-ranking semantics is not necessarily total, i.\,e. not every pair of arguments is comparable.

\begin{definition}
An \emph{argument-ranking semantics} $\rho$ is a function, which maps an AF $F=(A,R)$ to a preorder\footnote{A preorder is a (binary) relation that is \emph{reflexive} and \emph{transitive}.} $\succeq_{F}^{\rho}$ on $A$.   
\end{definition}
Intuitively, $a \succeq_{F}^{\rho} b$ means, that $a$ is at least as strong as $b$ in $F$. 
We define the usual abbreviations as follows;
\begin{itemize}
    \item $a \succ^{\rho}_{F} b$ denotes \emph{strictly stronger}, i.\,e. $a \succeq^{\rho}_{F} b$ and $b \not\succeq^{\rho}_{F} a$;
    \item $a \simeq^{\rho}_{F} b$ denotes \emph{equally strong}, i.\,e. $a \succeq^{\rho}_{F} b$ and $b \succeq^{\rho}_{F} a$;
    \item $a \bowtie^\rho_F b$ denotes \emph{incomparability} so neither $a \succeq^{\rho}_{F} b$ nor $b \succeq^{\rho}_{F} a$.
\end{itemize}

A number of different argument-ranking semantics can be found in the literature (see \cite{DBLP:conf/aaai/BonzonDKM16} for an overview). One example of such argument-ranking semantics is the \emph{h-categoriser argument-ranking semantics} \cite{besnardh01}, which takes into account the direct attackers of an argument to calculate its strength value.
\begin{definition}[\cite{besnardh01}]
Let $F=(A,R)$. The \emph{h-categoriser function} $Cat: A \rightarrow (0,1]$ is defined as:
\begin{align*}
Cat(a)=\frac{1}{1+ \sum_{b \in a^{-}_{F}} Cat(b)} 
\end{align*}
The \emph{h-categoriser argument-ranking semantics} defines a ranking $\succeq^{Cat}_{F}$ on $A$ s.t. for $a,b \in A$, $a\succeq^{Cat}_{F} b$ iff $Cat(a) \geq Cat(b)$.
\end{definition}
Pu et al. \cite{PuLZL14} have shown, that the h-categoriser argument-ranking semantics is well defined, i.\,e., an h-categoriser function exists and is unique for every AF. 
\begin{example}\label{ex:h-cat}
    Consider $F_4$ from Example \ref{ex:af_example}. We want to rank the arguments of $F_4$ with the h-categoriser argument-ranking semantics. Since argument $a$ is unattacked we know that $Cat(a) = 1$. The remaining values can be calculated based on $Cat(a)$. The resulting ranking wrt.\ the h-categoriser argument-ranking semantics is: 
    \begin{align*}
        a \succ^{Cat}_{F_4} g \succ^{Cat}_{F_4} d \succ^{Cat}_{F_4} e \succ^{Cat}_{F_4} b \succ^{Cat}_{F_4} c \succ^{Cat}_{F_4} f
    \end{align*}
\end{example}
In the example above, we can identify some flaws of the h-categoriser argument-ranking semantics. Argument $c$ is ranked low even though it is part of the stable extension $\{a,c,g\}$, while argument $e$ is ranked high even though it attacks itself. Thus, the h-categoriser argument-ranking semantics does not necessarily capture the acceptance in terms of classical extension semantics very well. In general, argument-ranking semantics usually lack the expressiveness of a classical extension semantics, in particular because they cannot represent multiple acceptable extensions. But the possibility of multiple extensions is essential for many extension semantics and argumentation would be trivial without it in many cases.

In order to evaluate and compare argument-ranking semantics a number of principles were proposed. We recall some basic ones, which we will need later. Note that this list is not complete and more principles can be found in the literature, see \cite{DBLP:conf/sum/AmgoudB13a, DBLP:conf/aaai/BonzonDKM16}. 
Before we introduce the principles, we need a few more notations. Let $F = (A,R)$ be an AF.
A \emph{path} of length $l_P=n$ between two arguments $a,b \in A$ is a sequence of arguments $P(a,b)= (a_0,a_1,...,a_{n})$ with $(a_i,a_{i+1})\in R$ for all $i$ with $a_0=a$ and $a_n = b$. The \emph{connected components} $cc(F)$ of $F$ are the maximal subgraphs $F'=(A',R')$, where for every pair of arguments $a,b \in A'$ there exists an undirected path, i.\,e., there is $P_u(a,b)= (a=a_0,a_1,...,a_{n-1},a_n =b)$ s.t. for every $i$ there is either $(a_i,a_{i+1})\in R$ or $(a_{i+1},a_i) \in R$.
An \emph{isomorphism} $\gamma$ between two argumentation frameworks $F= (A, R)$ and $F'=(A', R')$ is a bijective function $\gamma: A \rightarrow A'$ such that $(a, b) \in R$ iff $(\gamma(a), \gamma(b)) \in R'$ for all $a,b \in A$.

\begin{definition}
    \label{def:ranking_principles}
    An argument-ranking semantics $\rho$ satisfies the respective principle iff for all AFs $F=(A,R)$ and any $a,b \in A$:
    \begin{description}
 \item[Abstraction (\emph{Abs}).] Names of arguments should not influence the ranking. \\
        For a pair of AFs $F=(A,R)$ and $F'=(A',R')$ and every isomorphism $\gamma:F \rightarrow F'$, we have $a \succeq^\rho_F b$ iff $\gamma(a) \succeq^\rho_{F'} \gamma(b)$.
        \item[Independence (\emph{In}).] Unconnected arguments should not influence a ranking. \\
        For every $F'=(A',R') \in cc(F)$ and for all $a,b \in A'$: $a \succeq^\rho_F b$ iff $a \succeq^\rho_{F'} b$.  
        \item[Void Precedence (\emph{VP}).] Unattacked arguments should be stronger then attacked ones.\\
        If $a^-_F = \emptyset$ and $b^-_F \neq \emptyset$ then $a \succ^\rho_F b$.
         \item[Non-attacked Equivalence (\emph{NaE)}] Two unattacked arguments should be equally strong. \\
        If $a^-_F = b^-_F = \emptyset$ then $a \simeq^\rho_F b$.
    \end{description}
\end{definition}

%-------------------------------------------------------------------------------------------------------

\section{Extension-Ranking Semantics}\label{sec:extension-ranking}
%In the following section we will introduce the notion of \emph{extension-ranking semantics}.

An extension semantics $\sigma$ provides a simple way to assess whether a given set of arguments $E$ is acceptable: either $E$ is a $\sigma$-extension or it is not. Such a binary classification can be seen as a drawback of an extension semantics. Let us illustrate this shortcoming with the following example.
\begin{example}
    Consider again $F_4$ from Example \ref{ex:af_example}. As shown before $E_1= \{a\}$, 
    $E_2= \{a, g\}$,
    $E_3= \{a, c, g\}$, and
    $E_4= \{a, d, g\}$ are the complete extensions, but all these sets can be considered equal in a reasoning process based on complete semantics. Only by using a different extension semantics, such as preferred or stable semantics, can these sets be distinguished in terms of their plausibility of acceptance. 

    If we look at sets that do not satisfy the complete extension semantics, such as $\{d\}$ or $\{a,b\}$, we see that these sets can be considered equally with respect to the complete semantics. However, $\{d\}$ is an admissible set, while $\{a,b\}$ is not even conflict-free. So it is reasonable to say that $\{d\}$ is a more plausible to be accepted than $\{a,b\}$.
    %In other words, a classical extension semantics only distinguishes extensions from non-extensions. What we need is a way to compare arbitrary sets of arguments, including those that are not extensions under a classical extension semantics.
\end{example}

In order to define a new kind of semantics with more expressiveness than binary classification, we will take a more general perspective on this issue by considering preorders over sets of arguments. The resulting preorder will allow us to rank sets based on their plausibility of acceptance, and to determine the plausibility of acceptance of each set with respect to other sets. In general, not every pair of sets of arguments must be necessarily comparable; if we have two disjoint sets of arguments, then it is not always reasonable to force a decision about the relationship between these two sets. %\tr{Simplest example: given an AF $(\{a, b\}, \{ (a, b), (b, a) \})$ the sets $\{a\}$ and $\{b\}$ are incomparable since there's no reason to prefer $a$ over $b$ or v.v.)}
Therefore, we consider only preorders and do not necessarily impose the totality of the relations. 
These preorders will allow us to state whether one set is closer to being acceptable than another. 
\begin{definition}
Let $F=(A,R)$ be an AF.
    An \emph{extension ranking} on $F$ is a preorder $\sqsupseteq$ over the power set $2^A$ of all arguments. An \emph{extension-ranking semantics} $\tau$ is a function that maps $F$ to an extension ranking $\sqsupseteq^{\tau}_{F}$ on $F$.  
\end{definition}
For an AF $F=(A,R)$, an extension-ranking semantics $\tau$, an extension ranking $\sqsupseteq^{\tau}_{F}$, $E,E'\subseteq A$, and for $E \sqsupseteq^{\tau}_{F} E'$ we say that $E$ is \emph{at least as plausible to be accepted as} $E'$ with respect to $\tau$ in $F$. We introduce the usual abbreviations: 
\begin{itemize}
    \item $E$ is \emph{strictly more plausible to be accepted than $E'$}, denoted $E \sqsupset^{\tau}_{F} E'$, if $E \sqsupseteq^{\tau}_{F} E'$ but not $ E' \sqsupseteq^{\tau}_{F} E$;
    \item $E$ and $E'$ are \emph{equally plausible to be accepted}, denoted $E \equiv^{\tau}_{F} E'$, if $E \sqsupseteq^{\tau}_{F} E'$ and $E' \sqsupseteq^{\tau}_{F} E$;
    \item   $E$ and $E'$ are \emph{incomparable}, denoted $E \asymp^{\tau}_{F} E'$, if neither $E \sqsupseteq^{\tau}_{F} E'$ nor $E' \sqsupseteq^{\tau}_{F} E$. 
\end{itemize}

In order to relate extension-ranking semantics to extension semantics later, the most plausible sets of an extension-ranking, will be of interest.
\begin{definition}
Let $F= (A,R)$ be an AF. 
    We denote by $\maxpl_{\tau}(F)$ the maximal (or \emph{most plausible}) elements of the extension ranking $\sqsupseteq^{\tau}_{F}$, i.\,e., $\maxpl_{\tau}(F) = \{E \subseteq A \mid \nexists E' \subseteq A \text{ with } E' \sqsupset^{\tau}_{F} E\}$.
\end{definition}

Extension-ranking semantics provide an expressive semantical framework for ranking sets of arguments. In fact, extension semantics can be used directly to define a very naive instance of extension-ranking semantics, as follows:
\begin{definition}
Let $F=(A,R)$ be an AF.
    Given an extension semantics $\sigma$, we define the \emph{least-discriminating extension-ranking semantics} wrt. $\sigma$, denoted $\mathsf{LD}^{\sigma}$ by: 
    \begin{itemize}
        \item $E \sqsupset^{\LD^{\sigma}}_{F} E'$ if $E \in \sigma(F)$ and $E' \notin \sigma(F)$;
        \item and $E \equiv^{\LD^{\sigma}}_{F} E'$, if $E, E' \in \sigma(F)$ or $E, E' \notin \sigma(F)$.
    \end{itemize}
\end{definition}

\begin{example}
We continue Example \ref{ex:af_example} and consider the least-discriminating extension-ranking semantics $\LD^{\co}$ wrt.\ the complete semantics $\co$. Then we have 
\begin{align*}
    \maxpl_{\LD^\co}(F_4)= \{\{a\}, \{a, g\}, \{a, c, g\}, \{a, d, g\}\}
\end{align*}
and every other set is ranked below these sets (so the complete extensions are the most plausible sets).
If we consider the least-discriminating extension-ranking semantics $\LD^{\pr}$ wrt.\ the preferred semantics $\pr$ we get
\begin{align*}
    \maxpl_{\LD^\pr}(F_4)= \{\{a, c, g\}, \{a, d, g\}\}
\end{align*}
\end{example}

We see that the least-discriminating extension-ranking semantics always gives a binary classification. %, i.e., extension-ranking semantics generalises extension semantics.
Our goal for the remainder of this paper is to provide a finer distinction between sets of arguments than the least-discriminating extension-ranking semantics. 
%We focus on extension-ranking semantics that refine the least-discriminating extension-ranking semantics wrt. $\sigma$ by providing a more fine-grained differentiation of those sets of arguments that are \emph{not} $\sigma$-extensions. In other words, for any two sets of arguments that are not $\sigma$-extensions, we want to be able to say whether one of them is ``closer'' to be a $\sigma$-extension than the other one. 
%\mt{what about differentiating $\sigma$-extensions? This was advertised before (albeit implicitly) and should at least be discussed. KS: In Sec 7+8 we rank $\sigma$-extensions as well. Therefore this block should be removed.}

%-----------------------------------------------------------------------------------------------------------

\section{Principles for Extension-Ranking Semantics} \label{sec:principles}
We follow a \emph{principle-based approach} \cite{DBLP:journals/flap/TorreV17} to develop and analyse extension-ranking semantics. For this purpose, we define a number of general principles to describe various aspects of an intuitively well-behaved extension-ranking semantics and verify whether an extension-ranking semantics satisfies a principle or not. These principles also help us to compare different extension-ranking semantics with each other. Some principles are adapted from argument-ranking semantics literature \cite{DBLP:conf/aaai/BonzonDKM16}, others are defined specifically for extension-ranking semantics. 

Classical extension semantics already feature many desirable properties for an argumentative evaluation. The formalisation of extension-ranking semantics should therefore be compatible with these classical extension semantics, insofar as the constraints of valid and invalid extensions should be reflected in the extension-rankings. There are two aspects to this. First, for a given semantics $\sigma$, we expect that only $\sigma$-extensions are allowed to be the most plausible sets. Second, since different $\sigma$-extensions are indistinguishable by $\sigma$, all $\sigma$-extensions should be most plausible sets. We formalise the first demand via a principle called \emph{$\sigma$-soundness} and the second via a principle called \emph{$\sigma$-completeness}. Taking both demands together, we obtain the principle of \emph{$\sigma$-generalisation}.

\begin{definition}%[$\sigma$\emph{-generalisation}]
Let $\sigma$ be an extension semantics and $\tau$ an extension-ranking semantics. $\tau$ satisfies
\begin{itemize}
\item $\sigma$\emph{-soundness} iff for all AFs $F$: $\maxpl_{\tau}(F) \subseteq \sigma(F)$.
\item $\sigma$\emph{-completeness} iff for all AFs $F$: $\maxpl_{\tau}(F) \supseteq \sigma(F)$.
\item $\sigma$\emph{-generalisation} iff $\tau$ satisfies both $\sigma$-soundness and $\sigma$-completeness.
\end{itemize}
\end{definition}
Thus, for an extension-ranking semantics that satisfies $\sigma$-generalisation, we have that its best ranked sets coincides with the set of all $\sigma$-extensions of the given AF. 
%The main motivation of this work is to generalise the extension-based reasoning process in abstract argumentation. The $\sigma$-generalisation principle directly models this motivation. This principle enforces that extension-ranking semantics provide a distinction between $\sigma$-extensions and non-$\sigma$-extensions.% Because of this principle it should be noted, that only extension-ranking semantics are considered, which provide a differentiation of non-extensions. 

\begin{example}
    Consider $F_4$ from Example \ref{ex:af_example} and let $\tau$ be an extension-ranking semantics. If $\tau$ satisfies $\co$-generalisation then we have $\maxpl_\tau(F_4)=\{\{a\},\\\{a,g\},\{a,c,g\},\{a,d,g\}\}$, while if $\tau$ satisfies $\pr$-generalisation we have $$\maxpl_\tau(F_4)=\{\{a,c,g\},\{a,d,g\}\}$$
\end{example}

The plausibility of acceptance of a set of arguments should not depend on arguments not connected with this set. %\tr{Previous sentence is vague. What about: }
This means that adding or removing unconnected arguments does not change anything in the relationship of any of the other arguments like the principles \emph{Crash-resistance} and \emph{Non-interference} by van der Torre and Vesic \cite{DBLP:journals/flap/TorreV17} suggest. Extension-ranking semantics should also exhibit this behaviour. % \tr{Looks like there should be a reference here to  similar principles for other semantics.} 
For any two unconnected AFs, the relationship between sets after combining these two AFs into one AF should coincide with the relationship in the two smaller AFs, i.e. if a set $E$ is ranked better than another one $E'$ before combining these two AFs, then $E$ should still be ranked better than $E'$ in the combined AF. This demand is formalised via the \emph{composition} principle. Furthermore, by splitting an AF into its unconnected components the relationships between sets should not change. Thus, if a set is at least as plausible to be accepted as another set in an AF, then removing unconnected parts of the AF should not change the relationship between those two sets. The principle \emph{decomposition} models this demand.
 
\begin{definition}
    Let $\tau$ be an extension-ranking semantics.
    \begin{itemize}
        \item $\tau$ satisfies \emph{composition} if for every AF $F$ s.t. $F= F_1 \cup F_2 = (A_1,R_1) \cup (A_2,R_2)$ with $A_1 \cap A_2 = \emptyset$ and $E, E' \subseteq A_1 \cup A_2$ it holds that
        $$\mbox{if }\left\{
 \begin{array}{c}
 	E \cap A_1 \sqsupseteq_{F_1}^{\tau} E' \cap A_1\\
	E \cap A_2 \sqsupseteq_{F_2}^{\tau} E' \cap A_2
  \end{array}
  \right\}\mbox{ then }E \sqsupseteq^{\tau}_{F} E'.$$
  
    \item $\tau$ satisfies \emph{decomposition} if for every AF $F$ s.t. $F= F_1 \cup F_2 = (A_1,R_1) \cup (A_2,R_2)$ with $A_1 \cap A_2 = \emptyset$ and $E, E' \subseteq A_1 \cup A_2$ it holds that
     $$\mbox{if } E \sqsupseteq^{\tau}_{F} E' \mbox{ then }\left\{ 
 \begin{array}{c}
 	E \cap A_1 \sqsupseteq_{F_1}^{\tau} E' \cap A_1\\
	E \cap A_2 \sqsupseteq_{F_2}^{\tau} E' \cap A_2
  \end{array}
  \right\}.$$
 \end{itemize}
\end{definition}

Composition states that if one set of arguments is at least as plausible to be accepted as another in a local view, then that relationship should hold in the global view. Decomposition, on the other hand, focuses on keeping the relationships of the global view intact in the local view. 

\begin{example}\label{ex:comp}
    Consider $F_4$ from Example \ref{ex:af_example} (depicted again in Figure \ref{tikz:f4_recalled}) and $F_5= (\{h,i,j\},\{(h,i),(i,j)\})$ as depicted in Figure \ref{tikz:comp}.
    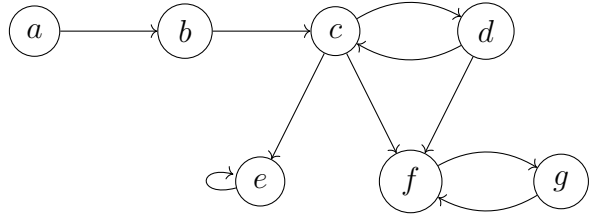
\begin{figure}
    \centering
 
 \scalebox{1}{
\begin{tikzpicture}

\node (a1) at (0,0) [circle, draw,minimum size= 0.65cm] {$a$};
\node (a2) at (2,0) [circle, draw,minimum size= 0.65cm] {$b$};
\node (a3) at (4,0) [circle, draw,minimum size= 0.65cm] {$c$};
\node (a4) at (6,0) [circle, draw,minimum size= 0.65cm] {$d$};
\node (a5) at (3,-2) [circle, draw,minimum size= 0.65cm] {$e$};
\node (a6) at (5,-2) [circle, draw,minimum size= 0.65cm] {$f$};
\node (a7) at (7,-2) [circle, draw,minimum size= 0.65cm] {$g$};

\path[<-] (a2) edge  (a1);
\path[<-] (a3) edge  (a2);
\path[->, bend left] (a3) edge  (a4);
\path[->, bend left] (a4) edge (a3); 

\path[->] (a3) edge (a5);
\path[->] (a3) edge (a6);
\path[->] (a4) edge (a6);
\path[->, bend left] (a6) edge  (a7);
\path[->, bend left] (a7) edge (a6); 

\path[->] (a5) edge  [loop left] node {} ();

\end{tikzpicture}}
   \caption{Recalled AF $F_4$ from Example \ref{ex:af_example}.}
    \label{tikz:f4_recalled}
\end{figure}
      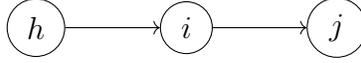
\begin{figure}
    \centering
\scalebox{1}{
\begin{tikzpicture}
 
\node (a1) at (0,0) [circle, draw,minimum size= 0.65cm] {$h$};
\node (a2) at (2,0) [circle, draw,minimum size= 0.65cm] {$i$};
\node (a3) at (4,0) [circle, draw,minimum size= 0.65cm] {$j$};
\path[<-] (a2) edge  (a1);
\path[<-] (a3) edge  (a2);

\end{tikzpicture}}
   \caption{AF $F_5$ from Example \ref{ex:comp}.}
    \label{tikz:comp}
\end{figure}
   Then these two AFs are disjoint, hence for any extension-ranking semantics $\tau$ the relationship in $F_4 \cup F_5$ should be the same as in $F_4$ resp. $F_5$. Consider sets $\{a,c,g,h,j\}$ and $\{b,c,h,i\}$, then $\{a,c,g,h,j\}$ is conflict-free in $F_4$ and $F_5$ ($\{a,c,g\} \in \cf(F_4)$ and $\{h,j\} \in \cf(F_5)$), while $\{b,c,h,i\}$ has conflicts in both $F_4$ and $F_5$, so it is reasonable to say that $\{a,c,g,h,j\}$ is more plausible to be accepted than $\{b,c,h,i\}$ in both $F_4$ and $F_5$ ($\{b,d\} \notin \cf(F_4)$ and $\{h,i\} \notin \cf(F_5)$), i.\,e., $\{a,c,g\} \sqsupset^\tau_{F_4} \{b,c\}$ resp. $\{h,j\} \sqsupset^\tau_{F_5} \{h,i\}$ for an extension-ranking semantics $\tau$, which generalises $\LD^\cf$. So, for composition to be satisfied it has to hold that $\{a,c,g,h,j\} \sqsupset^\tau_{F_4 \cup F_5} \{b,c,h,i\}$.

    Next, we see that $\{a,c,g,h,j\}$ is conflict-free and $\{b,c,h,i\}$ is not conflict-free wrt. $F_4 \cup F_5$, hence for any extension-ranking semantics $\tau$ that generalises $\LD^\cf$, we have $\{a,c,g,h,j\} \sqsupset^\tau_{F_4 \cup F_5} \{b,c,h,i\}$. To satisfy decomposition the relationship between $\{a,c,g,h,j\}$ and $\{b,c,h,i\}$ should stay the same wrt. $F_4$ and $F_5$, i.\,e., $\{a,c,g\} \sqsupset^\tau_{F_4} \{b,c\}$ and $\{h,j\} \sqsupset^\tau_{F_5} \{h,i\}$.
\end{example}

Caminada \cite{DBLP:conf/jelia/Caminada06} discussed that an important property of rational accounts of argumentation is \emph{reinstatement}, i.\,e., the ability to make an attacked argument acceptable by attacking its attackers. Complete semantics implements this property in a strict fashion: if an argument can be reinstated by a set of arguments, it must be included in that set. Therefore, a set of arguments containing an argument it defends is more plausible to be accepted than a set without that argument.
We define two different versions of this principle, a \emph{weak} and a \emph{strong} version. The weak version states that by adding a defended argument (which also introduces no further conflicts) we do not reduce the plausibility of acceptance of any set of arguments. The strong version ensures, that by adding a defended argument to a set, the plausibility of acceptance of that set will be strictly higher than before.

\begin{definition}
Let $\tau$ be an extension-ranking semantics.
\begin{itemize}
\item $\tau$ satisfies \emph{weak reinstatement} iff for all AF $F = (A,R)$ with $E \subseteq A$ it holds that $a \in \mathcal{F}_{F}(E)$, $a \notin E$ and $a  \notin (E^- \cup E^+)$ implies $E \cup \{a\} \sqsupseteq_{F}^{\tau} E$.
\item $\tau$ satisfies \emph{strong reinstatement} iff for all $F = (A,R)$ with $E \subseteq A$ it holds that $a \in \mathcal{F}_{F}(E)$, $a \notin E$ and $a  \notin (E^- \cup E^+)$ implies $E \cup \{a\} \sqsupset_{F}^{\tau} E$.
\end{itemize}
\end{definition}
Note that the condition $a  \notin (E^- \cup E^+)$ above is needed in order to not add additional conflicts to the set (which may again lower the plausibility of acceptance of the set of arguments).
\begin{example}
    Consider $F_5$ from Example \ref{ex:comp}. Then argument $j$ is defended by argument $h$, so for any extension-ranking semantics $\tau$ that satisfies weak reinstatement it has to hold: $\{h,j\} \sqsupseteq^{\tau}_{F_5} \{h\}$.
    For strong reinstatement the relationship between $\{h,j\}$ and $\{h\}$ has to be strict.

    If we consider the set $\{h,i\}$, then we might think that $\{h,i,j\} \sqsupseteq^{\tau}_{F_5} \{h,i\}$ has to hold as well, since $j$ is defended by $\{h,i\}$. However, by adding argument $j$ into the set we create more conflicts since $j \in \{h,i\}^+$. Thus, we can not enforce the relationship between $\{h,i,j\}$ and $\{h,i\}$. We could even argue that $\{h,i\} \sqsupset^{\tau}_{F_5} \{h,i,j\}$ should hold, since $\{h,i\}$ has fewer conflicts than $\{h,i,j\}$.
\end{example}

In general, changes to an AF will affect the inferences drawn. However, the resulting changes in the reasoning should be intuitive. For example, adding an attacker to an argument should not increase the strength of that argument. An extension-ranking semantics should also respect this behaviour. 
For an extension-ranking semantics, it should hold that given two sets $E$ and $E'$, where $E$ is already more plausible to be accepted than $E'$, adding attacks from $E$ to $E'$ should not make $E$ less plausible to be accepted than $E'$. So an extension-ranking semantics should be \emph{robust} against the addition of attacks, which intuitively should not worsen the plausibility of acceptance of a set. 
 The robustness of labelling-based semantics has already been discussed by Rienstra et al. \cite{DBLP:journals/argcom/RienstraSTL20}.
\begin{definition}
    Let $\tau$ be an extension-ranking semantics. $\tau$ satisfies \emph{addition robustness} if for all AF $F= (A,R)$ and $E,E' \subseteq A$ with $E \sqsupseteq^\tau_{F} E'$ it holds that $E \sqsupseteq^\tau_{F'} E'$ where $F'=(A,R')$ with $R'= R \cup \{(a,b)\}$ s.t. $a \in E$, and $b \in E' \setminus E$.
\end{definition} 
\begin{example}
    Consider $F_4$ from Example \ref{ex:af_example} and sets $\{a,c,g\}$ and $\{b,c,e\}$. $\{a,c,g\}$ is the stable extension, while $\{b,c,e\}$ is not even conflict-free. Let $\tau$ be an arbitrary extension-ranking semantics s.t. $\{a,c,g\} \sqsupset^\tau_{F_4} \{b,c,e\}$. If we create a new AF $F_4'$ by adding an attack from $a$ to $e$, then we add an additional reason to reject $\{b,c,e\}$. Thus, if $\tau$ satisfies \emph{addition robustness}, then $\{a,c,g\} \sqsupset^\tau_{F_4'} \{b,c,e\}$.
\end{example}

In abstract argumentation, the content of each argument is neglected and only the relationships between arguments are relevant.  It should make no difference whether an argument is called $a$ or 
$\alpha$, as long as the structure of the AF is the same. Argument-ranking semantics follow this demand in a strict way, where the principle \emph{Abstraction} \cite{DBLP:conf/sum/AmgoudB13a} is defined, which states that the names of the arguments are not relevant for the resulting argument ranking. For extension-ranking semantics, we define a principle similar in spirit to \emph{Abstraction}, called \emph{syntax independence}.
\begin{definition} \label{des:s}
  An extension-ranking semantics $\tau$ satisfies \emph{syntax independence} if for every pair of AFs $F= (A, R)$, $F'=(A', R')$ and for every isomorphism $\gamma: A \rightarrow A'$, for all $E, E' \subseteq A$, we have $E \sqsupseteq_{F}^{\tau} E'$ iff $\gamma(E) \sqsupseteq_{F'}^{\tau}~\gamma(E')$.
\end{definition}

We conclude our discussion of principles for extension-ranking semantics, by analysing the least-discriminating extension-ranking semantics from the previous section wrt.\ them. %By definition, the least-discriminating extension-ranking semantics satisfies their corresponding generalisation principle. 

The least-discriminating extension-ranking semantics wrt.\ an extension semantics $\sigma$ satisfies $\sigma$-generalisation by definition, since all $\sigma$-extensions are ranked better than non-$\sigma$-extensions. For example, for $\LD^\ad$, all admissible sets are the most plausible sets, and every non-admissible set is ranked below these sets. However, $\LD^\sigma$ usually does not satisfy $\sigma'$-generalisation for $\sigma\neq\sigma'$. For example, $\LD^\ad$ does not satisfy $\co$-soundness, as the following example shows.
\begin{example}
    Consider $F_4$ from Example \ref{ex:af_example}. Then $\{d\}$ is an admissible set implying $\{d\} \in \maxpl_{\LD^\ad}(F_4)$, however $\{d\}$ is not a complete set, since argument $a$ is defended by $\{d\}$. So, $\LD^\ad$ violates $\co$-soundness. For $\LD^\co$ we can also easily see that $\ad$-completeness is violated. Consider again set $\{d\}$. $\{d\} \notin \maxpl_{\LD^\co}(F_4)$, but $\{d\} \in \ad(F_4)$ so $\{d\}$ is supposed to be a most plausible set wrt. $\ad$.
\end{example}

\begin{proposition}\label{thm:LD_gen}
    $\LD^\sigma$ satisfies $\sigma$-generalisation for $\sigma \in \{\cf,\ad,\co,\pr,\gr, \sst\}$.
\end{proposition}
%
%\begin{proof}
 %   Let $F= (A,R)$ be an AF and $\sigma \in \{\cf,\ad,\co,\pr,\gr,\sst\}$. For every set $E \in \sigma(F)$, by definition it holds that $E \sqsupset^{\LD^\sigma}_{F} E'$ for every $E' \notin \sigma(F)$ and also $\LD^\sigma$ produces only a binary classification. This implies, that every $\sigma$-extension is among the most plausible sets and every non-$\sigma$-extension is not among the most plausible sets. Hence, $\sigma$-generalisation is satisfied.
%\end{proof}
$\LD^\st$ violates $\st$-soundness, since there exist AFs without any stable extension. However, for these AFs $\sqsupseteq^{\LD^{\st}}$ is still well-defined, just not very expressive. Every set is equally plausible to be accepted and therefore every set is among the most plausible sets. 
\begin{example}\label{ex:oddcycle}
    Consider the odd cycle $F_6= (\{a,b,c\},\{(a,b),(b,c),(c,a)\}$, as depicted in Figure \ref{tikz:oddcycle}.
      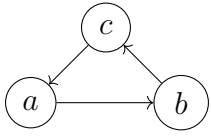
\begin{figure}
    \centering
\scalebox{1}{
\begin{tikzpicture}
 
\node (a) at (0,0) [circle, draw,minimum size= 0.65cm] {$a$};
\node (b) at (2,0) [circle, draw,minimum size= 0.65cm] {$b$};
\node (c) at (1,1) [circle, draw,minimum size= 0.65cm] {$c$};
\path[->] (a) edge  (b);
\path[->] (b) edge  (c);
\path[->] (c) edge  (a);

\end{tikzpicture}}
   \caption{AF $F_6$ from Example \ref{ex:oddcycle}.}
    \label{tikz:oddcycle}
\end{figure}
    This AF has no stable extension. So for any pair of sets $E, E' \subseteq \{a,b,c\}$ it holds $E \equiv^{\LD^{\st}}_{F_6} E'$, since $E, E' \notin \st(F_6)$. Thus every set is equally ranked and every set is among the most plausible sets.   
\end{example}
The stable semantics is not the only extension semantics with such a behaviour. 
 To be precise, for any extension semantics $\sigma$ for which there are AFs without $\sigma$-extensions, we cannot define any extension-ranking semantics satisfying $\sigma$-soundness. 
\begin{proposition}\label{thm:no st-gen}
Let $\sigma$ be an extension semantics s.t. there exists an AF $F=(A,R)$ s.t. $\sigma(F)= \emptyset$, then there is no extension-ranking semantics satisfying $\sigma$-soundness. 
 \end{proposition}
%\begin{proof}
 %   Let $\sigma$ be an extension semantics and $F=(A,R)$ an AF s.t. $\sigma(F)= \emptyset$. Let $\tau$ be an extension-ranking semantics, for $\tau$ to satisfy $\sigma$-generalisation it holds that $max_{\tau}(F)= \sigma(F)$ and for any preorder there always exists a maximal set, i.e. $max_{\tau}(F)\neq \emptyset$. However, $\sigma(F)= \emptyset$ therefore $\sigma$-generalisation is always violated. 
%\end{proof}

By combining two disjoint AFs into a new AF the acceptance of each set does not change, therefore $\LD^\sigma$ satisfies composition for all considered semantical notions.  %\mt{semi-stable semantics? This also applies to the rest of this section}
\begin{proposition}
    $\LD^\sigma$ satisfies composition for $\sigma \in \{\cf,\ad,\co,\pr, \gr, \st, \sst\}$.
\end{proposition}

Decomposition is violated by all considered least-discriminating extension-ranking semantics. We show this with the following counterexample.
\begin{example}\label{ex:LD_decomp}
    Let $F_7= (\{a,b,c,d\},\{(a,b),(c,d)\})$ and the two sets $E= \{a,b,c\}$ and $E'= \{a,c,d\}$. As depicted in Figure \ref{tikz:LD_decomp}.
       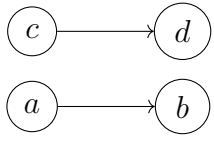
\begin{figure}
    \centering
\scalebox{1}{
\begin{tikzpicture}
 
\node (a) at (0,0) [circle, draw,minimum size= 0.65cm] {$a$};
\node (b) at (2,0) [circle, draw,minimum size= 0.65cm] {$b$};
\node (c) at (0,1) [circle, draw,minimum size= 0.65cm] {$c$};
\node (d) at (2,1) [circle, draw,minimum size= 0.65cm] {$d$};

\path[->] (a) edge  (b);
\path[->] (c) edge  (d);
\end{tikzpicture}}
   \caption{AF $F_7$ from Example \ref{ex:LD_decomp}.}
    \label{tikz:LD_decomp}
\end{figure}
    
 $F_7$ can be partitioned into two disjoint AFs $F_{7,1}= (\{a,b\},\{(a,b)\})$ and $F_{7,2}= (\{c,d\},\{(c,d)\})$. Both $E$ and $E'$ are not conflict-free in $F_7$, hence $E \equiv^{\LD^\sigma}_{F_7} E'$ for any extension semantics that is based on conflict-freeness, i.\,e. $\sigma \in \{\cf,\ad,\co,\pr,\gr, \st,\sst\}$. But $E' \cap \{a,b\}$ is conflict-free in $F_{7,1}$ while $E \cap \{c,d\}$ is conflict-free in $F_{7,2}$, hence $E'\cap \{a,b\} \sqsupset^{\LD^\sigma}_{F_{7,1}} E \cap \{a,b\}$ resp. $E\cap \{c,d\} \sqsupset^{\LD^\sigma}_{F_{7,2}} E' \cap \{c,d\}$  for $\sigma \in \{\cf,\ad,\co,\pr,\gr,\st, \sst\}$. So, $\LD^\sigma$ violates decomposition for $\sigma \in \{\cf,\ad,\co,\pr,\gr,\st, \sst\}$.
\end{example}

Adding an argument, which is defended, does not yield a rejection of the respective set wrt.\ any extension semantics, hence $\LD^\sigma$ satisfies weak reinstatement for all considered semantical notions.
\begin{proposition}
    $\LD^{\sigma}$ satisfies weak reinstatement for $\sigma \in \{\cf,\ad,\co,\pr,\gr,\st, \\\sst \}$.
\end{proposition}

While all considered least-discriminating extension-ranking semantics satisfy weak reinstatement, for every $\sigma \in \{\cf,\ad,\co,\pr,\gr,\st,\sst\}$, $\LD^{\sigma}$ violates strong reinstatement. We show this with a counterexample. 
 \begin{example}\label{ex:LD_st_rein}
     Let $F_8 = (\{a,b\},\{(a,a)\})$ be an AF and we consider $E = \{a\}$, as depicted in Figure \ref{tikz:LD_st_rein}.
        \begin{figure}
    \centering
\scalebox{1}{
\begin{tikzpicture}
 
\node (a) at (0,0) [circle, draw,minimum size= 0.65cm] {$a$};
\node (b) at (2,0) [circle, draw,minimum size= 0.65cm] {$b$};

\path[->] (a) edge  [loop left] node {} ();
\end{tikzpicture}}
   \caption{AF $F_8$ from Example \ref{ex:LD_st_rein}.}
    \label{tikz:LD_st_rein}
\end{figure}
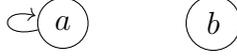
     Since $E$ is not conflict-free it follows that $E \notin \sigma(F_8)$ for $\sigma \in \{\ad,\co,\pr,\gr,\st,\sst\}$, however $b \in \mathcal{F}_{F_8}(E)$, $b \notin E$ and $b \notin (E^- \cup E^+)$. So based on strong reinstatement it would hold, that $E \cup \{b\} \sqsupset_{F_8}^{\LD^{\sigma}} E$, but $E \cup \{b\}$ is also not conflict-free. Hence, $E \cup \{b\} \notin \sigma(F_8)$. Therefore $E \cup \{b\} \notin \maxpl_{\LD^{\sigma}}(F_8)$ and $E \cup \{b\} \equiv_{F_8}^{\LD^{\sigma}}$ E. So, $\LD^{\sigma}$ does not satisfies strong reinstatement for $\sigma \in \{\cf,\ad,\co,\pr,\gr,\st,\sst\}$.  
 \end{example}

Adding an attack to an already defeated set of arguments cannot improve the status of that set, i.\,e. if that set was not conflict-free, admissible, grounded or stable, then that set cannot become conflict-free, admissible, grounded or stable. 
So, we see that the least-discriminating extension-ranking semantics with respect to conflict-freeness and admissibility or grounded and stable extension semantics satisfies addition robustness. 
\begin{proposition}\label{thm:LD_ar}
    $\LD^{\sigma}$ satisfies addition robustness for $\sigma \in \{\cf,\ad, \gr,\st\}$.
\end{proposition}

 For complete, preferred and semi-stable semantics, we can give a counterexample showing that addition robustness is violated, since by adding an outgoing attack from a set, additional arguments maybe defended.
 \begin{example}\label{ex:pr_robustness}
     Let $F_9= (A,R)$ be the AF depicted in Figure \ref{tikz:pr_robustness}. 
     \begin{figure}
    \centering
 \scalebox{1}{
\begin{tikzpicture}
\node (a) at (0,0) [circle, draw,minimum size= 0.65cm] {$a$};
\node (b) at (2,0) [circle, draw,minimum size= 0.65cm] {$b$};
%\node (c) at (4,0) [circle, draw] {$c$};
\node (d) at (4,2) [circle, draw,minimum size= 0.65cm] {$d$};
\node (e) at (2,2) [circle, draw,minimum size= 0.65cm] {$e$};
\node (f) at (0,2) [circle, draw,minimum size= 0.65cm] {$c$};
\node (ff) at (-2,0)  [circle, draw,minimum size= 0.65cm] {$f$};
\node (g) at (-2,2)  [circle, draw,minimum size= 0.65cm] {$g$};

\path[->,bend left] (a) edge  (f);
\path[->,bend left] (f) edge  (a);
\path[->, dashed] (a) edge  (b);
\path[->] (f) edge  (e);
\path[->] (e) edge  (b);
\path[->] (b) edge  (d);
\path[->] (d) edge  (e);
\path[->] (a) edge  (ff);
\path[->] (f) edge  (g);
\path[->] (ff) edge  [loop left] node {} ();
\path[->] (g) edge  [loop left] node {} ();
\end{tikzpicture}}
   \caption{AF $F_9$ from Example \ref{ex:pr_robustness}, where the dashed attack between $a$ and $b$, $(a,b)$, is added later to obtain $F_{9}'$}
    \label{tikz:pr_robustness}
\end{figure}
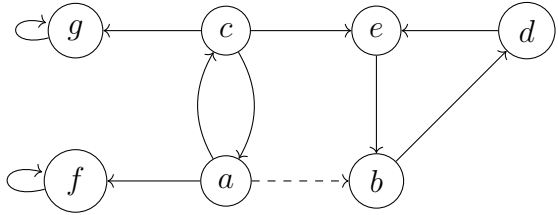
Consider sets $E= \{a\}$ and $E' = \{b,c\}$, both these sets are complete, preferred and semi-stable. However, if we add the attack $(a,b)$, then argument $d$ is defended by $\{a\}$ and therefore $E$ is no longer complete, preferred and semi-stable, while $\{b,c\}$ is still complete, preferred and semi-stable. This shows that $E \sqsupseteq^{\LD^{\{\co,\pr,\sst\}}}_{F_9} E'$ but $E' \sqsupset^{\LD^{\{\co,\pr,\sst\}}}_{F_9'} E$, which violates addition robustness. 
 \end{example}

 \begin{table}[]
 \resizebox{\textwidth}{!}{
     \centering
     \begin{tabular}{|l||c|c|c|c|c|c|c|}
     \hline
     Principles & $\LD^{\cf}$ & $\LD^{\ad}$ & $\LD^{\co}$ & $\LD^{\pr}$ & $\LD^{\gr}$ & $\LD^{\st}$ & $\LD^{\sst}$ \\
     \hline
        $\sigma$-generalisation  & \checkmark ($\sigma= \cf$) &  \checkmark  ($\sigma= \ad$) & \checkmark  ($\sigma= \co$) & \checkmark  ($\sigma= \pr$) & \checkmark  ($\sigma= \gr$)  & X & \checkmark  ($\sigma= \sst$)\\
        composition  & \checkmark &  \checkmark & \checkmark & \checkmark & \checkmark & \checkmark  & \checkmark \\
        decomposition  & X &  X & X & X & X & X & X\\
        weak reinstatement & \checkmark &  \checkmark & \checkmark & \checkmark & \checkmark & \checkmark  & \checkmark\\
        strong reinstatement  & X &  X & X & X & X & X & X \\
        addition robustness & \checkmark &  \checkmark & X & X & \checkmark & \checkmark & X  \\
 %       \INMAX  & X &  X & X & X & X & X & X \\
        syntax independence & \checkmark &  \checkmark & \checkmark & \checkmark & \checkmark & \checkmark & \checkmark  \\ \hline
     \end{tabular} }
     \caption{Principles satisfied by $\LD^\sigma$ for $\sigma \in \{\cf,\ad,\co,\pr,\gr,\st,\sst\}$.}
     \label{tab:principle_LD}
 \end{table} 

%Since $\LD^\sigma$ satisfies weak reinstatement and violates strong reinstatement we know that \INMAX\ is violated. 
 
It is clear by definition that $\LD^\sigma$ satisfies syntax independence for $\sigma \in \{\cf,\ad, \co, \pr, \gr, \st, \sst\}$. 
The principles satisfied by $\LD^\sigma$ are summarised in Table \ref{tab:principle_LD}. 
%-----------------------------------------------------------------------------------------------------
\section{Base Relations}\label{sec:Base functions}
%\mt{add some overview on the structure of this section}
%One of the advantages of abstract argumentation frameworks is the simplicity of the approach. \tr{This sounds strange to my ears. Is it really simple? Is simplicity really an advantage? (simple with respect to what?). I would remove it.} With this in mind, 

Next, we propose a general framework for defining extension-ranking semantics based on \emph{simple} relations \emph{between sets of arguments}. Each of these relations models a single aspect of argumentative reasoning.  On their own, these relations are not very expressive, but aggregating them yield  generalisations of various extension semantics (as we will see in the next section). 

In the remainder of this section, however, we will focus on these relations, which we will call \emph{base relations}. We start with set-based relations, where two sets of arguments are compared based on the set of conflicts (\CF), the set of undefended arguments (\UD), the set of defended and not contained arguments (\DN), or the set of unattacked arguments (\UA) the two sets do induce. In the second part of this section, we discuss cardinality-based variations of the aforementioned base relations as well as further variations.

Before introducing each base relation, we define the general notion of a base relation.\begin{definition}
    Let $F=(A,R)$ be an AF. We call a binary relation $\tau$ \emph{base relation} iff $\sqsupseteq^\tau_F \subseteq 2^A \times 2^A$. 
\end{definition}
%\tr{Two things: (1) the term base function is probably confusing. Since it maps every AF to a ranking, why not call it base ranking?  (2) I think the definition above is incorrect, because it implies that every set is at least as strong as every other set. What about: A base function $\tau$ maps every AF $F = (A, R)$ to a preorder $\sqsupseteq^\tau_F \subseteq 2^A \times 2^A$?
%KS: Problem: The base relations are not preorders i.e. not transitive}
In other words, a base relation $\tau$ gives us the relation of two sets of arguments $E$ and $E'$ with respect to an AF $F$, i.e. we can say whether $E$ is at least as plausible to be accepted as $E'$, or whether these two sets are equally plausible to be accepted or even incomparable to each other.

By definition, every extension-ranking semantics is a base relation, which means that the least-discriminating extension-ranking semantics $\LD^\sigma$ are also base relations. 

\subsection{Set-based base relations}\label{subsec:Dungean}
There are certain central notions in argumentative evaluations of abstract argumentation frameworks that are implemented in a binary manner in extension semantics. For example, \emph{conflict-freeness} requires that extensions contain no conflicts. So sets do either fulfil this requirement or they do not. When striving for a more fine-grained assessment of the plausibility of acceptance of sets of arguments, we also need a more fine-grained assessment of these notions. In what follows, we define base relations that model the central aspects of \emph{conflicts}, \emph{admissibility}, \emph{completeness} and \emph{stability} by assessing how well a given set fulfils that particular aspect. 

The first aspect we consider is \emph{conflicts}, whose absence is usually regarded as a desirable property of any extension semantics. Acceptable sets of arguments should be internally consistent and should not contradict themselves. 
More concretely, a conflict-free set of arguments is regarded as more plausible to be accepted than a non-conflict-free set. Generalising this idea, we deem a set $E$ to be more plausible to be accepted than another set $E'$ if $E$ has strictly fewer conflicts than $E'$ (wrt.\ set inclusion). 
 We model this aspect with the following \emph{$\sqsupseteq^\CF$ base relation}.
\begin{definition}\label{def:orderings_cf}
Let $F=(A,R)$ and $E \subseteq A$. Define \emph{\CF } via
\begin{align*}
\CF_F(E)&= \{(a,b) \in R | a,b \in E\} 
\end{align*}
and the corresponding \emph{\CF\ base relation} $\sqsupseteq_{F}^{\CF}$ via
\begin{align*}
    E \sqsupseteq_{F}^{\CF} E' & \mbox{ iff }\CF_F(E) \subseteq \CF_F(E')
\end{align*}
\end{definition}
\begin{example}
 Consider $F_{4}$ from Example~\ref{ex:af_example} (depicted again in Figure \ref{tikz:af1_recalled}). Let $\{b,c\}$ and $\{b,c,d\}$. We have $\CF_{F_4}(\{b,c\}) = \{(b,c)\}$ and \\$\CF_{F_4}(\{b,c,d\})= \{(b,c),(c,d),(d,c)\}$. Hence $\{b,c\}$ is more plausible to be accepted than $\{b,c,d\}$ with respect to their conflicts: $$\{b,c\} \sqsupset_{F_{4}}^{\CF} \{b,c,d\}$$

\begin{figure}
    \centering
 
 \scalebox{1}{
\begin{tikzpicture}

\node (a1) at (0,0) [circle, draw,minimum size= 0.65cm] {$a$};
\node (a2) at (2,0) [circle, draw,minimum size= 0.65cm] {$b$};
\node (a3) at (4,0) [circle, draw,minimum size= 0.65cm] {$c$};
\node (a4) at (6,0) [circle, draw,minimum size= 0.65cm] {$d$};
\node (a5) at (3,-2) [circle, draw,minimum size= 0.65cm] {$e$};
\node (a6) at (5,-2) [circle, draw,minimum size= 0.65cm] {$f$};
\node (a7) at (7,-2) [circle, draw,minimum size= 0.65cm] {$g$};

\path[<-] (a2) edge  (a1);
\path[<-] (a3) edge  (a2);
\path[->, bend left] (a3) edge  (a4);
\path[->, bend left] (a4) edge (a3); 

\path[->] (a3) edge (a5);
\path[->] (a3) edge (a6);
\path[->] (a4) edge (a6);
\path[->, bend left] (a6) edge  (a7);
\path[->, bend left] (a7) edge (a6); 

\path[->] (a5) edge  [loop left] node {} ();

\end{tikzpicture}}
   \caption{Recalled AF $F_4$ from Example \ref{ex:af_example}.}
    \label{tikz:af1_recalled}
\end{figure}
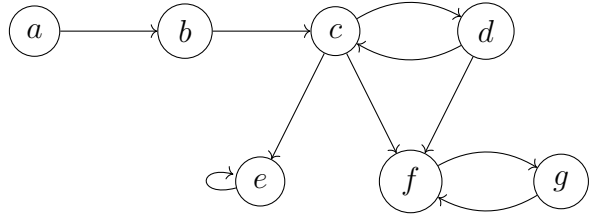
\end{example}

Another generally desirable property is \emph{admissibility}. Sets of arguments should defend themselves against any threat.
More precisely, we consider a set of arguments that defends all its elements to be more plausible to be accepted than a set that contains at least one undefended argument. Thus a set $E$ is more plausible to be accepted than another set $E'$ if $E$ has strictly fewer undefended arguments than $E'$ (wrt. set inclusions). The following \emph{$\sqsupseteq^\UD$ base relation} ($\UD$ stands for ``undefended'') captures this idea, where $\mathcal{F}$ is the characteristic function as defined in Definition \ref{def:semantics}.
%\mt{another variant would be not to collect the undefended arguments in the set, but the unattacked attackers outside the set (which is also not the same as UA, which collects all unattacked arguments outside the set); what happens with this notion? \\ KS: This notion should be equivalent to UD. Lets us call the new notion UATK then $\forall a \in \UD(E) \exists B \subseteq a^- \text{ s.t. } B \subseteq UATK(E)$ and $\forall b \in UATK(E) \exists A \subseteq b^+ \text{ s.t. } A \subseteq UD(E)$\\ The definition of UD is easier to understand, hence I would keep that definition and not discuss additional variations of that function.}
\begin{definition}\label{def:orderings_ud}
Let $F=(+A,R)$ and $E \subseteq A$. Define \emph{$\UD$} via
\begin{align*}
\UD_F(E)&= E \setminus \mathcal{F}_{F}(E) 
\end{align*}
and the corresponding \emph{$\UD$ base relation} $\sqsupseteq_{F}^{\UD}$ via
\begin{align*}
    E \sqsupseteq_{F}^{\UD} E' & \mbox{ iff }\UD_F(E) \subseteq \UD_F(E')
\end{align*}
\end{definition}
\begin{example}
    Consider $F_4$ from Example \ref{ex:af_example} and sets $\{b\}$ and $\{b,f\}$. Both arguments $b$ and $f$ are not defended by $ \{b\}$ resp. $\{b,f\}$, however $\{b\}$ contains less undefended arguments therefore $$ \{b\} \sqsupset^{\UD}_{F_4} \{b,f\}$$
\end{example}

A complete extension contains every argument it defends. Thus a set containing all its defended arguments (without adding conflicts) is more plausible to be accepted than a set not containing an argument that is actually defended.
More generally, a set $E$ is more plausible to be accepted than another set  $E'$ if there are fewer arguments \emph{consistently defended} by $E$ and not contained in $E$ than there are for $E'$ (wrt. set inclusion). By \emph{consistent defence} we mean arguments that are defended by $E$ and do not attack $E$ or are attacked by $E$. 
In order to adequately model this notion of \emph{consistent defence}, we need a more general notion of defence than that provided by the characteristic function $ \mathcal{F}_{F}$.
\begin{example}\label{ex:ex3}
Consider the AF $F_{10}$ depicted in Figure~\ref{tikz:af3} and the set $E = \{b\}$.
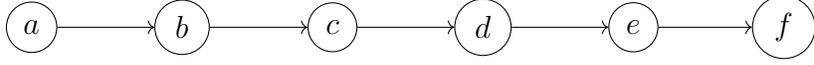
\begin{figure}
    \centering
 
 \scalebox{1}{
\begin{tikzpicture}

\node (a1) at (0,0) [circle, draw,minimum size= 0.65cm] {$a$};
\node (a2) at (2,0) [circle, draw,minimum size= 0.65cm] {$b$};
\node (a3) at (4,0) [circle, draw,minimum size= 0.65cm] {$c$};
\node (a4) at (6,0) [circle, draw,minimum size= 0.65cm] {$d$};
\node (a5) at (8,0) [circle, draw,minimum size= 0.65cm] {$e$};
\node (a6) at (10,0) [circle, draw,minimum size= 0.65cm] {$f$};

\path[->] (a1) edge  (a2);
\path[->] (a2) edge  (a3);
\path[->] (a3) edge  (a4);
\path[->] (a4) edge  (a5);
\path[->] (a5) edge  (a6);

\end{tikzpicture}}
   \caption{Abstract argumentation framework $F_{10}$ from Example \ref{ex:ex3}.}
    \label{tikz:af3}
\end{figure}
The question we want to address is: Which arguments are defended by $\{b\}$? Observe that $\mathcal{F}_{F_{10}}(E)=\{a,d\}$. Argument $d$ is only defended because of $b$, however argument $a$ is also defended by $\{b\}$ despite attacking $b$. So, we state that argument $d$ is defended despite that fact that we do not consider its defender $b$ acceptable. Thus, the characteristic function does not provide an intuitive answer here.  
Next, if we look at the defended arguments of $\{a,d\}$ we receive $\mathcal{F}_{F_{10}}(\{a,d\})=\{a,c,f\}$. The defence of argument $d$ is lost despite being defended before and again we have an argument $f$, which defence is depended on the acceptance of $b$, but this argument is not part of the set. 
\end{example}
The example shows that we need to modify the characteristic function to ignore any incoming attack on our starting set, in this case the attack from argument $a$, and compute the defended arguments assuming the starting set is acceptable. So in our example the expected result is $\{b,d\}$. Iterating the notion of consistent defence, we end up with the set $\{b,d,f\}$, which models the intuition that $d$ and $f$ are consistently defended if we assume that $\{b\}$ is accepted.  

The function $\mathcal{F}^*$ implements the above intuition as follows.
\begin{definition}
     The function $\mathcal{F}^*_{F}: 2^{A} \rightarrow 2^{A}$ is defined via
\begin{align*}
	\mathcal{F}^*_{F}(E) = \bigcup^{\infty}_{i = 1} \mathcal{F}^*_{i,F}(E)
\end{align*}
with
    \begin{align*}
    \mathcal{F}^*_{1,F}(E) &= E\\
    \mathcal{F}^*_{i,F}(E)&= \mathcal{F}^*_{i-1,F}(E) \cup (\mathcal{F}_{F}(\mathcal{F}^*_{i-1,F}(E))\setminus E^-)
\end{align*}
\end{definition}
Applying $\mathcal{F}^*$ to Example~\ref{ex:ex3} we receive the set $\mathcal{F}^*_{F_{10}}(\{b\})=\{b,d,f\}$ as desired.
In other words, $\mathcal{F^*}$ works as follows: The calculation starts with a set of arguments $E$, which is treated as admissible. Then every defended argument by $E$ is added and every attacker of $E$ is removed. Defended arguments are added until a fixed point is reached.
\begin{proposition}\label{prop:f_star fixpoint}
    Let $F=(A,R)$ be an AF and $E \subseteq A$, then there is $k \in \mathbb{N}$ with $$E \subseteq \mathcal{F}^*_{1,F}(E) \subseteq \mathcal{F}^*_{2,F}(E) \subseteq \dots \subseteq \mathcal{F}^*_{k-1,F}(E) \subseteq \mathcal{F}^*_{k,F}(E) =  \mathcal{F}^*_{k+1,F}(E) = \dots$$ 
\end{proposition}

This proposition shows that $\mathcal{F}^*_F$ is well-defined as $\mathcal{F}^*_{i,F}$ always reaches a fixed point.

Before we return to the discussion of base relations, let us analyse $\mathcal{F}^*$ a little more to make the intuition of its behaviour clearer. Since $\mathcal{F}^*$ should be a generalisation of $\mathcal{F}$, it should hold that these two functions return the same arguments if the input is admissible. 
\begin{proposition}\label{thm:f=fstart}
    Let $F=(A,R)$ be an AF, $E \subseteq A$ s.t. $E \in \ad(F)$, then $\mathcal{F}_F(E)= \mathcal{F}^*_F(E)$.
\end{proposition}

One frequently used property of $\mathcal{F}_F$ is that the grounded extension is the least fixed point of that function. For $\mathcal{F}^*_F$ we see that this is also the case.
\begin{proposition}
    Let $F=(A,R)$ be an AF and $E = \gr(F)$, then $E$ is the least fixed point of $\mathcal{F}^*_{F}$.
\end{proposition}
%\begin{proof}
 %   Let $F=(A,R)$ be an AF and $E = \gr(F)$. To prove that $E$ is the least fixed point of $\mathcal{F}^*_{F}$ we can modify the algorithm to calculate the grounded extension (for details see \cite{baroni2018handbook}). This algorithm starts with the empty set, i.e. $\mathcal{F}_{F}(\emptyset)$ and add in every step all defended arguments into that set until a fixed point is reached. Since $\emptyset \in \ad(F)$ we have $\mathcal{F}_{F}(\emptyset) = \mathcal{F}^*_{F}(\emptyset)$ like already proven in Proposition \ref{thm:f=fstart} and if $\mathcal{F}_{F}(\emptyset)$ reaches a fixed point $\mathcal{F}^*_{F}(\emptyset)$ reaches a fixed point as well. Since the grounded extension $E$ is also an admissible extension we have $\mathcal{F}_{F}(E) = \mathcal{F}^*_{F}(E)$.    
  %  This fixed point of $\mathcal{F}_{F}^*(\emptyset)$ coincided with the grounded extension $E$ and therefore $E$ is the least fixed point of $\mathcal{F}^*_{F}$.
%\end{proof}

Proposition \ref{thm:f=fstart} can be used to show that any complete extension is also a fixed point of $\mathcal{F}^*$.
\begin{proposition}
     Let $F=(A,R)$ be an AF and $E \in \co(F)$, then $\mathcal{F}^*_F(E)= E$.
\end{proposition}
%\begin{proof}
 %   Let $F=(A,R)$ be an AF and $E \in \co(F)$, then because of the definition of complete semantics we know $E \in \ad(F)$ and therefore $\mathcal{F}_F(E)= \mathcal{F}^*_F(E)$. Because $E$ already contains every argument it defends no more arguments will be added when applying the characteristic function, i.e. $\mathcal{F}_F(E) \subseteq E$. Proposition \ref{thm:f=fstart} gives us the other direction and therefore $\mathcal{F}^*_F(E)= E$.
%\end{proof}
%\mt{it would be good to analyse $F^*$ a bit to get an intuition on its behaviour. For example, $F^*$ should be the same as $F$ whenever the input is admissible, right? Also, the least fixed point of $F^*$ is also the grounded extension, right? Any complete extension is also a fixed point of $F^*$? etc.}
%\ks{Done  but comment kept to keep it in mind and maybe more things could be discussed.}
Next, let us have a look at an example that illustrates the behaviour of $\mathcal{F}^*$ in more detail.
\begin{example}\label{ex:fstar}
    Consider $F_4$ from Example \ref{ex:af_example} and the set $\{b\}$. Argument $b$ is attacked by argument $a$, so $b^-_{F_4}= \{a\}$. To calculate all consistently defended arguments by $\{b\}$ we use $\mathcal{F}^*_{F_4}(\{b\})$. Based on the definition of $\mathcal{F}^*$ we have $\mathcal{F}^*_{1,F_4}(\{b\})= \{b\}$. Since arguments $a$ and $d$ are defended by $\{b\}$ and argument $b$ is not defended, we have $\mathcal{F}_{F_4}(\{b\}) = \{a,d\}$ and therefore $\mathcal{F}^*_{2,F_4}(\{b\})= \{b\} \cup (\{a,d\} \setminus \{a\}) = \{b,d\}$. For the next iteration of $\mathcal{F}^*$ we see that $a$, $d$ and $g$ are defended by $\{b,d\}$ i.e. $\mathcal{F}_{F_4}(\{b,d\}) = \{a,d,g\}$ and therefore $\mathcal{F}^*_{3,F_4}(\{b\})= \{b,d\} \cup (\{a,d,g\} \setminus \{a\}) = \{b,d,g\}$. Since, $\{b,d,g\}$ also defends arguments $a$, $d$ and $g$ we reach a fixed point i.e. $\mathcal{F}^*_{4,F_4}(\{b\})= \{b,d,g\}$. Combing every iteration we get: $\mathcal{F}^*_{F_4}(\{b\}) = \{b,d,g\}$.
    \end{example}

Finally, we define the $\sqsupseteq^\DN$ \emph{base relation}, that models the concept of \emph{consistent defence} (\emph{$\DN$} meaning ``consistently defended and not in").
\begin{definition}\label{def:orderings_dn}
Let $F=(A,R)$ and $E \subseteq A$. We define \emph{$\DN$} via 
\begin{align*}
    \DN_F(E)= \mathcal{F}^*_{F}(E) \setminus E
\end{align*}
and the corresponding \emph{$\DN$ base relation} $\sqsupseteq_{F}^{\DN}$ via
\begin{align*}
    E \sqsupseteq_{F}^{\DN} E' & \mbox{ iff }\DN_F(E) \subseteq \DN_F(E').
\end{align*}
\end{definition}
\begin{example}
    Let us continue with Example \ref{ex:fstar}, where $F_4$ from Example \ref{ex:af_example} was used. We calculated already $\mathcal{F}^*_{F_4}(\{b\}) = \{b,d,g\}$, so using \DN\ we get $\DN_{F_4}(\{b\}) = \{d,g\}$. Next, let us look at the set $\{b,d\}$, here we have $\mathcal{F}^*_{F_4}(\{b,d\}) = \{b,d,g\}$, hence $\DN_{F_4}(\{b,d\}) = \{g\}$. Therefore we have: $$\{b,d\} \sqsupset^{\DN}_{F_4} \{b\}$$    
    \end{example}

The last property we will take a look at is \emph{stability}. A set is stable, if it attacks every argument not contained in it. To generalise this concept, we say that a set $E$ is more plausible to be accepted than another set $E'$ if $E$ strictly attacks more arguments than $E'$ (wrt. set inclusion). We capture this aspect with the \emph{$\sqsupseteq^\UA$ base relation} ($\UA$ meaning ``unattacked''). 
\begin{definition}\label{def:orderings_ua}
Let $F=(A,R)$ and $E \subseteq A$. Define \emph{$\UA$} via
\begin{align*}
\UA_F(E)= \{a \in A \setminus E | \neg\exists b \in E: (b,a) \in R\} 
\end{align*}
and the corresponding \emph{$\UA$ base relation} $\sqsupseteq_{F}^{\UA}$ via
\begin{align*}
    E \sqsupseteq_{F}^{\UA} E' & \mbox{ iff }\UA_F(E) \subseteq \UA_F(E')
\end{align*}
\end{definition}
\begin{example}
    Consider $F_4$ from Example \ref{ex:af_example} and the two sets $\{a,c\}$ and $\{a,d\}$. Argument $b$ is attacked by $a$ and $f$ is attacked by both $c$ and $d$, while argument $g$ is not attacked by either of the sets. Argument $e$ is only attacked by $c$. To summarise: $\UA_{F_4}(\{a,c\}) = \{g\}$ and $\UA_{F_4}(\{a,d\}) = \{e,g\}$, showing: $$\{a,c\} \sqsupset^{\UA}_{F_4} \{a,d\}$$ 
\end{example}

\subsection{Cardinality-based base relations}\label{subsec:cardinality}

So far we have used subset comparisons to compare sets of arguments with the base relations. However, other comparison methods are imaginable. A typical alternative to subset comparisons is \emph{cardinality}. Thus, instead of comparing two sets based on set-inclusion, we compare them by the amount of elements these sets contain. 
 The general idea of each base relation remains the same, but instead of comparing two sets in terms of subsets, these sets are compared in terms of cardinality. Each of the base relations still models an aspect of extension-based reasoning, for example ${\mathsf{c\text{-}}\CF}$ compares two sets based on their number of conflicts. With these cardinality-based comparisons we compare any two sets of arguments.
To avoid repeating each definition in detail we define a generalised cardinality-based base relation $\sqsupseteq^{\mathsf{c}\text{-}\mathsf{\tau}}$ for $\tau\in \{\CF, \UD, \DN, \UA\}$.

\begin{definition}
    Let $F=(A,R)$ be an AF and $E,E'\subseteq A$. We define $\mathsf{c}$-$\tau$ for $\tau\in \{\CF, \UD, \DN, \UA\}$ as the \emph{cardinality-based base relation}  $\sqsupseteq^{\mathsf{c}\text{-}\tau}_{F}$ via
    $$E \sqsupseteq^{\mathsf{c}\text{-}\tau}_{F} E' \text{ iff } |\tau_F(E)| \leq |\tau_F(E')|$$
\end{definition}

\begin{example}
    Consider $F_4$ from Example \ref{ex:af_example} and sets $\{c,g\}$ and $\{b,f\}$. Argument $g$ is defended by $\{c,g\}$, while $c$ is not defended, while both $b$ and $f$ are not defended by $\{b,f\}$. However, these two sets are incomparable with respect to \UD, since $\{c\}$ is not a subset of $\{b,f\}$. But using the cardinality-based version $\mathsf{c}$-\UD, these two sets can be compared and we get: $$\{c,g\} \sqsupset^{\mathsf{c}\text{-}\UD}_{F_4} \{b,f\}$$ 
\end{example}

A number of other set comparators, such as similarity measures, can be considered to define a base relation similar to $\sqsupseteq^\tau$ and $\sqsupseteq^{\mathsf{c}\text{-}\tau}$ for $\tau \in \{\CF,\UD,\DN,\UA\}$. 
A full discussion is beyond the scope of this paper and we refer to future work for such a discussion.  

Next, we want to discuss the behaviour of the sets of credulously and skeptically arguments under a given semantics in a small example and highlight shortcomings of these sets.

\begin{example}\label{ex:acceptance}
    Let $F_{11} = (\{a,b,c,d\},\{(a,b),(b,c),(c,d),(d,a)\})$ be an AF, as depicted in Figure \ref{tikz:acceptance}. Then every argument is credulously accepted wrt. e.g. preferred semantics, but no argument is skeptically accepted wrt. preferred semantics. So just using credulous or skeptical acceptance with respect to preferred semantics does not present us with any way to differentiate the arguments.
    
        \begin{figure}
    \centering
 
 \scalebox{1}{
\begin{tikzpicture}

\node (a) at (0,0) [circle, draw,minimum size= 0.65cm] {$a$};
\node (b) at (2,0) [circle, draw,minimum size= 0.65cm] {$b$};
\node (c) at (2,2) [circle, draw,minimum size= 0.65cm] {$c$};
\node (d) at (0,2) [circle, draw,minimum size= 0.65cm] {$d$};

\path[->] (a) edge  (b);
\path[->] (b) edge  (c);
\path[->] (c) edge  (d);
\path[->] (d) edge  (a);

\end{tikzpicture}}
   \caption{AF $F_{11}$ from Example \ref{ex:acceptance}.}
    \label{tikz:acceptance}
\end{figure}
\end{example}
The example above shows that credulous or skeptical acceptance of arguments is too limiting and should be extended. Konieczny et al. \cite{DBLP:conf/ecsqaru/KoniecznyMV15} proposed several ways to compare extensions, such that we can state that one extension is ``better'' than another one. These comparisons then allow us to restrict the credulous or skeptical acceptance of arguments to only the ``best'' extensions. Next, we extend the definitions of Konieczny et al. \cite{DBLP:conf/ecsqaru/KoniecznyMV15} to be applicable to arbitrary pairs of argument sets and not only to $\sigma$-extensions and in turn define new base relations. 

First, we need to recall the notion of \emph{strong defence}. Using the standard notion of defence, arguments are able to defend themselves against their attackers. Such self-defence has been criticised in the past \cite{DBLP:journals/ai/BaroniG07} and so Baroni and Giacomin \cite{DBLP:journals/ai/BaroniG07} introduced a stronger notion of defence called \emph{strong defence}.
\begin{definition}\label{def:strdef}
    Let $F=(A,R)$ be an AF, $a \in A$, and $E \subseteq A$, argument $a$ is \emph{strongly defended} by $E$ (denoted as $sd_{F}(a,E, E')$) from $E' \subseteq A$ iff for every argument $b \in E'$ with $(b,a) \in R$ there is an argument $c \in E \setminus \{a\}$ s.t. $(c,b)\in R$ and $sd_{F}(c,E\setminus \{a\}, E')$.
\end{definition}

In other words, an argument $a$ is \emph{strongly defended} by a set $E$ if $E$ contains a defender of $a$, that is not $a$ for every attacker of $a$, and every defender $c$ must be defended by $E$ without $a$ and itself. With $sd_{F}(E,E')$ we refer to the strongly defended arguments by $E$ from $E'$ in AF $F$. Using the notion of strong defence, extension semantics such as \emph{strongly complete} can be defined in a similar way to Definition \ref{def:semantics} \cite{DBLP:journals/ai/BaroniG07}.

Next, we define two base relations inspired by Konieczny et al. \cite{DBLP:conf/ecsqaru/KoniecznyMV15}.
\begin{definition}\label{def:Pairwise comparison criteria}
    Let $F= (A,R)$ be an AF, $E,E' \subseteq A$ two sets of arguments
    We define $\tau \in \{\nonatt,\strdef\}$ as base relations $\tau$ via:
    \begin{itemize}
        \item $E \sqsupseteq^{\nonatt}_{F} E'$ if $|E \setminus (E^-_{F} \cap E')|\geq |E' \setminus (E'^-_F \cap E)|$.
        %the number of arguments in $E$ not attacked by $E'$ is greater than or equal to the number of arguments in $E'$ not attacked by arguments of $E$.
        \item $E \sqsupseteq^{\strdef}_{F} E'$ if $|sd_F(E,E')| \geq |sd_F(E',E)|$. 
       % the number of arguments in $E$ strongly defended from $E'$ by $E$ is greater then or equal to the number of arguments in $E'$ strongly defended from $E$ by $E'$.
    %    \item \textcolor{blue}{$E \unlhd^{\tau\text{-}delarg}_{AF} E'$ if the cardinality of any largest subset $S$ of $E$ such that if all the attacks from $S$ to $E'$ are deleted, then $E \preceq^{\tau}_{(E \cup E',R_{\downarrow E\cup E'} \setminus \{(s, E')| s \in S\})} E'$ 
     %  is greater than or equal to the cardinality of any largest subset of $S'$ of $E'$ such that if all the attacks from $S'$ to $E$ are deleted, then $E' \preceq^{\tau}_{(E \cup E',R_{\downarrow E\cup E'}\setminus \{(s, E')| s \in S'\})} E$.}
      % \item \textcolor{blue}{$E \unlhd^{\tau\text{-}delatt}_{AF} E'$ if the maximal number of attacks from $E$ to $E'$ that can be deleted s.t. $E \preceq^{\tau}_{(E \cup E',R_{\downarrow E\cup E'} )} E'$ is greater than or equal to the maximal number of attacks from $E'$ to $E$ that can be deleted s.t. $E' \preceq^{\tau}_{(E \cup E',R_{\downarrow E\cup E'})} E$.}
    \end{itemize}
    \end{definition}
  %  \todo[inline]{think about subset variation. At first glance not super interesting since we only get results if $E \cap E' \neq \emptyset$}
    Intuitively, sets with more unattacked arguments are more plausible to be accepted. $\nonatt$ models this idea directly for a pair of sets. So if a set $E$ is attacked less by $E'$ than vice versa, this means that $E$ has less threats coming from $E'$ than vice versa, and therefore $E$ should be more plausible to be accepted than $E'$.
$\strdef$ extends this idea further by not only looking at unattacked arguments, but requiring strong defence instead. $E$ needs to strongly defend itself against $E'$. 

 \begin{example}\label{ex:nonatt}
       Consider $F_4$ from Example \ref{ex:af_example}. The sets $\{a,c,g\}$ and $\{a,d\}$ are both admissible sets. $\{a,c,g\}$ attacks argument $d$, while $\{a,d\}$ attacks only $c$, that means that $\{a,c,g\}$ contains more unattacked arguments than $\{a,d\}$ and therefore: $$\{a,c,g\} \sqsupset^{\nonatt}_{F_4} \{a,d\}$$
       
       For \strdef\ we look at the sets $\{a,b,f\}$ and $\{g\}$, here $\{g\}$ is an admissible set, while $\{a,b,f\}$ is not conflict-free. However, argument $a$ is strongly defended by $\{a,b,f\}$ from $\{g\}$, while $g$ is not strongly defended by $\{g\}$ from the attack from argument $f$. Therefore the number of strongly defended arguments in $\{a,b,f\}$ is higher than in $\{g\}$, so: $$\{a,b,f\} \sqsupset^{\strdef}_{F_4} \{g\}$$ 
    \end{example}
    In Example \ref{ex:nonatt} we see that $\sqsupseteq^{\nonatt}$ and $\sqsupseteq^{\strdef}$ behave differently from $\sqsupseteq^{\UD}$, for $\sqsupseteq^\UD$ we have $\{a,c,g\} \equiv^{\UD}_{F_4} \{a,d\}$ respectively $\{g\} \sqsupset^{\UD}_{F_4} \{a,b,f\}$.

    Note that $\sqsupseteq^{\nonatt}$ and $\sqsupseteq^{\strdef}$ are not transitive, therefore these two base relations do not induce extension rankings, which is a disadvantage since extension rankings are supposed to define a kind of preference ordering i.e. one set is more plausible to be accepted than another set and therefore this should be transitive.
\begin{example}
    Recall $F_6= (\{a,b,c\}, \{(a,b),(b,c),(c,a)\})$ from Example \ref{ex:oddcycle}. Consider $\{a\}$, $\{b\}$, and $\{c\}$. If we compare these three sets with \nonatt\, then we get $\{a\} \sqsupset^{\nonatt}_{F_6} \{b\}$, and $\{b\} \sqsupset^{\nonatt}_{F_6} \{c\}$, so based on \emph{transitivity} of extension rankings, we assume that $\{a\} \sqsupset^{\nonatt}_{F_6} \{c\}$ holds as well, which is false since argument $c$ is not attacked by $a$, while $c$ attacks $a$, so $\{c\} \sqsupset^{\nonatt}_{F_6} \{a\}$ and therefore violating \emph{transitivity}. The same behaviour holds for $\sqsupseteq^{\strdef}_{F_6}$.
 \end{example}

Besides violating transitivity, $\sqsupseteq^{\nonatt}$ and $\sqsupseteq^{\strdef}$ have another disadvantage. The internal conflicts of sets are ignored, so it does not matter, that a conflict-free set is more plausible to be accepted than a conflicting one. Even worse, the set containing every argument $A$ is ranked ``better'' than any other set. 
\begin{proposition}\label{prop:A_min_nonatt,strdef}
    For any AF $F=(A,R)$ we have $A \sqsupseteq^{\tau}_{F} E$ for any $E \subset A$ with $\tau \in \{\nonatt,\strdef\}$.
\end{proposition}
%\begin{proof}
 %   Let $F=(A,R)$ be any AF and $E \subset A$ a set of arguments. 
  %  \begin{description}
   %     \item[``$\sqsupseteq^{\nonatt}$'':] For every argument $a \in E$, which not attacked by $A$ it has to hold, that $a^-_{F}= \emptyset$, so $a$ has to be unattacked. However, since $E \subset A$ we know that $a \in A$ and therefore the number of unattacked arguments inside $E$ is smaller or equal to the number of unattacked arguments in $A$. Therefore $A \sqsupseteq^{\nonatt}_{F} E$ for every $E \subset A$.
    %    \item[``$\sqsupseteq^{\strdef}$'':] For set $E$ to contain any strongly defended arguments from $A$, $E$ needs to contain an unattacked argument. Like already discussed in the case above, $A$ contains every unattacked argument as well. Additionally if an argument is strongly defended by $E$ from $A$, we can use the same reasoning to show the strong defence by $A$ from $E$. Therefore the number of strongly defended arguments by $E$ is lower or equal to the number of strongly defended arguments in $A$. Hence,  $A \sqsupseteq^{\strdef}_{F} E$ for every $E \subset A$. \qedhere
    %\end{description}
%\end{proof}
The set containing all arguments $A$ is always among the most conflicting sets (for any AF with attacks), and should therefore be among the least acceptable sets, since conflict-freeness is a desirable property for reasoning with AFs. 
All of Dung's semantics \cite{DBLP:journals/ai/Dung95} use conflict-freeness as a baseline. Other extension semantics such as \emph{naive}, \emph{stage} or \emph{CF2} \cite{baroni2018handbook, verheij1996two, DBLP:journals/logcom/Verheij03, DBLP:conf/ecsqaru/BaroniG03, DBLP:journals/ai/BaroniGG05a} are based on conflict-freeness. Thus, $A$ should be ranked worse than any other set and should not be considered acceptable once we have attacks. However, only for $\sqsupseteq_{F}^{\CF}$ can we be certain, that $A$ is not among the most plausible sets. Later we will discuss solutions to avoid $A$ being in the most plausible sets. 

%---------------------------------------------------------------------
\section{Aggregation of base relations}\label{sec:combination}
In the previous section, we already hinted at our intention to propose a general framework for defining extension-ranking semantics by aggregating base relations. Several base relations are already applicable as extension-ranking semantics, but these relations focus only on a single aspect of comparing two sets of arguments and fail to encompass the entirety of argumentative evaluations. 

%\tr{I would not use an example here. The text below motivates the definitions, it does not demonstrate them. The explanation can also be simplified: each base function focuses on one aspect (minimising conflcits, maximising self-defense...) but to generalise the existing semantics, we need to combine these, just as admissibility is the combination of conflict-free and self-defense.}

 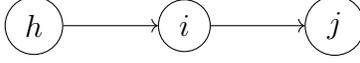
\begin{figure}
    \centering
 
 \scalebox{1}{
\begin{tikzpicture}

\node (a1) at (0,0) [circle, draw,minimum size= 0.65cm] {$h$};
\node (a2) at (2,0) [circle, draw,minimum size= 0.65cm] {$i$};
\node (a3) at (4,0) [circle, draw,minimum size= 0.65cm] {$j$};

\path[<-] (a2) edge  (a1);
\path[<-] (a3) edge  (a2);

\end{tikzpicture}}
   \caption{Recalled AF $F_5$ from Example \ref{ex:comp}.}
    \label{tikz:intro_recall}
\end{figure}

 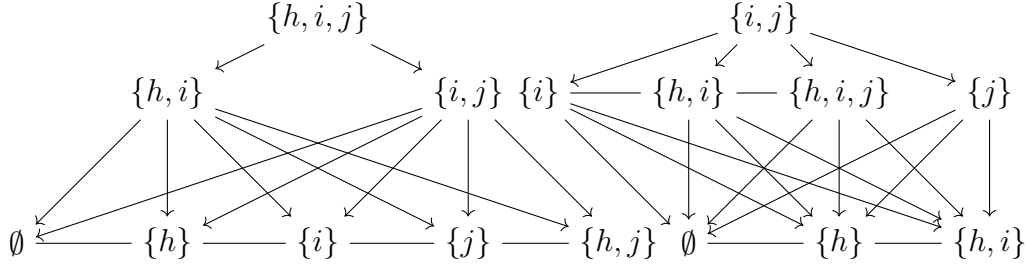
\begin{figure}
   % \centering
    \begin{minipage}[t]{0.4\textwidth}
 \scalebox{1}{
\begin{tikzpicture}

\node (abc) at (0,0) [] {$\{h,i,j\}$};
\node (ab) at (2,-1) [] {$\{i,j\}$};
\node (bc) at (-2,-1) [] {$\{h,i\}$};
\node (b) at (0,-3) [] {$\{i\}$};
\node (c) at (2,-3) [] {$\{j\}$};
\node (ac) at (4,-3) [] {$\{h,j\}$};
\node (empty) at (-4,-3) [] {$\emptyset$};
\node (a) at (-2,-3) [] {$\{h\}$};

\path[->] (abc) edge  (ab);
\path[->] (abc) edge  (bc);

\path[->] (ab) edge  (b);
\path[->] (ab) edge  (a);
\path[->] (ab) edge  (empty);
\path[->] (ab) edge  (c);
\path[->] (ab) edge  (ac);

\path[->] (bc) edge  (b);
\path[->] (bc) edge  (a);
\path[->] (bc) edge  (empty);
\path[->] (bc) edge  (c);
\path[->] (bc) edge  (ac);

\path[-] (empty) edge  (a);
\path[-] (a) edge  (b);
\path[-] (b) edge  (c);
\path[-] (c) edge  (ac);

\end{tikzpicture}}
    \end{minipage}
    \hspace{1cm}
     \begin{minipage}[t]{0.4\textwidth}
 \scalebox{1}{
\begin{tikzpicture}

\node (abc) at (1,-1) [] {$\{h,i,j\}$};
\node (ab) at (-1,-1) [] {$\{h,i\}$};
\node (bc) at (0,0) [] {$\{i,j\}$};
\node (b) at (-3,-1) [] {$\{i\}$};
\node (c) at (3,-1) [] {$\{j\}$};
\node (ac) at (3,-3) [] {$\{h,i\}$};
\node (empty) at (-1,-3) [] {$\emptyset$};
\node (a) at (1,-3) [] {$\{h\}$};

\path[->] (bc) edge  (c);
\path[->] (bc) edge  (b);
\path[->] (bc) edge  (ab);
\path[->] (bc) edge  (abc);

\path[-] (b) edge  (ab);
\path[-] (ab) edge  (abc);

\path[->] (b) edge  (empty);
\path[->] (b) edge  (a);
\path[->] (b) edge  (ac);

\path[->] (ab) edge  (empty);
\path[->] (ab) edge  (a);
\path[->] (ab) edge  (ac);

\path[->] (abc) edge  (empty);
\path[->] (abc) edge  (a);
\path[->] (abc) edge  (ac);

\path[->] (c) edge  (empty);
\path[->] (c) edge  (a);
\path[->] (c) edge  (ac);

\path[-] (a) edge  (empty);
\path[-] (a) edge  (ac);

\end{tikzpicture}}
    \end{minipage}
   \caption{Lattices depicting $\sqsupseteq^{\CF}_{F_5}$ (left) and $\sqsupseteq^{\UD}_{F_5}$ (right) from Example \ref{ex:combinitation_intro}, where $E \rightarrow E'$ means $E' \sqsupset^{\tau}_{F_5} E$ and $E - E'$ means $E' \equiv^{\tau}_{F_5} E$}\label{tikz:lattice_cfud}
\end{figure}

\begin{example}\label{ex:combinitation_intro}
Recall $F_5$ from Example \ref{ex:comp} (depicted again in Figure \ref{tikz:intro_recall}). Since $\{j\}$ is conflict-free and $\{h,i,j\}$ is not conflict-free we argued that $\{j\}$ is more plausible to be accepted than $\{h,i,j\}$. 
Using the base relations $\sqsupseteq^{\CF}_{F_5}$ gives us our intended result, i.\,e.  $\{j\} \sqsupset^{\CF}_{F_5} \{h,i,j\}$.
However, $\sqsupseteq^\CF$ does not fully model argumentative reasoning about this AF, since $\{j\}$ is equally as plausible to be accepted as the admissible set $\{h,j\}$ wrt. $\sqsupseteq^\CF_{F_5}$. But we argue that this should not be the case, $\{j\}$ is not as plausible to be accepted as any admissible set. Using $\sqsupseteq^{\UD}_{F_5}$ gives us the intended result of $\{h,j\} \sqsupset^{\UD}_{F_5} \{j\}$, but the sets $\{j\}$ and $\{h,i,j\}$ are incomparable wrt. \UD. Thus, both $\sqsupseteq^{\CF}_{F_5}$ and $\sqsupseteq^{\UD}_{F_5}$ alone are not enough to fully capture argumentative reasoning.
Figure \ref{tikz:lattice_cfud} depicts $\sqsupseteq^{\CF}_{F_5}$ and $\sqsupseteq^{\UD}_{F_5}$ as lattices.

By aggregating these two base relations, using on base relation after another, we can finally capture the entire reasoning process for $F_5$ and present a preorder one sets of arguments based on their plausibility of acceptance. In Figure \ref{tikz:lattice_intro} the resulting preorder of $\sqsupseteq^{\CF,\UD}_{F_5}$ is depicted as a lattice, where first $\sqsupseteq^{\CF}_{F_5}$ is used and then, if $\sqsupseteq^{\CF}_{F_5}$ returns equality, $\sqsupseteq^{\UD}_{F_5}$ is used.

        \begin{figure}
    \centering
 \scalebox{1}{
\begin{tikzpicture}

\node (abc) at (0,0) [] {$\{h,i,j\}$};
\node (ab) at (-2,-1) [] {$\{h,i\}$};
\node (bc) at (2,-1) [] {$\{i,j\}$};
\node (b) at (-2,-2) [] {$\{i\}$};
\node (c) at (2,-2) [] {$\{j\}$};
\node (ac) at (-2,-3) [] {$\{h,j\}$};
\node (empty) at (0,-3) [] {$\emptyset$};
\node (a) at (2,-3) [] {$\{h\}$};

\path[->] (abc) edge  (ab);
\path[->] (abc) edge  (bc);

\path[->] (bc) edge  (b);
\path[->] (bc) edge  (c);

\path[->] (ab) edge  (b);
\path[->] (ab) edge  (c);

\path[->] (b) edge  (ac);
\path[->] (b) edge  (empty);
\path[->] (b) edge  (a);

\path[->] (c) edge  (ac);
\path[->] (c) edge  (empty);
\path[->] (c) edge  (a);

\path[-] (ac) edge  (empty);
\path[-] (empty) edge  (a);
\end{tikzpicture}}
   \caption{$\sqsupseteq^{\CF,\UD}_{F_5}$ from Example \ref{ex:combinitation_intro} depicted as a lattice.}\label{tikz:lattice_intro}
\end{figure}
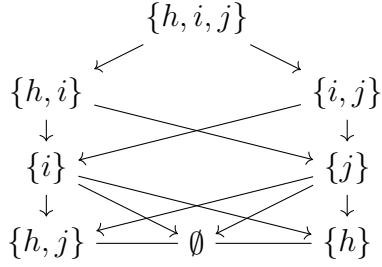
\end{example}
The example above shows that combinations of base relations are needed to define extension-ranking semantics in order to fully capture argumentative reasoning. Next, we propose a general framework for defining extension-ranking semantics using a sequence of base relations. 
\begin{definition}
    Let $F=(A,R)$ be an AF and $(\tau_1,\dots, \tau_n)$ a sequence of base relations. A \emph{aggregation method} $\mathcal{C}$ takes $(\tau_1,\dots, \tau_n)$ and returns an extension-ranking over $2^A$. 
\end{definition}
In the remainder of this section, we discuss a number of aggregation methods (\emph{Lexicographic Combination} in Section \ref{subsec:Lexi} and \emph{Voting Methods} in Section \ref{sec:voting}) and analyse the resulting extension-ranking semantics. 

\subsection{Lexicographic Combination} \label{subsec:Lexi}
In Example \ref{ex:combinitation_intro}, we aggregate $\sqsupseteq^\CF$ and $\sqsupseteq^\UD$ by applying these two base relations one after the other. If we reverse the order of using $\sqsupseteq^\UD$ and $\sqsupseteq^\CF$, then the set containing every argument $A$ is ranked quite highly, because $A$ is ranked highly wrt. $\sqsupseteq^\UD$. To avoid such behaviour we start with  $\sqsupseteq^\CF$ and refine the resulting ranking by applying other base relations one after the other. Such an aggregation is known as \emph{lexicographic combination} $\lex$. 
\begin{definition}
    Let $F=(A,R)$ be an AF, $E,E'\subseteq A$ and $(\sqsupseteq^{\tau_1}_F,\dots,\sqsupseteq_F^{\tau_n})$ be sequence of base relations. The \emph{lexicographic combination} $\lex(\sqsupseteq^{\tau_1}_F,\dots,\sqsupseteq^{\tau_n}_F)$ is defined via:  $E \sqsupset_F^{\lex(\tau_1,\dots,\tau_n)} E'$ iff there exists $i$ s.t. $E \sqsupset_F^{\tau_i} E'$ and for all $j < i, E \equiv_F^{\tau_j} E'$, and $E \equiv_F^{\lex(\tau_1,\dots,\tau_n)} E'$ iff for all $i, E \equiv_F^{\tau_i} E'$. 
\end{definition}
\begin{example}
   Let us look again at Example \ref{ex:combinitation_intro}. Formally the lattice in Figure \ref{tikz:lattice_intro} is produced by lexicographily combining  $\sqsupseteq^\CF_{F_5}$ and $\sqsupseteq^\UD_{F_5}$ i.e. $\lex(\sqsupseteq^\CF_{F_5}, \sqsupseteq^\UD_{F_5})$. For the three sets $\{h,i\}$, $\{i\}$ and $\{h,j\}$ we have: $$\{h,j\}\equiv^\CF_{F_5} \{i\} \sqsupset^\CF_{F_5} \{h,i\}$$ Thus, in-order to differentiate $\{h,j\}$ and $\{i\}$ we use $\sqsupseteq^\UD_{F_5}$ next. So, $\{h,j\} \sqsupset^\UD_{F_5} \{i\}$ and this entails the final preorder: $$\{h,j\} \sqsupset^{\lex(\CF,\UD)}_{F_5} \{i\} \sqsupset^{\lex(\CF,\UD)}_{F_5} \{h,i\}$$ 
\end{example}

Next, we look at certain sequences of base relations and study their behaviour. 
\subsubsection{Admissible Extension-ranking Semantics}\label{subsubsec:r-ad}
Revisiting the definition of admissibility, we see that a set $E$ is admissible if and only if $E$ has no internal conflicts and contains no undefended argument. These two aspects are captured by $\CF$ and $\UD$ respectively. Indeed, the lexicographic combination of the two base relations $\sqsupseteq^\CF$ and $\sqsupseteq^\UD$ provide a generalisation of the admissible extensions semantics.   
\begin{definition}\label{defn:adm-ranking-semantics}
    Let $F=(A,R)$ be an AF and $E,E'\subseteq A$. Define the \emph{admissbile extension-ranking semantics} $\rAd$ via $$E \sqsupseteq^{\rAd}_F E' \text{ iff } E \sqsupseteq^{\lex(\CF,\UD)}_F E'$$
\end{definition}
In other words, $E$ is more plausible to be accepted than $E'$ if $E$ is more plausible to be accepted than $E'$ wrt. $\sqsupseteq^\CF$ or in case of equality $E$ is more plausible to be accepted than $E'$ wrt. $\sqsupseteq^\UD$. 
We rename $\lex(\sqsupseteq^\CF,\sqsupseteq^\UD)$ to $\rAd$ to show the behaviour of the preorder and also to be inline with other works where $\rAd$ was already defined \cite{DBLP:conf/ijcai/SkibaRTHK21}.

\begin{example}\label{ex:rAD}
 Consider $F_{4}$ from Example~\ref{ex:af_example}. If we compare the sets $\{a,b\}$, $\{b\}$ and $\{b,f\}$, we see that $\{b\}$ and $\{b,f\}$ are conflict-free, while $\{a,b\}$ is not. Therefore we use $\sqsupseteq^\CF_{F_4}$  and compare these sets using the conflicts base relation. Comparing $\{b\}$ and $\{b,f\}$ further we see that both these sets are not admissible, however using \UD\ we get: $\UD_{F_4}(\{b\}) = \{b\}$ and $\UD_{F_4}(\{b,f\})= \{b,f\}$, so $\{b\} \sqsupset_{F_{4}}^{\rAd} \{b,f\}$. Therefore set $\{b\}$ is closer to be admissible than $\{b,f\}$. Extending our investigation to the three sets introduced we get:
 $$\{b\} \sqsupset_{F_{4}}^{\rAd} \{b,f\}  \sqsupset_{F_{4}}^{\rAd} \{a,b\} $$

 If we comparing the admissible sets $\{a,c,g\}$ and $\{a,d,g\}$ we have: 
 \begin{align*}
     &\CF_{F_4}(\{a,c,g\}) = \CF_{F_4}(\{a,d,g\}) = \emptyset \\
     &\UD_{F_4}(\{a,c,g\}) = \UD_{F_4}(\{a,d,g\}) = \emptyset 
 \end{align*}
 So, $\{a,c,g\} \equiv_{F_{4}}^{\rAd} \{a,d,g\}$ and more importantly $\{a,c,g\}, \{a,d,g\} \in \maxpl_{\rAd}(F_4)$.
 
 For the two conflict-free sets $\{b,g\}$ and $\{a,f\}$ we get $\UD_{F_4}(\{b,g\})= \{b\}$ and $\UD_{F_4}(\{a,f\})= \{f\}$ therefore these two sets are incomparable, i.e.\ $$\{b,g\} \asymp_{F_{4}}^{\rAd} \{a,f\}$$
 \end{example}
The full preorder of $\sqsupseteq^\rAd_{F_5}$ from Example \ref{ex:combinitation_intro} can be found in Figure \ref{tikz:lattice_intro}. We see that the three admissible sets $\{h,j\}$, $\emptyset$, and  $\{h\}$ are the most plausible sets, every not admissible set is ranked worse. Additionally, the set $\{i\}$ is closer to be admissible than $\{h,i\}$, because it is more plausible to be accepted wrt. $\sqsupseteq^\rAd_{F_5}$.

Next, we show the compliance of the admissible extension-ranking semantics with the principles defined in Section \ref{sec:principles}. 

First, we show that $\rAd$ is indeed a generalisation of admissibility. 
$\rAd$ is based on the lexicographic combination of $\sqsupseteq^\CF$ and $\sqsupseteq^\UD$, so first checking for minimal conflicts and then minimal undefended arguments. If both these relations return $\emptyset$ for a set $E$, then $E$ is admissible. 
\begin{proposition}
$\rAd$ satisfies $\ad$-generalisation.
\end{proposition}

The following two lemmas will be useful to prove (de)composition. 

\begin{lemma}\label{lem:composition_decomposition_1}
For $\tau \in \{\CF, \UD, \DN, \UA\}$, if $F_1=(A_1,R_1)$, $F_2=(A_2,R_2)$ and $F = F_1 \cup F_2= (A_1, R_1) \cup (A_2,R_2)$ with $A_1 \cap A_2 = \emptyset$ then $\tau_{F_1}(E\cap A_1) \cup \tau_{F_2}(E \cap A_2)= \tau_F(E)$ for every $E \subseteq A_1 \cup A_2$.
\end{lemma}

\begin{lemma}\label{lem:composition_decomposition_2}
    The base relation $\sqsupseteq^\tau_{F}$ satisfies composition and decomposition for $\tau \in \{\CF, \UD, \DN, \UA\}$.
\end{lemma}

So, these two lemmas show that the base relations do indeed satisfy composition and decomposition. Using this information we can easily show that the aggregation of these base relations also satisfy composition and decomposition.  

The lexicographic combination of $\sqsupseteq^\CF$ and $\sqsupseteq^\UD$ satisfies composition and decomposition. 
\begin{proposition}
    $\rAd$ satisfies composition and decomposition.
\end{proposition}
%\begin{proof}
 %   Follows from Lemma \ref{lem:composition_decomposition_2} together with Definition \ref{defn:adm-ranking-semantics}.
%\end{proof}

The \emph{Fundamental Lemma} states that the addition of a defended argument $a$ into an admissible set $E$ does not destroy the admissibility of $E$, hence weak reinstatement is satisfied while strong reinstatement is violated. 
\begin{proposition}\label{prop:rAD_reinstatement}
   $\rAd$ satisfies weak reinstatement.
\end{proposition}
%\begin{proof} Let $F= (A,R)$ be an AF and $E \subseteq A$.
 %   Suppose $a \in \mathcal{F}_{F}(E), a \notin E$ and $a \notin (E^- \cup E^+)$. Then $\CF_F(E \cup \{a\})= \CF_F(E)$. What remains is to prove that $\UD_F(E \cup \{a\}) \subseteq \UD_F(E)$. Suppose $x \in \UD_F(E \cup \{a\})$. Then $x \in E \cup \{a\}$ and $x \notin \mathcal{F}_{F}(E \cup \{a\})$. Because $a \in \mathcal{F}_{F}(E \cup \{a\})$ it follows that $x \in E$. Furthermore, since $x \notin \mathcal{F}_{F}(E \cup \{a\})$ we also have $x \notin \mathcal{F}_{F}(E)$. This implies that $x \in \UD_F(E)$. We thus have that $E \cup \{a\} \equiv^{\CF}_{F} E$ and $E \cup \{a\} \sqsupseteq^{\UD}_{F} E$ and hence, $E \cup \{a\} \sqsupseteq^{\rAd}_{F} E$.
%\end{proof}

\begin{example}\label{ex:AD_not_reinstatment}
Consider the argumentation framework $F_5= \{\{h,i,j\}, \{(h,i),\\(i,j)\}\}$ from Example \ref{ex:comp}. The sets $E_1 = \{h\}$ and $E_2 = \{h,j\}$ are both admissible. By applying the admissible extension-ranking semantics, we get $E_1 \equiv_{F_5}^{\rAd} E_2$. However, $j$ is the only suitable candidate to be reinstated for $E_1$. But $E_1 \cup \{j\} \not \sqsupset_{F_5}^{\rAd} E_1$, as they are equally plausible to be accepted. So, the admissible extension-ranking semantics does not satisfy strong reinstatement.
\end{example}

Adding an outgoing attack into an AF will not remove any admissible sets, hence addition robustness is satisfied.
\begin{proposition}
    $\rAd$ satisfies addition robustness.
\end{proposition}

%Since $\rAd$ satisfies weak reinstatement and violates strong reinstatement we know that $\INMAX$ is violated as well. 

\subsubsection{Complete Extension-ranking Semantics}\label{subsubsec:r-co}
In the previous subsection, we discussed a generalisation of admissibility, now the next step is clear. Can we generalise the remaining extension semantics in a similar way? We start with the complete extension semantics. Complete extensions are admissible sets, where every defended argument is included. So, a generalisation of the complete extension semantics should be based on the admissible extension-ranking semantics and refined by favouring sets that contain more defended arguments. The behaviour of the complete extension semantics can be modelled with a lexicographic combination of $\sqsupseteq^\rAd$ and $\sqsupseteq^\DN$. 
\begin{definition}\label{defn:co-ranking-semantics}
    Let $F= (A,R)$ be an AF and $E,E' \subseteq A$. Define the \emph{complete extension-ranking semantics} $\rCo$ via:
    $$E \sqsupseteq^{\rCo}_F E' \text{ iff } E \sqsupseteq^{\lex(\CF,\UD,\DN)}_F E'$$
\end{definition}
So, if $E$ is strictly more plausible to be accepted than $E'$ wrt. $\sqsupseteq^{\rAd}$ then this relationship holds for $\rCo$ as well, only if the two sets are equally plausible to be accepted wrt. $\sqsupseteq^{\rAd}$, the base relation $\sqsupseteq^\DN$ is used to. 
\begin{example}\label{ex:rCO}
 Consider $F_{4}$ from Example~\ref{ex:af_example} and sets $\{b\}$, $\{g\}$ and $\{d\}$. While $\{g\}$ and $\{d\}$ are both admissible set $\{b\}$ is not admissible, so we can use $\sqsupseteq^\rAd$ to receive $\{g\} \sqsupset^{\rCo}_{F_4} \{b\}$ respectively $\{d\} \sqsupset^{\rCo}_{F_4} \{b\}$. To compare $\{g\}$ and $\{d\}$ further we look at $\sqsupseteq^\DN$. Argument $a$ is defended by both these sets, while $\{d\}$ also defends $g$ in addition to $a$. So, $\DN_{F_4}(\{g\}) = \{a\}$ and $\DN_{F_4}( \{d\})= \{a,g\}$ and therefore $\{g\} \sqsupset_{F_{4}}^{\rCo} \{d\}$. 
 Since $\sqsupseteq^\rCo$ is built on top of $\sqsupseteq^\rAd$ we can extend the result of Example \ref{ex:rAD} further:
 $$ \{g\} \sqsupset_{F_{4}}^{\rCo} \{d\} \sqsupset^{\rCo}_{F_4} \{b\} \sqsupset_{F_{4}}^{\rCo} \{b,f\}  \sqsupset_{F_{4}}^{\rCo} \{a,b\} $$
 So among these sets $\{g\}$ is the closed set to be complete, while not being complete itself. 
\end{example}
\begin{example}\label{ex:lattice_rco}
    Consider again $F_5$ from Example \ref{ex:combinitation_intro}. The lattice in Figure \ref{tikz:lattice_rco} depicts the preorder $\sqsupseteq^{\rCo}_{F_5}$. The first three level are the same as in the preorder of $\sqsupseteq^{\rAd}_{F_5}$, however with $\sqsupseteq^\rCo_{F_5}$ we can differentiate the sets $\emptyset$, $\{h\}$, and $\{h,j\}$. The complete set $\{h,j\}$ is the most plausible set, however between the two admissible sets $\emptyset$ and $\{h\}$ we can say that $\{h\}$ is closer to be complete than $\emptyset$.
           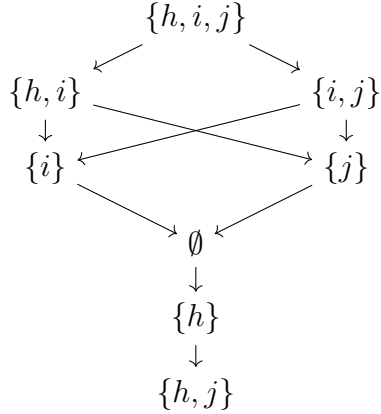
\begin{figure}
    \centering
 \scalebox{1}{
\begin{tikzpicture}

\node (abc) at (0,0) [] {$\{h,i,j\}$};
\node (ab) at (-2,-1) [] {$\{h,i\}$};
\node (bc) at (2,-1) [] {$\{i,j\}$};
\node (b) at (-2,-2) [] {$\{i\}$};
\node (c) at (2,-2) [] {$\{j\}$};
\node (ac) at (0,-5) [] {$\{h,j\}$};
\node (empty) at (0,-3) [] {$\emptyset$};
\node (a) at (0,-4) [] {$\{h\}$};

\path[->] (abc) edge  (ab);
\path[->] (abc) edge  (bc);

\path[->] (bc) edge  (b);
\path[->] (bc) edge  (c);

\path[->] (ab) edge  (b);
\path[->] (ab) edge  (c);

\path[->] (b) edge  (empty);

\path[->] (c) edge  (empty);

\path[->] (a) edge  (ac);
\path[->] (empty) edge  (a);
\end{tikzpicture}}
   \caption{$\sqsupseteq^{\rCo}_{F_5}$ from Example \ref{ex:lattice_rco} depicted as a lattice, where $E \rightarrow E'$ means $E' \sqsupset^{\rCo}_{F_5} E$.}\label{tikz:lattice_rco}
\end{figure}
\end{example}

The complete semantics is a refinement of admissibility, hence refining $\sqsupseteq^\rAd$ with the help of an additional base relation here $\sqsupseteq^\DN$, is intuitive. This combination is indeed a generalisation of the complete extension semantics. 
\begin{proposition}
$\rCo$ satisfies $\co$-generalisation.
\end{proposition}

Like for $\rAd$ the lexicographic combination of $\sqsupseteq^\CF$, $\sqsupseteq^\UD$, and $\sqsupseteq^\DN$ does not violate composition and decomposition. 
\begin{proposition}
    $\rCo$ satisfies composition and decomposition.
\end{proposition}

Similar to the extension-based version of the complete semantics, the strong reinstatement property is satisfied. 
\begin{proposition}\label{prop:rCO_reinstatement}
   $\rCo$ satisfies strong reinstatement. 
\end{proposition}

Since the addition of an attack can defend additional arguments, we see that the complete extension-ranking semantics violates addition robustness, because a set can change from complete to not complete by adding one outgoing attack.
\begin{example}\label{ex:co_violates_robustness}
\begin{figure}
    \centering
 
 \scalebox{1}{
\begin{tikzpicture}

\node (a) at (0,0) [circle, draw,minimum size= 0.65cm] {$a$};
\node (b) at (0,2) [circle, draw,minimum size= 0.65cm] {$b$};
\node (c) at (2,0) [circle, draw,minimum size= 0.65cm] {$c$};
\node (d) at (2,2) [circle, draw,minimum size= 0.65cm] {$d$};
\node (e) at (4,0) [circle, draw,minimum size= 0.65cm] {$e$};
\node (f) at (6,0) [circle, draw,minimum size= 0.65cm] {$f$};
%\node (g) at (0,0) [circle, draw] {$g$};

\path[->, bend left] (a) edge  (c);
\path[->, bend left] (c) edge  (a);
\path[->, bend left] (b) edge  (d);
\path[->, bend left] (d) edge  (b);

\path[->] (c) edge  (e);
\path[->] (d) edge  (e);
\path[->] (e) edge  (f);

\path[->, dashed] (a) edge  (b);

\end{tikzpicture}}
   \caption{Abstract argumentation framework $F_{12}$ from Example \ref{ex:co_violates_robustness}, where the dashed attack from $a$ to $b$ is added to obtain $F_{12}'$.}
    \label{tikz:robustness}
\end{figure}
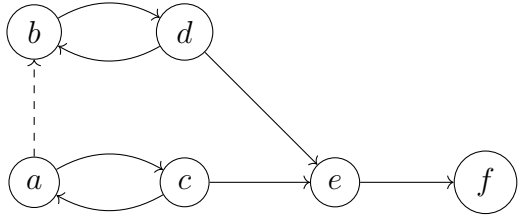
    Let $F_{12}$ be the AF depicted in Figure \ref{tikz:robustness}. Consider the two sets $E = \{a\}$ and $E'= \{b,c\}$. $E$ is a complete extension, while $E'$ is not a complete extension, so $E \sqsupset^{\rCo}_{F_{12}} E'$. So, if addition robustness is satisfied, we can add an attack between $E$ and $E'$ and keep this relationship. Say we add $(a,b)$ into $F_{12}$ to create $F_{12}'$, like depicted with the dashed line in Figure \ref{tikz:robustness}. Then $E$ is no longer a complete set, since arguments $d$ and $f$ are defended so $\DN_{F_{12}'}(E) = \{d,f\}$, while for $E'$ we have $\DN_{F_{12}'}(E') = \{f\}$, hence $E' \sqsupset^{\rCo}_{F_{12}'} E$, which violates addition robustness.
\end{example}

%The complete extension semantics is a well-known semantics that violates \emph{I-maximality} and therefore we know that $\rCo$ violates \INMAX\ as-well. 

\subsubsection{Grounded Extension-ranking Semantics}\label{subsubsec:r-gr}
The grounded extension is defined as the minimal complete extension, so it is natural to refine the complete extension-ranking semantics by a notion of minimality to obtain a generalisation of the grounded extension semantics. Thus, before we can define the grounded extension-ranking semantics we need to define a base relation that models the notion of minimality.
\begin{definition}\label{def:min_bf}
Let $F=(A,R)$ be an AF and $E,E' \subseteq A$. Define the \emph{minimality base relation} \Min\ via:
$$E \sqsupseteq^{\Min}_F E' \text{ iff } E \subseteq E'$$
\end{definition}
In other words, a set $E$ is more plausible to be accepted than $E'$ if $E$ is smaller (wrt. subset comparisons) than $E'$. Note, that this base relation is only useful in combination with other base relation, since the structure of the underlying AF is completely irrelevant to establish the plausibility of acceptance of a set. 

Now, we can define a generalisation of the grounded extension semantics by refining $\sqsupseteq^\rCo$ with $\sqsupseteq^{\Min}$.
\begin{definition}\label{defn:gr-ranking-semantics}
    Let $F= (A,R)$ be an AF and $E,E' \subseteq A$. Define \emph{grounded extension-ranking semantics} $\rGr$ via:
    $$ E \sqsupseteq^{\rGr}_F E' \text{ iff } E \sqsupseteq^{\lex(\CF,\UD,\DN,\Min)}_F E'$$
\end{definition}

       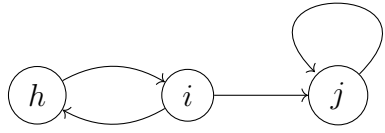
\begin{figure}
    \centering
 
 \scalebox{1}{
\begin{tikzpicture}

\node (a1) at (0,0) [circle, draw,minimum size= 0.65cm] {$h$};
\node (a2) at (2,0) [circle, draw,minimum size= 0.65cm] {$i$};
\node (a3) at (4,0) [circle, draw,minimum size= 0.65cm] {$j$};

\path[->,bend left] (a2) edge  (a1);
\path[->,bend left] (a1) edge  (a2);

\path[<-] (a3) edge  (a2);
\path[->, loop] (a3) edge  (a3);

\end{tikzpicture}}
   \caption{AF $F_{13}$ from Example \ref{ex:lattice_rgr}.}
    \label{tikz:rgr}
\end{figure}

\begin{example}\label{ex:rGR}
 Consider $F_{4}$ from Example~\ref{ex:af_example} and sets $\{g\}$, $\{a,g\}$ and $\{a,c,g\}$. $\{a,g\}$ and $\{a,c,g\}$ are both complete extensions, while $\{g\}$ is not complete. Hence, $\rCo$ gives us $\{a,g\} \sqsupset_{F_{4}}^{\rGr} \{g\}$ and $\{a,c,g\} \sqsupset_{F_{4}}^{\rGr} \{g\}$. $\{a,g\}$ and $\{a,c,g\}$ are not grounded, however $\{a,g\} \subset \{a,c,g\}$, so $\{a,g\} \sqsupset_{F_{4}}^{\rGr} \{a,c,g\}$. Using the results from Example \ref{ex:rCO} we can extend the ranking:
  $$\{a,g\} \sqsupset_{F_{4}}^{\rGr} \{a,c,g\} \sqsupset_{F_{4}}^{\rGr}  \{g\} \sqsupset_{F_{4}}^{\rGr} \{d\} \sqsupset^{\rGr}_{F_4} \{b\} \sqsupset_{F_{4}}^{\rGr} \{b,f\}  \sqsupset_{F_{4}}^{\rGr} \{a,b\} $$ 
\end{example}

\begin{example}\label{ex:lattice_rgr}
    Let $F_{13}= (\{h,i,j\},\{(h,i),(i,h),(i,j),(j,j)\})$ as depicted in Figure \ref{tikz:rgr}. The full preorders of $\sqsupseteq^{\rCo}_{F_{13}}$ (left) and $\sqsupseteq^{\rGr}_{F_{13}}$ (right) are depicted as lattices in Figure \ref{tikz:lattice_rgr}. The main difference between these two preorders is that $\emptyset$ is ranked better than $\{h\}$ and $\{i\}$ in $\sqsupseteq^{\rGr}_{F_{13}}$. While $\{h\}$ and $\{i\}$ are not the grounded extensions, these two sets are still considered more plausible to be accepted than every other set except $\emptyset$, because they are complete extensions. 
     
        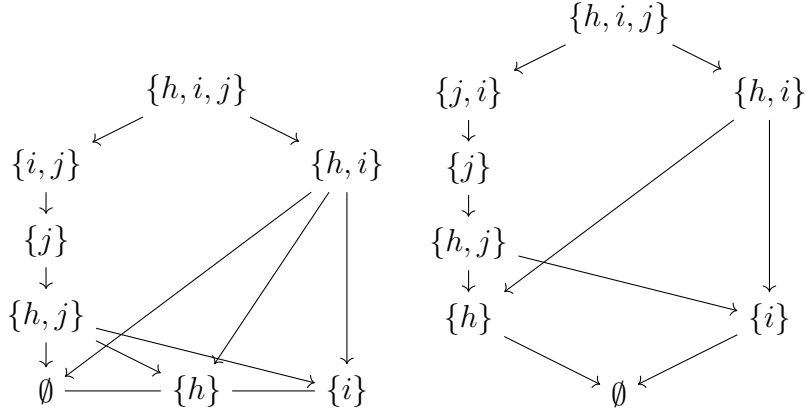
\begin{figure}
    \centering
    \begin{minipage}[t]{0.4\textwidth}
 \scalebox{1}{
\begin{tikzpicture}

\node (abc) at (0,0) [] {$\{h,i,j\}$};
\node (ab) at (2,-1) [] {$\{h,i\}$};
\node (bc) at (-2,-1) [] {$\{i,j\}$};
\node (b) at (2,-4) [] {$\{i\}$};
\node (c) at (-2,-2) [] {$\{j\}$};
\node (ac) at (-2,-3) [] {$\{h,j\}$};
\node (empty) at (-2,-4) [] {$\emptyset$};
\node (a) at (0,-4) [] {$\{h\}$};

\path[->] (abc) edge  (ab);
\path[->] (abc) edge  (bc);

\path[->] (bc) edge  (c);
\path[->] (c) edge  (ac);

\path[->] (ac) edge  (empty);
\path[->] (ac) edge  (a);
\path[->] (ac) edge  (b);

\path[->] (ab) edge  (b);
\path[->] (ab) edge  (a);
\path[->] (ab) edge  (empty);
\path[-] (empty) edge  (a);
\path[-] (b) edge  (a);

\end{tikzpicture}}
    \end{minipage}
     \begin{minipage}[t]{0.4\textwidth}
 \scalebox{1}{
\begin{tikzpicture}

\node (abc) at (0,0) [] {$\{h,i,j\}$};
\node (ab) at (2,-1) [] {$\{h,i\}$};
\node (bc) at (-2,-1) [] {$\{j,i\}$};
\node (b) at (2,-4) [] {$\{i\}$};
\node (c) at (-2,-2) [] {$\{j\}$};
\node (ac) at (-2,-3) [] {$\{h,j\}$};
\node (empty) at (0,-5) [] {$\emptyset$};
\node (a) at (-2,-4) [] {$\{h\}$};

\path[->] (abc) edge  (ab);
\path[->] (abc) edge  (bc);

\path[->] (bc) edge  (c);
\path[->] (c) edge  (ac);

\path[->] (ac) edge  (a);
\path[->] (ac) edge  (b);

\path[->] (ab) edge  (b);
\path[->] (ab) edge  (a);
\path[->] (a) edge  (empty);
\path[->] (b) edge  (empty);

\end{tikzpicture}}
    \end{minipage}
   \caption{Lattices depicting $\sqsupseteq^{\rCo}_{F_{13}}$ (left) and $\sqsupseteq^{\rGr}_{F_{13}}$ (right) from Example \ref{ex:lattice_rgr}.}\label{tikz:lattice_rgr}
\end{figure}

\end{example}

Note that by definition of the grounded extension-ranking semantics is a refinement of the complete extension-ranking semantics, since a set of arguments $E$ that is more plausible to be accepted than a other set $E'$ with respect to the complete extension-ranking semantics, is also more plausible to be accepted than $E'$ with respect to the grounded extension-ranking semantics. Thus the complete and the grounded extension-ranking semantics do behave to each other as the classic Dung complete and grounded extension semantics (where the grounded extension semantics is a refinement of the complete extension semantics).

The unique grounded extension is the smallest set among the complete extensions, hence creating a preorder based on $\sqsupseteq^{\rCo}$, which is then refined based on the size of the sets entails a generalisation of the grounded semantics. 
\begin{proposition}
$\rGr$ satisfies $\gr$-generalisation.
\end{proposition}

Since $\rGr$ is based on $\rCo$ we can use the results of $\rCo$ for the proofs of the remaining principles. Since $\rCo$ satisfies composition, decomposition, and strong reinstatement, $\rGr$ satisfies these principles as well and violates addition robustness. 

\begin{proposition}
    $\rGr$ satisfies composition and decomposition.
\end{proposition}
%\begin{proof}
 %   Follows from Lemma \ref{lem:composition_decomposition_2} together with Definition \ref{defn:gr-ranking-semantics}.
%\end{proof}

\begin{proposition}
   $\rGr$ satisfies strong reinstatement. 
\end{proposition}
%\begin{proof}
 %   This follows directly from Proposition \ref{prop:rCO_reinstatement}.
%\end{proof}

\begin{proposition}
    $\rGr$ violates addition robustness. 
\end{proposition}
%\begin{proof}
 %   Since, $\sqsupseteq^{\rGr}_{F}$ is based on $\sqsupseteq^{\rCo}_{F}$ we can use Example \ref{ex:co_violates_robustness} to show that  $\rGr$ violates addition robustness as well.
  %  \end{proof}

\subsubsection{Preferred Extension-ranking Semantics}\label{subsubsec:r-pr}
While the grounded extension semantics uses a minimisation of sets, the preferred extension semantics uses a maximisation to reason. More precisely, the preferred extensions are the maximally admissible sets. Therefore, to generalise the preferred extension semantics, a notion of maximisation is needed. 
\begin{definition}
    Let $F= (A,R)$ be an AF and $E,E' \subseteq A$. We define the \emph{maximality base relation} \Max\ via:
    $$E \sqsupseteq^{\Max}_F E' \text{ iff } E \supseteq E'$$
\end{definition}
So, a superset $E$ is more plausible to be accepted than its subset $E'$.
Similar to the minimality base relation \Min\ defined in Definition \ref{def:min_bf}, \Max\ should only be used in combination with other base relations.

To define a generalisation of the preferred extension semantics we lexicographicaly combine $\sqsupseteq^\rAd$ with $\sqsupseteq^{\Max}$. 
\begin{definition}\label{defn:pr-ranking-semantics}
    Let $F=(A,R)$ be an AF and $E,E' \subseteq A$. Define \emph{preferred extension-ranking semantics} $\rPr$ via:
    $$E \sqsupseteq^{\rPr}_F E' \text{ iff } E \sqsupseteq^{\lex(\CF,\UD,\Max)}_F E'$$
\end{definition}

\begin{example}\label{ex:lattice_rpr}
 We continue Example \ref{ex:af_example}. Consider sets $\{b\}$, $\{a\}$ and $\{a,g\}$. $\{a\}$ and $\{a,g\}$ are admissible sets and $\{b\}$ is not admissible. So, based on $\rAd$ we have $\{a\} \sqsupset_{F_{4}}^{\rPr}  \{b\}$ and  $\{a,g\} \sqsupset_{F_{4}}^{\rPr}  \{b\}$. Additionally, both $\{a\}$ and $\{a,g\}$ are not preferred, however $ \{a\} \subset \{a,g\}$, so $\{a,g\} \sqsupset_{F_{4}}^{\rPr} \{a\}$ meaning that $\{a,g\}$ is closer to be preferred than $\{a\}$. Extending the ranking from Example \ref{ex:rAD} we get:
  $$\{a,g\} \sqsupset_{F_{4}}^{\rPr} \{a\} \sqsupset_{F_{4}}^{\rPr}  \{b\} \sqsupset_{F_{4}}^{\rPr} \{b,f\}  \sqsupset_{F_{4}}^{\rPr} \{a,b\} $$
\end{example}
\begin{example}
    Consider $F_{13}$ from Example \ref{ex:lattice_rgr}. The preorder $\sqsupseteq^{\rPr}_{F_{13}}$ is depicted in Figure \ref{tikz:lattice_rpr}. If we compare the lattice for $\sqsupseteq^{\rPr}_{F_{13}}$ with the lattice for $\sqsupseteq^{\rCo}_{F_{13}}$, we see that the most plausible sets wrt. $\sqsupseteq^{\rCo}_{F_{13}}$ are split. $\emptyset$ is ranked worse than $\{h\}$ and $\{i\}$, which holds true for the preferred extension semantics as well, since $\emptyset$ is not a preferred extension, while $\{h\}$ and $\{i\}$ are. 
    \begin{figure}
    \centering
 \scalebox{1}{
\begin{tikzpicture}

\node (abc) at (0,0) [] {$\{h,i,j\}$};
\node (ab) at (2,-1) [] {$\{h,i\}$};
\node (bc) at (-2,-1) [] {$\{i,j\}$};
\node (b) at (2,-5) [] {$\{i\}$};
\node (c) at (-2,-2) [] {$\{j\}$};
\node (ac) at (-2,-3) [] {$\{h,j\}$};
\node (empty) at (0,-4) [] {$\emptyset$};
\node (a) at (-2,-5) [] {$\{h\}$};

\path[->] (abc) edge  (ab);
\path[->] (abc) edge  (bc);

\path[->] (bc) edge  (c);
\path[->] (c) edge  (ac);

\path[->] (ac) edge  (empty);

\path[->] (ab) edge  (empty);
\path[->] (empty) edge  (a);
\path[->] (empty) edge  (b);
\path[-] (a) edge  (b);

\end{tikzpicture}}

   \caption{Lattices depicting $\sqsupseteq^{\rPr}_{F_{13}}$ from Example \ref{ex:lattice_rgr}.}\label{tikz:lattice_rpr}
\end{figure}

\end{example}
Sometimes the preferred extensions are defined as maximal complete extensions. For extension semantics this change in definition does not change the acceptance of any set, since every preferred extension is also a complete extension. However, if we define the preferred extension-ranking semantics based on completeness and not on admissibility we obtain different extension rankings. 
\begin{definition}\label{defn:co-pr-ranking-semantics}
    Let $F=(A,R)$ and $E, E' \subseteq A$. We define the \emph{complete-preferred extension-ranking semantics} \emph{\rCoPr} via
\begin{align}
E \sqsupseteq_{F}^{\rCoPr} E' & \mbox{ iff }E \sqsupseteq_{F}^{\lex(\CF,\UD,\DN,\Max)} E'
\end{align}
\end{definition}
\begin{example}\label{ex:rCoPr}
    Consider $F_4$ from Example \ref{ex:af_example} and sets $\{d\}$ and $\emptyset$. If we use $\rPr$ we obtain $$ \{d\} \sqsupset_{F_4}^{\rPr}\emptyset $$ since both these sets are admissible and $\{d\}$ is the bigger set. However, using $\rCoPr$ gives us $$\emptyset \sqsupset_{F_4}^{\rCoPr} \{d\}$$ since both these sets are not complete but $\DN_{F_4}(\{d\})= \{a,g\}$ and $\DN_{F_4}(\emptyset)= \{a\}$. So, the switch from admissibility to completeness does indeed change our rankings.

    For a second example consider sets $\{a,g\}$, $\{a,c\}$ and $\emptyset$. All three sets are admissible therefore they are equally ranked with respect to $\sqsupseteq^\rAd_{F_4}$. However, $\emptyset \subset \{a,g\}$ and $\{a,c\}$, so $ \{a,g\} \sqsupset_{F_4}^{\rPr} \emptyset $ and $ \{a,c\} \sqsupset_{F_4}^{\rPr}\emptyset $, but $\{a,g\}$ and $\{a,c\}$ are incomparable with respect to $\subseteq$, so:
    $$ \{a,g\} \asymp^{\rPr}_{F_4} \{a,c\} \sqsupset_{F_4}^{\rPr} \emptyset $$
    For $\sqsupseteq^\rCoPr_{F_4}$ we compare $\{a,g\}$, $\{a,c\}$ and $\emptyset$ via $\sqsupseteq^\DN$. $\{a,g\}$ is a complete extension so $\DN_{F_4}(\{a,g\})=\emptyset$ the other two sets are not complete i.e. $\DN_{F_4}(\{a,c\})= g$ and $\DN_{F_4}(\emptyset)= a$. So, among these three sets $\{a,g\}$ is the closed to be preferred, while the other two sets are incomparable. 
     $$ \{a,g\} \sqsupset_{F_4}^{\rCoPr}  \{a,c\} \asymp^{\rCoPr}_{F_4}  \emptyset $$
     
\end{example}
The example above shows that $\rPr$ and $\rCoPr$ are not refinements of each other and they produce different rankings. 

Despite the differences between $\rPr$ and $\rCoPr$, both satisfy $\pr$-generalisation, so both these versions are generalisations of the preferred extension semantics.
Thus, depending on the context and application, either version can be used as a generalisation of the preferred extension semantics. In deciding which version to use, the importance of completeness of sets must be questioned. If sets that are consistently defended and contain more arguments are to be favoured, then $\rCoPr$ should be used, otherwise $\rPr$ can be used instead.
\iffalse
\mt{what about other properties of classical semantics and how they could be translated to extension-ranking semantics? 
%For example, the grounded extension is a subset of every complete/preferred extension. By definition, it holds that the grounded extension-ranking is a refinement of the complete extension-ranking, but this should still be given as an explicit result.
Are there other relationships? For example, do the rankings agree on the order of non-extensions? Is there some operation one can perform to ''collapse'' the complete extension-ranking to the grounded extension-ranking (so mimicking the intersection operation)? What happens if we apply this operation to the preferred-extension ranking? Do we get something similar as an ''ideal extension-ranking''? What is the relationship between the grounded extension-ranking and the complete-preferred extension-ranking?}
\todo[inline]{if time: think about other properties from extension semantics that can be translated to extension-ranking semantics}
\fi
%A preferred extension is a maximally admissible set, so refining $\sqsupseteq^{\rAd}_{F}$ (or $\sqsupseteq^{\rCo}_{F}$) with maximisation will yield a generalisation of the preferred extension semantics. 
\begin{proposition}
$\rPr$ and $\rCoPr$ are satisfying $\pr$-generalisation.
\end{proposition}

Similar to the other cases a lexicographic combination of $\sqsupseteq^{\rAd}_{F}$ resp.  $\sqsupseteq^{\rCo}_{F}$ and maximisation does satisfy composition and decomposition. 
\begin{proposition}
    $\rPr$ and $\rCoPr$ are satisfying composition and decomposition.
\end{proposition}

The main difference between $\sqsupseteq^{\rAd}_{F}$ and $\sqsupseteq^{\rPr}_{F}$ (resp. $\sqsupseteq^{\rCoPr}_{F}$) is the maximisation in the preferred case. So, adding an argument, which is defended and does not create new conflicts, into the set should make the set more plausible to be accepted based on preferred reasoning.
\begin{proposition}
   $\rPr$ and $\rCoPr$ are satisfying strong reinstatement. 
\end{proposition}
%\begin{proof}
%We only show it for $\rPr$, however the proof for $\rCoPr$ is similar. 

 %   Let $a \in \mathcal{F}_{F}(E)$, $a \notin E$ and $a \notin (E^- \cup E^+)$. Then Proposition \ref{prop:rAD_reinstatement} implies, that $\{a\} \cup E \sqsupseteq^{\rPr}_{F} E$. It also hold that $E \subset E \cup \{a\}$. Hence, $E \cup \{a\} \sqsupset^{\rPr}_{F} E$. So the preferred extension-ranking semantics does satisfy strong reinstatement.
%\end{proof}

To show that addition robustness is violated by $\rPr$ and $\rCoPr$ we can use Example \ref{ex:pr_robustness} again. 

\subsubsection{(Semi)-Stable Extension-ranking Semantics}\label{subsubsec:r-sst}
For a set of arguments to be stable, it must be conflict-free and attack all arguments not contained in that set. 
We have already shown in Proposition \ref{thm:no st-gen} that there are no extension-ranking semantics, that satisfy $\st$-soundness and therefore also $\st$-generalisation, so although we could consider defining an extension-ranking semantics by lexicographically combining $\sqsupseteq^{\CF}$ and $\sqsupseteq^{\UA}$, the resulting semantics will not be $\st$-sound. 
However, a generalisation of the semi-stable extensions semantics is possible. Recall that a set $E$ is semi-stable if it is complete and maximises $E \cup E^+$. We achieve a generalisation by lexicographically combining $\sqsupseteq^{\rCo}$ and $\sqsupseteq^{\UA}$.
\begin{definition}\label{defn:ss-ranking-semantics}
    Let $F=(A,R)$ be an AF and $E,E' \subseteq A$. We define the \emph{semi-stable extension-ranking semantics} $\rSst$ via:
    $$E \sqsupseteq^{\rSst}_F E' \text{ iff } E \sqsupseteq^{\lex(\CF,\UD,\DN,\UA)}_{F} E'$$
\end{definition}
\begin{example}
 Consider $F_{4}$ from Example~\ref{ex:af_example} and sets $\{a,g\}$, $\{a,d,g\}$, and $\{g\}$. $\{a,g\}$ and $\{a,d,g\}$ are both complete extensions, while $\{g\}$ is not complete. So, $\sqsupseteq^{\rCo}_{F_4}$ returns $\{a,g\} \sqsupset_{F_{4}}^{\rSst} \{g\}$ and $\{a,d,g\} \sqsupset_{F_{4}}^{\rSst} \{g\}$. To compare $\{a,g\}$ and $\{a,d,g\}$ further we use $\sqsupseteq^{\UA}_{F_4}$. Both these sets are not stable, but $\UA_{F_4}(\{a,g\}) = \{c,d,e\}$ and $\UA_{F_4}(\{a,d,g\})= \{e\}$, so $\{a,d,g\} \sqsupset_{F_{4}}^{\rSst} \{a,g\}$. If we compare the stable extension $\{a,c,g\}$ with $\{a,d,g\}$ we get: 
$$\{a,c,g\} \sqsupset_{F_{4}}^{\rSst} \{a,d,g\}$$
because $\UA_{F_4}(\{a,c,g\}) = \emptyset$, in particular this also entails that $\{a,c,g\} \in \maxpl_{\rSst}(F_4)$. 
 Since $\sqsupseteq^\rSst_{F_4}$ is based on $\sqsupseteq^\rCo_{F_4}$ we can extend the ranking calculated in Example \ref{ex:rCO}: 
 \begin{align*}   \{a,c,g\} \sqsupset_{F_{4}}^{\rSst} \{a,&\, d,g\} \sqsupset_{F_{4}}^{\rSst} \{a,g\} \sqsupset_{F_{4}}^{\rSst}  \{g\} \\ &\sqsupset_{F_{4}}^{\rSst} \{d\}  \sqsupset^{\rSst}_{F_4} \{b\} \sqsupset_{F_{4}}^{\rSst} \{b,f\}  \sqsupset_{F_{4}}^{\rSst} \{a,b\}
 \end{align*} 
 \end{example}
 The example shows nicely the behaviour of this family of extension-ranking semantics. If we divide the resulting ranking into parts, we have first the (semi-)stable extension $\{a,c,g\}$ then complete extensions which are not semi-stable $\{a,d,g\}$ and $\{a,g\}$, after that admissible sets $\{g\}$ and $\{d\}$ followed by conflict-free sets $\{b\}$ and $\{b,f\}$ and ending with the conflicting set $\{a,b\}$. Using only extension semantics such a grouping of sets would be possible as well, however using extension-ranking semantics we can compare sets inside these groups, i.\,e. $\{a,d,g\}$ is closer to be semi-stable than $\{a,g\}$ or $\{b\}$ is closer to be admissible than $\{a,f\}$. 

\begin{example}\label{ex:lattice_rsst}
Consider $F_{13}$ from Example \ref{ex:lattice_rgr}. The corresponding lattice to $\sqsupseteq^{\rSst}_{F_{13}}$ is depicted in Figure \ref{tikz:lattice_rsst}. Comparing $\sqsupseteq^{\rSst}_{F_{13}}$  with $\sqsupseteq^{\rCo}_{F_{13}}$, we see that the most plausible set of $\sqsupseteq^{\rCo}_{F_{13}}$ are divided. $\{i\}$ is more plausible than $\{h\}$ and $\{h\}$ is more plausible than $\emptyset$. So, both $\{h\}$ and $\emptyset$ are both complete extensions but none of them is a semi-stable extension, but we show that $\{h\}$ is closer to be semi-stable than $\emptyset$.

 \begin{figure}
    \centering
 \scalebox{1}{
\begin{tikzpicture}

\node (abc) at (0,0) [] {$\{h,i,j\}$};
\node (ab) at (2,-1) [] {$\{h,i\}$};
\node (bc) at (-2,-1) [] {$\{i,j\}$};
\node (b) at (0,-6) [] {$\{i\}$};
\node (c) at (-2,-2) [] {$\{j\}$};
\node (ac) at (-2,-3) [] {$\{h,j\}$};
\node (empty) at (0,-4) [] {$\emptyset$};
\node (a) at (0,-5) [] {$\{h\}$};

\path[->] (abc) edge  (ab);
\path[->] (abc) edge  (bc);

\path[->] (bc) edge  (c);
\path[->] (c) edge  (ac);

\path[->] (ac) edge  (empty);

\path[->] (ab) edge  (empty);
\path[->] (empty) edge  (a);
\path[->] (a) edge  (b);

\end{tikzpicture}}

   \caption{Lattices depicting $\sqsupseteq^{\rSst}_{F_{13}}$ from Example \ref{ex:lattice_rsst}.}\label{tikz:lattice_rsst}
\end{figure}
\end{example}

A property of the semi-stable extension semantics is that the semi-stable extensions do coincide with the stable extensions if they exist. For the semi-stable extension-ranking semantics, we have a similar behaviour: If stable extensions exists, then the most plausible sets of the semi-stable extension-ranking semantics are the stable extensions.
\begin{proposition}\label{prop:r-sst=st}
    Let $F=(A,R)$ be an AF s.t. $\st(F) \neq \emptyset$, then $\maxpl_{\rSst}(F) = \st(F)$.
\end{proposition}
Proposition \ref{prop:r-sst=st} also implies that $\rSst$ satisfies $\st$-completeness, since if $\st(F) = \emptyset$, then $\maxpl_{\rSst}(F) \supseteq \emptyset$. 

%\tr{About table 2 (also table 3): A weakness here is that, if we ignore the $\sigma$-generalisation principle, then the complete, grounded, preferred and sst ranking semantics satisfy exactly the same principles. This raises the question: are there no principles that characterise their differences? Example: does r-pr satisfy some generalisation of in-maximality? Does r-gr satisfy some generalisation of in-minimality? What about directionality? (In the standard setting satisfied by all but the sst semantics)? It would be nice if some of these principles could be studied too. Given that this is a journal paper, I suspect this is something that reviewers might object to. It looks like in-minimality/maximality could be an easy addition since these follow directly from the lexicographic ordering.}

Refining the $\sqsupseteq^{\rCo}$ with $\sqsupseteq^\UA$ does maximise the range of each set and therefore represents a generalisation of the semi-stable semantics.  
\begin{proposition}
    $\rSst$ satisfies $\sst$-generalisation.
\end{proposition}

The proofs of the next three propositions are based on the proofs for $\rCo$. 
\begin{proposition}
    $\rSst$ satisfies composition and decomposition.
\end{proposition}
%\begin{proof}
 %   Follows from Lemma \ref{lem:composition_decomposition_2} together with Definition \ref{defn:ss-ranking-semantics}.
%\end{proof}

\begin{proposition}
   $\rSst$ satisfies strong reinstatement. 
\end{proposition}
%\begin{proof}
 %   This follows directly from Proposition \ref{prop:rCO_reinstatement}.
%\end{proof}

\begin{proposition}
    $\rSst$ violates addition robustness. 
\end{proposition}
%\begin{proof}
 %   Since, $\sqsupseteq^{\rSst}_{F}$ is based on $\sqsupseteq^{\rCo}_{F}$ we can use Example \ref{ex:co_violates_robustness} to show that  $\rSst$ violates addition robustness as well.
  %  \end{proof}

It is clear by definition, that all extension-ranking semantics satisfy \emph{Syntax Independence}.

Table \ref{tab:principle_rsigma} summarises the results of this Subsection.

 \begin{table}[]
     \centering
      \resizebox{\textwidth}{!}{
     \begin{tabular}{|l||c|c|c|c|c|}
     \hline
     Principles & $\rAd$ & $\rCo$& $\rGr$ &  $\rPr$  &$\rSst$ \\
     \hline
        $\sigma$-generalisation  & \checkmark  ($\sigma= \ad$)  &  \checkmark  ($\sigma= \co$) & \checkmark  ($\sigma= \gr$) & \checkmark  ($\sigma= \pr$) & \checkmark  ($\sigma= \sst$) \\
        composition  & \checkmark &  \checkmark & \checkmark & \checkmark & \checkmark  \\
        decomposition  & \checkmark &  \checkmark & \checkmark & \checkmark & \checkmark \\
        weak reinstatement & \checkmark &  \checkmark & \checkmark & \checkmark & \checkmark \\
        strong reinstatement  & X &  \checkmark & \checkmark & \checkmark & \checkmark\\
        addition robustness & \checkmark &  X & X & X & X   \\
    %   \INMAX & X &  X & ? & ? & ?   \\
        syntax independence & \checkmark &  \checkmark & \checkmark & \checkmark & \checkmark \\ \hline
     \end{tabular} }
     \caption{Principles satisfied by $r\text{-}\sigma$ for $\sigma \in \{\ad, \co, \gr, \pr, \sst\}$.}
     \label{tab:principle_rsigma}
 \end{table}

\subsubsection{Cardinality-based Instances}\label{subsubsec:cardinality_ext_ranking}
%In Subsection \ref{subsec:cardinality} we proposed a different family of base relations to avoid incompatibilities of sets. Instead of comparing two sets using subset-comparisons we used the cardinality. For each of the Dungean base relations we proposed a cardinality-based variation.
The cardinality-based variations of the base relations from Section~\ref{subsec:cardinality} can also be used to define generalisations of extension semantics.

    \begin{definition}
    Let $F=(A,R)$ and $E,E' \subseteq A$. Define the \emph{cardinality-based extension-ranking semantics} $\mathsf{r\text{-}c\text{-}\sigma}$ for $\sigma \in \{\ad,\co,\gr,\pr,\sst\}$ via
    \begin{itemize}
        \item $E\sqsupseteq_{F}^{\mathsf{r\text{-}c\text{-}\ad}} E'$ iff $E \sqsupseteq_{F}^{\lex(\mathsf{c\text{-}\CF,c\text{-}\UD})} E'$. 
        \item $E\sqsupseteq_{F}^{\mathsf{r\text{-}c\text{-}\co}} E'$ iff $E \sqsupseteq_{F}^{\lex(\mathsf{c\text{-}\CF,c\text{-}\UD,c\text{-}\DN})} E'$.
       \item $E\sqsupseteq_{F}^{\mathsf{r\text{-}c\text{-}\gr}} E'$ iff $E \sqsupseteq_{F}^{\lex(\mathsf{c\text{-}\CF,c\text{-}\UD,c\text{-}\DN},\Min)} E'$.
       \item $E\sqsupseteq_{F}^{\mathsf{r\text{-}c\text{-}\pr}} E'$ iff $E \sqsupseteq_{F}^{\lex(\mathsf{c\text{-}\CF,c\text{-}\UD},\Max)} E'$.
       \item $E\sqsupseteq_{F}^\mathsf{{r\text{-}c\text{-}\sst}} E'$ iff $E \sqsupseteq_{F}^{\lex(\mathsf{c\text{-}\CF,c\text{-}\UD,c\text{-}\DN,c\text{-}\UA})} E'$.
    \end{itemize}
\end{definition}

\begin{example}\label{ex:crAD}
    Consider $F_4$ from Example \ref{ex:af_example} and sets $\{c,g\}$ and $\{b,f\}$. Argument $g$ is defended by $\{c,g\}$, while $c$ is not defended, while both $b,f$ are not defended by $\{b,f\}$. However, $\sqsupseteq^\rAd_{F_4}$ returns incompatibility between these two sets, since $\{c\}$ is not a subset of $\{b,f\}$. But using the cardinality-based version of $\rAd$ we can compare these two sets with each other and state $\{c,g\} \sqsupset_{F_4}^{\mathsf{r\text{-}c\text{-}\ad}} \{b,f\}$.

    %While $\sqsupseteq^{\mathsf{r\text{-}c\text{-}ad}}_{F_4}$ can compare two sets, which are incomparable for $\sqsupseteq^\rAd_{F_4}$, this extension-ranking semantics returns the same relationship for two for $\sqsupseteq^\rAd_{F_4}$ comparable sets in this case.
    Consider sets $\{a,b\}$, $\{b\}$ and $\{b,f\}$, then $\{a,b\}$ is not conflict-free, so $\{b\} \sqsupset_{F_4}^{\mathsf{\mathsf{r\text{-}c\text{-}\ad}}} \{a,b\}$ and $\{b,f\} \sqsupset_{F_4}^{\mathsf{r\text{-}c\text{-}\ad}} \{a,b\}$. To compare $\{b\}$ and $\{b,f\}$ further we look at $\sqsupseteq^{\mathsf{r\text{-}c\text{-}\UD}}_{F_4}$, argument $b$ is not defended by $\{b\}$ so $|\UD_{F_4}(\{b\})|= |\{b\}| = 1$. For $\{b,f\}$ we have $|\UD_{F_4}(\{b,f\})|= |\{b,f\}| = 2$ and this entails: $\{b\} \sqsupset_{F_4}^{\mathsf{r\text{-}c\text{-}\ad}} \{b,f\}$. So, the resulting ranking between these three sets coincides with the one for $\sqsupseteq^\rAd_{F_4}$ in Example \ref{ex:rAD}.
    $$\{b\} \sqsupset_{F_{4}}^{\mathsf{r\text{-}c\text{-}\ad}} \{b,f\}  \sqsupset_{F_{4}}^{\mathsf{r\text{-}c\text{-}\ad}} \{a,b\} $$
    \end{example}
    \begin{example}\label{ex:lattice_rcad}
    Consider $F_{14}$ as depicted in Figure \ref{tikz:rcad}. Extension rankings for $\sqsupseteq_{F_{14}}^{\rAd}$ and $\sqsupseteq_{F_{14}}^{\mathsf{r\text{-}c\text{-}\ad}}$ are depicted in Figure \ref{tikz:lattice_rcad}. The difference between the two rankings is that $\{h,i\}$ and $\{i,j\}$ are incomparable for $\sqsupseteq_{F_{14}}^{\rAd}$, while for $\sqsupseteq_{F_{14}}^{\mathsf{r\text{-}c\text{-}\ad}}$ these two sets are comparable and we can say that $\{i,j\}$ is more plausible to be accepted than $\{h,i\}$.
           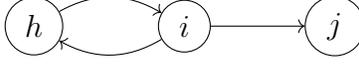
\begin{figure}
    \centering
 
 \scalebox{1}{
\begin{tikzpicture}

\node (a1) at (0,0) [circle, draw,minimum size= 0.65cm] {$h$};
\node (a2) at (2,0) [circle, draw,minimum size= 0.65cm] {$i$};
\node (a3) at (4,0) [circle, draw,minimum size= 0.65cm] {$j$};

\path[->,bend left] (a2) edge  (a1);
\path[->,bend left] (a1) edge  (a2);

\path[<-] (a3) edge  (a2);

\end{tikzpicture}}
   \caption{AF $F_{14}$ from Example \ref{ex:lattice_rcad}.}
    \label{tikz:rcad}
\end{figure}

        \begin{figure}
  %  \centering
    \begin{minipage}[t]{0.4\textwidth}
 \scalebox{0.9}{
\begin{tikzpicture}

\node (abc) at (0,0) [] {$\{h,i,j\}$};
\node (ab) at (2,-1) [] {$\{h,i\}$};
\node (bc) at (-2,-1) [] {$\{i,j\}$};
\node (b) at (-1,-3) [] {$\{i\}$};
\node (c) at (0,-2) [] {$\{j\}$};
\node (ac) at (-3,-3) [] {$\{h,i\}$};
\node (empty) at (1,-3) [] {$\emptyset$};
\node (a) at (3,-3) [] {$\{h\}$};

\path[->] (abc) edge  (ab);
\path[->] (abc) edge  (bc);

\path[->] (bc) edge  (c);
\path[->] (ab) edge  (c);

\path[->] (c) edge  (empty);
\path[->] (c) edge  (a);
\path[->] (c) edge  (b);
\path[->] (c) edge  (ac);

\path[-] (empty) edge  (a);
\path[-] (empty) edge  (b);
\path[-] (b) edge  (ac);

\end{tikzpicture}}
    \end{minipage}
    \hspace{1cm}
     \begin{minipage}[t]{0.4\textwidth}
 \scalebox{0.9}{
\begin{tikzpicture}

\node (abc) at (0,0) [] {$\{h,i,j\}$};
\node (ab) at (0,-1) [] {$\{h,i\}$};
\node (bc) at (0,-2) [] {$\{i,j\}$};
\node (b) at (-1,-4) [] {$\{i\}$};
\node (c) at (0,-3) [] {$\{j\}$};
\node (ac) at (-3,-4) [] {$\{h,j\}$};
\node (empty) at (1,-4) [] {$\emptyset$};
\node (a) at (3,-4) [] {$\{h\}$};

\path[->] (abc) edge  (ab);

\path[->] (bc) edge  (c);
\path[->] (ab) edge  (bc);

\path[->] (c) edge  (empty);
\path[->] (c) edge  (a);
\path[->] (c) edge  (b);
\path[->] (c) edge  (ac);

\path[-] (empty) edge  (a);
\path[-] (empty) edge  (b);
\path[-] (b) edge  (ac);

\end{tikzpicture}}
    \end{minipage}
   \caption{Lattices depicting $\sqsupseteq^{\rAd}_{F_{14}}$ (left) and $\sqsupseteq_{F_{14}}^{\mathsf{r\text{-}c\text{-}\ad}}$ (right) from Example \ref{ex:lattice_rcad}.}\label{tikz:lattice_rcad}
\end{figure}
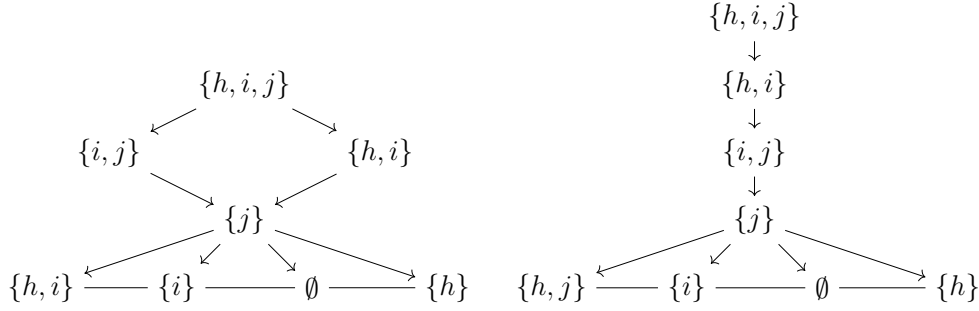
    \end{example}

    In Example \ref{ex:crAD} we see that $\sqsupseteq^{\mathsf{r\text{-}c\text{-}ad}}_{F_4}$ and $\sqsupseteq^\rAd_{F_4}$ coincide in the second part. This behaviour is always the case for the whole family of $\mathsf{r\text{-}c\text{-}\sigma}$ extension-ranking semantics when $\sqsupseteq^\mathsf{r\text{-}\sigma}$ can compare two sets. 
\begin{proposition}
  Let $F= (A,R)$ be an AF and $E,E' \subseteq A$, if $E \sqsupseteq^{\mathsf{\mathsf{r\text{-}\sigma}}}_F E'$ then $E \sqsupseteq^{\mathsf{r\text{-}c\text{-}\sigma}}_F E'$ for $\sigma \in \{\ad,\co,\gr,\pr,\sst\}$.
\end{proposition}

In other words, the family of extension-ranking semantics $\mathsf{r\text{-}c\text{-}\sigma}$ is a refinement of $\mathsf{r\text{-}\sigma}$, because former family fixes the non-totality issues of $\sqsupseteq^{\mathsf{r\text{-}\sigma}}$ for $\sigma \in \{\ad,\co,\sst\}$, but coincides with $\sqsupseteq^{\mathsf{r\text{-}\sigma}}$ when this extension-ranking semantics can compare two sets. 
 
Next, we  look at the principles that the $\mathsf{r\text{-}c\text{-}\sigma}$ extension-ranking semantics satisfy. Starting with the $\sigma$-generalisation principle, we see that the cardinality-based versions of the extension-ranking semantics $\mathsf{r\text{-}c\text{-}\sigma}$ are also generalisations of the extension semantics. 
\begin{proposition}
    $\mathsf{r\text{-}c\text{-}\sigma}$ satisfies $\sigma$-generalisation $\sigma \in \{\ad,\co,\gr,\pr,\sst\}$.
\end{proposition}
%\begin{proof}
%For any $\tau \in \{\CF,\UD,\DN,\UA\}$ the best value a set can have is $\emptyset$, for the cardinality-based versions of these relations the best value to be returned is $0$ and only the empty set can have cardinality of $0$. 
 %   Hence, we know that if and only if $\tau$ returns $\emptyset$ the cardinality-based versions of $\mathsf{c\text{-}\tau}$ returns $0$. Thus, we can use the same proof ideas of $\mathsf{r\text{-}\sigma}$ to show that $\mathsf{r\text{-}c\text{-}\sigma}$ satisfies $\sigma$-generalisation. 
%\end{proof}

So, both $\mathsf{r\text{-}\sigma}$ and $\mathsf{r\text{-}c\text{-}\sigma}$ are generalisations of extension semantics. Therefore, both families of approaches return the same most plausible sets, the $\sigma$-extensions. Note that $\mathsf{r\text{-}\sigma}$ does not return total preorders, since the subsets relation can yield incompatibilities, while the cardinality-based versions return total preorders, since we can always compare the number of elements inside two sets. 
Looking at the other properties, we see that $\mathsf{r\text{-}c\text{-}\sigma}$ satisfies composition. 
\begin{proposition}
    $\mathsf{r\text{-}c\text{-}\sigma}$ satisfies composition for $\sigma \in \{\ad,\co,\gr,\pr,\sst\}$.
\end{proposition}
%\begin{proof}
 %   Let $F_1=(A_1,R_1)$, $F_2=(A_2,R_2)$ and $F = F_1 \cup F_2= (A_1, R_1) \cup (A_2,R_2)$ be AFs with $A_1 \cap A_2 = \emptyset$. We know that $|\tau_{F_1}(E\cap A_1)| + |\tau_{F_2}(E \cap A_2)| = |\tau_F(E)|$ for any base function $\tau \in \{\CF, \UD, \DN, \UA\}$. Assume $E \cap A_1 \sqsupseteq^{\mathsf{r\text{-}c\text{-}\sigma}}_{F_1} E' \cap A_1$ and $E \cap A_2 \sqsupseteq^{\mathsf{r\text{-}c\text{-}\sigma}}_{F_2} E' \cap A_2$. So, there is at least one $\tau' \in \{\CF, \UD, \DN,\\ \UA\}$  s.t. $|\tau_{F_1}'(E\cap A_1)| \leq |\tau_{F_1}'(E'\cap A_1)|$ and  $|\tau_{F_2}'(E\cap A_2)| \leq |\tau_{F_2}'(E'\cap A_2)|$. Therefore, $|\tau_{F_1}'(E\cap A_1)| + |\tau_{F_2}'(E\cap A_2)| \leq |\tau_{F_1}'(E'\cap A_1)| + |\tau_{F_2}'(E'\cap A_2)|$ and this entails $|\tau_F'(E)| \leq |\tau_F'(E')|$. Hence, composition is satisfied.
%\end{proof}
However, for decomposition we can see that $\mathsf{r\text{-}c\text{-}\sigma}$ violates it for $\sigma \in \{\ad,\co,\gr,\pr,\sst\}$. 
\begin{example}\label{ex:decomp_rc}
    Assume $F_{15} = (\{a,b,c,d,e\}, \{(a,b),(a,c),(d,e)\})$ as depicted in Figure \ref{tikz:decomp_rc} and $E = \{d,e\}$ and $E' = \{a,b,c\}$. 
     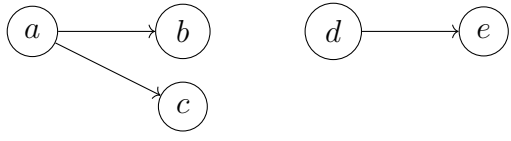
\begin{figure}
    \centering
 
 \scalebox{1}{
\begin{tikzpicture}

\node (a) at (0,0) [circle, draw,minimum size= 0.65cm] {$a$};
\node (b) at (2,0) [circle, draw,minimum size= 0.65cm] {$b$};
\node (c) at (2,-1) [circle, draw,minimum size= 0.65cm] {$c$};

\node (d) at (4,0) [circle, draw,minimum size= 0.65cm] {$d$};
\node (e) at (6,0) [circle, draw,minimum size= 0.65cm] {$e$};

\path[->] (a) edge  (b);
\path[->] (a) edge  (c);
\path[->] (d) edge  (e);

\end{tikzpicture}}
   \caption{AF $F_{15}$ from Example \ref{ex:decomp_rc}.}
    \label{tikz:decomp_rc}
\end{figure}
    
    We can split this AF into two disjoint AFs $F_{15,1}= (\{a,b,c\},\{(a,b), (a,c)\})$ and $F_{15,2}= (\{d,e\},\{(d,e)\})$. Since $E'$ has two conflicts and $E$ has only one it is clear that $E \sqsupset^{\CF}_{F_{15}} E'$ and since every $\mathsf{r\text{-}c\text{-}\sigma}$ is based on $\sqsupseteq^\CF$ the same holds for $\sigma \in \{\ad,\co,\gr,\pr,\sst\}$. If we now take a closer look at the sub-AFs $F_{15,1}$ and $F_{15,2}$, we see that $E \cap \{a,b,c\}$ is conflict-free in $F_{15,1}$, while $E' \cap \{a,b,c\}$ has two conflicts, so  $E \cap \{a,b,c\} \sqsupset^{\CF}_{F_{15,1}} E' \cap \{a,b,c\}$. However, $E'\cap \{d,e\}$ is conflict-free in $F_{15,2}$ and $E \cap \{d,e\}$ has one conflict in $F_{15,2}$, so $E' \cap \{d,e\} \sqsupset^{\CF}_{F_{15,2}} E \cap \{d,e\}$ and this shows that $\mathsf{c\text{-}\CF}$ violates decomposition and therefore $\mathsf{r\text{-}c\text{-}\sigma}$ also violates decomposition. 
\end{example}
So, while we gain total preorders using cardinality-based extension-ranking semantics, we lose decomposition. In a sense, we lose the local view of each independent AF.

Based on our results for $\rAd$, we can show that the admissible version of the cardinality-based extension-ranking semantics also only satisfies weak reinstatement but not the strong version.
\begin{proposition}
     $\mathsf{r\text{-}c\text{-}\ad}$ satisfies weak reinstatement and violates strong reinstatement.
\end{proposition}
%\begin{proof}
 %   Let $F= (A,R)$ be an AF and $E \subseteq A$. Suppose $a \in \mathcal{F}_{F}$, $a \notin E$, and $a \notin E^- \cup E^+$. In Proposition \ref{prop:rAD_reinstatement}, we have shown that $\UD_F(E \cup \{a\}) \subseteq \UD_F(E)$ and therefore also $|\UD_F(E \cup \{a\})| \leq |\UD_F(E)|$. This proves that $\mathsf{r\text{-}c\text{-}\ad}$ satisfies weak reinstatement.

  %  For the violation of strong reinstatement we can use Example \ref{ex:AD_not_reinstatment} again. 
%\end{proof}

For the remaining semantics, we can show that the cardinality-based extension-ranking semantics satisfies strong reinstatement. 
\begin{proposition}
$\mathsf{r\text{-}c\text{-}\sigma}$ satisfies strong reinstatement for $\sigma \in \{\co,\gr,\pr,\sst\}$.
\end{proposition}
%\begin{proof}
 %   Let $F= (A,R)$ be an AF and $E \subseteq A$. Suppose $a \in \mathcal{F}_{F}$, $a \notin E$, and $a \notin E^- \cup E^+$. In Proposition \ref{prop:rCO_reinstatement}, we have shown that $\DN_F(E \cup \{a\}) \subset \DN_F(E)$. This also shows that $|\DN_F(E \cup \{a\})| < |\DN_F(E)|$. Hence, $\mathsf{r\text{-}c\text{-}\sigma}$ satisfies strong reinstatement and because $E \cup \{a\} \sqsupset^{\mathsf{r\text{-}c\text{-}\sigma}}_{F} E$ this also shows the satisfaction of strong reinstatement for the remaining semantics.
%\end{proof}

Addition robustness is only satisfied by $\rAd$. Similar results hold for the cardinality-based version $\mathsf{r\text{-}c\text{-}\ad}$.
\begin{proposition}
   $\mathsf{r\text{-}c\text{-}\ad}$ satisfies addition robustness.
\end{proposition}
%\begin{proof}
 %   Let $F= (A,R)$ be an AF and $E, E' \subseteq A$ with $E \sqsupseteq^{\mathsf{r\text{-}c\text{-}\ad}}_{F} E'$. Let $F' = (A, R \cup \{(a,b)\})$ for $a \in E$ and $b \notin E$, $b \in E'$. We know that the addition of $(a,b)$ does not add any new conflict into $E$ and does not remove any conflict from $E'$, therefore $|\CF_F(E)|= |\CF_{F'}(E)|$ and $|\CF_F(E')| \leq |\CF_{F'}(E')|$. Hence, addition robustness holds for $\sqsupseteq^{c\text{-}\CF}$.
    
  %  It remains to show, that addition robustness holds for $\sqsupseteq^{c\text{-}\UD}$ as well. We already know that no defence in $E$ is removed therefore $|\UD_{F'}(E)| \leq |\UD_F(E)|$ and we also know that $E'$ can not defend more arguments in $F'$ than in $F$, otherwise additional conflicts have to be added. Hence, $E \sqsupseteq^{\mathsf{r\text{-}c\text{-}\ad}}_{F'} E'$.
%\end{proof}

For $\sigma \in \{\co,\gr,\pr,\sst\}$ the counterexamples used before (Examples \ref{ex:pr_robustness} and \ref{ex:co_violates_robustness}) can be used again to show that $\mathsf{r\text{-}c\text{-}\sigma}$ also violates addition robustness. For example, in Example \ref{ex:co_violates_robustness} it holds that $\DN_{F_{12}}(E=\{a\})= \emptyset$ and $\DN_{F_{12}}(E'= \{b,c\})= \{d\}$. After adding an attack originating from $E$ to $E'$, the consistently defended arguments of $E$ changed to $\DN_{F_{12}'}(E)= \{d,f\}$, while the consistently defended arguments of $E'$ remained the same, i.\,e., $\DN_{F_{12}'}(E')= \{d\}$. So the relation between $E$ and $E'$ in $F_{12}'$ changed to $E' \sqsupset^{\mathsf{r\text{-}c\text{-}\sigma}}_{F_{12}'} E$.

Finally, it is clear by definition that $\mathsf{r\text{-}c\text{-}\sigma}$ satisfies syntax independence for $\sigma \in \{\ad,\co,\gr,\pr,\sst\}$, since the names of the individual arguments are irrelevant for calculating the rankings. 

So, the cardinality-based extension-ranking semantics do behave similar to their subset variations, we just trade decomposition for totality. Thus, depending on the task at hand, we can choose which version is more appropriate by evaluating between totality and decomposition. 

Table \ref{tab:principle_rcsigma} summarises the results of this subsection.
\begin{table}[]
     \centering
      \resizebox{\textwidth}{!}{
     \begin{tabular}{|l||c|c|c|c|c|}
     \hline
     Principles & $\mathsf{r\text{-}c\text{-}\ad}$ & $\mathsf{r\text{-}c\text{-}\co}$& $\mathsf{r\text{-}c\text{-}\gr}$ &  $\mathsf{r\text{-}c\text{-}\pr}$  &$\mathsf{r\text{-}c\text{-}\sst}$ \\
     \hline
        $\sigma$-generalisation  & \checkmark  ($\sigma= \ad$)&  \checkmark ($\sigma= \co$)& \checkmark ($\sigma= \gr$)& \checkmark ($\sigma= \pr$)& \checkmark ($\sigma= \sst$)\\
        composition  & \checkmark &  \checkmark & \checkmark & \checkmark & \checkmark  \\
        decomposition  & X  & X & X & X & X \\
        weak reinstatement & \checkmark &  \checkmark & \checkmark & \checkmark & \checkmark \\
        strong reinstatement  & X &  \checkmark & \checkmark & \checkmark & \checkmark\\
        addition robustness & \checkmark &  X & X & X & X   \\
        syntax independence & \checkmark &  \checkmark & \checkmark & \checkmark & \checkmark \\ \hline
     \end{tabular} }
     \caption{Principles satisfied by $\mathsf{r\text{-}c\text{-}\sigma}$.}
     \label{tab:principle_rcsigma}
 \end{table}

%The cardinality-based extension-ranking semantics and $\rCoPr$ show that there are several ways to define extension-ranking semantics that generalise extension semantics. However, presenting a complete overview of all possible lexicographic combinations of base relations is beyond the scope of this paper. Sequences of base relations with interesting properties will be presented in future work.

\subsection{Voting Methods}\label{sec:voting}
\emph{Voting theory} is an area of research that focuses on finding the winner of an election. %, such as a presidential election or a local mayoral election.
An \emph{election} consists of a set $\mathcal{A}$ of $m$ \emph{alternatives} such as the mayoral candidates, and a set of \emph{voters} $\mathcal{N}= \{1,2,\dots, n\}$. Each voter $i$ cast a \emph{vote} as a linear ordering $\succsim_i$ of $\mathcal{A}$, where $\succsim_i$ is \emph{transitive} (if $a \succsim_i b$ and $b \succsim_i c$, then $a \succsim_i$ for all $a,b,c \in \mathcal{A}$), \emph{complete} ($a \succsim_i b$ or $b \succsim_i a$ for all $a\neq b \in \mathcal{A}$), \emph{reflexive} ($a \succsim_i a$ for all $a \in \mathcal{A}$), and \emph{antisymmetric} (if $a \succsim_i b$ and $a \succsim_i b$, then $a =b$, for all $a, b \in \mathcal{A}$). The linear order $\succsim_i$ represents voter's preference over the alternatives, i.\,e., $a \succsim_i b$ means that voter $i$ prefers or votes for alternative $a$ over alternative $b$. 
A \emph{voting rule} determines the winner(s) of an election, that is, the most preferred alternative based on the votes of the voters. A number of voting rules can be found in the literature, such as the \emph{plurality voting rule}, which selects as the winner the alternatives that is placed at the top of the most votes, or the \emph{Borda rule}, where alternatives receive points according to their position in each vote, then the alternative with the highest sum of points wins the election. For more information on voting theory and other voting rules, we refer interested readers to \cite{DBLP:reference/choice/2016}. 
In this paper we focus on the \emph{Copeland rule}, which looks at the number of pairwise comparisons that an alternative wins versus the number of pairwise comparisons that the alternative loses. The alternative with the best win/loss record wins the overall election. 

 Konieczny et al. \cite{DBLP:conf/ecsqaru/KoniecznyMV15} were inspired by the Copeland rule to propose methods for comparing $\sigma$-extensions. \emph{Pairwise comparison criteria} were proposed, and a $\sigma$-extension $E$ is favoured over $E'$ if $E$ has a better win/loss record than $E'$ wrt. these criteria. In the rest of this subsection, we generalise Konieczny et al.'s \cite{DBLP:conf/ecsqaru/KoniecznyMV15} definitions to propose a family of extension-ranking semantics that aggregates base relations differently from the lexicographic combination. Our previously defined base relations take on the role of pairwise comparison criteria and we aggregate the pairwise comparisons into a single preorder using a rule such as the \emph{Copeland Rule} (but other aggregation rules can be applied in the same way).
\begin{definition}
Let $F=(A,R)$ be an AF, $E,E' \subseteq A$ and $(\sqsupseteq^{\tau_1}_F,\dots,\sqsupseteq^{\tau_n}_F)$ be a sequence of base relations. We define the \emph{Copeland-based combination} $\cope(\sqsupseteq^{\tau_1}_F,\dots,\sqsupseteq^{\tau_n}_F)$ via:
\begin{align*}
&E \sqsupseteq^{\cope(\tau_1,\dots,\tau_n)}_{F} E' \text{ iff }
\\ &\Sigma^n_{i=1} |\{\mathcal{E} \subseteq A | E \sqsupseteq^{\tau_i}_{F}  \mathcal{E}\}| - |\{\mathcal{E'} \subseteq A | \mathcal{E'} \sqsupseteq^{\tau_i}_{F} E\}| \geq 
\\ &\Sigma^n_{i=1} |\{\mathcal{E} \subseteq A | E' \sqsupseteq^{\tau_i}_{F} \mathcal{E} \}| - |\{\mathcal{E'} \subseteq A | \mathcal{E'} \sqsupseteq^{\tau_i}_{F} E'\}|     
\end{align*}
\end{definition}
In other words, a set $E$ is more plausible to be accepted than $E'$ if the number of sets ranked worse is greater than the number of sets ranked worse than $E'$. 
This can be translated to say that if we consider all base relations, $E$ has a better balance of wins and losses than $E'$. 
%Other voting rules can be used to define an aggregation method. An in-depth investigation of all possible aggregation methods based on voting rules is beyond the scope of this paper and will be done in future work.
\begin{example}\label{ex:cope_CF_UD}
Consider $F_5$ from Example \ref{ex:comp}. Let us look at the sets $\{h,i\}$ and $\{i\}$, then $\{h,i\}$ wins only against $\{h,i,j\}$ wrt. $\sqsupseteq^\CF_{F_5}$ while that set loses against $5$ other sets $(\{h\}, \{i\}, \{j\}, \{h,j\}, \emptyset)$, so $\{h,i\}$ receives a value of $-4$, while $\{i\}$ receives a value of $3$. So $\{i\} \sqsupset^{\cope(\CF)}_{F_5} \{h,i\}$. However, for $\sqsupseteq^\UD_{F_5}$ we have the values: $-2$ for $\{h,i\}$ and $-2$ for $\{i\}$, hence $\{h,i\} \equiv^{\cope(\UD)}_{F_5} \{i\}$. So the base relations individually rank these two sets differently. If we combine these scores we get: $-6$ for $\{h,i\}$ and $1$ for $\{i\}$, thus $$\{i\} \sqsupset^{\cope(\CF,\UD)}_{F_5} \{h,i\}$$ 

The full extension ranking of $\sqsupseteq_{F_{5}}^{\cope(\CF,\UD)}$ is depicted as a lattice in Figure \ref{tikz:lattice_cope_CF_UD} on the right. Comparing the extension ranking of $\sqsupseteq_{F_{5}}^{\cope(\CF,\UD)}$ with $\sqsupseteq_{F_{5}}^{\rAd}$ we see a few differences. For example the set containing every argument $\{h,i,j\}$ is more plausible to be accepted than $\{i,j\}$ wrt. $\sqsupseteq_{F_{5}}^{\cope(\CF,\UD)}$, while for $\sqsupseteq_{F_{5}}^{\rAd}$, $\{h,i,j\}$ is the least plausible set.

So, these two different aggregation methods induce different extension rankings.
\begin{figure}
    \begin{center}
    \begin{minipage}[t]{0.4\textwidth}
 \scalebox{1}{
\begin{tikzpicture}

\node (abc) at (0,0) [] {$\{h,i,j\}$};
\node (ab) at (-2,-1) [] {$\{h,i\}$};
\node (bc) at (2,-1) [] {$\{i,j\}$};
\node (b) at (-2,-2) [] {$\{i\}$};
\node (c) at (2,-2) [] {$\{j\}$};
\node (ac) at (-2,-3) [] {$\{h,j\}$};
\node (empty) at (0,-3) [] {$\emptyset$};
\node (a) at (2,-3) [] {$\{h\}$};

\path[->] (abc) edge  (ab);
\path[->] (abc) edge  (bc);

\path[->] (bc) edge  (b);
\path[->] (bc) edge  (c);

\path[->] (ab) edge  (b);
\path[->] (ab) edge  (c);

\path[->] (b) edge  (ac);
\path[->] (b) edge  (empty);
\path[->] (b) edge  (a);

\path[->] (c) edge  (ac);
\path[->] (c) edge  (empty);
\path[->] (c) edge  (a);

\path[-] (ac) edge  (empty);
\path[-] (empty) edge  (a);
\end{tikzpicture}}
    \end{minipage}
     \begin{minipage}[t]{0.4\textwidth}
 \scalebox{1}{
\begin{tikzpicture}

\node (abc) at (0,-1) [] {$\{h,i,j\}$};
\node (ab) at (0,-2) [] {$\{h,i\}$};
\node (bc) at (0,0) [] {$\{i,j\}$};
\node (b) at (-1,-3) [] {$\{i\}$};
\node (c) at (1,-3) [] {$\{j\}$};
\node (ac) at (-2,-4) [] {$\{h,j\}$};
\node (empty) at (0,-4) [] {$\emptyset$};
\node (a) at (2,-4) [] {$\{h\}$};

\path[->] (bc) edge  (abc);

\path[->] (abc) edge  (ab);

\path[->] (ab) edge  (c);
\path[->] (ab) edge  (b);

\path[-] (b) edge  (c);

\path[->] (b) edge  (ac);
\path[->] (b) edge  (a);
\path[->] (b) edge  (empty);

\path[->] (c) edge  (ac);
\path[->] (c) edge  (a);
\path[->] (c) edge  (empty);

\path[-] (empty) edge  (a);
\path[-] (empty) edge  (ac);

\end{tikzpicture}}
    \end{minipage}
   \caption{Lattices depicting $\sqsupseteq^{\rAd}_{F_5}$ (left) and $\sqsupseteq_{F_{5}}^{\cope(\CF,\UD)}$ (right) from Example \ref{ex:cope_CF_UD}.}\label{tikz:lattice_cope_CF_UD}
           
    \end{center}
\end{figure}
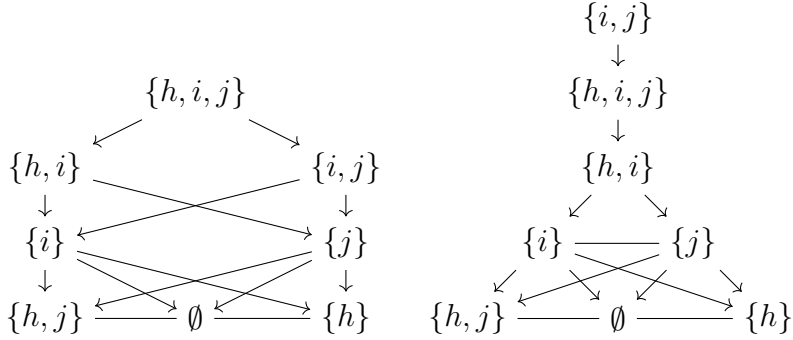
\end{example}
The two extension-ranking semantics $\sqsupseteq^{\cope(\nonatt)}$ and $\sqsupseteq^{\cope(\strdef)}$ are generalisations of the \emph{Copeland-based extensions} $\mathsf{CBE}_{\sigma,\gamma}$ defined by Konieczny et al. \cite{DBLP:conf/ecsqaru/KoniecznyMV15}. Although $\sqsupseteq^{\nonatt}$ and $\sqsupseteq^{\strdef}$ are not \emph{transitive} relations $\sqsupseteq^{\cope(\nonatt)}$ and $\sqsupseteq^{\cope(\strdef)}$ are still transitive and reflexive. In general, by using \cope\ we use any sequence of base relations to construct an extension ranking.
\begin{proposition}\label{prop:cope_extension-ranking}
Let $F=(A,R)$ be an AF. 
    $\sqsupseteq^{\cope(\tau_1,\dots, \tau_n)}_F$ is \emph{reflexive} and \emph{transitive} for any sequence of base relations $(\tau_1,\dots, \tau_n)$.
\end{proposition}

While \lex\ fully retains the input relation when applied to a single base relation, i.\,e., for any AF $F=(A,R)$ it holds that $\lex(\sqsupseteq^\tau_F)= \sqsupseteq^\tau_F$, this behaviour does not always hold for \cope. 

\begin{example}\label{ex:cope_not_id}
Consider $F_4$ from Example \ref{ex:af_example} and sets $\{b,d\}$ and $\{c,f\}$. 
Argument $b$ is not attacked by $\{c,f\}$ while $d$ attacked both $c$ and $f$ this implies $\{b,d\} \sqsupset^{\nonatt}_{F_4} \{c,f\}$, however using \cope\ we get $$\{c,f\} \sqsupset^{\cope(\nonatt)}_{F_4} \{b,d\}$$ So, the relationship between $\{b,d\}$ and $\{c,f\}$ changes by using \cope.

Next, consider sets $\{b,f\}$ and $\{c,f\}$, now arguments $b$ and $f$ are strongly defended by $\{b,f\}$ from $\{c,f\}$, while $c$ and $f$ are not strongly defended by $\{c,f\}$ from $\{b,f\}$, thus $\{b,f\} \sqsupset^{\strdef}_{F_4} \{c,f\}$. However, using \cope\ we get: $$\{c,f\} \sqsupset^{\cope(\strdef)}_{F_4} \{b,f\}$$ So, again the relationship between the two sets changes.
\end{example}

\begin{example}
    Consider $F_5$ from Example \ref{ex:comp} and sets $\{h,i\}$ and $\{i,j\}$. These two sets are incomparable wrt. \CF, i.e. $\{h,i\} \asymp^{\CF}_{F_5} \{i,j\}$. However, with the help of \cope\ these two sets are comparable and equally plausible $\{h,i\} \equiv^{\cope(\CF)}_{F_5} \{i,j\}$. Hence, with the help of \cope\ we can now compare two previously incomparable sets. But the relationship between these two sets has changed. 
\end{example}
The two examples above show that \cope\ cannot be used as an identity function for every base relation, but $\sqsupseteq^{\nonatt}$, $\sqsupseteq^{\strdef}$, and $\sqsupseteq^\CF$ lack important properties, namely \emph{transitivity} ($\sqsupseteq^{\nonatt}$, $\sqsupseteq^{\strdef}$) and \emph{totality} ($\sqsupseteq^\CF$). If the underlying base relation $\tau$ satisfies these two properties, then $\cope(\sqsupseteq^\tau)$ coincides with $\sqsupseteq^\tau$. 
\begin{proposition}\label{prop:cope_not_ID}
Let $F=(A,R)$ be an AF. 
    If $\sqsupseteq^\tau$ is total and transitive, then $\cope(\sqsupseteq^\tau_F)=\sqsupseteq^\tau_F$.
\end{proposition}

Next, we investigate compliance to our principles for some simple sequences of base relations. We start with $\cope(\sqsupseteq^\CF,\sqsupseteq^\UD)$. As we discussed earlier, the lexicographic combination of $\sqsupseteq^\CF$ and $\sqsupseteq^\UD$ generalises classical admissibility. Lexicographic combination is not the only aggregation method for which $\sqsupseteq^\CF$ and $\sqsupseteq^\UD$ generalise admissibility. Combining these two base relations with the Copeland-based combination also generalises admissibility. 
\begin{proposition}
    $\cope(\sqsupseteq^\CF,\sqsupseteq^\UD)$ satisfies $\ad$-generalisation.
    \end{proposition}
   
Since $\ad$-generalisation is satisfied, we know that $\cope(\sqsupseteq^\CF,\sqsupseteq^\UD)$ cannot be a generalisation of any other extension semantics.  

While $\cope(\sqsupseteq^\CF,\sqsupseteq^\UD)$ generalises admissibility, composition is violated. 
\begin{example}\label{ex:cope(CF,UD)_comp}
    Let $F_{15}= (\{a,b,c,d,e,f,g\}, \{(a,b),(b,a),(c,a),(a,c),(b,c),\\(c,b),(d,d),(d,e),(e,d),(d,f),(f,g)\}$ be an AF, as depicted in Figure \ref{tikz:cope(CF,UD)_comp}. $F_{15}$ can be split into two disjoint AFs $F_{15,1}= (\{a,b,c\}, \{(a,b),(b,a),(a,c),(b,c),\\(c,b)\})$ and $F_{15,2}=(\{d,e,f,g\}, \{(d,d),(d,e),(e,d),(d,f),(f,g)\}$. Consider the sets $\{a,d,e\}$ and $\{b,e,g\}$, then we have $\{a\}~\equiv^{\cope(\CF,\UD)}_{F_{15,1}} \{b\}$ and  \\$\{e,g\}~ \sqsupset^{\cope(\CF,\UD)}_{F_{15,2}} \{d,e\}$. So, $\{b,e,g\}$ should be at least as plausible to be accepted as $\{a,d,e\}$ in $F$ wrt. $\sqsupseteq^{\cope(\CF,\UD)}_{F_{15}}$. However, we have $$\{a,d,e\} \sqsupset^{\cope(\CF,\UD)}_{F_{15}} \{b,e,g\}$$ So, composition is violated.
 
      \begin{figure}
    \centering
 
 \scalebox{1}{
\begin{tikzpicture}

\node (a) at (0,0) [circle, draw,minimum size= 0.65cm] {$a$};
\node (b) at (2,0) [circle, draw,minimum size= 0.65cm] {$b$};
\node (c) at (1,2) [circle, draw,minimum size= 0.65cm] {$c$};

\node (d) at (3,0) [circle, draw,minimum size= 0.65cm] {$d$};
\node (e) at (5,0) [circle, draw,minimum size= 0.65cm] {$e$};

\node (f) at (5,1) [circle, draw,minimum size= 0.65cm] {$f$};
\node (g) at (6,1) [circle, draw,minimum size= 0.65cm] {$g$};

\path[->,bend left] (a) edge  (b);
\path[->,bend left] (b) edge  (a);
\path[->,bend left] (a) edge  (c);
\path[->,bend left] (c) edge  (a);
\path[->,bend left] (c) edge  (b);
\path[->,bend left] (b) edge  (c);

\path[->,bend left] (d) edge  (e);
\path[->,bend left] (e) edge  (d);
\path[->, loop] (d) edge  (d);

\path[->] (d) edge  (f);
\path[->] (f) edge  (g);

\end{tikzpicture}}
   \caption{AF $F_{15}$ from Example \ref{ex:cope(CF,UD)_comp}.}
    \label{tikz:cope(CF,UD)_comp}
\end{figure}
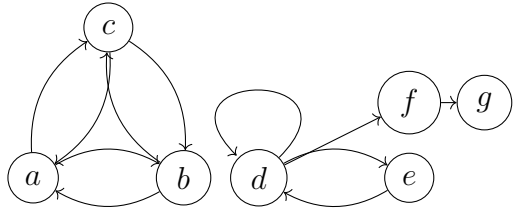
\end{example}

Decomposition is also violated by $\cope(\sqsupseteq^\CF,\sqsupseteq^\UD)$.
\begin{example}\label{ex:cope(CF,UD)_decomp}
     Let $F_{16}= (\{a,b,c,d,e\}, \{(a,b),(a,c),(d,e)\})$ be an AF as depicted in Figure \ref{tikz:cope(CF,UD)_decomp}. This AF can be partitioned into two AFs $F_{16,1}=(\{a,b,c\},\{(a,b),(a,c)\})$ and $F_{16,2}=(\{d,e\}, \{(d,e)\})$. Consider sets $\{a,e\}$ and $\{c,d\}$, then $\{c,d\} \equiv^{\cope(\CF,\UD)}_{F_{16}} \{a,e\}$ and $\{d\}\sqsupset^{\cope(\CF,\UD)}_{F_{16,2}} \{e\}$. However, we also have  $\{a\}\sqsupset^{\cope(\CF,\UD)}_{F_{16,1}} \{c\}$ and thus violating decomposition.

         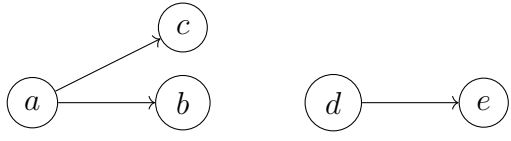
\begin{figure}
    \centering
 
 \scalebox{1}{
\begin{tikzpicture}

\node (a) at (0,0) [circle, draw,minimum size= 0.65cm] {$a$};
\node (b) at (2,0) [circle, draw,minimum size= 0.65cm] {$b$};
\node (c) at (2,1) [circle, draw,minimum size= 0.65cm] {$c$};

\node (d) at (4,0) [circle, draw,minimum size= 0.65cm] {$d$};
\node (e) at (6,0) [circle, draw,minimum size= 0.65cm] {$e$};

\path[->] (a) edge  (b);
\path[->] (a) edge  (c);

\path[->] (d) edge  (e);

\end{tikzpicture}}
   \caption{AF $F_{16}$ from Example \ref{ex:cope(CF,UD)_decomp}.}
    \label{tikz:cope(CF,UD)_decomp}
\end{figure}
\end{example}

Adding a defended argument to a set does not change the number of comparisons the set wins with respect to its conflicts or undefended arguments. Thus weak reinstatement is satisfied.
\begin{proposition}\label{prop:cope_CFUD_weak_reinst}
     $\cope(\sqsupseteq^\CF,\sqsupseteq^\UD)$ satisfies weak reinstatement.
\end{proposition}

In the proof for Proposition \ref{prop:cope_CFUD_weak_reinst}, we have shown that $E \cup \{a\} \equiv^{\cope(\CF,\UD)}_F E$ holds, so strong reinstatemant is violated. 

Despite the fact that ${\cope(\sqsupseteq^\CF,\sqsupseteq^\UD)}$ satisfies $\ad$-generalisation, addition robustness is violated.
\begin{example}\label{ex:cope_cfud_addrob}
    Let $F_{17}=(A,R)=(\{a,b,c,d,e,f,g\},\{(b,c),(b,d), (g,f),\\(g,e),(f,e),(e,f)\})$ be an AF, as depicted in Figure \ref{tikz:cope_cfud_addrob}. Consider sets $\{a,b,c,d\}$ and $\{a,e,f\}$, then $\{a,e,f\} \sqsupset^{\cope(\CF,\UD)}_{F_{17}} \{a,b,c,d\}$, therefore we can add the attack between $a$ and $b$. So consider AF $F_{17}' = (A, R \cup \{(a,b)\})$, if addition robustness is satisfied, then $\{a,e,f\} \sqsupset^{\cope(\CF,\UD)}_{F_{17}'} \{a,b,c,d\}$ has to hold as well, but this is not the case since $$\{a,b,c,d\} \sqsupset^{\cope(\CF,\UD)}_{F_{17}'} \{a,e,f\}$$ Hence, addition robustness is violated. 
  \begin{figure}
    \centering
 
 \scalebox{1}{
\begin{tikzpicture}

\node (a) at (0,0) [circle, draw,minimum size= 0.65cm] {$a$};
\node (b) at (2,0) [circle, draw,minimum size= 0.65cm] {$b$};
\node (c) at (3,1) [circle, draw,minimum size= 0.65cm] {$c$};
\node (d) at (3,-1) [circle, draw,minimum size= 0.65cm] {$d$};

\node (e) at (5,-1) [circle, draw,minimum size= 0.65cm] {$e$};
\node (f) at (7,-1) [circle, draw,minimum size= 0.65cm] {$f$};
\node (g) at (6,1) [circle, draw,minimum size= 0.65cm] {$g$};

\path[->, dashed] (a) edge  (b);
\path[->] (b) edge  (c);
\path[->] (b) edge  (d);

\path[->,bend left] (e) edge  (f);
\path[->,bend left] (f) edge  (e);
\path[->] (g) edge  (f);
\path[->] (g) edge  (e);

\end{tikzpicture}}
   \caption{AF $F_{17}$ from Example \ref{ex:cope_cfud_addrob}, where attack $(a,b)$ is added later to obtain $F_{17}'$.}
    \label{tikz:cope_cfud_addrob}
\end{figure}
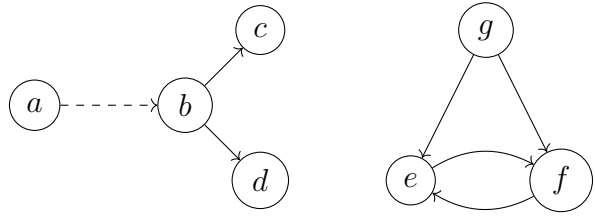
\end{example}

In Example \ref{ex:cope_not_id}, we showed that $\sqsupseteq^{\cope(\nonatt)}$ and $\sqsupseteq^{\cope(\strdef)}$ are not equal to $\sqsupseteq^{\nonatt}$ and $\sqsupseteq^{\strdef}$, so it is interesting to study the properties of $\cope(\sqsupseteq^{\nonatt})$ and $\cope(\sqsupseteq^{\strdef})$. 

Proposition \ref{prop:A_min_nonatt,strdef} already shows that the set containing all arguments is among the best sets for $\sqsupseteq^{\nonatt}$ and $\sqsupseteq^{\strdef}$, this behaviour gives us an indicator that $\cope(\sqsupseteq^{\nonatt})$ and $\cope(\sqsupseteq^{\strdef})$ do not satisfy $\sigma$-generalisation for any conflict-free based extension semantics $\sigma$. 
\begin{example}
    Consider $F_5 = (\{h,i,j\},\{(h,i),(i,j)\})$ from Example \ref{ex:comp}. For both $\sqsupseteq^{\cope(\nonatt)}$ and $\sqsupseteq^{\cope(\strdef)}$ the set $\{h,i,j\}$ is among the best sets, i.\,e., $\{h,i,j\} \in \maxpl_{\cope(\tau)}(F_5)$ for $\tau \in \{\nonatt, \strdef\}$. Hence, $\sigma$-generalisation is violated for $\sigma \in \{\cf,\ad,\co,\gr,\pr,\st, \sst\}$.
\end{example}

Similar to $\cope(\sqsupseteq^\CF,\sqsupseteq^\UD)$, $\cope(\sqsupseteq^{\nonatt})$ and $\cope(\sqsupseteq^{\strdef})$ also violate composition.
\begin{example}\label{ex:cope_nonatt_comp}
    Consider $F_{18}= (\{a,b,c,d,e\}, \{(a,b),(b,a),(c,d),(c,e)\}$ as depicted in Figure \ref{tikz:cope_nonatt_comp}. $F_{18}$ can be splitted into two disjoint AFs $F_{18,1}= (\{a,b\}, \{(a,b),(b,a)\})$ and $F_{18,2}=(\{c,d,e\}, \{(c,d),(c,e)\}$. Consider sets $\{a,c\}$ and $\{b,c,d,e\}$. In $F_{18,1}$, $\{a\}$ wins the comparison wrt. $\sqsupseteq^{\nonatt}_{F_{18,1}}$ and $\sqsupseteq^{\strdef}_{F_{18,1}}$ against $\emptyset$ and $\{b\}$, while $\{b\}$ wins the comparison against $\emptyset$ and $\{a\}$, thus $\{a\}$ and $\{b\}$ have the same win/lose record in $F_{18,1}$ wrt. $\sqsupseteq^{\cope({\nonatt})}_{F_{18,1}}$ and $\sqsupseteq^{\cope({\strdef})}_{F_{18,1}}$. In $F_{18,2}$ similar holds, so $\{a\} \equiv^{\cope(\tau)}_{F_{18,1}} \{b\}$ and $\{c\} \equiv^{\cope(\tau)}_{F_{18,2}} \{c,d,e\}$ for $\tau \in \{\nonatt,\strdef\}$. In $F_{18}$ $\{a,c\}$ wins 27 comparisons and loses 15 comparisons wrt. $\sqsupseteq^{\nonatt}_{F_{18}}$ and $\sqsupseteq^{\strdef}_{F_{18}}$, $\{b,c,d,e\}$ wins 27 comparisons and loses 11 comparisons wrt. $\sqsupseteq^{\nonatt}_{F_{18}}$ and $\sqsupseteq^{\strdef}_{F_{18}}$. Thus, $\{b,c,d,e\}$ has a better win/lose ratio than $\{a,c\}$ and therefore $\{b,c,d,e\} \sqsupset^{\cope(\tau)}_{F_{18}} \{a,c\}$ for $\tau \in \{\nonatt,\strdef\}$ showing that composition is violated.  

      \begin{figure}
    \centering
 
 \scalebox{1}{
\begin{tikzpicture}

\node (a) at (0,0) [circle, draw,minimum size= 0.65cm] {$a$};
\node (b) at (1.5,0) [circle, draw,minimum size= 0.65cm] {$b$};
\node (c) at (3,0) [circle, draw,minimum size= 0.65cm] {$c$};
\node (d) at (4,1) [circle, draw,minimum size= 0.65cm] {$d$};
\node (e) at (4,0) [circle, draw,minimum size= 0.65cm] {$e$};

\path[->, bend left] (b) edge  (a);
\path[->, bend left] (a) edge  (b);

\path[->] (c) edge  (d);
\path[->] (c) edge  (e);

\end{tikzpicture}}
   \caption{AF $F_{18}$ from Example \ref{ex:cope_nonatt_comp}.}
    \label{tikz:cope_nonatt_comp}
\end{figure}
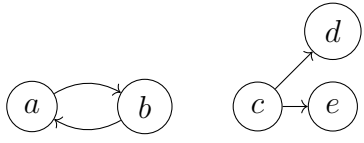
\end{example}

Decomposition is also violated by  $\cope(\sqsupseteq^{\nonatt})$ and $\cope(\sqsupseteq^{\strdef})$.
\begin{example}\label{ex:cope_nonatt_decomp}
     Let $F_{16}= (\{a,b,c,d,e\}, \{(a,b),(a,c),(d,e)\})$ be an AF from Example \ref{ex:cope(CF,UD)_decomp}. This AF can be partitioned into two AFs $F_{16,1}=(\{a,b,c\},\{(a,b),\\(a,c)\})$ and $F_{16,2}=(\{d,e\}, \{(d,e)\})$. Consider sets $\{a,b,e\}$ and $\{c,d\}$, then for $\tau \in \{\nonatt,\strdef\}$ we have $\{a,b,e\} \sqsupset^{\cope(\tau)}_{F_{16}} \{c,d\}$, but we also have $\{a,b\} \sqsupset^{\cope(\tau)}_{F_{16,1}} \{c\}$ and $\{d\} \sqsupset^{\cope(\tau)}_{F_{16,2}} \{e\}$.
\end{example}

Since $\sqsupseteq^{\nonatt}$ and $\sqsupseteq^{\strdef}$ only rely on the attackers, we show that adding a defended argument only increases the strength of a set. 
\begin{proposition}
    ${\cope(\sqsupseteq^\tau)}$ satisfies strong reinstatement for $\tau \in \{\nonatt,\\ \strdef\}$.
\end{proposition}

$\cope(\sqsupseteq^\tau)$ violates addition robustness for $\tau \in \{\nonatt, \strdef\}$.
\begin{example}\label{ex:cope_nonatt_addrob}
    Consider AF $F_{19}=(\{a,b,c,d,e,f\}, \{(c,a),(d,b)\})$ and the two sets $\{a,d,e\}$ and $\{a,b,e,f\}$. These two sets are equally plausible to be accepted wrt. $\cope(\sqsupseteq^\tau)$ for $\tau \in \{\nonatt,\strdef\}$, i.\,e., $\{a,d,e\} \equiv^{\cope(\tau)}_{F_{19}} \{a,b,e,f\}$. So, if addition robustness is satisfied, then the attack $(a,b)$ can be added and it should not hold that $\{a,b,e,f\} \sqsupset^{\cope(\tau)}_{F_{19}'} \{a,d,e\}$, where $F'= (\{a,b,c,d,e,f\}, \{(c,a),(d,b)\} \cup \{(a,b)\})$. However, this is the case for $F_{19}'$, therefore $\cope(\sqsupseteq^\tau)$ violates addition robustness for $\tau \in \{\nonatt,\strdef\}$.
     \begin{figure}
    \centering
 
 \scalebox{1}{
\begin{tikzpicture}

\node (a) at (0,0) [circle, draw,minimum size= 0.65cm] {$a$};
\node (b) at (0,2) [circle, draw,minimum size= 0.65cm] {$b$};
\node (c) at (2,0) [circle, draw,minimum size= 0.65cm] {$c$};
\node (d) at (2,2) [circle, draw,minimum size= 0.65cm] {$d$};
\node (e) at (4,0) [circle, draw,minimum size= 0.65cm] {$e$};
\node (e) at (4,2) [circle, draw,minimum size= 0.65cm] {$f$};

\path[->] (c) edge  (a);
\path[->] (d) edge  (b);

\path[->, dashed] (a) edge  (b);

\end{tikzpicture}}
   \caption{AF $F_{19}$ from Example \ref{ex:cope_nonatt_addrob}, where the attack $(a,b)$ is added later to obtain $F_{19}'$.}
    \label{tikz:cope_nonatt_addrob}
\end{figure}
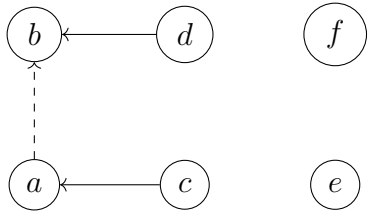
\end{example}

It is clear by definition that all extension-ranking semantics discussed in this section satisfy syntax independence. Table \ref{tab:principle_cope} summarises the results of this section.

 \begin{table}[]
     \centering
     \begin{tabular}{|l||c|c|c|}
     \hline
     Principles & $\cope(\sqsupseteq^\CF,\sqsupseteq^\UD)$ & $\cope(\sqsupseteq^{\nonatt})$& $\cope(\sqsupseteq^{\strdef})$  \\
     \hline
        $\sigma$-generalisation  & \checkmark ($\sigma= \ad$)&  X & X  \\
        composition  & X & X & X  \\
        decomposition  & X & X & X  \\
        weak reinstatement & \checkmark &  \checkmark & \checkmark  \\
        strong reinstatement  & X &  \checkmark & \checkmark\\
        addition robustness & X & X & X   \\
        syntax independence & \checkmark &  \checkmark & \checkmark  \\ \hline
     \end{tabular}
     \caption{Principles satisfied by $\cope(\sqsupseteq^\CF,\sqsupseteq^\UD)$,  $\cope(\sqsupseteq^{\nonatt})$, and $\cope(\sqsupseteq^{\strdef})$.}
     \label{tab:principle_cope}
 \end{table}

%--------------------------------------------------------------
\section{Other Extension-ranking Semantics}\label{sec:additional combinations}
All of our approaches to define an extension-ranking semantics involve applying relations to the entire set of arguments, and we evaluate the entire set as a single entity. However, argument-ranking semantics \cite{DBLP:conf/sum/AmgoudB13a} examine the individual arguments and determine their individual strength.  Consequently, it makes sense to assess the plausibility of acceptance of a set of arguments by considering the quality or strength of its components i.\,e., a set $E$ is more plausible to be accepted than $E'$ if $E$ contains ``better'' arguments. 

In this section we use argument-ranking semantics to introduce new families of extension-ranking semantics. In Section \ref{subsec:AR} we use the induced preorder over the set of argument of an argument-ranking semantics to define extension-ranking semantics, while in Section \ref{subsec:numerical_evaluation} we use the calculated strength values of the individual arguments by an argument-ranking semantics to construct an extension-ranking semantics.

\subsection{Using Argument-ranking Semantics}\label{subsec:AR}
Argument-ranking semantics determine the quality of each argument within an AF. The higher an argument is ranked, the better. A set containing a large number of highly ranked arguments should be considered more plausible to be accepted compared to a set containing worse ranked arguments. Bonzon et al. \cite{DBLP:conf/kr/BonzonDKM18} proposed to use argument-ranking semantics to compare $\sigma$-extensions. By generalising their definitions, we define a new base relation, where a set $E$ is more plausible to be accepted than $E'$ if $E$ contains more better ranked arguments than $E'$. 
\begin{definition}
    Let $F=(A,R)$ be an AF and $\rho$ an argument-ranking semantics. For $E,E' \subseteq A$, we define the  \emph{argument qualtiy count} $\mathcal{N}_{\rho,F}(E,E')$ between $E$ and $E'$ wrt. $\rho$ via:
    $$\mathcal{N}_{\rho,F}(E,E')= |\{(a,b)\text{ s.t. } a \succ^{\rho}_F b \text{ with } a \in E \text{ and } b \in E'\}|$$ 
    and the corresponding $\mathcal{N}_\rho$ \emph{base relation} $\sqsupseteq^{\mathcal{N}_\rho}$ via:
    $$ E \sqsupseteq^{\mathcal{N}_\rho}_F E' \text{ iff } \mathcal{N}_{\rho,F}(E,E') \geq \mathcal{N}_{\rho,F}(E',E)$$
\end{definition}
In other words, the argument quality count is the number of arguments of $E$ that are ranked better than arguments of $E'$ with respect to an argument-ranking semantics. 
\begin{example}\label{ex:Nrho}
    Consider $F_4$ from Example \ref{ex:af_example} recalled again in Figure \ref{tikz:af1_recall}.
    \begin{figure}
    \centering
 
 \scalebox{1}{
\begin{tikzpicture}

\node (a1) at (0,0) [circle, draw,minimum size= 0.65cm] {$a$};
\node (a2) at (2,0) [circle, draw,minimum size= 0.65cm] {$b$};
\node (a3) at (4,0) [circle, draw,minimum size= 0.65cm] {$c$};
\node (a4) at (6,0) [circle, draw,minimum size= 0.65cm] {$d$};
\node (a5) at (3,-2) [circle, draw,minimum size= 0.65cm] {$e$};
\node (a6) at (5,-2) [circle, draw,minimum size= 0.65cm] {$f$};
\node (a7) at (7,-2) [circle, draw,minimum size= 0.65cm] {$g$};

\path[<-] (a2) edge  (a1);
\path[<-] (a3) edge  (a2);
\path[->, bend left] (a3) edge  (a4);
\path[->, bend left] (a4) edge (a3); 

\path[->] (a3) edge (a5);
\path[->] (a3) edge (a6);
\path[->] (a4) edge (a6);
\path[->, bend left] (a6) edge  (a7);
\path[->, bend left] (a7) edge (a6); 

\path[->] (a5) edge  [loop left] node {} ();

\end{tikzpicture}}
   \caption{Recalled AF $F_4$ from Example \ref{ex:af_example}.}
    \label{tikz:af1_recall}
\end{figure}
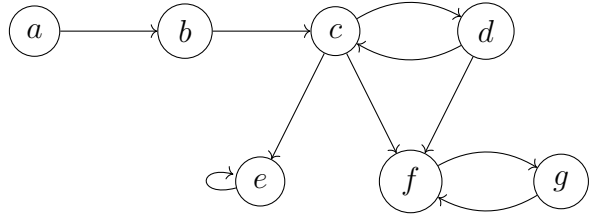
    We recall the corresponding ranking wrt. h-categoriser as presented in Example \ref{ex:h-cat}: $$a \succ^{Cat}_{F_4} g \succ^{Cat}_{F_4} d \succ^{Cat}_{F_4} e \succ^{Cat}_{F_4} b \succ^{Cat}_{F_4} c \succ^{Cat}_{F_4} f$$ 
    Consider the sets $\{a,c,g\}$ and $\{a,b,g\}$, the first set is a stable extension, while the second one is not conflict-free. To calculate $\mathcal{N}_{Cat,F}(\{a,c,g\},\{a,b,g\})$, we see that $a$ is the best ranked argument, so $a \succ^{Cat}_{F_4} b$, $a \succ^{Cat}_{F_4} g$ but also $a \succ^{Cat}_{F_4} c$. $g$ is ranked better than $b$ and $c$. So, $\mathcal{N}_{Cat,F_4}(\{a,c,g\},\{a,b,g\}) = 3$. However, we also have $b \succ^{Cat}_{F_4} c$ and therefore $\mathcal{N}_{Cat, F_4}(\{a,b,g\},\{a,c,g\}) = 4$ implying $\{a,b,g\} \sqsupset^{\mathcal{N}_{Cat}}_{F_4} \{a,c,g\}$. Hence, although $\{a,c,g\}$ is a stable extension, that set is ranked worse than a conflict set $\{a,b,g\}$, i.e.
    $$\{a,b,g\} \sqsupset^{\mathcal{N}_{Cat}}_{F_4} \{a,c,g\}$$
\end{example}
The example above shows that $\sqsupseteq^{\mathcal{N}_\rho}$ ignores internal conflicts of the set. However, the set containing every argument $A$ is not always among the most plausible sets. 
\begin{example}
    Let us continue Example \ref{ex:Nrho}, consider the two sets $\{a\}$ and $\{a,b,c,d,e,f,g\}$, then argument $a$ is ranked better than every other argument, thus $\mathcal{N}_{Cat,F_4}(\{a\},\{a,b,c,d,e,f,g\}) = 6$ while  $\mathcal{N}_{Cat,F_4}(\{a,b,c,d,e,f,g\},\\ \{a\}) = 0$. Therefore: $$\{a\} \sqsupset^{\mathcal{N}_{Cat}}_{F_4} \{a,b,c,d,e,f,g\}$$
\end{example}

In addition to ignoring internal conflicts, the base relation $\sqsupseteq^{\mathcal{N}_{Cat}}$ is not transitive.
\begin{example}
    Let us continue Example \ref{ex:Nrho} and sets $\{a,f\}$, $\{d,g\}$, and $\{a,c\}$. Then $\{a,f\} \equiv^{\mathcal{N}_{Cat}}_{F_4} \{d,g\}$ and $\{d,g\} \equiv^{\mathcal{N}_{Cat}}_{F_4} \{a,c\}$, but $\{a,c\} \sqsupset^{\mathcal{N}_{Cat}}_{F_4} \{a,f\}$, this shows that  $\sqsupseteq^{\mathcal{N}_{Cat}}_{F_4}$ is not transitive.
\end{example}

As shown in Proposition \ref{prop:cope_extension-ranking} the Copeland-based combination 
can be used to denote an extension-ranking semantics based on $\sqsupseteq^{\mathcal{N}_{\rho}}$ ($\cope(\sqsupseteq^{\mathcal{N}_\rho})$).

\begin{example}\label{ex:lattice_cope_NCat}
  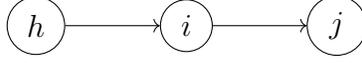
\begin{figure}
    \centering
 
 \scalebox{1}{
\begin{tikzpicture}

\node (a1) at (0,0) [circle, draw,minimum size= 0.65cm] {$h$};
\node (a2) at (2,0) [circle, draw,minimum size= 0.65cm] {$i$};
\node (a3) at (4,0) [circle, draw,minimum size= 0.65cm] {$j$};

\path[<-] (a2) edge  (a1);
\path[<-] (a3) edge  (a2);

\end{tikzpicture}}
   \caption{Recalled AF $F_5$ from Example \ref{ex:comp}.}
    \label{tikz:intro_recall2}
\end{figure}
     Recall $F_5$ from Example \ref{ex:comp} and recalled in Figure \ref{tikz:intro_recall2}. The set $\{h\}$ is the only most plausible set for $\sqsupseteq^{\cope(\mathcal{N}_{Cat})}_{F_5}$. Notice that $$\{h,i,j\} \equiv^{\cope(\mathcal{N}_{Cat})}_{F_5} \emptyset$$ So this extension-ranking semantics does behave differently to the previously discussed semantics. The complete extension ranking is depicted in Figure \ref{tikz:lattice_cope_NCat}.
      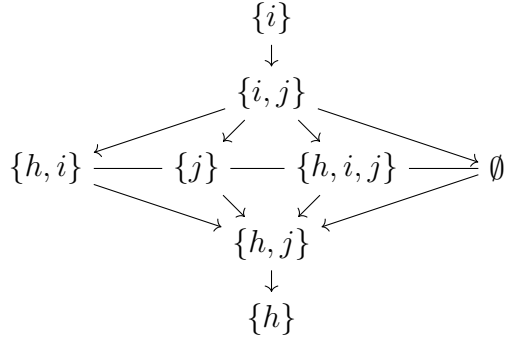
\begin{figure}
    \centering
 \scalebox{1}{
\begin{tikzpicture}

\node (abc) at (1,-2) [] {$\{h,i,j\}$};
\node (ab) at (-3,-2) [] {$\{h,i\}$};
\node (bc) at (0,-1) [] {$\{i,j\}$};
\node (b) at (0,0) [] {$\{i\}$};
\node (c) at (-1,-2) [] {$\{j\}$};
\node (ac) at (0,-3) [] {$\{h,j\}$};
\node (empty) at (3,-2) [] {$\emptyset$};
\node (a) at (0,-4) [] {$\{h\}$};

\path[->] (b) edge  (bc);

\path[->] (bc) edge  (c);
\path[->] (bc) edge  (ab);
\path[->] (bc) edge  (abc);
\path[->] (bc) edge  (empty);

\path[-] (ab) edge  (c);
\path[-] (abc) edge  (c);
\path[-] (abc) edge  (empty);

\path[->] (ab) edge  (ac);
\path[->] (c) edge  (ac);
\path[->] (abc) edge  (ac);
\path[->] (empty) edge  (ac);

\path[->] (ac) edge  (a);

\end{tikzpicture}}

   \caption{Lattices depicting $\sqsupseteq^{\cope(\mathcal{N}_{Cat})}_{F_5}$ from Example \ref{ex:lattice_cope_NCat}.}\label{tikz:lattice_cope_NCat}
\end{figure}
\end{example}

Next, we investigate the behaviour of $\cope(\sqsupseteq^{\mathcal{N}_{Cat}})$ in detail and we see that except for syntax independence this extension-ranking semantics violates every principle.
In Example \ref{ex:lattice_cope_NCat} we see that the set $\{h,j\}$ is not among the most plausible sets despite being admissible and even a stable extension. This shows that  $\cope(\sqsupseteq^{\mathcal{N}_{Cat}})$ violates  $\sigma$-generalisation for all $\sigma \in \{\cf, \ad, \co, \gr, \pr, \st, \sst\}$. 

To show the violation for composition and decomposition, we use Example \ref{ex:cope_nonatt_comp} and Example \ref{ex:cope(CF,UD)_decomp} again.
\begin{example}
    Consider $F_{18}$ from Example \ref{ex:cope_nonatt_comp} and sets $\{a,c\}$ and $\{a,b,c\}$. Then we have $\{a\} \equiv^{\cope(\mathcal{N}_{Cat})}_{F_{18,1}} \{a,b\}$ and $\{c\} \equiv^{\cope(\mathcal{N}_{Cat})}_{F_{18,2}} \{c\}$, however we have $\{a,c\} \sqsupset^{\cope(\mathcal{N}_{Cat})}_{F_{18}} \{a,b,c\}$. So composition is violated. 
\end{example}
\begin{example}
    Consider $F_{16}$ from Example \ref{ex:cope(CF,UD)_decomp} and sets $\{a,b,e\}$ and $\{c,d\}$. Then $\{c,d\} \sqsupset^{\cope(\mathcal{N}_{Cat})}_{F_{16}} \{a,b,e\}$ and $\{d\} \sqsupset^{\cope(\mathcal{N}_{Cat})}_{F_{18,2}} \{e\}$, but $\{a,b\} \sqsupset^{\cope(\mathcal{N}_{Cat})}_{F_{18,1}} \{c\}$. Therefore decomposition is violated.
\end{example}

In contrast to $\cope(\sqsupseteq^\CF,\sqsupseteq^\UD)$, $\cope(\sqsupseteq^{\nonatt})$, and $\cope(\sqsupseteq^{\strdef})$, \\ $\cope(\sqsupseteq^{\mathcal{N}_{Cat}})$ violates weak reinstatement and therefore also strong reinstatement. 
\begin{example}
    Recall $F_5$ from Example \ref{ex:comp} and recalled in Figure \ref{tikz:intro_recall2}. Consider set $\{h\}$, then argument $j$ is defended by $h$ and can be reinstated. So $\{h,j\} \sqsupseteq^{\cope(\mathcal{N}_{Cat})}_{F_5} \{h\}$ should hold, however this is not the case, therefore weak reinstatement is violated. 
\end{example}

Addition robustness is also violated by $\cope(\sqsupseteq^{\mathcal{N}_{Cat}})$.
\begin{example}\label{ex:cope_Nrho_addrob}
    Let $F_{20}= (\{a,b,c,d,e,f\}, \{(b,c),(b,e),(c,e),(d,e)\})$ be an AF, as depicted in Figure \ref{tikz:cope_Nrho_addrob}. Consider the two sets $\{a,e\}$ and $\{b,c,d,f\}$. For these two sets it holds that $\{a,e\} \equiv^{\cope(\mathcal{N}_{Cat})}_{F_{20}} \{b,c,d,f\}$. So we can add the attack $(a,b)$ and create $F_{20}'=  (\{a,b,c,d,e,f\}, \{(b,c),(b,e),(c,e),(d,e)\} \cup \{(a,b)\})$. However in $F_{20}'$ we have: $\{b,c,d,f\} \sqsupset^{\cope(\mathcal{N}_{Cat})}_{F_{20}'} \{a,e\}$. So addition robustness is violated. 
    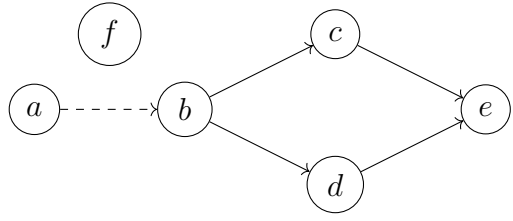
\begin{figure}
    \centering
 
 \scalebox{1}{
\begin{tikzpicture}

\node (a) at (0,0) [circle, draw,minimum size= 0.65cm] {$a$};
\node (b) at (2,0) [circle, draw,minimum size= 0.65cm] {$b$};
\node (c) at (4,1) [circle, draw,minimum size= 0.65cm] {$c$};
\node (e) at (4,-1) [circle, draw,minimum size= 0.65cm] {$d$};
\node (d) at (6,0) [circle, draw,minimum size= 0.65cm] {$e$};
\node (f) at (1,1) [circle, draw,minimum size= 0.65cm] {$f$};

\path[->] (b) edge  (c);
\path[->] (b) edge  (e);
\path[->] (c) edge  (d);
\path[->] (e) edge  (d);

\path[->,dashed] (a) edge  (b);

\end{tikzpicture}}
   \caption{AF $F_{20}$ from Example \ref{ex:cope_Nrho_addrob}, where attack $(a,b)$ is added later to obtain $F'_{20}$.}\label{tikz:cope_Nrho_addrob}
\end{figure}
\end{example}

Since the names of the arguments do not have any influence on the h-categoriser ranking, we see that $\cope(\sqsupseteq^{\mathcal{N}_{Cat}})$ satisfies syntax independence.

\begin{table}[]
     \centering
     \begin{tabular}{|l||c|}
     \hline
     Principles & $\cope(\sqsupseteq^{\mathcal{N}_{Cat}})$ \\
     \hline
        $\sigma$-generalisation  & X  \\
        composition  & X  \\
        decomposition  & X  \\
        weak reinstatement & X  \\
        strong reinstatement  & X \\
        addition robustness & X \\
        syntax independence &  \checkmark  \\ \hline
     \end{tabular}
     \caption{Principles satisfied by $\cope(\sqsupseteq^{\mathcal{N}_{Cat}})$.}
     \label{tab:principle_cope_NCat}
 \end{table}
In Table \ref{tab:principle_cope_NCat} we see that $\cope(\sqsupseteq^{\mathcal{N}_{Cat}})$ only satisfies syntax independence. Therefore we will not discuss this family of extension-ranking semantics further.

%------------------------------
Amgoud and Ben-Naim \cite{DBLP:conf/sum/AmgoudB13a} also discussed how to compare two sets of arguments using an argument-ranking semantics. The notion of \emph{group comparison} states that a set of arguments $E$ is at least as plausible to be accepted as another set $E'$ if we can find for each argument of $E'$ an argument of $E$ which is stronger with respect to an argument-ranking semantics.

\begin{definition}
  Let $F=(A,R)$ be an AF, $E,E' \subseteq A$, and $\rho$ an argument-ranking semantics.
    We define the \emph{group comparison extension-ranking semantics} ($gc_\rho$) as follows:
    $$E \sqsupseteq^{gc_\rho}_F E' \text{ iff } \forall a \in E' \text{ there is a } b\in E \text{ s.t. } b \succeq^{\rho}_F a$$  
\end{definition}

First note that $\sqsupseteq^{gc_\rho}$ is transitive.
\begin{proposition}
    Let $F=(A,R)$ be an AF, $E,E',E'' \subseteq A$, and $\rho$ be an argument-ranking semantics. If $E \sqsupseteq^{gc_\rho}_F E'$ and $E' \sqsupseteq^{gc_\rho}_F E''$ then $E \sqsupseteq^{gc_\rho}_F E''$.
\end{proposition}

\begin{example}
    Consider $F_4$ from Example \ref{ex:af_example} and the sets $\{a,b,c\}$ and $\{d,e,g\}$. We use h-categoriser as the underlying argument-ranking semantics, then $a \succ^{Cat}_{F_4} d$,  $a \succ^{Cat}_{F_4} e$, and  $a \succ^{Cat}_{F_4} g$, so $\{a,b,c\} \sqsupseteq^{gc_{Cat}}_{F_4} \{d,e,g\}$. 
    If we compare these results with $\sqsupseteq^{\mathcal{N}_{Cat}}_{F_4}$ we see that $\mathcal{N}_{Cat,F_4}(\{a,b,c\},\{d,e,g\})= 3$ and $\mathcal{N}_{Cat,F_4}(\{d,e,g\}, \{a,b,c\})= 6$, so $\{d,e,g\} \sqsupseteq^{\mathcal{N}_{Cat}}_{F_4} \{a,b,c\}$ and also $\{d,e,g\} \sqsupseteq^{\cope(\mathcal{N}_{Cat})}_{F_4} \{a,b,c\}$. Thus these two semantics do behave differently.     
\end{example}

\begin{example}\label{ex:lattice_gcCat}
    Recall $F_5$ from Example \ref{ex:comp}. We use the h-categoriser argument-ranking semantics, then the set $\{h,i,j\}$ is among the most plausible sets with respect to $\sqsupseteq^{gc_{Cat}}_{F_5}$. The two admissible sets $\{h\}$ and $\{h,j\}$ are also among the most plausible sets. The complete extension ranking can be found in Figure \ref{tikz:lattice_gcCat}.
    \begin{figure}
    \centering
 \scalebox{1}{
\begin{tikzpicture}

\node (abc) at (-3,-4) [] {$\{h,i,j\}$};
\node (ab) at (1,-4) [] {$\{h,i\}$};
\node (bc) at (2,-2) [] {$\{i,j\}$};
\node (b) at (0,-1) [] {$\{i\}$};
\node (c) at (-2,-2) [] {$\{j\}$};
\node (ac) at (-1,-4) [] {$\{h,j\}$};
\node (empty) at (0,0) [] {$\emptyset$};
\node (a) at (3,-4) [] {$\{h\}$};

\path[->] (empty) edge  (b);

\path[->] (b) edge  (c);
\path[->] (b) edge  (bc);
\path[-] (c) edge  (bc);

\path[->] (c) edge  (abc);
\path[->] (c) edge  (ac);
\path[->] (c) edge  (ab);
\path[->] (c) edge  (a);

\path[->] (bc) edge  (abc);
\path[->] (bc) edge  (ac);
\path[->] (bc) edge  (ab);
\path[->] (bc) edge  (a);

\path[-] (abc) edge  (ac);
\path[-] (ac) edge  (ab);
\path[-] (ab) edge  (a);

\end{tikzpicture}}

   \caption{Lattices depicting $\sqsupseteq^{gc_{Cat}}_{F_5}$ from Example \ref{ex:lattice_gcCat}.}\label{tikz:lattice_gcCat}
\end{figure}
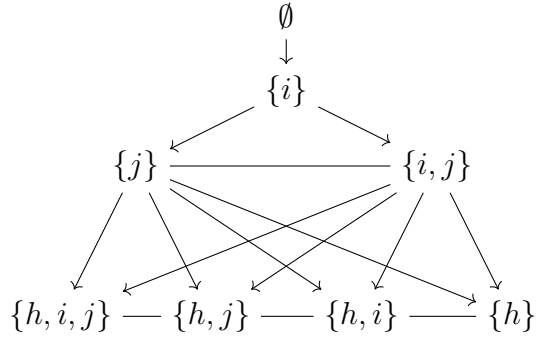
\end{example}

We begin the investigation of the satisfied principles of  $gc_\rho$ by showing that the set containing all arguments $A$ is among the most plausible sets for any argument-ranking semantics $\rho$.
\begin{proposition}\label{prop:A_max_gc}
    Let $F=(A,R)$ be an AF and $\rho$ an argument-ranking semantics, then $A \in \maxpl_{gc_{\rho}}(F)$.     
\end{proposition}

Since $A$ is among the most plausible sets for $\sqsupseteq^{gc_\rho}$ we know that $\sigma$-generalisation is violated for every argument-ranking semantics $\rho$. 

The best ranked argument inside a set dominates the entire set. Therefore composition is satisfied if the underlying argument-ranking semantics satisfies Independence.
\begin{proposition}
    $gc_\rho$ satisfies composition if $\rho$ satisfies Independence.
\end{proposition}

While the dominance of the best ranked argument is helpful to satisfy composition, this dominance hurt the satisfaction of decomposition.
\begin{proposition}
     $gc_\rho$ violates decomposition if $\rho$ satisfies Void Precedence, Non-attacked Equality, and Independence.
\end{proposition}

Next, we show that weak reinstatement is satisfied.
\begin{proposition}\label{prop:gc_wreinst}
    $gc_\rho$ satisfies weak reinstatement.
\end{proposition}

The proof of Proposition \ref{prop:gc_wreinst} shows that $E \cup \{a\} \sqsupset^{gc_{\rho}}_F E$ cannot be guaranteed for every AF $F$, hence strong reinstatement is violated. 

Adding attacks to an argument decreases the strength of an argument and in turn increase the strength of the arguments attacked by the newly attacked argument. This behaviour results in the violation of addition robustness by $gc_{Cat}$. 
\begin{example}\label{ex:gc_cat_addrob}
Let $F_{21}= (\{a,b,c,d,e,f\},\{(b,c),(c,d),(e,a),(e,b),(e,e)\})$ be an AF, as depicted in Figure \ref{tikz:gc_cat_addrob}. Consider the two sets $\{a,d\}$ and $\{b,c\}$, then we have $\{a,d\} \equiv^{gc_{Cat}}_{F_{21}} \{b,c\}$ since $$a \simeq^{Cat}_{F_{21}} b \simeq^{Cat}_{F_{21}} c \simeq^{Cat}_{F_{21}} d \simeq^{Cat}_{F_{21}} e \simeq^{Cat}_{F_{21}} $$ So we can add attack $(a,b)$ to $F_{21}$ and create $F_{21}' =  (\{a,b,c,d,e,f\},\\ \{(b,c),(c,d),(e,a),(e,b),(e,e)\} \cup \{(a,b)\})$. However, this addition changes the argument ranking to $$c  \succ^{Cat}_{F_{21}'} a \simeq^{Cat}_{F_{21}'} e \succ^{Cat}_{F_{21}'}  d \succ^{Cat}_{F_{21}'} b$$ and this implies $\{b,c\} \sqsupset^{gc_{Cat}}_{F_{21}'} \{a,d\}$ and therefore addition robustness is violated. 
    
 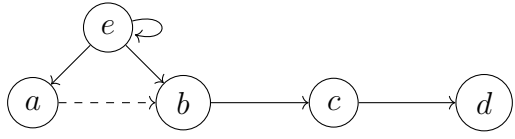
\begin{figure}
    \centering
 
 \scalebox{1}{
\begin{tikzpicture}

\node (a) at (0,0) [circle, draw,minimum size= 0.65cm] {$a$};
\node (b) at (2,0) [circle, draw,minimum size= 0.65cm] {$b$};
\node (c) at (4,0) [circle, draw,minimum size= 0.65cm] {$c$};
\node (e) at (1,1) [circle, draw,minimum size= 0.65cm] {$e$};
\node (d) at (6,0) [circle, draw,minimum size= 0.65cm] {$d$};
%\node (f) at (4,0) [circle, draw,minimum size= 0.65cm] {$f$};

\path[->] (b) edge  (c);
%\path[->] (b) edge  (f);
\path[->] (c) edge  (d);
\path[->] (e) edge  (a);
\path[->] (e) edge  (b);

\path[->,dashed] (a) edge  (b);
\path[->] (e) edge  [loop right] node {} ();

\end{tikzpicture}}
   \caption{AF $F_{21}$ from Example \ref{ex:gc_cat_addrob}, where attack $(a,b)$ is added later to obtain $F'_{21}$.}\label{tikz:gc_cat_addrob}
\end{figure}
\end{example}
Example \ref{ex:gc_cat_addrob} can be used to show the violation of addition robustness if a different argument-ranking semantics is used, like for example the \emph{Burden-based semantics} \cite{DBLP:conf/sum/AmgoudB13a} or the argument-ranking semantics by Matt and Toni \cite{DBLP:conf/jelia/MattT08}.

Finally we show that syntax independence is satisfied by $gc_\rho$.
\begin{proposition}
     $gc_\rho$ satisfies syntax independence if $\rho$ satisfies Abstraction.
\end{proposition}
%\begin{proof}
 %   Let $F= (A,R)$ and $F'=(A',R')$ be AFs such that there is a isomorphism $\gamma: A \rightarrow A'$ with $F' = \gamma(F)$. Then if $\rho$ satisfies Abstraction we have for all $a,b \in A$ s.t. $a \succeq^\rho_F b$ then $\gamma(a) \succeq^\rho_{\gamma(F)} \gamma(b)$. So, the underling argument rankings are the same for $F$ and $F'$, this implies that for two sets $E, E' \subseteq A$ with $E \sqsupseteq^{gc_{\rho}}_F E'$ then $\gamma(E) \sqsupseteq^{gc_{\rho}}_{\gamma(F)} \gamma(E')$. 
%\end{proof}

Since the h-categoriser argument-ranking semantics satisfies Abstraction, Independence, Void Precedence, and Non-attacked Equivalence, $gc_{Cat}$ satisfies composition and syntax independence and violates decomposition like depicted in Table \ref{tab:principle_gc_Cat}.

\begin{table}[]
     \centering
     \begin{tabular}{|l||c|}
     \hline
     Principles & $gc_{Cat}$ \\
     \hline
        $\sigma$-generalisation  & X  \\
        composition  & \checkmark  \\
        decomposition  & X  \\
        weak reinstatement & \checkmark  \\
        strong reinstatement  & X \\
        addition robustness & X \\
        syntax independence &  \checkmark  \\ \hline
     \end{tabular}
     \caption{Principles satisfied by $gc_{Cat}$.}
     \label{tab:principle_gc_Cat}
 \end{table}

\iffalse
\ks{
First idea for group comparisons extension-ranking semantics.
Strict version:
\begin{definition}
    Let $F=(A,R)$ be an AF, $E,E' \subseteq A$, and $\rho$ an argument-ranking semantics.
    We define the \emph{strict group comparison extension-ranking semantics} ($s\text{-}gc_\rho$) as follows:
    $$E \sqsupset^{s\text{-}gc_\rho}_F E' \text{ iff } E \sqsupseteq^{gc_\rho}_F E' \text{ and } (|E| > |E'| \text{ or } \exists a \in E' \text{ s.t. } \exists b \in E \text{ with } b\succ^\rho_F a)$$
\end{definition}
}
\fi

\subsection{Numerical Evaluation Functions}\label{subsec:numerical_evaluation}
In the previous section, arguments that are not part of the set are ignored when the set is evaluated. To illustrate the problem of this behaviour, consider the following example.
\begin{example}
    Consider $F_4$ from Example \ref{ex:af_example} and the sets $\{c\}$ and $\{d\}$. If we use $\sqsupseteq^{\mathcal{N}_{Cat}}$ to compare these two sets, we see that $\{d\} \sqsupset^{\mathcal{N}_{Cat}}_{F_4} \{c\}$. However, the existence of the argument $b$ is completely ignored. We could remove the argument $b$ from $F_4$ and the relation between these two sets will not change, although the existence of $b$ is important for the rejection of $\{c\}$. Without the attack of $b$ on $c$, the set $\{c\}$ is admissible, so the reasoning for the modified AF changes. Thus, focusing only on a local view is not appropriate in abstract argumentation, since the acceptance, and thus also the plausibility of acceptance of a set should be evaluated in the context of the complete AF. 
\end{example}
Konieczny et al. \cite{DBLP:conf/ecsqaru/KoniecznyMV15} proposed using the number of times an argument within a set is part of a $\sigma$-extension, and aggregating the resulting sequences to compare two $\sigma$-extensions. We use the same idea to define an extension-ranking semantics. We evaluate sets of arguments based on the quality of the arguments they contain, and an argument is considered ``good'' if it is contained in a large number of $\sigma$-extensions, preferably in every $\sigma$-extension. Prior to introducing an extension-ranking semantics, it is necessary to articulate criteria that determine when an argument is qualifies as ``good''.

\begin{definition}
    Let $F= (A,R)$ be an AF, $a \in A$. A function $\varepsilon_F(a): A \rightarrow \mathbb{N}$ is called a \emph{numerical evaluation function}. $\varepsilon_F(a)$ gives $a$ a numerical value in the context of $F$. By extending this definition, we get for an set of arguments $E = \{a,b,\dots\}$ a sequence $\varepsilon_F(E)= \{\varepsilon_F(a),\varepsilon_F(b),\dots\}$.   
\end{definition}
The \emph{numerical evaluation functions} gives each argument a numerical value corresponding to its strength.
Konieczny et al. \cite{DBLP:conf/ecsqaru/KoniecznyMV15} proposed a numerical evaluation function based on the number of times an argument is part of a $\sigma$-extension. 

\begin{definition}[\cite{DBLP:conf/ecsqaru/KoniecznyMV15}]\label{def:ne}
    Let $F= (A,R)$ be an AF, $\sigma$ an extension semantics and $a \in A$. $ne_{\sigma,F}(a)$ is the number of $\sigma$-extension $a$ is part of, i.\,e., $ne_{\sigma,F}(a)=|\{E \in \sigma(F) | a \in E\}|$. For any set $E = \{a,b,\dots\}$ we define its \emph{support sequence} as $vsupp_{\sigma,F}(E)= \{ne_{\sigma,F}(a),ne_{\sigma,F}(b),\dots\}$.  
\end{definition}

For each set of arguments $E$, we return a sequence containing for each argument $a \in E$ the number of occurrences of $a$ within a $\sigma$-extension, i.\,e. the number of times argument $a$ is acceptable with respect to $\sigma$.
\begin{example}\label{ex:ne_co}
    Consider $F_4$ from Example \ref{ex:af_example}. The corresponding complete extensions are $\{a\},\{a,g\},\{a,c,g\}$, and $\{a,d,g\}$. So the corresponding sequences for $ne_{\co,F_4}$ are $ne_{\co,F_4}(\{a\})= \{4\}$,  $ne_{\co,F_4}(\{a,g\})= \{4,3\}$, \\$ne_{\co,F_4}(\{a,c,g\})= \{4,1,3\}$, and $ne_{\co,F_4}(\{a,d,g\})= \{4,1,3\}$.
\end{example}

Now we can compute a sequence for each set based on the strength of the arguments it contains. To finally reason with these sequences 
we need methods for comparing two sequences. One way is to first aggregate each sequence and then compare the aggregated values.

\begin{definition}
An \emph{aggregation function} $\Box$ maps a sequence of numbers $v= \{v_1,\dots, v_n\}$ to a single value, i.\,e., for every $n \in \mathbb{N}$, $\Box: \mathbb{N}^n \rightarrow \mathbb{N}$ s.t.:
\begin{itemize}
    \item if $v_i \geq v_i'$, then $\Box\{v_1,\dots, v_i, \dots, v_n\} \geq \Box\{v_1,\dots, v'_i, \dots, v_n\}$
    \item $\Box(x)=x$
\end{itemize}
\end{definition}

 We denote by $sum(v)$ the sum of all elements of the sequence $v$, i.\,e., $sum(v) = \Sigma_{i=1}^n v_i$. $max(v)$ returns the maximum element of the sequence $v$ and $min(v)$ the minimum element. $leximin(v)$ rearranges the sequence $v$ so that the elements are ordered in decreasing order and $leximax(v)$ rearranges the sequence $v$ so that the elements are ordered in increasing order. 
    Note that besides $sum(v)$, $max(v)$, $min(v)$, $leximin(v)$, and $leximax(v)$, there are a number of other aggregation functions that can be used to aggregate sequences (see for example the work of Dubois et al. \cite{DBLP:journals/fss/DuboisFP96} for a discussion).

    Finally, using aggregated sequences, we can compare sets of arguments based on the strength of the arguments they contain. For $sum(v)$, $max(v)$ and $min(v)$ the order $\geq^{\Box}$ is clear, for $leximax(v)$ and $leximin(v)$ we use the lexicographic order $\geq^{\{leximin,leximax\}}$ for natural numbers, i.\,e., for two sequences $v,v'$ in ascending (descending) order $v \geq^{leximin} v'$ ($v \geq^{leximax} v'$) if and only if there exists $i$ s.t. $v_i > v'_i$ and for all $j < i$, $v_j = v_j'$, and $v =^{\{leximin,leximax\}} v'$ if and only if for all $i$, $v_i = v_i'$. Thus, a set with a better aggregated sequence is more plausible to be accepted than other sets. Using this observation, we define an extension-ranking semantics.

    \begin{definition}
    Let $F=(A,R)$ be an AF, $\varepsilon$ a numerical evaluation function, $E,E' \subseteq A$ two sets of arguments and $\Box$ an aggregation function. We define the \emph{order-based extension-ranking semantics} ${\OBE_{\varepsilon,\Box}}$ via
    $$E \sqsupseteq^{\OBE_{\varepsilon,\Box}}_{F} E' \text{ iff } \Box(\varepsilon_F(E)) \geq^\Box  \Box(\varepsilon_F(E'))$$
\end{definition}
Note that $\OBE_{\varepsilon,\Box}$ is only an extension-ranking semantics if $\Box$ is transitive.

\begin{example}\label{ex:OBE_neco}
    Let us consider $F_4$ from Example \ref{ex:af_example} and the following sets $\{a\},\{a,e,g\},\{a,c,g\}$, and $\{a,d,g\}$. So, the corresponding sequences for $ne_{\co}$ are $ne_{\co,F_4}(\{a\})= \{4\}$,  $ne_{\co,F_4}(\{a,e,g\})= \{4,0,3\}$,  $ne_{\co,F_4}(\{a,c,g\})= \{4,1,3\}$, and $ne_{\co,F_4}(\{a,d,g\})= \{4,1,3\}$. When we use $sum(v)$ as the underlying aggregating function, the resulting ranking is:
    $$\{a,c,g\} \equiv^{\OBE_{ne_{\co},sum}}_{F_4} \{a,d,g\} \sqsupset^{\OBE_{ne_{\co},sum}}_{F_4} \{a,e,g\} \sqsupset^{\OBE_{ne_{\co},sum}}_{F_4} \{a\}$$ 
  Using $max$ we see that all these four sets are equally plausible to be accepted, since they all contain $a$, which is a skeptically accepted argument w.r.t. the complete extension semantics. So the choice of the aggregation function is important. 
\end{example}
\begin{example}\label{ex:lattice_OBE_nead,sum}
Consider $F_5$ from Example \ref{ex:comp}. We use $ne_\ad$ as the numerical evaluation function and $sum$ as the aggregation function, then the resulting extension ranking $\sqsupseteq^{\OBE_{ne_\ad,sum}}_{F_5}$ is depicted in Figure \ref{tikz:lattice_OBE_nead,sum}. While the admissible set $\{h,j\}$ is still among the most plausible sets the other admissible sets $\{h\}$ and $\emptyset$ are not. Also, the set containing every argument $\{h,i,j\}$ is among the most plausible sets.

 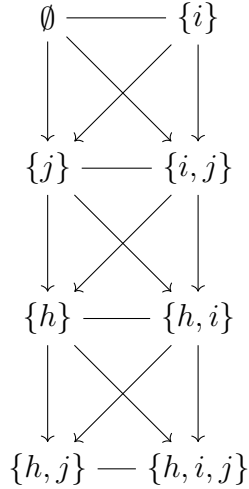
\begin{figure}
    \centering
 \scalebox{1}{
\begin{tikzpicture}

\node (abc) at (2,-6) [] {$\{h,i,j\}$};
\node (ab) at (2,-4) [] {$\{h,i\}$};
\node (bc) at (2,-2) [] {$\{i,j\}$};
\node (b) at (2,0) [] {$\{i\}$};
\node (c) at (0,-2) [] {$\{j\}$};
\node (ac) at (0,-6) [] {$\{h,j\}$};
\node (empty) at (0,0) [] {$\emptyset$};
\node (a) at (0,-4) [] {$\{h\}$};

\path[-] (b) edge  (empty);
\path[->] (b) edge  (bc);
\path[->] (b) edge  (c);
\path[->] (empty) edge  (bc);
\path[->] (empty) edge  (c);

\path[-] (bc) edge  (c);
\path[->] (bc) edge  (ab);
\path[->] (bc) edge  (a);
\path[->] (c) edge  (a);
\path[->] (c) edge  (ab);

\path[-] (a) edge  (ab);
\path[->] (a) edge  (ac);
\path[->] (a) edge  (abc);
\path[->] (ab) edge  (ac);
\path[->] (ab) edge  (abc);

\path[-] (ac) edge  (abc);
\end{tikzpicture}}

   \caption{Lattices depicting $\sqsupseteq^{\OBE_{ne_\ad,sum}}_{F_5}$ from Example \ref{ex:lattice_OBE_nead,sum}.}\label{tikz:lattice_OBE_nead,sum}
\end{figure}
\end{example}

Next, we investigate the principles $\OBE_{ne_\sigma,\Box}$ satisfies.

In Example \ref{ex:lattice_OBE_nead,sum} we have already seen that the set containing all arguments $A$ is among the most plausible sets for  $\sqsupseteq^{\OBE_{ne_\ad,sum}}_{F_5}$. We use this observation to show the violation of $\sigma$-generalisation, i.e. we show at $A$ is always among the most plausible sets for $\sigma \in \{\ad,\co,\gr,\pr,\sst\}$.

\begin{proposition}
    Let $F=(A,R)$ be an AF, then $A \in \maxpl_{\OBE_{ne_\sigma,\Box}}(F)$ for $\sigma \in \{\ad,\co,\gr,\pr,\sst\}$ and $\Box \in \{sum,max,leximax\}$.
\end{proposition}

For $\Box = \{min,leximin\}$ the set containing all arguments is not among the most plausible sets for $\OBE_{ne_\sigma,\Box}$, however $\sigma$-generalisation is still violated.
\begin{example}
    Consider $F_5$ from Example \ref{ex:comp}, then $\{h,j\}$ is a complete, preferred, and stable extension and also the grounded extension. However, $min(\{h\}) > min(\{h,j\})$ and therefore $\{h\} \sqsupset^{\OBE_{ne_\sigma,\Box}}_{F_5} \{h,j\}$ for all $\sigma \in \{\ad,\co,\gr,\pr,\sst\}$ and $\Box \in \{min,leximin\}$ and therefore $\sigma$-generalisation is violated.
\end{example}

The following lemmas will be helpful to show that $\OBE_{ne_{\sigma},\Box}$ satisfies composition for $\sigma \in \{\ad,\co,\pr,\st\}$ and $\Box \in \{leximax,leximin\}$.
\begin{lemma}\label{lemma:lexmax_union}
    Let $V,W,X,Y$ be sequences in descending order s.t. $V \sqsupseteq^{leximax} W$ and $X \sqsupseteq^{leximax} Y$. It holds that $c_1 * V \cup c_2 *X \sqsupseteq^{leximax} c_1 * W \cup c_2* Y$ for $c_1, c_2 \in \mathbb{N}$.
\end{lemma}

For leximin we can show the same behaviour using the same reasoning. 
\begin{lemma}\label{lemma:lexmin_union}
    Let $V,W,X,Y$ be sequences in descending order s.t. $V \sqsupseteq^{leximin} W$ and $X \sqsupseteq^{leximin} Y$. It holds that $c_1*V \cup c_2*X \sqsupseteq^{leximin} c_1*W \cup c_2*Y$ for $c_1,c_2 \in \mathbb{N}$.
\end{lemma}

\begin{lemma}\label{lemma:ne_comp}
    Let $F_1= (A_1, R_1)$ and $F_2= (A_2,R_2)$ be two AFs s.t. $A_1 \cap A_2 =\emptyset$ and $F = F_1 \cup F_2$ then it holds that $ne_{\sigma,F}(a) = (ne_{\sigma,F_1} * |\sigma(F_2)|) + (ne_{\sigma,F_2}(a) * |\sigma(F_1)|)$ for every $a \in A_1 \cup A_2$ and $\sigma \in \{\ad,\co, \gr, \pr,\st, \sst\}$.
\end{lemma}

\begin{proposition}
    $\OBE_{ne_\sigma, \Box}$ satisfies composition for $\sigma \in \{\ad,\co,\pr,\st, \sst\}$ and $\Box \in \{sum,max,leximax,min,leximin\}$.
\end{proposition}

For extension semantics $\sigma \in \{\ad, \co,\pr, \st, \sst\}$ and aggregation functions $\Box \in \{sum, max, leximax, min, leximin\}$ we can show that decomposition is violated for $\OBE_{ne_\sigma,\Box}$. 
\begin{example}\label{ex:OBE_NE:decomp}
    Let $F_{22}$ be the AF as depicted in Figure \ref{tikz:OBE_NE:decomp}. Consider the sets $\{a,e\}$ and $\{c,f,g\}$. $F_{22}$ can be partitioned into two disjoint AFs $F_{22,1}= (\{a,b,c\}, \{(a,b),(b,a),(a,c),(b,c)\}$ and $F_{22,2}= (\{d,e,f,g\}, \{(d,e),(e,f),\\(e,g),(f,g),(g,f)\}$. The corresponding sequences are: $ne_{\co,F_{22}}(\{a,e\})= (3,0)$,  $ne_{\co,F_{22,1}}(\{a\})= (1)$,  $ne_{\co,F_{22,2}}(\{e\})= (0)$, and  $ne_{\co,F_{22}}(\{c,f,g\})= (0,3,3)$, $ne_{\co,F_{22,1}}(\{c\})= (0)$, and $ne_{\co,F_{22,2}}(\{f,g\})= (1,1)$. 
   \begin{description}
       \item[``$\Box = sum$'':] We have $\{c,f,g\} \sqsupset^{\OBE_{ne_\co,sum}}_{F_{22}} \{a,e\}$, but $\{a\} \sqsupset^{\OBE_{ne_\co,sum}}_{F_{22,1}} \{c\}$. 
    
   \item[``$\Box \in \{max, leximax\}$'':] We have $\{a,e\} \equiv^{\OBE_{ne_\co,\Box}}_{F_{22}} \{c,f,g\}$, but  $\{a\} \sqsupset^{\OBE_{ne_\co,\Box}}_{F_{22,1}} \{c\}$ and $\{f,g\} \sqsupset^{\OBE_{ne_\co,\Box}}_{F_{22,2}} \{e\}$. 
    
      \item[``$\Box \in \{min, leximin\}$'':] We have $\{a,e\} \equiv^{\OBE_{ne_\co,\Box}}_{F_{22}} \{c,f,g\}$, but $\{a\} \sqsupset^{\OBE_{ne_\co,\Box}}_{F_{22,1}} \{c\}$ and $\{f,g\} \sqsupset^{\OBE_{ne_\co,\Box}}_{F_{22,2}} \{e\}$.  
 \end{description} 
 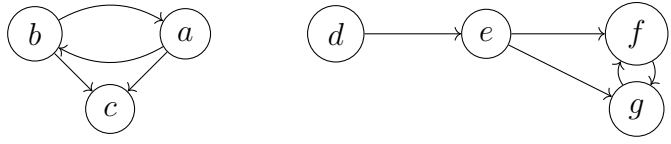
\begin{figure}
    \centering
 
 \scalebox{1}{
\begin{tikzpicture}

\node (a) at (0,0) [circle, draw,minimum size= 0.65cm] {$a$};
\node (b) at (-2,0) [circle, draw,minimum size= 0.65cm] {$b$};
\node (c) at (-1,-1) [circle, draw,minimum size= 0.65cm] {$c$};

\node (d) at (2,0) [circle, draw,minimum size= 0.65cm] {$d$};
\node (e) at (4,0) [circle, draw,minimum size= 0.65cm] {$e$};
\node (f) at (6,0) [circle, draw,minimum size= 0.65cm] {$f$};
\node (g) at (6,-1) [circle, draw,minimum size= 0.65cm] {$g$};

\path[->] (d) edge  (e);
\path[->] (e) edge  (f);
\path[->] (e) edge  (g);
\path[->,bend left] (f) edge  (g);
\path[->, bend left] (g) edge  (f);

\path[->,bend left] (a) edge  (b);
\path[->, bend left] (b) edge  (a);
\path[->] (a) edge  (c);
\path[->] (b) edge  (c);

\end{tikzpicture}}
   \caption{AF $F_{22}$ from Example \ref{ex:OBE_NE:decomp}.}\label{tikz:OBE_NE:decomp}
\end{figure}
\end{example}

For $\sigma = \gr$ we need a different counterexample for show that $\OBE_{ne_\gr,\Box}$ violates decomposition for $\Box \in \{sum,max,leximax, min,leximin\}$.
\begin{example}\label{ex:OBE_NE_gr:decomp}
    Let $F_{23} = (\{a,b,c,d,e,f\}, \{(a,b),(b,c),(d,e),(e,f)\})$ be an AF as depicted in Figure \ref{tikz:OBE_NE_gr:decomp}. Consider sets $\{a,c,e\}$ and $\{b,d,f\}$. Then $F_{23}$ can be partitioned into $F_{23,1}=  (\{a,b,c\}, \{(a,b),(b,c)\})$ and  $F_{23,2}= (\{d,e,f\}, \{(d,e),(e,f)\})$. We have $\{a,c,e\} \equiv^{\OBE_{ne_\gr,\Box}}_{F_{23}} \{b,d,f\}$ for \\$\Box \in \{sum,max,leximax,min,leximin\}$. However, we have  $\{a,c\} \sqsupset^{\OBE_{ne_\gr,\Box}}_{F_{23,1}} \{b\}$ and  $\{d,f\} \sqsupset^{\OBE_{ne_\gr,\Box}}_{F_{23,2}} \{e\}$ for $\Box \in \{sum,max,leximax,min,leximin\}$. Thus, decomposition is violated for $\OBE_{ne_\gr,\Box}$ with $\Box \in \{sum,max,leximax,\\ min,leximin\}$.

     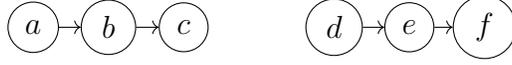
\begin{figure}
    \centering
 
 \scalebox{1}{
\begin{tikzpicture}

\node (a) at (0,0) [circle, draw,minimum size= 0.65cm] {$a$};
\node (b) at (1,0) [circle, draw,minimum size= 0.65cm] {$b$};
\node (c) at (2,0) [circle, draw,minimum size= 0.65cm] {$c$};

\node (d) at (4,0) [circle, draw,minimum size= 0.65cm] {$d$};
\node (e) at (5,0) [circle, draw,minimum size= 0.65cm] {$e$};
\node (f) at (6,0) [circle, draw,minimum size= 0.65cm] {$f$};

\path[->] (a) edge  (b);
\path[->] (b) edge  (c);
\path[->] (d) edge  (e);
\path[->] (e) edge  (f);
\end{tikzpicture}}
   \caption{AF $F_{23}$ from Example \ref{ex:OBE_NE_gr:decomp}.}\label{tikz:OBE_NE_gr:decomp}
\end{figure}
\end{example}

Since adding a defended argument into a set does not decrease the strength values of the other arguments $\OBE_{ne_\sigma,\Box}$ satisfies weak reinstatement for  $\Box \in \{sum, max\}$ . For $\Box = leximax$, $\OBE_{ne_\sigma,leximax}$ even satisfies strong reinstatement.
\begin{proposition}
    For $\Box \in \{sum, max\}$  $\OBE_{ne_\sigma,\Box}$ satisfies weak reinstatement and for $\Box = leximax$ satisfies strong reinstatement with $\sigma \in \{\ad, \co, \gr,\\ \pr, \st, \sst\}$.  
\end{proposition}

If the added argument $a$ is part less $\sigma$-extensions than the other arguments of the set it is defended by, then the addition of $a$ decreases the minimum and therefore weak reinstatement is violated by $\OBE_{ne_\sigma,\Box}$ for $\Box \in \{min, leximin\}$ for $\sigma \in \{\ad,\co,\pr,\st,\sst\}$
\begin{example}\label{ex:OBE_NE_min:weak_reinstat}
    Let $F_{24}$ be an AF as depicted in Figure \ref{tikz:OBE_NE_min:weak_reinstat}. Then every argument except $c$ is part of three complete extensions. Argument $c$ is defended by the set $\{a,e\}$, hence adding $c$ into $\{a,e\}$ should not lower the plausibility of acceptance of this set. However, $min(ne_{\co,F_{24}}(\{a,e\})) = 3$ but $min(ne_{\co,F_{24}}(\{a,c,e\})) = 1$, so $\{a,e\} \sqsupset^{\OBE_{ne_\co,\Box}}_{F_{24}} \{a,c,e\}$ for $\Box \in \{min,leximin\}$. So, weak reinstatement is violated. The same behaviour can also be observed for semantics $\sigma \in \{\ad,\pr,\st,\sst\}$.
 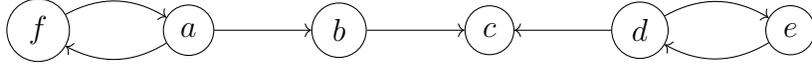
\begin{figure}
    \centering
 
 \scalebox{1}{
\begin{tikzpicture}

\node (a) at (0,0) [circle, draw,minimum size= 0.65cm] {$a$};
\node (b) at (2,0) [circle, draw,minimum size= 0.65cm] {$b$};
\node (c) at (4,0) [circle, draw,minimum size= 0.65cm] {$c$};
\node (d) at (6,0) [circle, draw,minimum size= 0.65cm] {$d$};
\node (e) at (8,0) [circle, draw,minimum size= 0.65cm] {$e$};
\node (f) at (-2,0) [circle, draw,minimum size= 0.65cm] {$f$};

\path[->] (a) edge  (b);
\path[->] (b) edge  (c);
\path[->] (d) edge  (c);
\path[->,bend left] (d) edge  (e);
\path[->, bend left] (e) edge  (d);

\path[->,bend left] (a) edge  (f);
\path[->, bend left] (f) edge  (a);

\end{tikzpicture}}
   \caption{AF $F_{24}$ from Example \ref{ex:OBE_NE_min:weak_reinstat}.}\label{tikz:OBE_NE_min:weak_reinstat}
\end{figure}
\end{example}

So, for only $\sigma = \gr$, $\OBE_{ne_\gr,\Box}$ for $\Box \in \{min,leximin\}$ satisfies weak reinstatement.   
\begin{proposition}\label{prop:OBE_ne_gr_min_weak_reinstat}
    $\OBE_{ne_\gr,\Box}$ for $\Box \in \{min,leximin\}$ satisfies weak reinstatement.   
\end{proposition}

The proof for Proposition \ref{prop:OBE_ne_gr_min_weak_reinstat} we show that $E \cup \{a\} \equiv^{\OBE_{ne_\gr,\Box}}_F E$ for $\Box \in \{min,leximin\}$, so strong reinstatement is violated. 

To show that addition robustness is violated by $\OBE_{ne_\sigma,\Box}$ for $\sigma \in \{\ad,\co,\gr,\\ \pr,\st,\sst\}$ and $\Box \in \{sum, max, leximax, min, leximin\}$ we need several counterexamples. 
We begin with $\OBE_{ne_\sigma,\Box}$ with $\sigma \in \{\co,\pr\}$ and $\Box \in \{sum, max, \\leximax, min, leximin\}$.

\begin{example}\label{ex:OBE_co_addrob}
    Let $F_{25}$ be the AF depicted in Figure \ref{tikz:OBE_co_addrob}. Consider sets $\{a,e\}$ and $\{b,g\}$, then for $\Box \in \{sum,leximax,min,leximax\}$ and $\sigma \in \{\co,\pr\}$, we have $\{a,e\} \equiv^{\OBE_{ne_\sigma,\Box}}_{F_{25}} \{b,g\}$. Hence, we can add the attack $(a,b)$ to obtain $F_{25}'$, however for this AF we get: $$\{b,g\} \sqsupset^{\OBE_{ne_\sigma,\Box}}_{F_{25}'} \{a,e\}$$ 

    For $\Box = max$ we consider the two sets $\{a\}$ and $\{b,c\}$, again we have $\{a\} \equiv^{\OBE_{ne_\sigma,max}}_{F_{25}} \{b,c\}$ and therefore we can add $(a,b)$, however we get:  $$\{b,c\} \sqsupset^{\OBE_{ne_\sigma,max}}_{F_{25}'} \{a\}$$ So, addition robustness is violated by $\OBE_{ne_\sigma,\Box}$ with $\sigma \in \{\co,\pr\}$ and $\Box \in \{sum,max,leximax,min,leximin\}$.

\begin{figure}
    \centering
 
 \scalebox{1}{
\begin{tikzpicture}

\node (a) at (0,0) [circle, draw,minimum size= 0.65cm] {$a$};
\node (b) at (0,1.5) [circle, draw,minimum size= 0.65cm] {$b$};
\node (c) at (2,0) [circle, draw,minimum size= 0.65cm] {$c$};
\node (d) at (2,1.5) [circle, draw,minimum size= 0.65cm] {$d$};
\node (e) at (3,0) [circle, draw,minimum size= 0.65cm] {$e$};
\node (f) at (4,0) [circle, draw,minimum size= 0.65cm] {$f$};
\node (bb) at (0,-1) [circle, draw,minimum size= 0.65cm] {$\bar{b}$};
\node (cc) at (2,-1) [circle, draw,minimum size= 0.65cm] {$\bar{c}$};
\node (ff) at (4,-1) [circle, draw,minimum size= 0.65cm] {$\bar{f}$};

\node (g) at (3,-2) [circle, draw,minimum size= 0.65cm] {$g$};

%\node (g) at (0,0) [circle, draw] {$g$};

\path[->, bend left] (a) edge  (c);
\path[->, bend left] (c) edge  (a);
\path[->, bend left] (b) edge  (d);
\path[->, bend left] (d) edge  (b);

\path[->] (c) edge  (e);
\path[->] (d) edge  (e);
\path[->] (e) edge  (f);

\path[->, bend right] (b) edge  (bb);
\path[->] (c) edge  (cc);
\path[->] (f) edge  (ff);

\path[->] (ff) edge  (g);
\path[->] (cc) edge  (g);
\path[->] (bb) edge  (g);

\path[->] (cc) edge  [loop left] node {} ();
\path[->] (bb) edge  [loop left] node {} ();
\path[->] (ff) edge  [loop left] node {} ();

\path[->, dashed] (a) edge  (b);

\end{tikzpicture}}
   \caption{Abstract argumentation framework $F_{25}$ from Example \ref{ex:OBE_co_addrob}, where the dashed attack from $a$ to $b$ is added to obtain $F_{25}'$.}
    \label{tikz:OBE_co_addrob}
\end{figure}
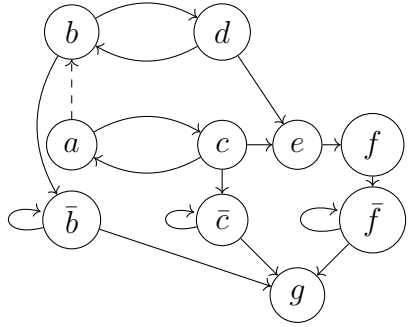
\end{example}

Next, we look at $\OBE_{ne_\sigma,\Box}$ with $\sigma = \ad$ and $\Box \in \{sum,leximax,leximin\}$.
\begin{example}\label{ex:OBE_ad_sum_addrob}
    Let $F_{26}$ be the AF depicted in Figure \ref{tikz:OBE_ad_sum_addrob}. Consider sets $\{a,d,e\}$ and $\{a,b,c\}$, then for $\sigma = \ad$ and $\Box \in \{sum,leximax,leximin\}$ we have $\{a,d,e\} \equiv^{\OBE_{ne_\ad,\Box}}_{F_{26}} \{a,b,c\}$. So, we obtain $F_{26}'$ by adding attack $(a,b)$, however this entails $\{a,b,c\} \sqsupset^{\OBE_{ne_\ad,\Box}}_{F_{26}'} \{a,d,e\}$. So, addition robustness is violated for $\OBE_{ne_\ad,\Box}$ with $\Box \in \{sum,leximax,leximin\}$.

    \begin{figure}
    \centering
 
 \scalebox{1}{
\begin{tikzpicture}

\node (a) at (0,0) [circle, draw,minimum size= 0.65cm] {$a$};
\node (b) at (2,0) [circle, draw,minimum size= 0.65cm] {$b$};
\node (c) at (4,0) [circle, draw,minimum size= 0.65cm] {$c$};

\node (d) at (6,0) [circle, draw,minimum size= 0.65cm] {$d$};
\node (e) at (7,0) [circle, draw,minimum size= 0.65cm] {$e$};

%\node (g) at (0,0) [circle, draw] {$g$};

\path[->] (b) edge  (c);

\path[->] (b) edge  [loop above] node {} ();
\path[->] (d) edge  [loop above] node {} ();
\path[->] (e) edge  [loop above] node {} ();

\path[->, dashed] (a) edge  (b);

\end{tikzpicture}}
   \caption{Abstract argumentation framework $F_{26}$ from Example \ref{ex:OBE_ad_sum_addrob}, where the dashed attack from $a$ to $b$ is added to obtain $F_{26}'$.}
    \label{tikz:OBE_ad_sum_addrob}
\end{figure}
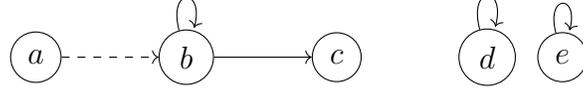
\end{example}

For $\OBE_{ne_\sigma,\Box}$ with $\sigma \in \{\st, \sst\}$ and $\Box \in \{sum,leximax,leximin\}$ we use the following example.
\begin{example}\label{ex:OBE_st_sum_addrob}
    Let $F_{27}$ be the AF depicted in Figure \ref{tikz:OBE_st_sum_addrob}. Consider sets $\{a,d,f\}$ and $\{b,c,e\}$, then for $\sigma \in \{\st, \sst\}$ and $\Box \in \{sum,leximax,leximin\}$ we get $\{a,d,f\} \equiv^{\OBE_{ne_\sigma,\Box}}_{F_{27}} \{b,c,e\}$. Hence, attack $(a,b)$ can be added to obtain $F_{27}'$, however this entails $\{b,c,e\} \sqsupset^{\OBE_{ne_\sigma,\Box}}_{F_{27}'} \{a,d,f\}$. So, addition robustness is violated for $\OBE_{ne_\sigma,\Box}$ with $\Box \in \{sum,leximax,leximin\}$.

    \begin{figure}
    \centering
 
 \scalebox{1}{
\begin{tikzpicture}

\node (a) at (0,0) [circle, draw,minimum size= 0.65cm] {$a$};
\node (b) at (2,0) [circle, draw,minimum size= 0.65cm] {$b$};
\node (c) at (2,-1.5) [circle, draw,minimum size= 0.65cm] {$c$};

\node (d) at (0,-1.5) [circle, draw,minimum size= 0.65cm] {$d$};
\node (e) at (4,-1.5) [circle, draw,minimum size= 0.65cm] {$e$};

\node (f) at (-2,0) [circle, draw,minimum size= 0.65cm] {$f$};

%\node (g) at (0,0) [circle, draw] {$g$};

\path[->] (b) edge  (c);
\path[->] (b) edge  (d);
\path[->] (b) edge  (e);
\path[->, bend left] (d) edge  (c);
\path[->, bend left] (c) edge  (d);
\path[->, bend right, in = -90, out = -60,  looseness=0.6] (d) edge  (e);
\path[->, bend left, in = 90, out = 60,  looseness=0.6] (e) edge  (d);

\path[->] (d) edge  (a);
\path[->] (a) edge  (f);

\path[->] (f) edge  [loop above] node {} ();

\path[->, dashed] (a) edge  (b);

\end{tikzpicture}}
   \caption{Abstract argumentation framework $F_{27}$ from Example \ref{ex:OBE_st_sum_addrob}, where the dashed attack from $a$ to $b$ is added to obtain $F_{27}'$.}
    \label{tikz:OBE_st_sum_addrob}
\end{figure}
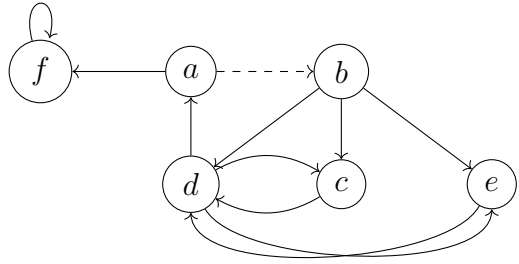
\end{example}

We continue with $\OBE_{ne_\sigma,max}$ with $\sigma \in \{\ad,\st, \sst\}$.
\begin{example}\label{ex:OBE_max_addrob}
    Let $F_{28}$ be the AF depicted in Figure \ref{tikz:OBE_max_addrob}. Consider sets $\{a,e\}$ and $\{b,f\}$, then for $\OBE_{ne_\ad,max}$ we get $\{a,e\} \equiv^{\OBE_{ne_\ad,max}}_{F_{28}} \{b,f\}$. Hence, attack $(a,b)$ can be added to obtain $F_{28}'$, but for this AF we get: $$\{b,f\} \sqsupset^{\OBE_{ne_\ad,max}}_{F_{28}'} \{a,e\}$$
    
    For $\sigma \in \{\st, \sst\}$, we consider sets $\{a\}$ and $\{b,f\}$, then $\{a\} \equiv^{\OBE_{ne_\sigma,max}}_{F_{28}} \{b,f\}$, but $\{b,f\} \sqsupset^{\OBE_{ne_\sigma,max}}_{F_{28}'} \{a\}$.
    So, addition robustness is violated for $\OBE_{ne_\sigma,max}$ with $\sigma \in \{\ad,\st, \sst\}$.

    \begin{figure}
    \centering
 
 \scalebox{1}{
\begin{tikzpicture}

\node (a) at (0,0) [circle, draw,minimum size= 0.65cm] {$a$};
\node (b) at (2,0) [circle, draw,minimum size= 0.65cm] {$b$};
\node (c) at (1,-1) [circle, draw,minimum size= 0.65cm] {$c$};

\node (d) at (4,0) [circle, draw,minimum size= 0.65cm] {$d$};
\node (e) at (6,0) [circle, draw,minimum size= 0.65cm] {$e$};

\node (f) at (3,-1) [circle, draw,minimum size= 0.65cm] {$f$};

%\node (g) at (0,0) [circle, draw] {$g$};

\path[->] (b) edge  (c);
\path[->] (c) edge  (a);
\path[->] (b) edge  (d);
\path[->, bend left] (b) edge  (f);
\path[->, bend left] (f) edge  (b);

\path[->] (d) edge  (e);
\path[->] (f) edge  (d);

\path[->, dashed] (a) edge  (b);

\end{tikzpicture}}
   \caption{Abstract argumentation framework $F_{28}$ from Example \ref{ex:OBE_max_addrob}, where the dashed attack from $a$ to $b$ is added to obtain $F_{28}'$.}
    \label{tikz:OBE_max_addrob}
\end{figure}
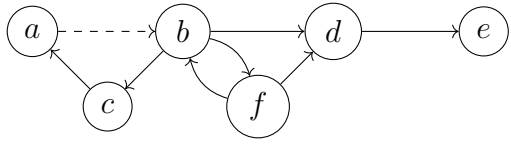
\end{example}

Next, we discuss $\OBE_{ne_\sigma,min}$ with $\sigma \in \{\ad,\st,\sst\}$ 
\begin{example}\label{ex:OBE_min_addrob}
    Let $F_{29}$ be the AF depicted in Figure \ref{tikz:OBE_min_addrob}. Consider sets $\{a,c\}$ and $\{b,c\}$, then for $\sigma = \{\ad,\st, \sst\}$ and $\Box = min$ we get $\{a,c\} \equiv^{\OBE_{ne_\sigma,min}}_{F_{29}} \{b,c\}$. Thus, attack $(a,b)$ can be added to obtain $F_{29}'$, however this entails $\{b,c\} \sqsupset^{\OBE_{ne_\sigma,min}}_{F_{29}'} \{a,c\}$. Therefore, addition robustness is violated for $\OBE_{ne_\sigma,min}$ with $\sigma = \{\ad,\st, \sst\}$.

    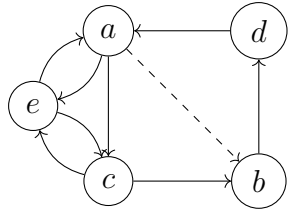
\begin{figure}
    \centering
 
 \scalebox{1}{
\begin{tikzpicture}

\node (a) at (0,0) [circle, draw,minimum size= 0.65cm] {$a$};
\node (b) at (2,-2) [circle, draw,minimum size= 0.65cm] {$b$};
\node (c) at (0,-2) [circle, draw,minimum size= 0.65cm] {$c$};

\node (d) at (2,0) [circle, draw,minimum size= 0.65cm] {$d$};
\node (e) at (-1,-1) [circle, draw,minimum size= 0.65cm] {$e$};

%\node (g) at (0,0) [circle, draw] {$g$};

\path[->] (a) edge  (c);
\path[->] (c) edge  (b);
\path[->] (b) edge  (d);
\path[->] (d) edge  (a);
\path[->, bend left] (a) edge  (e);
\path[->, bend left] (e) edge  (a);
\path[->, bend left] (c) edge  (e);
\path[->, bend left] (e) edge  (c);

\path[->, dashed] (a) edge  (b);

\end{tikzpicture}}
   \caption{Abstract argumentation framework $F_{29}$ from Example \ref{ex:OBE_min_addrob}, where the dashed attack from $a$ to $b$ is added to obtain $F_{29}'$.}
    \label{tikz:OBE_min_addrob}
\end{figure}
\end{example}

We continue with the counterexamples for $\sigma= \gr$.
\begin{example}\label{ex:OBE_gr_sum_addrob}
    Let $F_{30}$ be the AF depicted in Figure \ref{tikz:OBE_gr_sum_addrob}. Consider sets $\{a,e,f\}$ and $\{b,c,d\}$, then for $\sigma = \gr$ and $\Box = \{sum,leximax,leximin\}$ we get $\{a,e,f\} \equiv^{\OBE_{ne_\gr,\Box}}_{F_{30}} \{b,c,d\}$. Hence, attack $(a,b)$ can be added to obtain $F_{30}'$, however this entails $\{b,c,d\} \sqsupset^{\OBE_{ne_\gr,\Box}}_{F_{30}'} \{a,e,f\}$. So, addition robustness is violated for $\OBE_{ne_\gr,\Box}$ with $\Box = \{sum,leximax,leximin\}$.

    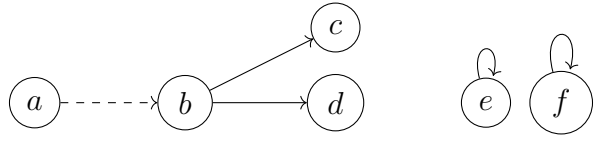
\begin{figure}
    \centering
 
 \scalebox{1}{
\begin{tikzpicture}

\node (a) at (0,0) [circle, draw,minimum size= 0.65cm] {$a$};
\node (b) at (2,0) [circle, draw,minimum size= 0.65cm] {$b$};
\node (c) at (4,1) [circle, draw,minimum size= 0.65cm] {$c$};

\node (d) at (4,0) [circle, draw,minimum size= 0.65cm] {$d$};
\node (e) at (6,0) [circle, draw,minimum size= 0.65cm] {$e$};
\node (f) at (7,0) [circle, draw,minimum size= 0.65cm] {$f$};

%\node (g) at (0,0) [circle, draw] {$g$};

\path[->] (b) edge  (c);
\path[->] (b) edge  (d);

\path[->] (f) edge  [loop above] node {} ();
\path[->] (e) edge  [loop above] node {} ();

\path[->, dashed] (a) edge  (b);

\end{tikzpicture}}
   \caption{Abstract argumentation framework $F_{30}$ from Example \ref{ex:OBE_gr_sum_addrob}, where the dashed attack from $a$ to $b$ is added to obtain $F_{30}'$.}
    \label{tikz:OBE_gr_sum_addrob}
\end{figure}
\end{example}

The final counterexample is for $\OBE_{ne_\gr,\Box}$ for $\Box \in \{max,min\}$.
\begin{example}\label{ex:OBE_gr_max_addrob}
    Let $F_{31}$ be the AF depicted in Figure \ref{tikz:OBE_gr_max_addrob}. Consider sets $\{a,d\}$ and $\{b,e\}$, then for $\sigma = \gr$ and $\Box = \{max,min\}$ we get $\{a,d\} \equiv^{\OBE_{ne_\gr,\Box}}_{F_{31}} \{b,e\}$. We obtain $F_{31}'$ by adding attack $(a,b)$, but this entails $\{b,e\} \sqsupset^{\OBE_{ne_\gr,\Box}}_{F_{31}'} \{a,d\}$. Therefore, addition robustness is violated for $\OBE_{ne_\gr,\Box}$ with $\Box = \{max,min\}$.

    \begin{figure}
    \centering
 
 \scalebox{1}{
\begin{tikzpicture}

\node (a) at (0,0) [circle, draw,minimum size= 0.65cm] {$a$};
\node (b) at (2,0) [circle, draw,minimum size= 0.65cm] {$b$};
\node (c) at (4,0) [circle, draw,minimum size= 0.65cm] {$c$};

\node (d) at (6,0) [circle, draw,minimum size= 0.65cm] {$d$};
\node (e) at (7,0) [circle, draw,minimum size= 0.65cm] {$e$};

%\node (g) at (0,0) [circle, draw] {$g$};

\path[->] (b) edge  (c);
\path[->] (c) edge  (d);

\path[->] (a) edge  [loop above] node {} ();

\path[->, dashed] (a) edge  (b);

\end{tikzpicture}}
   \caption{Abstract argumentation framework $F_{31}$ from Example \ref{ex:OBE_gr_max_addrob}, where the dashed attack from $a$ to $b$ is added to obtain $F_{31}'$.}
    \label{tikz:OBE_gr_max_addrob}
\end{figure}
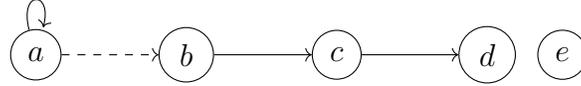
\end{example}

We know that the numerical evaluation function does not relay on the names of the arguments therefore syntax independence is satisfied. 
\begin{proposition}
     $\OBE_{ne_\sigma,\Box}$ satisfies syntax independence for $\Box \in \{sum,max,\\min,leximax, leximin\}$.
\end{proposition}

The satisfies principles by $\OBE_{ne_\sigma,\Box}$ for $\Box \in \{sum, max, leximax, min, \\leximin\}$ and $\sigma \in \{\ad, \co, \gr, \pr, \st, \sst\}$ are summarised in Table \ref{tab:principle_obe_ne}.
\begin{table}[]
     \centering
      \resizebox{\textwidth}{!}{
     \begin{tabular}{|l||c|c|c|c|c|}
     \hline
     Principles & $\OBE_{ne_\sigma,sum}$ & $\OBE_{ne_\sigma,max}$& $\OBE_{ne_\sigma,leximax}$& $\OBE_{ne_\sigma,min}$& $\OBE_{ne_\sigma,leximin}$ \\
     \hline
        $\sigma$-generalisation  & X & X & X & X & X \\
        composition  & \checkmark &  \checkmark &  \checkmark &  \checkmark &  \checkmark  \\
        decomposition  & X & X& X& X& X  \\
        weak reinstatement & \checkmark & \checkmark & \checkmark & X (\checkmark $\sigma = \gr$) & X (\checkmark $\sigma = \gr$) \\
        strong reinstatement  & X & X & \checkmark & X & X   \\
        addition robustness & X & X & X & X & X  \\
        syntax independence &  \checkmark &  \checkmark &  \checkmark &  \checkmark &  \checkmark  \\ \hline
     \end{tabular} }
     \caption{Principles satisfied by $\OBE_{ne_\sigma,\Box}$ for $\Box \in \{sum, max, leximax, min, leximin\}$ and $\sigma \in \{\ad, \co, \gr, \pr, \st, \sst\}$.}
     \label{tab:principle_obe_ne}
 \end{table}

A different option to evaluate the strength of individual arguments is to use argument-ranking semantics such as the h-categoriser \cite{besnardh01}. h-categoriser is part of a subfamily of argument ranking semantics called \emph{gradual semantics} \cite{DBLP:conf/sum/AmgoudB13a,DBLP:conf/kr/AmgoudBDV16}, where each argument is assigned a numerical strength value. Thus, these semantics can be used as numerical evaluation functions, i.\,e., we can calculate the h-categoriser value for each argument of a set to create a sequence. The gradual semantics are thus an alternative to $ne_\sigma$.
\begin{example}
Consider again $F_4$ from Example \ref{ex:af_example} as well as sets $\{a\}$, $\{a,e,g\}$, $\{a,c,g\}$ and $\{a,d,g\}$. In this example we use h-categoriser as a numerical evaluation function. The corresponding h-categoriser values are: $Cat(a)= 1$, $Cat(c)= 0.46$, $Cat(d)= 0.69$, $Cat(e)= 0.51$, and $Cat(g)= 0.74$. Thus the corresponding sequences are: $Cat(\{a\}) = \{1\}$, $Cat(\{a,e,g\}) = \{1,0.51,0.74\}$, $Cat(\{a,c,g\}) = \{1,0.46,0.74\}$ and $Cat(\{a,d,g\}) = \{1,0.69,\\0.74\}$. Using $sum(v)$ as the underlying aggregation function results in the following ranking: 
$$\{a,d,g\} \sqsupset^{\OBE_{Cat,sum}}_{F_4} \{a,e,g\} \sqsupset^{\OBE_{Cat,sum}}_{F_4} \{a,c,g\} \sqsupset^{\OBE_{Cat,sum}}_{F_4} \{a\}$$
The resulting ranking differs from Example \ref{ex:OBE_neco}, where $\{a,c,g\}$ is more plausible to be accepted than $\{a,e,g\}$. This shows that the choice of numerical evaluation function is important for the resulting ranking. 
\end{example}

Besides gradual semantics, there are argument-ranking semantics that do not produce numerical values for each argument. An example of this is the argument-ranking semantics \emph{burden-based semantics} ($Bbs$) of Amgoud and Ben-Naim \cite{DBLP:conf/sum/AmgoudB13a}, which gives us only a preorder about the strength of each argument, i.\,e., $a \succ^{Bbs}_F b$ means that $a$ is stronger than $b$ in $F$ wrt. $Bbs$, but we cannot say how much stronger $a$ is compared to $b$. Bonzon et al. \cite{DBLP:conf/kr/BonzonDKM18} proposed a way to give each argument a numerical strength value based on its position within a preorder.  
\begin{definition}\label{def:sv}
    Let $F= (A,R)$ be an AF  and $\rho$ be an argument-ranking semantics for AF. The \emph{numerical strength value} $sv^\rho_{F}(a): A \rightarrow \mathbb{N}$ of argument $a \in A$ in $F$ with respect to $\rho$
    is the length of the longest sequence of arguments $a_1,\dots,a_i$ s.t. $a_1 \succ^\rho_F \dots \succ^\rho_F a_i \succ^\rho_F a$, i.\,e., $sv^\rho_F(a)=i$ and $sv^\rho_F(a)=0$ if there is no $b \in A$ s.t. $b \succ^\rho_F a$.
%By extending these definitions we get for a set of arguments $E= \{a_1,...,a_n\}$ a vector containing the numerical strength values: $sv^\rho_{F}(E)= (sv^\rho_{F}(a_1),...,sv^\rho_{F}(a_n))$.
    \end{definition}
 So, argument $a$ gets the value of 0 if it is the best ranked argument, $b$ receives value 1 if $b$ is in the second position of the preorder and so on.
 \begin{example}
 Consider $F_4$ from Example \ref{ex:af_example}. The \emph{burden-based argument-ranking semantics} assess the strength of an argument in relation to the strength of its attackers. Let $\succeq_{lex}$ be the \emph{lexicographical preference order}, which for (possibly infinite) real-valued vectors $V=(V_1,V_2,\ldots)$ and $V'=(V'_1,V'_2,\ldots)$ is defined as $V \succ_{lex} V'$ iff there exists an $i$ s.t. $V_i < V'_i$ and $\forall j < i$, $V_j = V'_j$ (and $V \simeq_{lex} V'$ iff for all $i$, $V_i = V_i'$).
 Then the \emph{burden number} $bur_i(a)$ for argument $a\in A$ in iteration $i$ is defined as
    \begin{align*}
&bur_i(a):= \left\{\begin{array}{ll} 1   \hspace{2pt} &\text{if $i=0$}\\
          1+\sum_{b\in a^{-}_F} \frac{1}{bur_{i-1}(b)}  \hspace{2pt} &\text{otherwise}\end{array}\right.
\end{align*}
    Let $bur(a)=(bur_0(a),bur_1(a),bur_2(a),\ldots)$ and define the \emph{burden-based argument-ranking semantics} $\succeq^{Bbs}_{F}$ via
   $a \succeq^{Bbs}_{F} b$ iff $bur(a)~ \succeq_{lex} bur(b)$ for all $a,b \in A$.

     Since $a$ is unattacked we have $bur(a)=(1,1,\dots)$. Then the corresponding argument ranking is:
     $$a \succ^{Bbs}_{F_4} g \succ^{Bbs}_{F_4} d \succ^{Bbs}_{F_4} b \succ^{Bbs}_{F_4} e \succ^{Bbs}_{F_4} c \succ^{Bbs}_{F_4} f$$
      So, the numerical strength values $sv^{Bbs}_{F_4}$ are:
    \begin{align*}
        sv^{Bbs}_{F_4}(a)&= 0 &  sv^{Bbs}_{F_4}(b)&= 3 & sv^{Bbs}_{F_4}(c)&= 5 \\
        sv^{Bbs}_{F_4}(d)&= 2 &  sv^{Bbs}_{F_4}(e)&= 4 & sv^{Bbs}_{F_4}(f)&= 6 \\
        sv^{Bbs}_{F_4}(g)&= 1
    \end{align*}
     \end{example}
 Now we can compute a numerical strength sequence for each argument ranking semantics, which can be used as a \emph{numerical evaluation function} $\varepsilon$ for $\OBE_{\varepsilon,\Box}$. When the aggregation functions $sum$, $max$, $min$, $leximin$ and $leximax$ are used, the resulting extension ranking is a generalisation of the \emph{Rank-based extensions} defined by Bonzon et al. \cite{DBLP:conf/aaai/BonzonDKM16}.
 
While gradual semantics like h-categoriser semantics already give us numerical strength values for each argument, Definition \ref{def:sv} can be applied as well. However, the resulting extension rankings coincide for $\Box \in \{max, min, \\leximax, leximin\}$. 

\begin{proposition}\label{prop:sv=rho}
     For AF $F=(A,R)$, $E,E' \subseteq A$, and $\rho$ a gradual semantics. Then $E \sqsupseteq^{\OBE_{\rho,\Box}}_{F} E'$ iff $E \sqsupseteq^{\OBE_{sv^{\rho},\Box}}_{F} E'$ for $\Box \in \{max, min,leximax, leximin\}$.
 \end{proposition}
   
For aggregation function $\Box= sum$ Proposition \ref{prop:sv=rho} does not hold. 

\begin{example}
Consider $F_4$ from Example \ref{ex:af_example} and sets $\{a,c\}$ and $\{d,e\}$. Like shown in Example \ref{ex:h-cat} the corresponding h-categoriser ranking is: $$a \succ^{Cat}_{F_4} g \succ^{Cat}_{F_4} d \succ^{Cat}_{F_4} e \succ^{Cat}_{F_4} b \succ^{Cat}_{F_4} c \succ^{Cat}_{F_4} f$$ So, using the numerical strength value $sv^{Cat}$ we get: $sv^{Cat}_{F_4}(\{a,c\})=\{6,1\}$ and $sv^{Cat}_{F_4}(\{d,e\})=\{4,3\}$, thus these two sets are equally plausible to be accepted when using $sum(v)$ as aggregation function, i.e. $$\{a,c\} \equiv^{\OBE_{sv^{Cat},sum}}_{F_4} \{d,e\}$$

The h-categoriser values are $Cat(a)= 1$, $Cat(c)= 0.46$, $Cat(d)= 0.69$, and $Cat(e)= 0.51$. So, the aggregated values are: $sum(\{1,0.4\})) = 1.46$ and $sum(\{0.69,0.51\}) = 1.2$. Thus, $\{a,c\} \sqsupset^{\OBE_{Cat,sum}}_{F_4} \{d,e\}$. 
 \end{example}

 In the remainder of this section, we will only discuss $\OBE_{\rho,\Box}$ for $\Box \in \{sum,max,min,leximax,leximin\}$. A full investigation of $\OBE_{sv^{\rho},\Box}$ and especially $\OBE_{sv^{\rho},sum}$ we leave open for future work. 

Note that $\sqsupseteq^{\OBE_{\rho,max}}$ coincides with $\sqsupseteq^{gc_{\rho}}$ if $\rho$ is total, so these two extension-ranking semantics behave the same.
 \begin{proposition}
     Let $F=(A,R)$ be an AF and $\rho$ a total gradual semantics, then $\sqsupseteq^{\OBE_{\rho,max}}_F = \sqsupseteq^{gc_{\rho}}_F$.
 \end{proposition}
 
 So, $\OBE_{\rho,max}$ satisfies composition, weak reinstatement, and syntax independence and violates $\sigma$-generalisation, decomposition, strong reinstatement, and addition robustness if $\rho$ is total. 

 Next we show that the set containing every argument is among the most plausible sets for $\Box \in \{sum,leximax\}$.

 \begin{proposition}
     Let $F= (A,R)$ be an AF, then $A \in \maxpl_{\OBE_{\rho,\Box}}(F)$ for $\Box \in \{sum,leximax\}$.
 \end{proposition}

 So, $\OBE_{\rho,\Box}(F)$ for $\Box \in \{sum,leximax\}$ violates $\sigma$-generalisation for all discussed extension semantics. 
For $\Box \in \{min,leximin\}$, we can also show that $\sigma$-generalisation is violated if $\rho$ satisfies Void Precedence. 

\begin{example}
    Consider $F_5$ from Example \ref{ex:comp} and let $\rho$ be an argument-ranking semantics satisfying Void Precedence. Then we know that $\rho(i) > \rho(j)$, so $min(\rho(\{h,j\})) = \rho(j)$ and therefore $\{h\} \sqsupset^{\OBE_{\rho, \Box}}_{F_5} \{h,j\}$ for $\Box \in \{min,leximin\}$. However, $\{h,j\}$ is an admissible set and also a complete, preferred, grounded, stable and semi-stable extension for $F_5$, thus $\sigma$-generalisation is violated for $\sigma \in \{\ad,\co,\pr,\gr,\st, \sst \}$. 
\end{example}

Next, we show that  $\OBE_{\rho,\Box}$ satisfies composition for all previously discussed aggregation functions. However, we first need to redefine the Independence principle for argument-ranking semantics for gradual semantics like proposed by Amgoud et al. \cite{DBLP:conf/ijcai/AmgoudBDV17}. 
\begin{definition}
    A gradual semantics $\rho$ satisfies \emph{Independence for gradual semantics} if and only if for any two AFs $F = (A,R)$ and $F'=(A',R')$ such that $A \cap A' = \emptyset$ it holds that for all $a \in A$, $\rho_F(a) = \rho_{F \cup F'}(a)$. 
\end{definition}
In other words, unconnected arguments should not change the numerical strength value of an argument.
\begin{proposition}
     $\OBE_{\rho,\Box}$ satisfies composition for $\Box \in \{sum, max, leximax,\\ min, leximin\}$ if $\rho$ satisfies Independence for gradual semantics. 
\end{proposition}

For the violation of decomposition we can find a general result based on the principles a gradual semantics $\rho$ satisfies. Besides Independence for gradual semantics we also need the principles \emph{Equivalence} and \emph{Maximality} \cite{DBLP:conf/ijcai/AmgoudBDV17}.
\begin{definition}
     A gradual semantics $\rho$ satisfies \emph{Equivalence} if and only if for any AF $F = (A,R)$ and two arguments $a,b \in A$ there exits a bijective function $f$ from $a^-_F$ to $b^-_F$ s.t. for all $x \in a^-_F$, $\rho_F(x)= \rho_F(f(x))$, then $\rho_F(a) = \rho_F(b)$.
\end{definition}
Informally, if the attackers of two arguments have the same strength, then the strength of these two arguments should be the same.
\begin{definition}
     A gradual semantics $\rho$ satisfies \emph{Maximality} if and only if for any AF $F = (A,R)$ 
     and argument $a \in A$ s.t. $a^-_F = \emptyset$, then $\rho_F(a)=1$.
\end{definition}
In other words, unattacked arguments should receive the highest possible strength value.

\begin{proposition}
    $\OBE_{\rho,\Box}$ violates decomposition for $\Box \in \{sum, leximax,\\ min, leximax\}$ if $\rho$ satisfies Independence for gradual semantics, Equivalence, and Maximality. 
\end{proposition}

 Similar to $\OBE_{ne_{\sigma},\Box}$, $\OBE_{\rho,\Box}$ also satisfies weak reinstatement for $\Box = sum$ and strong reinstatement for $\Box = leximax$. 
 \begin{proposition}
     $\OBE_{\rho,\Box}$ satisfies weak reinstatement for $\Box = sum$ and strong reinstatement for $\Box = leximax$. 
 \end{proposition}
  
  For $\Box \in \{min, leximin\}$ we can show that weak reinstatement is violated for    $\OBE_{\rho,\Box}$. 
  \begin{example}
      Consider $F_5$ from Example \ref{ex:comp}. Then if gradual semantics $\rho$ satisfies Maximality, then $\rho_{F_5}(h) > \rho_{F_5}(j)$ and therefore $min(\rho_{F_5}(\{h\})) > min(\rho_{F_5}(\{h,j\}))$ and thus $\{h\} \sqsupseteq^{\OBE_{\rho,\Box}}_{F_5} \{h,j\}$ for $\Box \in \{min,leximin\}$. Hence, weak reinstatement is violated.
  \end{example}

To show that $\OBE_{Cat,\Box}$ violates addition robustness we need two counterexamples. We start with $\Box \in \{sum,max,leximax\}$.
\begin{example}\label{ex:obe_cat_max_addrob}
    Consider $F_{32} = (\{a,b,c,d,e,f\}, \{(b,c),(b,e),(c,d),(e,f),\\(g,a),(g,b),(g,g)\})$ as depicted in Figure \ref{tikz:obe_cat_max_addrob}. If we use h-categoriser, then every argument receives the same strength value. Consider the sets $\{a,d,f\}$ and $\{b,c,e\}$, then $\{a,d,f\} \equiv^{\OBE_{Cat, \Box}}_{F_{32}} \{b,c,e\}$ for $\Box \in \{sum,max, leximax\}$. Hence, we can add the attack $(a,b)$ to obtain $F_{32}'=  (\{a,b,c,d,e,f\}, \{(b,c),\\(b,e),(c,d),(e,f),(g,a),(g,b),(g,g)\} \cup \{(a,b)\})$, then the strength values of arguments $c$ and $e$ increase and the values of $b$, $d$ and $f$ decreases. It follows that $\{b,c,e\} \sqsupset^{\OBE_{Cat, \Box}}_{F_{32}'} \{a,d,f\}$ for $\Box \in \{sum,max, leximax\}$ and therefore we see that addition robustness is violated by $\OBE_{Cat,\Box}$ for   $\Box \in \{sum,max, leximax\}$.
    \begin{figure}
    \centering
 
 \scalebox{1}{
\begin{tikzpicture}

\node (a) at (0,0) [circle, draw,minimum size= 0.65cm] {$a$};
\node (b) at (2,0) [circle, draw,minimum size= 0.65cm] {$b$};
\node (c) at (6,1) [circle, draw,minimum size= 0.65cm] {$d$};
\node (e) at (1,1) [circle, draw,minimum size= 0.65cm] {$g$};
\node (d) at (4,0) [circle, draw,minimum size= 0.65cm] {$e$};
\node (f) at (4,1) [circle, draw,minimum size= 0.65cm] {$c$};
\node (g) at (6,0) [circle, draw,minimum size= 0.65cm] {$f$};

\path[->] (b) edge  (d);
\path[->] (b) edge  (f);
\path[->] (f) edge  (c);
\path[->] (e) edge  (a);
\path[->] (e) edge  (b);
\path[->] (d) edge  (g);

\path[->,dashed] (a) edge  (b);
\path[->] (e) edge  [loop right] node {} ();

\end{tikzpicture}}
   \caption{AF $F_{32}$ from Example \ref{ex:obe_cat_max_addrob}, where attack $(a,b)$ is added later to obtain $F'_{32}$.}\label{tikz:obe_cat_max_addrob}
\end{figure}
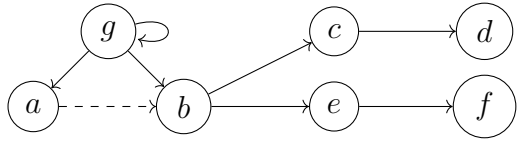
\end{example}

Next we look at $\Box \in \{min, leximin\}$.

\begin{example}\label{ex:obe_cat_min_addrob}
    Consider $F_{33} = (\{a,b,c,d,e\},\{(b,c),(b,d),(c,e),(d,e)\})$ as depicted in Figure \ref{tikz:obe_cat_min_addrob} and sets $\{a,e\}$ and $\{b,c\}$. Then $\{a,e\} \equiv^{\OBE_{Cat, \Box}}_{F_{33}} \{b,c\}$ for $\Box \in \{min, leximin\}$. Hence, we can add attack $(a,b)$ to obtain $F_{33}' =  (\{a,b,c,d,e\},\{(b,c),(b,d),(c,e),(d,e)\} \cup \{(a,b)\})$, however this yield to $\{b,c\} \sqsupset^{\OBE_{Cat, \Box}}_{F_{33}'} \{a,e\}$ and therefore addition robustness is violated by  $\OBE_{Cat,\Box}$ for $\Box \in \{min, leximin\}$.

      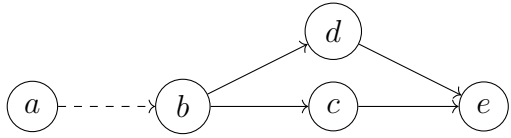
\begin{figure}
    \centering
 
 \scalebox{1}{
\begin{tikzpicture}

\node (a) at (0,0) [circle, draw,minimum size= 0.65cm] {$a$};
\node (b) at (2,0) [circle, draw,minimum size= 0.65cm] {$b$};
\node (c) at (4,0) [circle, draw,minimum size= 0.65cm] {$c$};
\node (d) at (4,1) [circle, draw,minimum size= 0.65cm] {$d$};
\node (e) at (6,0) [circle, draw,minimum size= 0.65cm] {$e$};

\path[->] (b) edge  (c);
\path[->] (b) edge  (d);
\path[->] (c) edge  (e);
\path[->] (d) edge  (e);

\path[->,dashed] (a) edge  (b);
\end{tikzpicture}}
   \caption{AF $F_{33}$ from Example \ref{ex:obe_cat_min_addrob}, where attack $(a,b)$ is added later to obtain $F'_{33}$.}\label{tikz:obe_cat_min_addrob}
\end{figure}
\end{example}

To satisfy syntax independence, we have to be sure that the underlying gradual semantics is not influenced by the names of the arguments. 
\begin{proposition}
    If gradual semantics $\rho$ satisfies Abstraction, then $\OBE_{\rho,\Box}$ satisfies syntax independence for $\Box \in \{sum, max, leximax, min, leximin\}$.
\end{proposition}
%\begin{proof}
 %     Let $F=(A,R)$, $F'=(A',R')$ be AFs, $E, E' \subseteq A$, and $\rho$ a gradual semantics. Assume isomorphism $\gamma: A \rightarrow A'$. If $\rho$ satisfies Abstraction, then $\rho(a) = \rho(\gamma(a))$ for every $a \in A$ and therefore syntax independence is satisfied.   
%\end{proof}

Since the h-categoriser semantics satisfies Abstraction, Independence for gradual semantics, Equivalence, Maximality, and Void Precedence and this semantics is total, Table \ref{tab:principle_obe_Cat} depicted the satisfied and violated principles by $\OBE_{Cat,\Box}$ for $\Box \in \{sum,max,leximax,min,leximin\}$. 

\begin{table}[]
     \centering
      \resizebox{\textwidth}{!}{
     \begin{tabular}{|l||c|c|c|c|c|}
     \hline
     Principles & $\OBE_{Cat,sum}$ & $\OBE_{Cat,max}$& $\OBE_{Cat,leximax}$& $\OBE_{Cat,min}$& $\OBE_{Cat,leximin}$ \\
     \hline
        $\sigma$-generalisation  & X & X & X & X & X \\
        composition  & \checkmark  & \checkmark & \checkmark & \checkmark & \checkmark  \\
        decomposition  & X & X& X& X& X  \\
        weak reinstatement & \checkmark & \checkmark & \checkmark & X& X  \\
        strong reinstatement  & X & X& \checkmark & X& X   \\
        addition robustness & X & X& X & X& X  \\
        syntax independence &  \checkmark &  \checkmark &  \checkmark &  \checkmark &  \checkmark  \\ \hline
     \end{tabular} }
     \caption{Principles satisfied by $\OBE_{Cat,\Box}$ for $\Box \in \{sum, max, leximax, min, leximin\}$.}
     \label{tab:principle_obe_Cat}
 \end{table}

\section{Related Work}\label{sec:related work}
Several works have addressed the problem of refining reasoning in abstract argumentation. Konieczny et al. \cite{DBLP:conf/ecsqaru/KoniecznyMV15} focus on refining the acceptance conditions of individual arguments by restricting when a set of arguments is considered acceptable. Bonzon et al. \cite{DBLP:conf/kr/BonzonDKM18} use argument-ranking semantics to refine extension semantics in order to restrict the acceptance of a set of arguments.  While both of these works present approaches to restricting sets, neither of them attempt to rank sets based on their plausibility of acceptance. Furthermore, neither approach distinguishes between two not acceptable sets. 
However, we have extended the definitions of Konieczny et al. \cite{DBLP:conf/ecsqaru/KoniecznyMV15} and Bonzon et al.  \cite{DBLP:conf/kr/BonzonDKM18} to define extension-ranking semantics in line with their ideas. 

Extensions of argumentation frameworks have been much discussed in recent years. Such as \emph{bipolar argumentation frameworks} (BAF) \cite{DBLP:conf/ecsqaru/CayrolL05a}, where there is an attack concept and a support concept, or \emph{weighted argumentation frameworks} (WAF) \cite{DBLP:journals/ai/DunneHMPW11}, where each attack and argument is given a weight. The study of methods for ranking arguments according to their (relative) degree of acceptability in these extended frameworks has received considerable attention in recent years (such as the work by Amgoud et.al \cite{DBLP:journals/ai/AmgoudDV22}). Defining extension-ranking notions for these extensions of argumentation frameworks is an interesting task, where the additional information can lead to new semantic notions for comparing two sets of arguments based on their plausibility of acceptance. 

Amgoud and Vesic \cite{DBLP:journals/ijar/AmgoudV14} proposed \emph{rich preference-based argumentation framework} (rich PAF), which are extensions of AFs where the framework takes as input, in addition to arguments and attacks, preferences about the arguments and sets of arguments. The extension-ranking semantics can be used to construct rich PAFs for any AF without the need for additional preferences over the argument sets. It is interesting to examine the resulting rich PAFs and to see which properties are satisfied by rich PAFs constructed with extension-ranking semantics. 

In Yun et al. \cite{DBLP:conf/comma/YunVCB18}, the authors develop a framework to identify the best sets among a selected set of sets of arguments, like all admissible sets. They use argument-ranking semantics to determine the best sets. Their definitions are in line with the extension-ranking semantics proposed in Section \ref{sec:additional combinations}, where argument-ranking semantics can be used to define an extension-ranking semantics. While the authors use an argument-ranking semantics and a lifting operator to identify the best sets among a selected set of sets of arguments, the converse problem, i.\,e. determining an argument ranking based on a ranking over sets of arguments, was briefly discussed by Skiba et al. \cite{DBLP:conf/ijcai/SkibaRTHK21}. These two discussions give us an insight into the relationship between argument-ranking semantics and extension-ranking semantics, so a closer look at these two reasoning approaches in abstract argumentation is interesting. 

The labelling-based semantics are a variation of extension semantics, where each argument gets a label of $\mathsf{in}$, $\mathsf{out}$, $\mathsf{undec}$. If an argument is labelled $\mathsf{in}$ it means that it is part of an $\sigma$-extension, while $\mathsf{out}$ means that the argument is attacked by an accepted argument and is therefore definitely not part of the extension. $\mathsf{undec}$ (short for undecided) means that no clear decision can be found. Based on these labels, an extension can be constructed by taking every argument labelled as $\mathsf{in}$. The papers \cite{DBLP:conf/comma/ArieliR14}, \cite{DBLP:phd/hal/Rienstra14} and \cite{DBLP:conf/comma/RienstraT18} considered ordering over labellings. The latter two papers discussed some general principles for these orders, namely \emph{conditional directionality} and \emph{SCC stratification}, which are the counterparts of the \emph{directionality} and \emph{SCC decomposability} principles used in abstract argumentation \cite{DBLP:journals/ker/BaroniCG11}. These approaches can be extended to define a new family of extension-ranking semantics. 

The work of Booth et al. \cite{DBLP:conf/comma/BoothKRT12} also explores a notion of generalising the acceptance of arguments. They introduced \emph{conditional acceptance functions}, which are weaker than labelling semantics, but allow the ordering of labels to indicate which labels are the ``most rational''. While the authors investigate generalisation of acceptance, their motivation is to better model dynamic processes. Extension-ranking semantics are not as useful for dynamic models, since in order to reason the full ranking is required. 

%---------------------------------------------------------------------------------------

\section{Conclusion}\label{sec:conclusion}
In this paper, we presented a general framework for ranking sets of arguments based on their plausibility of acceptance such that we can say that one set is ``closer'' to being acceptable than another set. For this purpose, we introduced \emph{extension-ranking semantics}, which are functions that induce a preorder over the powerset of arguments, allowing us to say that set $E$ is more plausible to be accepted than $E'$. We also introduced a number of principles for evaluating extension-ranking semantics, and also for guiding the development of new extension-ranking semantics. 

In order to define a general framework for the development of extension-ranking semantics, we considered several central aspects of argumentative reasoning and defined these aspects as simple base relations, each of which models one of these aspects. By combining these base relations, we propose a family of extension-ranking semantics, which turn out to be generalisations of Dung's classical extension semantics. 

We were also inspired by work focused on refining extension-based reasoning, and proposed approaches to ranking sets of arguments according to the quality of the arguments they contain.  We discussed different ways of estimating the quality or strength of each argument, and aggregation methods for determining the strength of the whole set.

If we compare the extension-ranking semantics discussed on the basis of the principles satisfied, we see that $LD^\sigma$, $\mathsf{r}$-$\sigma$, and $\mathsf{r}$-$\mathsf{c}$-$\sigma$ all satisfy their respective $\sigma$-generalisation principles.
In addition, $\cope(\sqsupseteq^{\cf},\sqsupseteq^{\UD})$ satisfies $\ad$-generalisation, these results show us that the proposed base relations $\sqsupseteq^\CF,\sqsupseteq^\UD, \sqsupseteq^\DN$, and $\sqsupseteq^\UA$ do actually model our intended behaviour, while the other extension-ranking semantics requires an additional prepossessing step to avoid the ignorance of internal conflicts.  
The fact that $LD^\sigma$, $\mathsf{r}$-$\sigma$, $gc_{Cat}$, $\OBE_{ne_{\sigma},\Box}$, and $\OBE_{Cat, \Box}$ satisfy composition shows us that the Copeland-based combination $\cope$ does not behave optimally, even though we can always construct a preorder using this aggregation method.
A strength of the $\mathsf{r}$-$\sigma$ extension-ranking semantics family is that it is the only one that satisfies decomposition, but the preoders are not total. 
The two reinstatement principles give us insight into which base relation or aggregation method we should use to create a good preorder. For $\cope$ we should use the base relations $\sqsupseteq^\CF,\sqsupseteq^\UD, \sqsupseteq^\DN, \sqsupseteq^\UA, \sqsupseteq^\nonatt$, and $\sqsupseteq^\strdef$ and avoid using $\sqsupseteq^{\mathcal{N}_{Cat}}$. When using the order-based extension-ranking semantics $\OBE_{\varepsilon,\Box}$ we should not use $\Box = \{min , leximin\}$ and preferably use $\Box = leximax$. 
Since only $LD^\cf$, $LD^\ad$, $LD^\gr$, $LD^\st$, $\rAd$, and $\mathsf{r\text{-}c\text{-}\ad}$ satisfy addition robustness we have to be careful when adding attacks to an AF in order to change the AF.
So, all in all $\mathsf{r}$-$\sigma$ and $\mathsf{r}$-$\mathsf{c}$-$\sigma$ behave the best with respect to our set of principles.

%Finally, we discussed the underlying problem of comparing two sets of arguments based on their plausibility of being accepted under the lens of computational complexity. For a number of extension-ranking semantics this problem is solvable in polynomial time, while for others we need an exponential amount of space.

At several points in this paper we have already discussed future work, such as discussing additional sequences of base relations or using similarity measures such as \emph{Jaccard similarity} to define base relations or extension-ranking semantics. Defining extension-ranking semantics for the extension of abstract argumentation frameworks such as \emph{Bipolar Argumentation Frameworks} or \emph{Weighted Argumentation Frameworks} is also interesting. The idea of ranking sets of arguments is interesting not only for the abstract setting, but also for the structured approach like \emph{ABA}, where the individual arguments are more expressive, but the sets of arguments respectively sets of assumptions are only accepted or rejected, with no in-between. Finally, extension-ranking semantics can be used to guide the enforcement of extensions in abstract argumentation. Highly ranked sets should be easier to enforce.

\medskip\noindent\textbf{Acknowledgements.}
The research reported here was supported by the Deutsche Forschungsgemeinschaft under grants 423456621 and 506604007.

%-----------------------------------------------------------------------------------------------------
\bibliographystyle{elsarticle-num}
%\bibliography{<your-bib-database>}
\bibliography{references}

\appendix
\section{Technical Proofs of Section \ref{sec:principles}}
%\subsection*{\ref{sec:principles}. Principles for Extension-Ranking Semantics}
\setcounter{proposition}{1}

\iffalse
\begin{lemma}%\label{lemma:wreinst+!sreinst-!INMAX}
    If extension-ranking semantics $\tau$ satisfies weak reinstatement and violates strong reinstatement, then $\tau$ violates \INMAX\ as-well. 
\end{lemma}
\begin{proof}
    Let $F=(A,R)$ be an AF and $\tau$ an extension-ranking semantics satisfying weak reinstatement and violating strong reinstatement. Assume $a \in \mathcal{F}_F(E)$, $a \notin E$ and $a \notin (E^- \cup E^+)$, then satisfying weak reinstatement and violating strong reinstatement implies $E \cup \{a\} \equiv^\tau_F E$ and this implies the violation of \INMAX.
\end{proof}

\begin{lemma}
    If extension-ranking semantics $\tau$ satisfies $\sigma$-generalisation and $\sigma$ violates \emph{I-maximality}, then $\tau$ violates \INMAX.
\end{lemma}
\begin{proof}
    Let $F= (A,R)$ be an AF, $\sigma$ an extension semantics violating \emph{I-maximality} and $\tau$ an extension-ranking semantics satisfying $\sigma$-generalisation. Assume there exists a pair of sets $E,E'$ with $E \subset E'$ s.t. $E,E' \in \sigma(F)$. Since $\tau$ satisfies $\sigma$-generalisation this means that $E,E' \in \maxpl_\tau(F)$ and therefore also $E \equiv^\tau_F E'$. Hence, \INMAX\ is violated.  
\end{proof}
\fi

\begin{proposition}%\label{thm:LD_gen}
    $\LD^\sigma$ satisfies $\sigma$-generalisation for $\sigma \in \{\cf,\ad,\co,\pr,\gr, \sst\}$.
\end{proposition}
\begin{proof}
    Let $F= (A,R)$ be an AF and $\sigma \in \{\cf,\ad,\co,\pr,\gr,\sst\}$. For every set $E \in \sigma(F)$, by definition it holds that $E \sqsupset^{\LD^\sigma}_{F} E'$ for every $E' \notin \sigma(F)$ and also $\LD^\sigma$ produces only a binary classification. This implies, that every $\sigma$-extension is among the most plausible sets and every non-$\sigma$-extension is not among the most plausible sets. Hence, $\sigma$-generalisation is satisfied.
\end{proof}

\begin{proposition}%\label{thm:no st-gen}
Let $\sigma$ be an extension semantics s.t. there exists an AF $F=(A,R)$ s.t. $\sigma(F)= \emptyset$, then there is no extension-ranking semantics satisfying $\sigma$-soundness. 
 \end{proposition}
\begin{proof}
    Let $\sigma$ be an extension semantics and $F=(A,R)$ an AF s.t. $\sigma(F)= \emptyset$. Let $\tau$ be an extension-ranking semantics, for $\tau$ to satisfy $\sigma$-generalisation it holds that $\maxpl_{\tau}(F)= \sigma(F)$ and for any preorder there always exists a maximal set, i.e. $\maxpl_{\tau}(F)\neq \emptyset$. However, $\sigma(F)= \emptyset$ therefore $\sigma$-generalisation is always violated. 
\end{proof}

\begin{proposition}
    $\LD^\sigma$ satisfies composition for $\sigma \in \{\cf,\ad,\co,\pr, \gr, \st, \sst\}$.
\end{proposition}
\begin{proof}
    Let $F= F_1 \cup F_2 = (A_1,R_1) \cup (A_2,R_2)$ with $A_1 \cap A_2 = \emptyset$ be an AF. 
    To prove the theorem for $E,E' \subseteq A_1 \cup A_2$ it has to hold that if $$\left\{
 \begin{array}{c}
 	E \cap A_1 \sqsupseteq_{F_1}^{\LD^{\sigma}} E' \cap A_1\\
	E \cap A_2 \sqsupseteq_{F_2}^{\LD^{\sigma}} E' \cap A_2
  \end{array}
  \right\}\mbox{ then }E \sqsupseteq^{\LD^\sigma}_{F} E'.$$
We look at each semantics case by case, and since the semantics are build on top of each other we will start with conflict-freeness and then admissibility. Since $\sqsupseteq^{\LD^\sigma}$ only has two levels, it is enough to show that if $E \cap A_i$ satisfies $\sigma$ in $F_1$ and $F_2$ for $i \in \{1,2\}$, then it also satisfies $\sigma$ in $F$ and if $E \cap A_i$ violates $\sigma$ in $F_1$ or $F_2$ for $i \in \{1,2\}$ then $E$ also violates $\sigma $ in F. Because if $E$ satisfies $\sigma$ in $F$, then there can not be any set $S$ ranked strictly better than $E$. If $E \cap A_i$ violates $\sigma$ in either $F_1$ or $F_2$ and $E \cap A_1 \sqsupseteq_{F_1}^{\LD^{\sigma}} E' \cap A_1$ resp. $E \cap A_2 \sqsupseteq_{F_2}^{\LD^{\sigma}} E' \cap A_2$, then $E' \cap A_i$ violates $\sigma$ in $F_1$ resp. $F_2$ aswell for $i \in \{1,2\}$. So if both $E$ and $E'$ are violating $\sigma$ in $F$, then neither of these sets can be ranked strictly better.   

\begin{description}
  \item[``$\sigma = \cf$'':] If $E\cap A_1$ resp. $E \cap A_2$ is conflict-free in $F_1$ and in $F_2$, then $E$ is also conflict-free in $F$, since the union of $F_1$ and $F_2$ does not create new conflicts. 
  
  If $E \cap A_1$ resp. $E \cap A_2$ is not conflict-free in $F_1$ or $F_2$ then the responsible conflict $(a,b) \in R_{\{1,2\}}$ with $a,b \in E$ is also part of $F$. Hence, $E$ is also not conflict-free in $F$. 

  \item[``$\sigma = \ad$'':] If $E \cap A_1$ resp. $E \cap A_2$ is admissible in $F_1$ and in $F_2$, then every argument in $E$ has to be defended in $F$ as well. Otherwise there has to be an argument attacking an argument in $E$ in $F$ which is not counterattacked, however this attacker has to exists in either $F_1$ or $F_2$ as well, resulting in the fact that $E \cap A_i$ can no longer be admissible in the respective $F_i$ for $i \in \{1,2\}$. So, since $E$ also has to be conflict-free we know that $E$ is admissible in $F$. 

  Wlog. assume $E \cap A_1$ is not admissible in $F_1$, then there exists an argument $a \in E$ for which there is no defender in $E$ and since $F_1$ and $F_2$ are disjoint argument $a$ is still not defended in $F$. So, $E$ is not admissible in $F$. 

  \item[``$\sigma = \co$'':] Assume $E \cap A_1$ resp. $E \cap A_2$ is complete in $F_1$ and $F_2$. If $E$ is not complete in $F$ there has to be an argument, which is defended by $E$ and not contained. However, this argument has to exists in either $F_1$ or $F_2$ as well, making $E \cap A_i$ no longer complete in the respective $F_i$ for $i \in \{1,2\}$. 
  
  Wlog. assume $E \cap A_1$ is not complete in $F_1$, then there exists an argument $a$ which is defended by $E$ and not contained in $E$. However, since $F_1$ and $F_2$ are disjoint this argument is still defended in $F$ by $E$, so $E$ can not be complete in $F$ as well.

  \item[``$\sigma = \pr$'':] Assume $E \cap A_1$ resp. $E \cap A_2$ is preferred in $F_1$ and $F_2$, then we know that $E$ is admissible in $F$. If $E$ is not preferred in $F$, then there has to be a preferred set $S \subseteq A_1 \cup A_2$ s.t. $E \subset S$. Meaning there is an argument $a \in S$ and $a \notin E$, which is defended by $S$ and $E$. However, $a$ has to exits in either $F_1$ or $F_2$ and therefore $E \cap A_i$ can not be maximal in that $F_i$ for $i \in \{1,2\}$.

  Wlog. assume $E \cap A_1$ is not preferred in $F_1$, so there is a preferred set $S \subseteq A_1$ s.t. $E \subset S$. Meaning there is an argument $a \in S$ and $a \notin E$, which is defended by $S$ and $E$. Since $F_1$ and $F_2$ are disjoint $a$ has to be defended by $E$ in $F$ aswell, therefore $E$ can not be maximal in $F$. 

  \item[``$\sigma= \gr$'':] Assume $E \cap A_1$ resp. $E \cap A_2$ is grounded in $F_1$ and $F_2$, then $E$ is the least fixed point of the characteristic function in both these AFs. Since $F_1$ and $F_2$ are disjoint every argument defended in $F_1$ and $F_2$ is also defended in $F$. So, $E$ is a fixed point of the characteristic function in $F$. Since, the set of unattacked arguments in $F$ is the same as the union of unattacked arguments in $F_1$ and $F_2$ we also know that $E$ is minimal, proving that $E$ is grounded in $F$.

  Wlog. assume $E \cap A_1$ is not grounded in $F_1$, then there is a set $S \subset E$, i.e. there is an argument $a \in E$ and $a \notin S$, but $S \cup \{a\}$ defends $a$. This argument $a$ is also part of $F$, since $F= F_1 \cup F_2$ holds as well as $F_1$ and $F_2$ are disjoint. So, $E$ is not minimal in $F$ as well and therefore not grounded.

  \item[``$\sigma = \st$'':] Assume $E \cap A_1$ resp. $E \cap A_2$ is stable in $F_1$ and $F_2$, then if $E$ is not stable in $F$, then there has to be an argument $a \in A_1 \cup A_2$ which is not attacked by $E$ and $a \notin E$. However, this $a$ has to exits in $F_1$ or $F_2$ making $E \cap A_i$ not stable in that $F_i$ for $i \in \{1,2\}$.

  Wlog. assume $E \cap A_1$ is not stable in $F_1$, then there has to be an argument $a \in A_1$ such that $a \notin E \cup E^+$. However, $a$ has to exists in $F$ aswell since $F= F_1 \cup F_2$ and $F_1$ and $F_2$ are disjoint. 

  \item[``$\sigma= \sst$'': ] Assume $E \cap A_1$ resp. $E \cap A_2$ is semi-stable in $F_1$ and $F_2$, then we know that $E$ is complete in $F$. If $E$ is not semi-stable in $F$, then there has to be a semi-stable set $S \subseteq A_1 \cup A_2$ s.t. $E \subset S$. Meaning there is an argument $a \in S \cup S^+_F$ and $a \notin E$, which is attacked by $S$ or in $S$. However, $a$ has to exits in either $F_1$ or $F_2$ and therefore $E \cap A_i$ can not be semi-stable in that $F_i$ for $i \in \{1,2\}$.

   Wlog. assume $E \cap A_1$ is not semi-stable in $F_1$, so there is a semi-stable set $S \subseteq A_1$ s.t. $E \subset S$. Meaning there is an argument $a \in S \cup S^+_{F_1}$ and $a \notin E \cup E^+_{F_1}$. Since $F_1$ and $F_2$ are disjoint $a$ can not be part of $E \cup E^+_{F}$ in $F$ aswell, therefore $E \cup E^+_F$ can not be maximal in $F$. 
  \qedhere
\end{description}

  \end{proof}

\begin{proposition}
    $\LD^{\sigma}$ satisfies weak reinstatement for $\sigma \in \{\cf,\ad,\co,\pr,\gr,\st,\sst \}$.
\end{proposition}
\begin{proof}
 Let $F = (A,R)$ be an AF and $E\subseteq A$ any set. Assume $a\in \mathcal{F}_{F}(E)$, $a \notin E$ and $a \notin (E^- \cup E^+)$. We will show that $E \not\sqsupset_{F}^{\LD^{\sigma}} E \cup \{a\}$. Since $\sqsupseteq^{\LD^{\sigma}}$ only has two layers, we only need to check for $E\in \sigma(F)$ and $E \notin \sigma(F)$. 
 \begin{enumerate}
     \item If $E \notin \sigma(F)$, then since $\sqsupseteq^{\LD^{\sigma}}$ has only two levels and $E \notin \maxpl_{\LD^{\sigma}}(F)$ we know that there is no set $S \subseteq A$ s.t. $E \sqsupset^{\LD^{\sigma}}_{F} S$.
     \item If $E \in \sigma(F)$, then $E \in \maxpl_{\LD^{\sigma}}(F)$. For $\sigma' \in \{\cf,\ad,\co,\pr,\st, \sst\}$ we know that $a$ does not add any conflict into the set $E$, hence we can freely add $a$ and we will not break the acceptance of $E$. 
     If $E \in \gr(F)$, then we know that $E = \mathcal{F}_{F}(E)$ and therefore there can not exists an argument $a \in A$ s.t. $a \in \mathcal{F}_{F}(E)$ and $a \notin E$. Thus, every $\LD^{\sigma}$ satisfies weak reinstatement.\qedhere
 \end{enumerate}
\end{proof}

\begin{proposition}
    $\LD^{\sigma}$ satisfies addition robustness for $\sigma \in \{\cf,\ad, \gr,\st\}$.
\end{proposition}
\begin{proof}
    Let $F= (A,R)$ be an AF and $E, E' \subseteq A$ with $a \in E$, and $b \in E' \setminus E$. We extend $F$ to $F'$ s.t. $F'= (A,R \cup \{(a,b)\})$. To show that $\LD^\sigma$ satisfies additional robustness it is enough to show that if $E \in \sigma(F)$ then $E \in \sigma(F')$ and if $E' \notin \sigma(F)$ then $E' \notin \sigma(F')$. Since $\sqsupseteq^{\LD^{\sigma}}_F$ has only two levels this implies if $E$ satisfies $\sigma$ in $F'$, then $E$ is also ranked among the best sets wrt. $\sqsupseteq^{\LD^{\sigma}}_{F'}$ and $E'$ can not be ranked strictly better then $E$.
    Additionally, if $E' \notin \sigma(F')$, then $E$ can not be ranked strictly worse than $E'$ wrt. $\sqsupseteq^{\LD^{\sigma}}_{F'}$.

\begin{description}
    \item[``$\sigma= \cf$'':] Let $E \in \cf(F)$, then it is clear that $E \in \cf(F')$ since no conflict was added wrt. E. So, for any set $E'' \subseteq A$, we have $E \sqsupseteq^{\LD^\cf}_{F'} E''$.
    
    Let $E' \notin \cf(F)$, then we do not remove any conflict in $F'$ wrt. $E'$. So, for any set $E'' \subseteq A$ we have $E'' \sqsupseteq^{\LD^\cf}_{F'} E'$. Thus, if $E$ is conflict-free this set stays conflict-free in $F'$ and if a set is not conflict-free, then this set stays not conflict-free. Hence, the relationship between $E$ and $E'$ stays the same wrt. $\sqsupseteq^{\LD^\cf}_F$. 

    \item[``$\sigma= \ad$'':] Let $E \in \ad(F)$. We know that $E$ is conflict-free in $F'$, hence we only need to show that every argument of $E$ is defended in $F'$. Assume $a' \in E$ is not defended in $F'$ by $E$. So, there is one attacker of $a'$ that is not defeated in $F'$ and we do not remove any attack, therefore the added attack $(a,b)$ has to be the reason. This implies that $a' = b$, which is impossible since $b \notin E$. Hence, $E$ has to be admissible in $F'$. 
    
    Next, let $E' \notin \ad(F)$ and assume $E' \in \ad(F')$, so there is one argument $b' \in E'$ which is not defended in $F$ by $E'$ but defended in $F'$. Let $c$ be the attacker of $b'$, then $E'$ does not contain any attacker of $c$ in $F$, hence the addition of $(a,b)$ will create a new attacker for $c$ implying $c = b$ and $a \in E'$, which entails that $E'$ is not conflict-free in $F'$ and therefore also not admissible. 
    So, if $E$ is admissible in $F$ it stays admissible and if $E'$ is not admissible, $E'$ stays not admissible implying that the relationship between $E$ and $E'$ stays the same in $F'$ i.e. $E \sqsupseteq^{\LD^\ad}_{F'} E'$.

 \item[``$\sigma= \gr$'':] 
    Assume $E' \in \gr(F')$ and $E' \notin \gr(F)$, then since $E'$ is not empty there has to be an unattacked argument $c\in E'$ and since $a \in b^-_{F'}$ we know $c \neq b$. Additionally, for every $c \neq a$ we know that $c^+_{F} = c^+_{F'}$ hence every argument defended by $c$ in $F'$ is also defended by $c$ in $F$. Since $E' \in \gr(F')$, we know that every argument in $E'$ is part of the least fixed point of $\mathcal{F}_{F'}(\{c\})$ and everything defended by $\{c\}$ in $F'$ is also defended by $\{c\}$ in $F$, therefore every argument, which is part of the least fixed point of $\mathcal{F}_{F'}(\{c\})$ also has to be part of the least fixed of $\mathcal{F}_{F}(\{c\})$, which implies that $E' \in \gr(F)$ contradicting the assumption.  

    Let $E \in \gr(F)$ then since the grounded extension is unique it holds that $E' \notin \gr(F)$ and therefore $E' \notin \gr(F')$. Hence it is impossible that $E' \sqsupset^{\LD^{\gr}}_{F'} E$ and showing that $\LD^{\gr}$ satisfies addition robustness. 
    
    \item[``$\sigma= \st$'':] Let $E\in \st(F)$, we know that $E \in \ad(F)$, hence also $E \in \ad(F')$. Additionally, we do not remove any attack switching over to $F'$, so $E^+_{F} \subseteq E^+_{F'}$ and since $E \cup E^+_{F} = A$ it has to hold that $E \cup E^+_{F'} = A$ and this implies $E \in \st(F')$. 
    
    Next, we show that if $E' \notin \st(F)$, then $E' \notin \st(F')$. For contradiction assume $E' \in \st(F')$. If $E' \notin \cf(F)$ or $E' \notin \ad(F)$ we are already done, hence let $E' \in \ad(F)$ and therefore also $E' \in \ad(F')$. So, there exist an argument $b' \in E'^+_{F'}$ s.t. $b' \notin E'^+_{F}$ and additionally it holds that $ E'^+_{F} \subset  E'^+_{F'}$. Since, the only attack we add is $(a,b)$ this implies $b = b'$ and $a \in E'$. However, this also implies that $E'$ is no longer conflict-free in $F'$ since $b \in E'$. Therefore $b'$ can not exist and therefore $E' \notin \st(F')$.
       \qedhere
   \end{description}

    \end{proof}

\section{Technical Proofs of Section \ref{sec:Base functions}}
  %  \subsection*{. Base Relations}

\subsection{Technical Proofs of Section \ref{subsec:Dungean}} 
    \begin{proposition}%\label{prop:f_star fixpoint}
    Let $F=(A,R)$ be an AF and $E \subseteq A$, then there is $k \in \mathbb{N}$ with $$E \subseteq \mathcal{F}^*_{1,F}(E) \subseteq \mathcal{F}^*_{2,F}(E) \subseteq \dots \subseteq \mathcal{F}^*_{k-1,F}(E) \subseteq \mathcal{F}^*_{k,F}(E) =  \mathcal{F}^*_{k+1,F}(E) = \dots$$ 
\end{proposition}
\begin{proof}
     Let $F=(A,R)$ be an AF and $E \subseteq A$. We show that $\mathcal{F}^*_{i,F}(E) \subseteq \mathcal{F}^*_{i+1,F}(E)$ via induction over $i \geq 0$.
     \begin{description}
         \item[base case ``$i= 1$''] $\mathcal{F}^*_{1,F}(E) = E \subseteq \mathcal{F}^*_{2,F}(E) = E \cup (\mathcal{F}_F(E) \setminus E^-_F)$

         Assume until $i$ it holds that $\mathcal{F}_{i-1,F}^\ast(E) \subseteq \mathcal{F}_{i,F}^\ast(E)$
         \item[induction step $i \rightarrow i+1$] $\mathcal{F}^*_{i-1,F}(E) \subseteq \mathcal{F}^*_{i,F}(E) $ holds because $\mathcal{F}^*_{i,F}(E) = \mathcal{F}^*_{i-1,F}(E) \cup (\mathcal{F}_F(\mathcal{F}^*_{i-1,F}(E)) \setminus E^-)$. Thus we only add arguments with $(\mathcal{F}_F(\mathcal{F}^*_{i-1,F}(E)) \setminus E^-)$ and do not remove anything from $\mathcal{F}^*_{i-1,F}(E)$. Then $\mathcal{F}^*_{i-1,F}(E) \cup (\mathcal{F}_F(\mathcal{F}^*_{i-1,F}(E)) \setminus E^-) \subseteq (\mathcal{F}^*_{i-1,F}(E) \cup (\mathcal{F}_F(\mathcal{F}^*_{i-1,F}(E)) \setminus E^-)) \cup (\mathcal{F}_F(\mathcal{F}^*_{i-1,F}(E) \cup (\mathcal{F}_F(\mathcal{F}^*_{i-1,F}(E)) \setminus E^-)) \setminus E^-)$ which is equivalent to $\mathcal{F}^*_{i,F} \subseteq \mathcal{F}^*_{i+1,F}$. Since $\mathcal{F}^*_{i,F} \subseteq \mathcal{F}^*_{i+1,F}$ holds for every $i$ there has to be a smallest $k$ s.t. $\mathcal{F}^*_{k,F} = \mathcal{F}^*_{k+1,F}$, because there is only a finite amount of arguments. Therefore $\mathcal{F}^*_{k,F} = \mathcal{F}^*_{k',F}$ for all $k' > k$. \qedhere
     \end{description}
\end{proof}

\begin{proposition}%\label{thm:f=fstart}
    Let $F=(A,R)$ be an AF, $E \subseteq A$ s.t. $E \in \ad(F)$, then $\mathcal{F}_F(E)= \mathcal{F}^*_F(E)$.
\end{proposition}
\begin{proof}
     Let $F=(A,R)$ be an AF and $E \subseteq A$ s.t. $E \in \ad(F)$. Applying $\mathcal{F}^*$ to $E$ we get: $\mathcal{F}^*_{1,F}(E)=E$ and $\mathcal{F}^*_{2,F}(E)=E \cup (\mathcal{F}_F(E) \setminus E^-)$.  $\mathcal{F}_F(E)$ is defined such that this function will not return any attacker of $E$ if $E$ is admissible, so we do not need to remove any attacker, so $\mathcal{F}_F(E) \setminus E^- = \mathcal{F}_F(E) $.
     The \emph{Fundamental Lemma} \cite{DBLP:journals/ai/Dung95} then implies $E \subseteq \mathcal{F}_F(E)$ if $E$ is admissible, so $E \cup \mathcal{F}_F(E)= \mathcal{F}_F(E)$. Therefore, $\mathcal{F}^*_{2,F}(E)=\mathcal{F}_F(E)$. We can use the same reasoning for any $i > 2$ and this entails that in every step $\mathcal{F}^*_F(E)$ return the same arguments as $\mathcal{F}_F(E)$.  
\end{proof}

\begin{proposition}
    Let $F=(A,R)$ be an AF and $E = \gr(F)$, then $E$ is the least fixed point of $\mathcal{F}^*_{F}$.
\end{proposition}
\begin{proof}
    Let $F=(A,R)$ be an AF and $E = \gr(F)$. To prove that $E$ is the least fixed point of $\mathcal{F}^*_{F}$ we can modify the algorithm to calculate the grounded extension (for details see \cite{baroni2018handbook}). This algorithm starts with the empty set, i.e. $\mathcal{F}_{F}(\emptyset)$ and add in every step all defended arguments into that set until a fixed point is reached. Since $\emptyset \in \ad(F)$ we have $\mathcal{F}_{F}(\emptyset) = \mathcal{F}^*_{F}(\emptyset)$ like already proven in Proposition \ref{thm:f=fstart} and if $\mathcal{F}_{F}(\emptyset)$ reaches a fixed point $\mathcal{F}^*_{F}(\emptyset)$ reaches a fixed point as well. Since the grounded extension $E$ is also an admissible extension we have $\mathcal{F}_{F}(E) = \mathcal{F}^*_{F}(E)$.    
    This fixed point of $\mathcal{F}_{F}^*(\emptyset)$ coincided with the grounded extension $E$ and therefore $E$ is the least fixed point of $\mathcal{F}^*_{F}$.
\end{proof}

\begin{proposition}
     Let $F=(A,R)$ be an AF and $E \in \co(F)$, then $\mathcal{F}^*_F(E)= E$.
\end{proposition}
\begin{proof}
    Let $F=(A,R)$ be an AF and $E \in \co(F)$, then because of the definition of complete semantics we know $E \in \ad(F)$ and therefore $\mathcal{F}_F(E)= \mathcal{F}^*_F(E)$. Because $E$ already contains every argument it defends no more arguments will be added when applying the characteristic function, i.e. $\mathcal{F}_F(E) \subseteq E$. Proposition \ref{thm:f=fstart} gives us the other direction and therefore $\mathcal{F}^*_F(E)= E$.
\end{proof}

\subsection{Technical Proofs of Section \ref{subsec:cardinality}}
%\subsubsection*{ Cardinality-based base relations}

\begin{proposition}%\label{prop:A_min_nonatt,strdef}
    For any AF $F=(A,R)$ we have $A \sqsupseteq^{\tau}_{F} E$ for any $E \subset A$ with $\tau \in \{\nonatt,\strdef\}$.
\end{proposition}
\begin{proof}
    Let $F=(A,R)$ be any AF and $E \subset A$ a set of arguments. 
    \begin{description}
        \item[``$\sqsupseteq^{\nonatt}$'':] For every argument $a \in E$, which not attacked by $A$ it has to hold, that $a^-_{F}= \emptyset$, so $a$ has to be unattacked. However, since $E \subset A$ we know that $a \in A$ and therefore the number of unattacked arguments inside $E$ is smaller or equal to the number of unattacked arguments in $A$. Therefore $A \sqsupseteq^{\nonatt}_{F} E$ for every $E \subset A$.
        \item[``$\sqsupseteq^{\strdef}$'':] For set $E$ to contain any strongly defended arguments from $A$, $E$ needs to contain an unattacked argument. Like already discussed in the case above, $A$ contains every unattacked argument as well. Additionally if an argument is strongly defended by $E$ from $A$, we can use the same reasoning to show the strong defence by $A$ from $E$. Therefore the number of strongly defended arguments by $E$ is lower or equal to the number of strongly defended arguments in $A$. Hence,  $A \sqsupseteq^{\strdef}_{F} E$ for every $E \subset A$. \qedhere
    \end{description}
\end{proof}

\section{Technical Proofs of Section \ref{sec:combination}}

\subsection{Technical Proofs of Section \ref{subsec:Lexi}}
\subsubsection{Technical Proofs of Section \ref{subsubsec:r-ad}}%. Admissible Extension-ranking Semantics}
\begin{proposition}
$\rAd$ satisfies $\ad$-generalisation.
\end{proposition}
\begin{proof} Let $F= (A,R)$ be an AF and $E \subseteq A$.
\begin{description}
   \item[``$\ad$-soundness'':] If $E\in \maxpl_{\rAd}(F)$ then $E$ is conflict-free and $\UD_F(E) = \emptyset$ and hence, $E$ defends all its elements. Therefore $E$ is admissible.

    \item[``$\ad$-completeness'':] Suppose $E$ is admissible. Then $\CF_F(E) = \emptyset$ and $\UD_F(E)= \emptyset$. Hence, there is no $E'$ s.t. $E' \sqsupset^{\CF}_{F} E$ or $E' \sqsupset^{\UD}_{F} E$. Hence, $E \in \maxpl_{\rAd}(F)$.\qedhere
\end{description}
\end{proof}

\begin{lemma}%\label{lem:composition_decomposition_1}
For $\tau \in \{\CF, \UD, \DN, \UA\}$, if $F_1=(A_1,R_1)$, $F_2=(A_2,R_2)$ and $F = F_1 \cup F_2= (A_1, R_1) \cup (A_2,R_2)$ with $A_1 \cap A_2 = \emptyset$ then $\tau_{F_1}(E\cap A_1) \cup \tau_{F_2}(E \cap A_2)= \tau_F(E)$ for every $E \subseteq A_1 \cup A_2$.
\end{lemma}
\begin{proof}
    We prove it for $\tau = \UD$ (others are similar). 
    Suppose $F_1=(A_1,R_1)$, $F_2=(A_2,R_2)$ and $F = F_1 \cup F_2= (A_1, R_1) \cup (A_2,R_2)$ with $A_1 \cap A_2 = \emptyset$. We prove that $\UD_{F_1}(E \cap A_1) \cup \UD_{F_2}(E \cap A_2) = \UD_F(E)$.
    \begin{description}
        \item[``$\subseteq$'':] Suppose wlog that argument $a \in \UD_{F_1}(E\cap A_1)$. Then $a \in E \cap A_1$ and $a \notin \mathcal{F}_{F_1}(E \cap A_1)$. Therefore, $a \in E$ and since $F_1$ and $F_2$ are disjoint, $a \notin \mathcal{F}_{F}(E)$. It follows that $a \in \UD_F(E)$.
        \item[``$\supseteq$'':] Suppose $a \in \UD_F(E)$. Furthermore, suppose wlog that $a \in F_1$. Then $a \in E$ and $a \notin \mathcal{F}_{F}(E)$. Hence, there is an $b \in A_1 \cup A_2$ s.t. $b$ attacks $a$ and $b$ is not attacked by $E$. It then follows that $b \in A_1$. Hence, $a \notin \mathcal{F}_{F_1}(E\cap A_1)$. It follows that $a \in \UD_{F_1}(E \cap A_1)$. \qedhere
    \end{description}
\end{proof}

\begin{lemma}%\label{lem:composition_decomposition_2}
    The base relation $\sqsupseteq^\tau_{F}$ satisfies composition and decomposition for $\tau \in \{\CF, \UD, \DN, \UA\}$.
\end{lemma}
\begin{proof}
    Let $\tau \in \{\CF, \UD, \DN, \UA\}$, $F_1=(A_1,R_1)$, $F_2=(A_2,R_2)$ and $F = F_1 \cup F_2= (A_1, R_1) \cup (A_2,R_2)$ with $A_1 \cap A_2 = \emptyset$. 
    \begin{description}
        \item[``Composition'':] Suppose $E \cap A_1 \sqsupseteq^{\tau}_{F_1} E'\cap A_1$ and $E \cap A_2 \sqsupseteq^\tau_{F_1} E'\cap A_2$. Then $\tau_{F_1}(E\cap A_1) \subseteq \tau_{F_1}(E'\cap A_1)$ and $\tau_{F_2}(E \cap A_2) \subseteq \tau_{F_2}(E' \cap A_2)$. Lemma \ref{lem:composition_decomposition_1} the implies that $\tau_F(E) \subseteq \tau_F(E')$ and hence $E \sqsupseteq^{\tau}_{F} E'$.
        \item[``Decomposition'':] Suppose $E \sqsupseteq^{\tau}_{F} E'$. Then $\tau_F(E) \subseteq \tau_F(E')$. Lemma \ref{lem:composition_decomposition_1} then implies that $\tau_{F_1}(E \cap A_1) \subseteq \tau_{F_1}(E' \cap A_1)$ and $\tau_{F_2}(E \cap A_2) \subseteq \tau_{F_2}(E' \cap A_2)$. Hence, $E \cap A_1 \sqsupseteq^{\tau}_{F_1} E' \cap A_1$ and $E \cap A_2 \sqsupseteq^{\tau}_{F_2} E' \cap A_2$.  \qedhere
    \end{description}
\end{proof}

\begin{proposition}
    $\rAd$ satisfies composition and decomposition.
\end{proposition}
\begin{proof}
    Follows from Lemma \ref{lem:composition_decomposition_2} together with Definition \ref{defn:adm-ranking-semantics}.
\end{proof}

\begin{proposition}%\label{prop:rAD_reinstatement}
   $\rAd$ satisfies weak reinstatement.
\end{proposition}
\begin{proof} Let $F= (A,R)$ be an AF and $E \subseteq A$.
    Suppose $a \in \mathcal{F}_{F}(E), a \notin E$ and $a \notin (E^- \cup E^+)$. Then $\CF_F(E \cup \{a\})= \CF_F(E)$. What remains is to prove that $\UD_F(E \cup \{a\}) \subseteq \UD_F(E)$. Suppose $x \in \UD_F(E \cup \{a\})$. Then $x \in E \cup \{a\}$ and $x \notin \mathcal{F}_{F}(E \cup \{a\})$. Because $a \in \mathcal{F}_{F}(E \cup \{a\})$ it follows that $x \in E$. Furthermore, since $x \notin \mathcal{F}_{F}(E \cup \{a\})$ we also have $x \notin \mathcal{F}_{F}(E)$. This implies that $x \in \UD_F(E)$. We thus have that $E \cup \{a\} \equiv^{\CF}_{F} E$ and $E \cup \{a\} \sqsupseteq^{\UD}_{F} E$ and hence, $E \cup \{a\} \sqsupseteq^{\rAd}_{F} E$.
\end{proof}

\begin{proposition}
    $\rAd$ satisfies addition robustness.
\end{proposition}
\begin{proof}
Let $F= (A,R)$ be an AF and $E, E' \subseteq A$ with $E \sqsupseteq^{\rAd}_{F} E'$. Let $F' = (A, R \cup \{(a,b)\}$ for $a \in E$, and $b \in E' \setminus E$.
\begin{description}
\item[``$\sqsupseteq^\CF_F$'':] The addition of $(a,b)$ will not add any new conflicts into $E$ so $\CF_F(E)=\CF_{F'}(E)$, additionally no conflict from $E'$ is deleted, so $\CF_F(E') \subseteq \CF_{F'}(E')$, so since $E \sqsupseteq^{\rAd}_{F} E'$, we know that $\CF_F(E) \subseteq \CF_F(E')$ therefore also $\CF_{F'}(E) \subseteq \CF_{F'}(E')$. This shows that $E \sqsupseteq^{\CF}_{F'} E'$.

\item[``$\sqsupseteq^\UD_F$'':] Assume $E \sqsupseteq^{\rAd}_{F} E'$ and $\CF_F(E) = \CF_{F}(E')$. Like show above we know that $\CF_{F'}(E) \subseteq \CF_{F'}(E')$, assume $\CF_{F'}(E) = \CF_{F'}(E')$. It remains to show that $\UD_{F'}(E) \subseteq \UD_{F'}(E')$. Since $E \sqsupseteq^{\rAd}_{F} E'$ and $\CF_F(E) = \CF_F(E')$, we know that $\UD_F(E) \subseteq \UD_F(E')$. The attack $(a,b)$ will not disable any defence from $E$, since $b \notin E$, hence $b$ can not be used to defend anything. Therefore $\UD_{F'}(E) \subseteq \UD_F(E)$, so we do not lose any defence w.r.t. $E$.

Next, we show that $E'$ can not defend more argument than before. Assume $b' \in \mathcal{F}_{F'}(E')$ and $b' \notin \mathcal{F}_{F}(E')$, this means there is one attack $(c,b') \in R$ s.t. $E'$ does not attack $c$. Since, only $(a,b)$ is added this attack is responsible for the defends of $b'$ in $F'$, therefore $a \in E'$ has to hold. However, this is a contradiction to $\CF_{F'}(E)=\CF_{F'}(E')$, so $b'$ can not be defended, hence $\UD_F(E') \subseteq \UD_{F'}(E')$ and since $\UD_F(E) \subseteq \UD_{F'}(E')$ we can follow that $\UD_{F'}(E) \subseteq \UD_{F'}(E')$ and therefore $E \sqsupseteq^{\rAd}_{F'} E'$. \qedhere   
\end{description}
\end{proof}

\subsubsection{Technical Proofs of Section \ref{subsubsec:r-co}}
%\subsubsection*{. Complete Extension-ranking Semantics}

\begin{proposition}
$\rCo$ satisfies $\co$-generalisation.
\end{proposition}
\begin{proof} Let $F= (A,R)$ be an AF and $E \subseteq A$.
\begin{description}
 \item[``$\co$-soundness'':] Suppose $E \in \maxpl_{\rCo}(F)$. We first prove that $E$ is admissible. Suppose $E$ is not admissible, then there is an admissible $E'$ s.t. $E' \sqsupset^{\rAd}_{F} E$. But this contradicts $E \in \maxpl_{\rCo}(F)$. Hence, $E$ is admissible. Next we prove that $E$ contains all defended arguments. Suppose the contrary. Then there is an $a \notin E$ and $a$ is defended by $E$. We then have $a \in \DN_F(E)$. Then for some complete extension $E'$ we have $\DN_F(E') = \emptyset \subset \DN_F(E)$, but this implies $E' \equiv^{\rAd}_{F} E$ and $E' \sqsupset^{\rCo}_{F} E$, contradicting $E \in \maxpl_{\rCo}(F)$. Hence, $E$ contains all arguments it defends. It follows that $E$ is complete.

\item[``$\co$-completeness'':] If $E$ is a complete extension of $F$. Suppose, towards contradiction, that $E \notin \maxpl_{\rCo}(F)$. Then there is an $E'$ s.t. $E' \sqsupset^{\rCo}_{F} E$. Since $E$ is admissible it then follows that $E'$ is admissible. This implies $E' \sqsupset^{\DN}_{F} E$ and hence $\DN_F(E') \subset \DN_F(E)$. Consequently there is an argument $a \in \DN_F(E)$, this implies that there is an argument defended by $E$ but not element of $E$, this contradicts the assumption that $E$ is complete. Hence, $E \in \maxpl_{\rCo}(F)$. \qedhere
\end{description}
\end{proof}

\begin{proposition}
    $\rCo$ satisfies composition and decomposition.
\end{proposition}
\begin{proof}
    Follows from Lemma \ref{lem:composition_decomposition_2} together with Definition \ref{defn:co-ranking-semantics}.
\end{proof}

\begin{proposition}%\label{prop:rCO_reinstatement}
   $\rCo$ satisfies strong reinstatement. 
\end{proposition}
\begin{proof}
    Let $a \in \mathcal{F}_{F}(E)$, $a \notin E$ and $a \notin (E^- \cup E^+)$. Proposition \ref{prop:rAD_reinstatement} implies that $\{a\} \cup E \sqsupseteq^{\rAd}_{F} E$. If $\{a\} \cup E \sqsupset^{\rAd}_{F} E$ we are done, so in the remainder we assume $\{a\} \cup E \equiv^{\rAd}_{F} E$ and prove that $\DN_F(\{a\} \cup E) \subset \DN_F(E)$.
    Since $a \in \mathcal{F}_F(E)$ and $a \notin \{E^- \cup E^+\}$ we have $a \in \mathcal{F}^*_F(E)$ and $a \in \mathcal{F}^*_F(E \cup \{a\})$. We know that $\mathcal{F}^*_F$ reaches a fixed point implying $\mathcal{F}^*_{F}(\{a\} \cup E) = \mathcal{F}^*_{F}(E)$. So, since $a \notin E$ we have $\mathcal{F}^*_{F}(\{a\} \cup  E) \setminus (E \cup \{a\}) \subset \mathcal{F}^*_{F}(E) \setminus E$ and thus $\DN_F(\{a\} \cup E) \subset \DN_F(E)$. This implies $E \cup \{a\} \sqsupset^{\rCo}_{F} E$. So, the complete extension-ranking semantics satisfies strong reinstatement. 
\end{proof}

\subsubsection{Technical Proofs of Section \ref{subsubsec:r-gr}}
%\subsubsection*{\ref{subsubsec:r-gr}. Grounded Extension-ranking Semantics}

\begin{proposition}
$\rGr$ satisfies $\gr$-generalisation.
\end{proposition}
\begin{proof} Let $F= (A,R)$ be an AF and $E \subseteq A$. 
\begin{description}
    \item[``$\gr$-soundness'':] Suppose $E \in \maxpl_{\rGr}(F)$. Then $E$ is complete. Suppose $E$ is not grounded. Then for the grounded extension $E'$ of $F$ we have $E' \subset E$, which implies $E' \sqsupset^{\rGr}_{F} E$. But this contradicts $E \in \maxpl_{\rGr}(F)$. Hence, $E$ is the grounded extension.

    \item[``$\gr$-completeness'':] Suppose $E$ is the grounded extension of $F$ and suppose $E' \sqsupset^{\rGr}_{F} E$. Then since $E$ is complete, $E'$ is also complete. Hence, we have $E' \subset E$. But this is impossible as it implies $E$ is not the grounded extension. Hence, $E \in \maxpl_{\rGr}(F)$. \qedhere
\end{description}
\end{proof}

\begin{proposition}
    $\rGr$ satisfies composition and decomposition.
\end{proposition}
\begin{proof}
    Follows from Lemma \ref{lem:composition_decomposition_2} together with Definition \ref{defn:gr-ranking-semantics}.
\end{proof}

\begin{proposition}
   $\rGr$ satisfies strong reinstatement. 
\end{proposition}
\begin{proof}
    This follows directly from Proposition \ref{prop:rCO_reinstatement}.
\end{proof}

\begin{proposition}
    $\rGr$ violates addition robustness. 
\end{proposition}
\begin{proof}
    Since, $\sqsupseteq^{\rGr}_{F}$ is based on $\sqsupseteq^{\rCo}_{F}$ we can use Example \ref{ex:co_violates_robustness} to show that  $\rGr$ violates addition robustness as well.
    \end{proof}

\subsubsection{Technical Proofs of Section \ref{subsubsec:r-pr}}
%\subsubsection*{\ref{subsubsec:r-pr}. Preferred Extension-ranking Semantics}

\begin{proposition}
$\rPr$ and $\rCoPr$ are satisfying $\pr$-generalisation.
\end{proposition}
\begin{proof}
Let $F= (A,R)$ be an AF and $E\subseteq A$.
We start with  $\rPr$. 
\begin{description}
   \item[``$\pr$-soundness'':] Suppose $E \in \maxpl_{\rPr}(F)$. Then $E$ is admissible. Suppose $E$ is not preferred. Then there is an admissible $E'$ such that $E \subset E'$. But this is a contradiction to $E \in \maxpl_{\rPr}(F)$. Hence, $E$ is a preferred extension.

    \item[``$\pr$-completeness'':] Suppose $E$ is a preferred extension of $F$ and suppose $E' \sqsupset^{\rPr}_{F} E$. Then since $E$ is admissible, $E'$ is also admissible. Hence, we have $E \subset E'$. But this is impossible as it implies that $E$ is not a preferred extension. Hence, $E \in \maxpl_{\rPr} (F)$. 
\end{description}
Next, we look at $\rCoPr$.
\begin{description}
   \item[``$\pr$-soundness'':] Suppose $E \in \maxpl_{\rCoPr}(F)$. Then $E$ is complete. Suppose $E$ is not preferred. Then there is an complete $E'$ such that $E \subset E'$. But this is a contradiction to $E \in \maxpl_{\rCoPr}(F)$. Hence, $E$ is a preferred extension.

    \item[``$\pr$-completeness'':] Suppose $E$ is a preferred extension of $F$ and suppose $E' \sqsupset^{\rCoPr}_{F} E$. Then since $E$ is complete, $E'$ is also complete. Hence, we have $E \subset E'$. But this is impossible as it implies that $E$ is not a preferred extension. Hence, $E \in \maxpl_{\rCoPr} (F)$.  \qedhere
\end{description}
\end{proof}

\begin{proposition}
    $\rPr$ and $\rCoPr$ are satisfying composition and decomposition.
\end{proposition}
\begin{proof}
    Follows from Lemma \ref{lem:composition_decomposition_2} together with Definition \ref{defn:pr-ranking-semantics} respectively Definition \ref{defn:co-pr-ranking-semantics}.
\end{proof}

\begin{proposition}
   $\rPr$ and $\rCoPr$ are satisfying strong reinstatement. 
\end{proposition}
\begin{proof}
We only show it for $\rPr$, however the proof for $\rCoPr$ is similar. 

    Let $a \in \mathcal{F}_{F}(E)$, $a \notin E$ and $a \notin (E^- \cup E^+)$. Then Proposition \ref{prop:rAD_reinstatement} implies, that $\{a\} \cup E \sqsupseteq^{\rPr}_{F} E$. It also hold that $E \subset E \cup \{a\}$. Hence, $E \cup \{a\} \sqsupset^{\rPr}_{F} E$. So the preferred extension-ranking semantics does satisfy strong reinstatement.
\end{proof}

\subsubsection{Technical Proofs of Section \ref{subsubsec:r-sst}}
%\subsubsection*{\ref{subsubsec:r-sst}. Semi)-Stable Extension-ranking Semantics}

\begin{proposition}
    Let $F=(A,R)$ be an AF s.t. $\st(F) \neq \emptyset$, then $\maxpl_{\rSst}(F) = \st(F)$.
\end{proposition}
\begin{proof}
   Let $F=(A,R)$ be an AF s.t. $\st(F) \neq \emptyset$. 
   \begin{description}
       \item[``$\subseteq$'':]
        Let $E \in \maxpl_{\rSst}(F)$ and $E' \in \st(F)$. This means that $\CF_F(E)= \UD_F(E) = \DN_F(E)= \emptyset$ and $E' \in \co(F)$ implying $\CF_F(E')= \UD_F(E') = \DN_F(E')= \emptyset$ as well. Additionally for $E'$ we know that $E'$ attacks everything not inside, hence $\UA_F(E')= \emptyset$, so in order to have $E \in \maxpl_{\rSst}(F)$ it has to hold that $\UA_F(E) = \emptyset$ as well, meaning that $E$ is conflict-free and attacks everything not inside, implying that $E$ is a stable extension.  
       \item[``$\supseteq$'':] Let $E \in \st(F)$, this means that $E$ is conflict-free and attacks every argument not in $E$, i.\,e. $\CF_F(E)= \UA_F(E) = \emptyset$. Additionally, we know that $E$ is complete therefore $\UD_F(E)= \DN_F(E)= \emptyset$. So, there can not be any set be ranked strictly more plausible to be accepted than $E$ wrt. \rSst, i.e. $E \in \maxpl_{\rSst}(F)$. \qedhere
   \end{description}
\end{proof}

\begin{proposition}
    $\rSst$ satisfies $\sst$-generalisation.
\end{proposition}
\begin{proof} Let $F= (A,R)$ be an AF and $E\subseteq A$.
\begin{description}
    \item[``$\sst$-soundness'':] Suppose $E \in \maxpl_{\rSst}(F)$. Then $E$ is complete. Suppose $E$ is not semi-stable. Then there is a complete $E'$ s.t. $E \cup E^+ \subset E' \cup E'^{+}$. But this implies $E' \sqsupset^{\UA}_{F} E$, contradicting $E \in \maxpl_{\rSst}(F)$. Hence, $E$ is a semi-stable extension.

    \item[``$\sst$-completeness'':] Suppose $E$ is a semi-stable extension of $F$ and suppose $E' \sqsupset^{\rSst}_{F} E$. Then since $E$ is complete, $E'$ is also complete. Hence, we have $E' \sqsupset^{\UA}_{F} E$. But this is impossible as it implies that $E \cup E^{+} \subset E' \cup E'^{+}$, which means that $E$ is not a semi-stable extension. Hence, $E \in \maxpl_{\rSst}(F)$. \qedhere
\end{description}
\end{proof}

\begin{proposition}
    $\rSst$ satisfies composition and decomposition.
\end{proposition}
\begin{proof}
    Follows from Lemma \ref{lem:composition_decomposition_2} together with Definition \ref{defn:ss-ranking-semantics}.
\end{proof}

\begin{proposition}
   $\rSst$ satisfies strong reinstatement. 
\end{proposition}
\begin{proof}
    This follows directly from Proposition \ref{prop:rCO_reinstatement}.
\end{proof}

\begin{proposition}
    $\rSst$ violates addition robustness. 
\end{proposition}
\begin{proof}
    Since, $\sqsupseteq^{\rSst}_{F}$ is based on $\sqsupseteq^{\rCo}_{F}$ we can use Example \ref{ex:co_violates_robustness} to show that  $\rSst$ violates addition robustness as well.
    \end{proof}

\subsubsection{Technical Proofs of Section \ref{subsubsec:cardinality_ext_ranking}}
%\subsection*{\ref{subsubsec:cardinality_ext_ranking}. Cardinality-based Instances}
    \begin{proposition}
  Let $F= (A,R)$ be an AF and $E,E' \subseteq A$, if $E \sqsupseteq^{\mathsf{\mathsf{r\text{-}\sigma}}}_F E'$ then $E \sqsupseteq^{\mathsf{r\text{-}c\text{-}\sigma}}_F E'$ for $\sigma \in \{\ad,\co,\gr,\pr,\sst\}$.
\end{proposition}
\begin{proof}
    Let $F= (A,R)$ be an AF and $E,E' \subseteq A$, assume $E \sqsupseteq^{\mathsf{r\text{-}\sigma}}_F E'$ for $\sigma  \in \{\ad,\co,\sst\}$, then there is one $\tau \in \{\CF,\UD,\DN, \UA\}$ s.t. $\tau_F(E) \subseteq \tau_F(E')$, this subset behaviour also entails $|\tau_F(E)| \leq |\tau_F(E')|$, so $E \sqsupseteq^{\mathsf{r\text{-}c\text{-}\sigma}}_F E'$.

    For $\sigma \in \{\gr,\pr\}$ we still use $\sqsupseteq^{\Min}$ respectively $\sqsupseteq^{\Max}$ so there is no difference between the two variations.   
\end{proof}

\begin{proposition}
    $\mathsf{r\text{-}c\text{-}\sigma}$ satisfies $\sigma$-generalisation $\sigma \in \{\ad,\co,\gr,\pr,\sst\}$.
\end{proposition}
\begin{proof}
For any $\tau \in \{\CF,\UD,\DN,\UA\}$ the best value a set can have is $\emptyset$, for the cardinality-based versions of these relations the best value to be returned is $0$ and only the empty set can have cardinality of $0$. 
    Hence, we know that if and only if $\tau$ returns $\emptyset$ the cardinality-based versions of $\mathsf{c\text{-}\tau}$ returns $0$. Thus, we can use the same proof ideas of $\mathsf{r\text{-}\sigma}$ to show that $\mathsf{r\text{-}c\text{-}\sigma}$ satisfies $\sigma$-generalisation. 
\end{proof}

\begin{proposition}
    $\mathsf{r\text{-}c\text{-}\sigma}$ satisfies composition for $\sigma \in \{\ad,\co,\gr,\pr,\sst\}$.
\end{proposition}
\begin{proof}
    Let $F_1=(A_1,R_1)$, $F_2=(A_2,R_2)$ and $F = F_1 \cup F_2= (A_1, R_1) \cup (A_2,R_2)$ be AFs with $A_1 \cap A_2 = \emptyset$. We know that $|\tau_{F_1}(E\cap A_1)| + |\tau_{F_2}(E \cap A_2)| = |\tau_F(E)|$ for any base relation $\tau \in \{\CF, \UD, \DN, \UA\}$. Assume $E \cap A_1 \sqsupseteq^{\mathsf{r\text{-}c\text{-}\sigma}}_{F_1} E' \cap A_1$ and $E \cap A_2 \sqsupseteq^{\mathsf{r\text{-}c\text{-}\sigma}}_{F_2} E' \cap A_2$. So, there is at least one $\tau' \in \{\CF, \UD, \DN,\\ \UA\}$  s.t. $|\tau_{F_1}'(E\cap A_1)| \leq |\tau_{F_1}'(E'\cap A_1)|$ and  $|\tau_{F_2}'(E\cap A_2)| \leq |\tau_{F_2}'(E'\cap A_2)|$. Therefore, $|\tau_{F_1}'(E\cap A_1)| + |\tau_{F_2}'(E\cap A_2)| \leq |\tau_{F_1}'(E'\cap A_1)| + |\tau_{F_2}'(E'\cap A_2)|$ and this entails $|\tau_F'(E)| \leq |\tau_F'(E')|$. Hence, composition is satisfied.
\end{proof}

\begin{proposition}
     $\mathsf{r\text{-}c\text{-}\ad}$ satisfies weak reinstatement and violates strong reinstatement.
\end{proposition}
\begin{proof}
    Let $F= (A,R)$ be an AF and $E \subseteq A$. Suppose $a \in \mathcal{F}_{F}$, $a \notin E$, and $a \notin E^- \cup E^+$. In Proposition \ref{prop:rAD_reinstatement}, we have shown that $\UD_F(E \cup \{a\}) \subseteq \UD_F(E)$ and therefore also $|\UD_F(E \cup \{a\})| \leq |\UD_F(E)|$. This proves that $\mathsf{r\text{-}c\text{-}\ad}$ satisfies weak reinstatement.

    For the violation of strong reinstatement we can use Example \ref{ex:AD_not_reinstatment} again. 
\end{proof}

\begin{proposition}
$\mathsf{r\text{-}c\text{-}\sigma}$ satisfies strong reinstatement for $\sigma \in \{\co,\gr,\pr,\sst\}$.
\end{proposition}
\begin{proof}
    Let $F= (A,R)$ be an AF and $E \subseteq A$. Suppose $a \in \mathcal{F}_{F}$, $a \notin E$, and $a \notin E^- \cup E^+$. In Proposition \ref{prop:rCO_reinstatement}, we have shown that $\DN_F(E \cup \{a\}) \subset \DN_F(E)$. This also shows that $|\DN_F(E \cup \{a\})| < |\DN_F(E)|$. Hence, $\mathsf{r\text{-}c\text{-}\sigma}$ satisfies strong reinstatement and because $E \cup \{a\} \sqsupset^{\mathsf{r\text{-}c\text{-}\sigma}}_{F} E$ this also shows the satisfaction of strong reinstatement for the remaining semantics.
\end{proof}

\begin{proposition}
   $\mathsf{r\text{-}c\text{-}\ad}$ satisfies addition robustness.
\end{proposition}
\begin{proof}
    Let $F= (A,R)$ be an AF and $E, E' \subseteq A$ with $E \sqsupseteq^{\mathsf{r\text{-}c\text{-}\ad}}_{F} E'$. Let $F' = (A, R \cup \{(a,b)\})$ for $a \in E$ and $b \notin E$, $b \in E'$. We know that the addition of $(a,b)$ does not add any new conflict into $E$ and does not remove any conflict from $E'$, therefore $|\CF_F(E)|= |\CF_{F'}(E)|$ and $|\CF_F(E')| \leq |\CF_{F'}(E')|$. Hence, addition robustness holds for $\sqsupseteq^{c\text{-}\CF}$.
    
    It remains to show, that addition robustness holds for $\sqsupseteq^{c\text{-}\UD}$ as well. We already know that no defence in $E$ is removed therefore $|\UD_{F'}(E)| \leq |\UD_F(E)|$ and we also know that $E'$ can not defend more arguments in $F'$ than in $F$, otherwise additional conflicts have to be added. Hence, $E \sqsupseteq^{\mathsf{r\text{-}c\text{-}\ad}}_{F'} E'$.
\end{proof}

\subsection{Technical Proofs of Section \ref{sec:voting}}

\begin{proposition}%\label{prop:cope_extension-ranking}
Let $F=(A,R)$ be an AF. 
    $\sqsupseteq^{\cope(\tau_1,\dots, \tau_n)}_F$ is \emph{reflexive} and \emph{transitive} for any sequence of base relations $(\tau_1,\dots, \tau_n)$.
\end{proposition}
\begin{proof}
    Let $F=(A,R)$ be an AF and $(\tau_1,\dots, \tau_n)$ a sequence of base relations. 
    \begin{description}
        \item[``reflexive'':]  For any $E\subseteq A$ it has to hold that $E \equiv^{\cope(\tau_1,\dots, \tau_n)}_F E$. Let $w$ be the number of sets ranked worse than $E$ wrt. to $\sqsupseteq^{\tau_i}_F$ and $l$ be the number of sets ranked better than $E$  wrt. to  $\sqsupseteq^{\tau_i}_F$, then for $E \equiv^{\cope(\tau_1,\dots, \tau_n)}_F E$ to hold, we have to have $w - l = w - l$ for every $1 \leq i \leq n$, which is clear by definition. 
        \item[``transitive'':] Let $E,E',E'' \subseteq A$ and $E \sqsupset^{\cope(\tau_1,\dots, \tau_n)}_F E'$ and $E' \sqsupset^{\cope(\tau_1,\dots, \tau_n)}_F E''$, then $E$ has a better win/loss record than $E'$ and $E'$ has a better win/loss record than $E''$. Let $w$, $w'$, $w''$ be the corresponding number of wins minus the number of loses of $E$, $E'$ and $E''$ respectively. So, $w > w'$ and $w' > w''$. Then by definition $w > w''$, which shows that $E \sqsupset^{\cope(\tau_1,\dots, \tau_n)}_F E''$. \qedhere
        \end{description}
\end{proof}

\begin{proposition}
Let $F=(A,R)$ be an AF. 
    If $\sqsupseteq^\tau$ is total and transitive, then $\cope(\sqsupseteq^\tau_F)=\sqsupseteq^\tau_F$.
\end{proposition}
\begin{proof}
    Let $F=(A,R)$ be an AF and $E, E' \subseteq A$. Let $\sqsupseteq^\tau$ be a base relation that is total and transitive.    
    First we show that if $E \sqsupseteq^{\tau}_F E'$ then $ E \sqsupseteq^{\cope(\tau)}_F E'$. Assume $E \sqsupseteq^{\tau}_F E'$, then for every $E''$ s.t. $E' \sqsupseteq^\tau_F E''$ we know because of transitivity we also have $E \sqsupseteq^\tau_F E''$. So, $E$ wins at least as many times as $E'$. If $E'' \sqsupseteq^\tau_F E$, then it also holds that $E'' \sqsupseteq^\tau_F E'$, so $E'$ loses at least as often as $E$. Thus, $|\{\mathcal{E} \subseteq A | E \sqsupseteq^{\tau}_{F}  \mathcal{E}\}| - |\{\mathcal{E'} \subseteq A | \mathcal{E'} \sqsupseteq^{\tau}_{F} E\}| \geq |\{\mathcal{E} \subseteq A | E' \sqsupseteq^{\tau}_{F} \mathcal{E} \}| - |\{\mathcal{E'} \subseteq A | \mathcal{E'} \sqsupseteq^{\tau}_{F} E'\}|$ implying $E \sqsupseteq^{\cope(\tau)}_F E'$.

    Next, we show the other direction. Let $E \sqsupseteq^{\cope(\tau)}_F E'$, and assume for contrary $E' \sqsupset^{\tau}_F E$, similarly to what we showed above this implies $E' \sqsupseteq^{\cope(\tau)}_F E$ and because of transitivity and totality of $\tau$ we know that the relationship has to be strict i.e. $E' \sqsupset^{\cope(\tau)}_F E$. Thus  $E \sqsupseteq^{\cope(\tau)}_F E'$ can only imply $E \sqsupseteq^{\tau}_F E'$ or $E \asymp^{\tau}_F E'$. However the second option is not possible since $\sqsupseteq^\tau_F$ is total. So, $E \sqsupseteq^{\cope(\tau)}_F E'$ implies  $E \sqsupseteq^{\tau}_F E'$.
\end{proof}

\begin{proposition}
    $\cope(\sqsupseteq^\CF,\sqsupseteq^\UD)$ satisfies $\ad$-generalisation.
    \end{proposition}
    \begin{proof}
        Let $F=(A,R)$ be an AF and $E \subseteq A$.
        \begin{description}
            \item[``$\ad$-completeness'':] Suppose $E$ is admissible, then $\CF_F(E)=\UD_F(E)=\emptyset$ and therefore $E \sqsupseteq^{\CF}_F E'$ and $E \sqsupseteq^{\UD}_F E'$ for any $E' \subseteq A$. So $E$ wins against every set of arguments and $E$ only loses against sets $\mathcal{E}$ with $\CF_F(\mathcal{E}) = \emptyset$ respectively $\UD_F(\mathcal{E})= \emptyset$. However, all these sets are also losing against $E$, so there is no set with less loses than $E$. $E$ wins the maximal amount of times and no set loses less than $E$, so there can not be any set with a better win/ loss record than $E$ showing $E \in \maxpl_{\cope(\CF,\UD)}(F)$.
            \item[``$\ad$-soundness'':] Suppose $E \in \maxpl_{\cope(\CF,\UD)}(F)$, so $E \sqsupseteq^{\cope(\CF,\UD)}_F \emptyset$. We know that $\CF(\emptyset)= \UD(\emptyset) = \emptyset$, which means that $\emptyset$ wins the maximal amount of times and there is no set which loses less than $\emptyset$, i.\,e., $\emptyset \in  \maxpl_{\cope(\CF,\UD)}(F)$. So, for $E \in \maxpl_{\cope(\CF,\UD)}(F)$ to hold $E$ has to have the same win/ loss record as $\emptyset$, which is only possible if $\CF(E)=\UD(E)= \emptyset$, which shows that $E$ is admissible. \qedhere
        \end{description}
    \end{proof}

    \begin{proposition}
     $\cope(\sqsupseteq^\CF,\sqsupseteq^\UD)$ satisfies weak reinstatement.
\end{proposition}
\begin{proof}
    Let $F= (A,R)$ be an AF and $E \subseteq A$. Consider $a \in \mathcal{F}_F, a\notin E$ and $a \notin (E^- \cup E^+)$. First we look at $\sqsupseteq^\CF_F$ and $\sqsupseteq^\UD_F$ one by one. 
    \begin{description}
        \item[``$\sqsupseteq^\CF$'':] Since $a \notin (E^- \cup E^+)$ we have $\CF_F(E)= \CF_F(E\cup \{a\})$. So, if $E \sqsupseteq^\CF_F E'$ for any $E' \subseteq A$ then $E \cup \{a\} \sqsupseteq^\CF_F E'$ holds as well. Thus, the win/ loss record of $E$ and $E \cup \{a\}$ are the same. 
    \item[``$\sqsupseteq^\UD$'':] Since $a \in \mathcal{F}_F(E)$ we have $\UD(E)= \UD(E \cup \{a\})$. Thus again the win/ loss records of these two sets are equal. 
    \end{description}
    Next, let $w_{\CF}$ and $w_\UD$ be the win/ loss records of $E$ and $w'_\CF$ and $w'_\UD$ the win/ loss records of $E \cup \{a\}$. Since $w_\CF = w'_\CF$ and $w_\UD = w'_\UD$, we know that $w_\CF + w_\UD = w'_\CF + w'_\UD$. Thus, $E \cup \{a\} \equiv^{\cope(\CF,\UD)}_F E$.  
\end{proof}

\begin{proposition}
    ${\cope(\sqsupseteq^\tau)}$ satisfies strong reinstatement for $\tau \in \{\nonatt, \strdef\}$.
\end{proposition}
\begin{proof}
    Let $F= (A,R)$ be an AF and $E \subseteq A$. Consider $a \in \mathcal{F}_F, a\notin E$ and $a \notin (E^- \cup E^+)$.
    \begin{description}
        \item[``$\tau = \nonatt$'':] We show that if $E \sqsupseteq^{\nonatt}_F E'$ then $E\cup \{a\} \sqsupseteq^{\nonatt}_F E'$ for $E' \subseteq A$. Everything that is attacked by $E$ is also attacked by $E \cup \{a\}$, i.\,e., $E^+ \subseteq (E \cup \{a\})^+$. If $a$ is attacked by $E'$, then since $a \in \mathcal{F}_F(E)$ we know that the attacker of $a$ is attacked by $E$, so the addition of $a$ does not add any new attacked arguments so $E \cup \{a\} \sqsupseteq^{\nonatt}_F E'$ if $E  \sqsupseteq^{\nonatt}_F E'$ entailing that $E \cup \{a\}$ wins atleast as often as $E$. 
        
        Next, we show that $E$ loses at least as often as $E \cup \{a\}$. Assume $E' \sqsupseteq^{\nonatt}_F E\cup \{a\}$, then like discussed before every argument attacked by $E'$ in $E\cup \{a\}$ is also attacked by $E'$ in $E$ except $a$, however since $a$ is defended by $E$ this argument does not create any problem, so $E' \sqsupseteq^{\nonatt}_F E$.
        Thus the win/lose ratio of $E \cup \{a\}$ is at least as good as $E$ and also $E \cup \{a\}$ wins against $E$ so $E \cup \{a\} \sqsupset^{\cope(\nonatt)}_F E$. 
        
        \item[``$\tau = \strdef$'':] Similar to the $\sqsupseteq^{\nonatt}$ case every argument attacked by $E$ is also attacked by $E\cup\{a\}$, hence we can use the same reasoning as above. Thus the addition of $a$ can not lower the number of arguments strongly defended by $E \cup \{a\}$ from $E'$, so $E \cup \{a\} \sqsupseteq^{\cope(\strdef)}_F E'$ if $E \sqsupseteq^{\cope(\strdef)}_F E'$. Therefore the win/lose record of $E \cup \{a\}$ is atleast as good as the win/lose record of $E$ and since $a \notin E^-_F$ we know $a$ is not attacked by $E$ and therefore strongly defended by $E \cup \{a\}$ from $E$ implying $E \cup \{a\} \sqsupset^{{\strdef}}_F E$ and therefore also $E \cup \{a\} \sqsupset^{\cope(\strdef)}_F E$. \qedhere
     \end{description}
\end{proof}

\section{Technical Proofs of Section \ref{sec:additional combinations}}
\subsection{Technical Proofs of Section \ref{subsec:AR}}
\begin{proposition}
    Let $F=(A,R)$ be an AF, $E,E',E'' \subseteq A$, and $\rho$ be an argument-ranking semantics. If $E \sqsupseteq^{gc_\rho}_F E'$ and $E' \sqsupseteq^{gc_\rho}_F E''$ then $E \sqsupseteq^{gc_\rho}_F E''$.
\end{proposition}
\begin{proof}
    Let $F=(A,R)$ be an AF, $E,E',E'' \subseteq A$, and $\rho$ be an argument-ranking semantics. Assume $E \sqsupseteq^{gc_\rho}_F E'$ and $E' \sqsupseteq^{gc_\rho}_F E''$. Then for every argument $a \in E''$ there is an argument $b \in E'$ s.t. $b \succeq^\rho_F a$. Since $E \sqsupseteq^{gc_\rho}_F E'$ we know that for this argument $b$ there is an argument $c \in E$ s.t. $c \succeq^\rho_F b$ and therefore also $c \succeq^\rho_F a$. So it holds that $E \sqsupseteq^{gc_\rho}_F E''$.
\end{proof}

\begin{proposition}
    Let $F=(A,R)$ be an AF and $\rho$ an argument-ranking semantics, then $A \in \maxpl_{gc_{\rho}}(F)$.     
\end{proposition}
\begin{proof}
    Let $F=(A,R)$ be an AF and $\rho$ an argument-ranking semantics. W.l.o.g. assume $a \in A$ is the best ranked argument wrt. $\rho$, i.e. $a \succ^\rho_F b$ for any $b \in A$. Then for every set of arguments $E \subset A$ argument $a$ is ranked at least as good as every argument of $E$ and $a \in A$, therefore $A \sqsupseteq^{gc_{\rho}}_F E$. Hence,  $A \in \maxpl_{gc_{\rho}}(F)$.
\end{proof}

\begin{proposition}
    $gc_\rho$ satisfies composition if $\rho$ satisfies Independence.
\end{proposition}
\begin{proof}
    Let $F = F_1 \cup F_2= (A_1, R_1) \cup (A_2,R_2)$ with $A_1 \cap A_2 = \emptyset$ and argument-ranking semantics $\rho$.
    Assume 
    $E,E' \subseteq A_1 \cup A_2$ with $E \cap A_1 \sqsupseteq^{gc_\rho}_{F_1} E' \cap A_1$ and $E \cap A_2 \sqsupseteq^{gc_\rho}_{F_2} E' \cap A_2$. Then for all $b \in E' \cap A_1$ there is an argument $a \in E\cap A_1$ s.t. $a \succeq^\rho_{F_1} b$. If $\rho$ satisfies Independence, then $a \succeq^\rho_F b$ as well. Similar holds for every $b' \in E'\cap A_2$. So, for every argument $b \in E'$ there is an argument $a \in E$ s.t. $a \succeq^\rho_F b$, therefore $E \sqsupseteq^{gc_\rho}_F E'$. 
\end{proof}

\begin{proposition}
     ${gc_\rho}$ violates decomposition if $\rho$ satisfies Void Precedence, Non-attacked Equality, and Independence.
\end{proposition}
\begin{proof}
     Let $F = F_1 \cup F_2= (A_1, R_1) \cup (A_2,R_2)$ with $A_1 \cap A_2 = \emptyset$ and argument-ranking semantics $\rho$ satisfies Void Precedence, Non-attacked Equality and Independence.
     Let $a,c \in A_1$ and $b,d \in A_2$, and $a$ and $b$ are unattacked, i.e. $a^-_F = b^-_F = \emptyset$, and $c,d$ are attacked. Then Void Precedence and Non-attacked Equality implies $a \simeq^\rho_F b \succ^\rho_F c$ and $a \simeq^\rho_F b \succ^\rho_F d$, therefore $\{a,d\} \equiv^{gc_\rho}_F \{b,c\}$. 
     Independence implies $a \succ^\rho_{F_1} c$ and $b \succ^\rho_{F_2} d$ and therefore $\{a\} \sqsupseteq^{gc_\rho}_{F_1} \{c\}$ and  $\{b\} \sqsupseteq^{gc_\rho}_{F_2} \{d\}$. This shows that decomposition is violated. 
\end{proof}

\begin{proposition}
    $gc_\rho$ satisfies weak reinstatement.
\end{proposition}
\begin{proof}
 Let $F= (A,R)$ be an AF and $E \subseteq A$. 
    Suppose $a \in \mathcal{F}_{F}(E), a \notin E$ and $a \notin (E^- \cup E^+)$. Then there are two cases. First assume $a \succ^\rho_F b$ for every $b \in E$, then $E \cup \{a\} \sqsupset^{gc_\rho}_F E$. Next, assume there is a $b \in E$ s.t. $b \succ^\rho_F a$. Since $\rho$ is reflexive by definition we have for every argument $c \in A$, $c \succeq^\rho_F c$ and therefore we have $E \cup \{a\} \equiv^{gc_{\rho}}_F E$.   
\end{proof}

\begin{proposition}
     $gc_\rho$ satisfies syntax independence if $\rho$ satisfies Abstraction.
\end{proposition}
\begin{proof}
    Let $F= (A,R)$ and $F'=(A',R')$ be AFs such that there is a isomorphism $\gamma: A \rightarrow A'$ with $F' = \gamma(F)$. Then if $\rho$ satisfies Abstraction we have for all $a,b \in A$ s.t. $a \succeq^\rho_F b$ then $\gamma(a) \succeq^\rho_{\gamma(F)} \gamma(b)$. So, the underling argument rankings are the same for $F$ and $F'$, this implies that for two sets $E, E' \subseteq A$ with $E \sqsupseteq^{gc_{\rho}}_F E'$ then $\gamma(E) \sqsupseteq^{gc_{\rho}}_{\gamma(F)} \gamma(E')$. 
\end{proof}

\subsection{Technical Proofs of Section \ref{subsec:numerical_evaluation}}
%\subsubsection*{\ref{subsec:numerical_evaluation}. Numerical Evaluation Functions}

\begin{proposition}
    Let $F=(A,R)$ be an AF, then $A \in \maxpl_{\OBE_{ne_\sigma,\Box}}(F)$ for $\sigma \in \{\ad,\co,\gr,\pr,\sst\}$ and $\Box \in \{sum,max,leximax\}$.
\end{proposition}
\begin{proof}
    Let $F=(A,R)$ be an AF and $\sigma \in \{\ad,\co,\gr,\pr,\sst\}$. 
    \begin{description}
        \item[``$\Box = sum$'':] Since $ne_{\sigma,F}(a) \geq 0$ for every $a \in A$, we know that $\Sigma_{a \in A} ne_{\sigma,F}(a) \geq \Sigma_{b \in E} ne_{\sigma,F}(b)$ for any $E \subset A$, therefore $A \in \maxpl_{\OBE_{ne_\sigma,sum}}(F)$.
        \item[``$\Box = max$'':] Let $a$ be $max(ne_{\sigma,F}(A))$, then there can not be an argument $b\in A$ with $ne_{\sigma,F}(b) > ne_{\sigma,F}(a)$ so for every $E \subseteq A$ we have $max(ne_{\sigma,F}(E)) \leq ne_{\sigma,F}(a)$ and therefore $A \in \maxpl_{\OBE_{ne_\sigma, max}}(F)$.
        \item[``$\Box = leximax$'':] Let $A = \{a_1, \dots, a_n\}$ be in descending order then for any $E \subset A$ there is at least one $a_i$ missing. Let $E = \{a_1, \dots, a_{i-1},a_{i+1},\dots, n\}$, then $ne_{\sigma,F}(a_i) \geq ne_{\sigma,F}(a_{i+1})$, which are compared after $i$ steps. In the $i+1$ step $ne_{\sigma,F}(a_{i+1})$ and $ne_{\sigma,F}(a_{i+2})$ are compared. At some point either the previous element is strictly bigger or $ne_{\sigma,F}(a_n)$ is compared to the blank element. Therefore $A \sqsupseteq^{\OBE_{ne_\sigma,leximax}} E$ and therefore  $A \in \maxpl_{\OBE_{ne_\sigma, leximax}}(F)$. \qedhere
    \end{description}
\end{proof}

\begin{lemma}%\label{lemma:lexmax_union}
    Let $V,W,X,Y$ be sequences in descending order s.t. $V \sqsupseteq^{leximax} W$ and $X \sqsupseteq^{leximax} Y$. It holds that $c_1 * V \cup c_2 *X \sqsupseteq^{leximax} c_1 * W \cup c_2* Y$ for $c_1, c_2 \in \mathbb{N}$.
\end{lemma}
\begin{proof}
 Let $V= (v_1,v_2, \dots), W= (w_1, w_2, \dots), X= (x_1, x_2, \dots),\\ Y= (y_1, y_2, \dots)$ be sequences in descending order s.t. $V  \sqsupseteq^{leximax} W$ and $X \sqsupseteq^{leximin} Y$. 

 First we can multiply the constants $c_1$ and $c_2$ with the individual sequences, i.e. $c_1 * V= (c_1* v_1,c_1*v_2, \dots), c_2 * W= (c_2* w_1, c_2 *w_2, \dots), c_1 * X= (c_1* x_1, c_1* x_2, \dots), c_2 * Y= (c_2* y_1, c_2* y_2, \dots)$. 

 We normalise the length of these sequences by appending $-\infty$ to shorter sequences, i.e. for two sequences $A = (a_1, \dots, a_{m_1}), B = (b_1, \dots, b_{m_2})$ s.t. $|A| > |B|$ we add $- \infty$ until both sequences have the same length $B' = (b_1,b_2,\dots, b_{m_2}, -\infty, \dots)$ s.t. $|A| = |B'|$. This addition does not change the order between $A$ and $B$. First assume $A \sqsupseteq^{leximax} B$, since $|A| > |B|$ these two sequences can not be equal. So, $A \sqsupset^{leximax} B$ and therefore there exists an $i$ s.t. $a_i = b_i$  for $i = 1, \dots, m$ and $a_{i+1} > b_{i+1}$. Since $B$ and $B'$ are equal up to position $m_2$, these pairs are also in $B'$, hence $A \sqsupset^{leximax} B'$. 
 Next assume $B \sqsupset^{leximax} A$ then there exists an $i$ s.t. $a_i = b_i$ for $i = 1, \dots m$ and $b_{i+1} > a_{i+1}$. For $b_j \in B$ we know that $b_j \in B'$ so this implies $B' \sqsupset^{leximax} A$.

We can assume w.l.o.g. $|c_1*V| = |c_1*W| = |c_2*X| = |c_2*Y|$. 
Let $c_1 * V \sqsupseteq^{leximax} c_1* W$, $c_2 * X \sqsupseteq^{leximax} c_2 * Y$ and $c_1 * V \cup c_2 * X = (u_1, u_2, \dots)$ and $c_1*W \cup c_2*Y = (z_1, z_2, \dots)$. We reformulate the definition of lexicographic order: Let $A= (a_1, a_2, \dots), B= (b_1, b_2, \dots)$ be sequences in descending order $A \sqsupseteq^{leximax} B \Leftrightarrow (a_1 > b_1) \lor (A^1_\downarrow \sqsupseteq^{leximax} B^1_\downarrow \land a_1 = b_1)$, where $A^i_\downarrow$ is the sequence where the first $i$ elements of $A$ are removed, i.e. $A^i_\downarrow = (a_{i+1}, \dots)$. 

We show $c_1 * V \cup c_2* X \sqsupseteq^{leximax} c_1 * W \cup c_2 * Y$ in four cases.
\begin{description}
    \item[Case: $c_1 * V \sqsupset^{leximax} c_1 *W $ and $c_2* X \sqsupset^{leximax} c_2*Y$: ] So, there exists an $i$ s.t. $i= 1, \dots m_1$, $c_1 *v_i = c_1* w_i$, and $c_1* v_{i+1} > c_1* w_{i+1}$ and there exists an $j$ s.t. $j= 1, \dots m_2$, $c_2* x_j = c_2* y_j$, and $c_2* x_{j+1} > c_2* y_{j+1}$, where $m_1,m_2 \in \mathbb{N}$. We can modify the sequences such that $m_1 = m_2$. W.l.o.g. assume $m_1 > m_2$, then we add $m_1 - m_2$ times $\infty$ at the beginning of $X$ and $Y$, such that the first differences in these sequences occur at $m_1$. This modification does not change the order, after this modification we have to make sure, that  $|V| = |W| = |X| = |Y|$ still holds.  

   The next step will be done via induction for $m_1 \rightarrow |V|$. 
    Case $m_1 = 1$: Then we have $c_1 * v_1 > c_1* w_1$ and $c_2* x_1 > c_2* y_1$. 
    If $u_1 = c_1* v_1$ and $z_1 = c_1* w_1$ or $u_1 = c_2* x_1$ and $z_1= c_2* y_1$, then it is clear that $u_1 > z_1$. 
    If $u_1 = c_1* v_1$ and $z_1 = c_2* y_1$ or $u_1 = c_2* x_1$ and $z_1= c_1* w_1$, then $c_1* v_1 \geq c_2* x_1 > c_2*y_1$ resp. $c_2 *x_1 \geq c_1* v_1 > c_1* w_1$.
    This shows that $u_1> z_1$ and therefore $V \cup X \sqsupseteq^{leximax} W \cup Y$.

    Assume that for all $m_1= 1, \dots, n$ it holds that $u_{m_1} = z_{m_1}$.

    Case $n \rightarrow n+1$: Since for all $i= 1, \dots, n$ we have $u_i = z_i$, it is enough to look at the $i+1$ elements, i.e. we use $(c_1*V \cup c_2*X)^i_\downarrow= (u_{i+1}, \dots)$ and $(c_1*W \cup c_2*Y)^i_\downarrow= (z_{i+1}, \dots)$. If $u_{i+1} \in c_1*V$ and $z_{i+1} \in c_1*W $ or $u_{i+1} \in c_2*X$ and $z_{i+1} \in c_2*Y$, then $u_{i+1} > z_{i+1}$ since $c_1* v_{i+1} > c_1* w_{i+1}$ resp. $c_2* x_{i+1} > c_2* y_{i+1}$. 
    If $u_{i+1} \in c_1*V$ and $z_{i+1} \in c_2*Y $ or $u_{i+1} \in c_2*X$ and $z_{i+1} \in c_1*W$, then we have $c_1* v_{i+1} \geq c_2* x_{i+1} > c_2* y_{i+1}$ resp. $c_2* x_{i+1} \geq c_1* v_{i+1} > c_1* w_{i+1}$. 
    This implies $c_1* V \cup c_2* X \sqsupseteq^{leximax} c_1* W \cup c_2* Y$. 

    \item[Case: $c_1*V \equiv^{leximax} c_1*W $ and $c_2*X \sqsupset^{leximax} c_2*Y$: ] So, there exists an $i$ s.t. $i= 1, \dots m_1$ $c_2*x_i = c_2*y_i$ and $c_2*x_{i+1} > c_2*y_{i+1}$ and for every $j= 1, \dots, |c_1 *V|$ we have $c_1*v_j = c_1*w_j$. We show $c_1 *V \cup c_2* X \sqsupseteq^{leximax}c_1*  W \cup c_2* Y$ via induction for $m \rightarrow |V|$.

    Case $m_1= 1$: Then we have $c_2*x_1 > c_2*y_1$. 
    If $u_1 = c_1*v_1$ and $z_1 = c_1*w_1$, then $u_1 = z_1$. 
    If $u_1 = c_2*x_1$ and $z_1 = c_2*y_1$, then $u_1 > z_1$.
    If $u_1 = c_1*v_1$ and $z_1 = c_2*y_1$, then $c_1*v_1 \geq c_2*x_1 > c_2*y_1$. 
    If $u_1 = c_2*x_1$ and $z_1 = c_1*w_1$, then $c_2*x_1 \geq c_1*v_1 = c_1*w_1$.
    So, $u_1 \geq z_1$ in all cases. 

    Assume that for all $m_1 = 1, \dots, n$ it holds that $u_{m_1} = z_{m_1}$.

    Case $n \rightarrow n+1$: Since for all $i= 1, \dots, n$ we have $u_i = z_i$, it is enough to look at the $i+1$ elements, i.e. we use $(c_1*V \cup c_2*X)^i_\downarrow= (u_{i+1}, \dots)$ and $(c_1*W \cup c_2*Y)^i_\downarrow= (z_{i+1}, \dots)$. 
    If $u_{i+1} = c_1*v_{i+1}$ and $z_{i+1} = c_1*w_{i+1}$, then $u_{i+1} = z_{i+1}$. 
    If $u_{i+1} = c_2*x_{i+1}$ and $z_{i+1} = c_2*y_{i+1}$, then $u_{i+1} > z_{i+1}$, since $c_2*x{i+1} > c_2*y_{i+1}$.
    If $u_{i+1} = c_1*v_{i+1}$ and $z_{i+1} = c_2*y_{i+1}$, then $c_1* v_{i+1} \geq c_2* x_{i+1} > c_2*y_{i+1}$. 
    If $u_{i+1} = c_2*x_{i+1}$ and $z_{i+1} = c_1*w_{i+1}$, then $c_2*x_{i+1} \geq c_1*v_{i+1} = c_1*w_{i+1}$.
     This implies $c_1*V \cup c_2*X \sqsupseteq^{leximax} c_2*W \cup c_2*Y$. 

     \item[Case: $c_1*V \sqsupset^{leximax} c_1*W $ and $c_2*X \equiv^{leximax} c_2*Y$: ] We can use the same reasoning as for case $c_1*V \equiv^{leximax} c_1*W $ and $c_2*X \sqsupset^{leximax} c_2*Y$. 
     \item[Case: $c_1*V \equiv^{leximax} c_1*W $ and $c_2*X \equiv^{leximax} c_2*Y$: ] Then for all $i = 1, \dots, |c_1*V|$ we have $c_1*v_i = c_1*w_i$ and $c_1*x_i = c_1*y_i$. 
     We show $c_1*V \cup c_2*X \sqsupseteq^{leximax} c_1*W \cup c_2*Y$ via induction for $m_1 \rightarrow |c_1*V|$.
Case $m_1= 1$: 
If $u_1 = c_1*v_1$ and $z_1 = c_1*w_1$ or  $u_1 = c_2*x_1$ and $z_1 = c_2*y_1$, then $u_1 = z_1$. 
    If $u_1 = c_1*v_1$ and $z_1 = c_2*y_1$ or $u_1 = c_2*x_1$ and $z_1 = c_1*w_1$, then $c_1*v_1 \geq c_2*x_1 = c_2*y_1$ resp. $c_1*x_1 \geq c_1*v_1 = c_1*w_1$. 

    Assume that for all $m_1 = 1, \dots, n$ it holds that $u_{m_1} = z_{m_1}$.
     Case $n \rightarrow n+1$: Since for all $i= 1, \dots, n$ we have $u_i = z_i$, it is enough to look at the $i+1$ elements, i.e. we use $(c_1*V \cup c_2*X)^i_\downarrow= (u_{i+1}, \dots)$ and $(c_1*W \cup c_2*Y)^i_\downarrow= (z_{i+1}, \dots)$.

    If $u_{i+1} = c_1*v_{i+1}$ and $z_{i+1} = c_1*w_{i+1}$ or  $u_{i+1} = c_2*x_{i+1}$ and $z_{i+1} = c_2*y_{i+1}$, then $u_{i+1} = z_{i+1}$. 
    If $u_{i+1} = c_1*v_{i+1}$ and $z_{i+1} = c_2*y_{i+1}$ or $u_{i+1} = c_2*x_{i+1}$ and $z_{i+1} = c_1*w_{i+1}$, then $c_1*v_{i+1} \geq c_2*x_{i+1} = c_2* y_{i+1}$ resp. $c_2* x_{i+1} \geq c_1*v_{i+1} = c_1*w_{i+1}$. 
     This implies $c_1*V \cup c_2*X \sqsupseteq^{leximax} c_1*W \cup c_2*Y$.  \qedhere
 \end{description}
\end{proof}

\begin{lemma}
    Let $V,W,X,Y$ be sequences in descending order s.t. $V \sqsupseteq^{leximin} W$ and $X \sqsupseteq^{leximin} Y$. It holds that $c_1*V \cup c_2*X \sqsupseteq^{leximin} c_1*W \cup c_2*Y$ for $c_1,c_2 \in \mathbb{N}$.
\end{lemma}

\begin{lemma}
    Let $F_1= (A_1, R_1)$ and $F_2= (A_2,R_2)$ be two AFs s.t. $A_1 \cap A_2 =\emptyset$ and $F = F_1 \cup F_2$ then it holds that $ne_{\sigma,F}(a) = (ne_{\sigma,F_1} * |\sigma(F_2)|) + (ne_{\sigma,F_2}(a) * |\sigma(F_1)|)$ for every $a \in A_1 \cup A_2$ and $\sigma \in \{\ad,\co, \gr, \pr,\st, \sst\}$.
\end{lemma}
\begin{proof}
     Let $F_1= (A_1, R_1)$ and $F_2= (A_2,R_2)$ be two AFs s.t. $A_1 \cap A_2 =\emptyset$ and $F = F_1 \cup F_2$ and  $\sigma \in \{\ad,\co,\pr,\st, \sst\}$. W.l.o.g. assume there is a $E \in \sigma(F_1)$ with $a \in E$. Consider any set $E'$ s.t. $E' \in \sigma(F_2)$. Important to note that $a \notin E'$. Every argument defended by $E$ in $F_1$ is still defended by $E$ in $F$, for $E'$ the same holds for their defended arguments in $F_2$. So, $E, E' \in \ad(F)$ and additionally since there can not be any attacks between $E$ and $E'$ we have $E \cup E' \in \ad(F)$. Hence we can combine $E$ with every admissible set of $F_2$ to get a new admissible set in $F$. Hence, $ne_{\ad,F}(a) = (ne_{\ad,F_1} * |\ad(F_2)|) + (ne_{\ad,F_2}(a) * |\ad(F_1)|)$. Note that $ne_{\ad,F_2}(a) = 0$. 

     The only arguments that are defended by $E$ in $F$ that were originally in $F_2$ are unattacked arguments. However these unattacked arguments have to be part of every complete extension of $F_2$ thus if $E \in \co(F_1)$ and $E' \in \co(F_2)$ then $E \cup E' \in \co(F)$. This implies $ne_{\co,F}(a) = (ne_{\co,F_1} * |\co(F_2)|) + (ne_{\co,F_2}(a) * |\co(F_1)|)$.

     For $E \in \gr(F_1)$ and $E' \in \gr(F_2)$, we know that every argument defended by $E$ in $F_1$ is also defended by $E$ in $F$, hence if $a \in \mathcal{F}_{F_1}(E)$ then $a \in \mathcal{F}_{F}(E)$. This implies $E \cup E' \in \gr(F)$ and therefore $ne_{\gr,F}(a) = (ne_{\gr,F_1} * |\gr(F_2)|) + (ne_{\gr,F_2}(a) * |\gr(F_1)|)$.

     If $E \in \pr(F_1)$ and $E' \in \pr(F_2)$, then we cannot add any argument into $E\cup E' $ without violation the admissibility of this set in $F$. Thus, $E \cup E'$ is a preferred extensions in $F$. Hence $ne_{\pr,F}(a) = (ne_{\pr,F_1} * |\pr(F_2)|) + (ne_{\pr,F_2}(a) * |\pr(F_1)|)$. 

     Any argument attacked by $E$ in $F_1$ is also attacked by $E$ in $F$ and every not attacked argument by $E$ in $F_1$ is also not attacked by $E$ in $F$. This implies $ne_{\sigma',F}(a) = (ne_{\sigma',F_1} * |\sigma'(F_2)|) + (ne_{\sigma',F_2}(a) * |\sigma'(F_1)|)$ for $\sigma' \in \{\st,\sst\}$. 
\end{proof}

\begin{proposition}
    $\OBE_{ne_\sigma, \Box}$ satisfies composition for $\sigma \in \{\ad,\co,\pr,\st, \sst\}$ and $\Box \in \{sum,max,leximax,min,leximin\}$.
\end{proposition}
\begin{proof}
    Let $F_1= (A_1, R_1)$ and $F_2= (A_2,R_2)$ be two AFs s.t. $A_1 \cap A_2 =\emptyset$, $F = F_1 \cup F_2$ and $\sigma \in \{\ad,\co, \gr, \pr,\st,\sst\}$
\begin{description}
    \item[``$\Box= sum$'':] Assume $E, E' \subseteq A_1 \cup A_2$ with $E \cap A_1 \sqsupseteq^{\OBE_{ne_\sigma, sum}}_{F_1} E' \cap A_1$ and  $E \cap A_2 \sqsupseteq^{\OBE_{ne_\sigma, sum}}_{F_2} E' \cap A_2$, then $\Sigma_{a \in E \cap A_1} ne_{\sigma,F_1}(a) \geq \Sigma_{b \in E' \cap A_1} ne_{\sigma,F_1}(b)$ and $\Sigma_{a' \in E \cap A_2} ne_{\sigma,F_2}(a') \geq \Sigma_{b' \in E' \cap A_2} ne_{\sigma,F_2}(b')$. 
    Because of Lemma \ref{lemma:ne_comp} we have 
    \begin{align*}
        \Sigma_E~ ne_{\sigma,F}(E) &=  \Sigma_{a \in E} (ne_{\sigma,F_1}(a) * |\sigma(F_2)|) + (ne_{\sigma,F_2}(a) * |\sigma(F_1)|)\\
        &= |\sigma(F_2)| *  \Sigma_{a \in E\cap A_1}~ ne_{\sigma,F_1}(a) + |\sigma(F_1)| *  \Sigma_{a' \in E \cap A_2}~ ne_{\sigma,F_2}(a') \\
        &\geq |\sigma(F_2)| *  \Sigma_{b \in E'\cap A_1}~ ne_{\sigma,F_1}(b) + |\sigma(F_1)| *  \Sigma_{b' \in E' \cap A_2}~ ne_{\sigma,F_2}(b') \\
        &= \Sigma_{E'}~ ne_{\sigma,F}(E') 
    \end{align*}
    So, we get $E \sqsupseteq^{\OBE_{ne_\sigma,sum}}_{F} E'$.
    
    \item[``$\Box= max$'':] Assume $E, E' \subseteq A_1 \cup A_2$ with $E \cap A_1 \sqsupseteq^{\OBE_{ne_\sigma, max}}_{F_1} E' \cap A_1$ and  $E \cap A_2 \sqsupseteq^{\OBE_{ne_\sigma, max}}_{F_2} E' \cap A_2$. 

    Let $max(ne_{\sigma,F}(E))= |\sigma(F_2)| *  ne_{\sigma,F_1}(a) +  |\sigma(F_1)| *  ne_{\sigma,F_2}(a) $ and \\$max(ne_{\sigma,F}(E'))= |\sigma(F_2)| *  ne_{\sigma,F_1}(b) +  |\sigma(F_1)| *  ne_{\sigma,F_2}(b)$. If $a \in A_1$ and $b \in A_1$ then we compare $|\sigma(F_2)| *  ne_{\sigma,F_1}(a)$ and $|\sigma(F_2)| *  ne_{\sigma,F_1}(b)$. Since $|\sigma(F_2)| \in \mathbb{N}$ we know that $ne_{\sigma,F_1}(a)= max(ne_{\sigma,F_1}(E \cap A_1))$ and therefore $ne_{\sigma,F_1}(a)\geq ne_{\sigma,F_1}(b)$ so $E \sqsupseteq^{\OBE_{ne_{\sigma},max}}_F E'$.

    If $a \in A_2$ and $b \in A_2$ we can use the same reasoning as above.

    For $a \in A_1$ and $b \in A_2$ we have $|\sigma(F_2)|* ne_{\sigma,F_1}(a) \geq |\sigma(F_1)| *ne_{\sigma,F_2}(a')$ for every $a' \in E \cap A_2$. Let $ne_{\sigma,F_2}(a') = max(ne_{\sigma,F_2}(E \cap A_2))$ then we have $|\sigma(F_1)| * ne_{\sigma,F_2}(a') \geq |\sigma(F_1)| * ne_{\sigma,F_2}(b)$ and therefore also  $|\sigma(F_2)| * ne_{\sigma,F_1}(a) \geq |\sigma(F_1)| * ne_{\sigma,F_2}(b)$. So $E \sqsupseteq^{\OBE_{ne_{\sigma}},max}_F E'$.

    For $a \in A_2$ and $b \in A_1$ we can use the same reasoning as for $a \in A_1$ and $b \in A_2$.
\item[``$\Box= leximax$'':] Assume $E, E' \subseteq A_1 \cup A_2$ with $E \cap A_1 \sqsupseteq^{\OBE_{ne_\sigma, leximax}}_{F_1} E' \cap A_1$ and  $E \cap A_2 \sqsupseteq^{\OBE_{ne_\sigma, leximax}}_{F_2} E' \cap A_2$. So we can use Lemma \ref{lemma:lexmax_union} to show $E \sqsupseteq^{\OBE_{ne_{\sigma}},leximax}_F E'$.

    \item[``$\Box= min$'':] Assume $E, E' \subseteq A_1 \cup A_2$ with $E \cap A_1 \sqsupseteq^{\OBE_{ne_\sigma, min}}_{F_1} E' \cap A_1$ and  $E \cap A_2 \sqsupseteq^{\OBE_{ne_\sigma, min}}_{F_2} E' \cap A_2$. 
       Let $min(ne_{\sigma,F}(E))= |\sigma(F_2)| *  ne_{\sigma,F_1}(a) +  |\sigma(F_1)| *  ne_{\sigma,F_2}(a) $ and $min(ne_{\sigma,F}(E'))= |\sigma(F_2)| *  ne_{\sigma,F_1}(b) +  |\sigma(F_1)| *  ne_{\sigma,F_2}(b)$. If $a \in A_1$ and $b \in A_1$ then we compare $|\sigma(F_2)| *  ne_{\sigma,F_1}(a)$ and $|\sigma(F_2)| *  ne_{\sigma,F_1}(b)$. Since $|\sigma(F_2)| \in \mathbb{N}$ we know that $ne_{\sigma,F_1}(a)= min(ne_{\sigma,F_1}(E \cap A_1))$ and therefore $ne_{\sigma,F_1}(a)\geq ne_{\sigma,F_1}(b)$ so $E \sqsupseteq^{\OBE_{ne_{\sigma},min}}_F E'$.

    If $a \in A_2$ and $b \in A_2$ we can use the same reasoning as above.

    For $a \in A_1$ and $b \in A_2$ we have $|\sigma(F_1)|* ne_{\sigma,F_2}(b) \leq |\sigma(F_2)| *ne_{\sigma,F_1}(b')$ for every $b' \in E \cap A_1$. Let $ne_{\sigma,F_1}(b') = min(ne_{\sigma,F_1}(E' \cap A_1))$ then we have $|\sigma(F_2)| * ne_{\sigma,F_1}(a) \geq |\sigma(F_1)| * ne_{\sigma,F_2}(b')$ and therefore also  $|\sigma(F_2)| * ne_{\sigma,F_1}(a) \geq |\sigma(F_2)| * ne_{\sigma,F_1}(b)$. So $E \sqsupseteq^{\OBE_{ne_{\sigma}},min}_F E'$.

  \item[``$\Box= leximin$'':] Assume $E, E' \subseteq A_1 \cup A_2$ with $E \cap A_1 \sqsupseteq^{\OBE_{ne_\sigma, leximin}}_{F_1} E' \cap A_1$ and  $E \cap A_2 \sqsupseteq^{\OBE_{ne_\sigma, leximin}}_{F_2} E' \cap A_2$. So we can use Lemma \ref{lemma:lexmin_union} to show $E \sqsupseteq^{\OBE_{ne_{\sigma}},leximin}_F E'$. \qedhere
\end{description}
\end{proof}

\begin{proposition}
    For $\Box \in \{sum, max\}$  $\OBE_{ne_\sigma,\Box}$ satisfies weak reinstatement and for $\Box = leximax$ satisfies strong reinstatement with $\sigma \in \{\ad, \co, \gr,\\ \pr, \st, \sst\}$. 
\end{proposition}
\begin{proof}
    Let $F= (A,R)$ be an AF and $E \subseteq A$. Suppose $a \in \mathcal{F}_{F}(E), a \notin E$ and $a \notin (E^- \cup E^+)$. 
    \begin{description}
        \item[``$\Box= sum$'':] Since $ne_{\sigma,F}(a) \geq 0$, we know that $\Sigma_{b \in E \cup \{a\}} ne_{\sigma_F}(b) \geq \Sigma_{c \in E} ne_{\sigma,F}(c)$ and therefore $E \cup \{a\} \sqsupseteq^{\OBE_{ne_{\sigma}},sum}_F E$. If $ne_{\sigma,F}(a) = 0$ then $E \cup \{a\} \equiv^{\OBE_{ne_{\sigma}},sum}_F E$, then strong reinstatement is violated.
        \item[``$\Box = max$'':] If $max(ne_{\sigma,F}(E\cup\{a\})= ne_{\sigma,F}(a)$ then $ne_{\sigma,F}(a) \geq max(ne_{\sigma,F}(E))$ and so $E \cup \{a\} \sqsupseteq^{\OBE_{ne_{\sigma}},max}_F E$. 

        If $max(ne_{\sigma,F}(E\cup\{a\})\neq ne_{\sigma,F}(a)$ then $max(ne_{\sigma,F}(E \cup \{a\}) = max(ne_{\sigma,F}(E))$  and therefore $E \cup \{a\} \equiv^{\OBE_{ne_{\sigma}},max}_F E$.
      
        \item[``$\Box = leximax$'':] Let $ne_{\sigma,F}(E)= \{a_1,\dots,a_n\}$ in descending order, then $ne_{\sigma,F}(E \cup \{a\})= \{a_1,\dots, a_{i-1},a, a_{i},\\ \dots,a_n\}$. If we compare $ne_{\sigma,F}(E)$ and $ne_{\sigma,F}(E \cup \{a\})$ with $leximax$, then in the first $i-1$ steps the same numbers are compared. We know that $a \geq a_i$, so $E \cup \{a\} \sqsupseteq^{\OBE_{ne_{\sigma}},leximax}_F E$. Since $|E \cup \{a\}| > |E|$ we know that at some points $j \leq i$ we have $a_j > a_{j+1}$. So $E \cup \{a\} \sqsupset^{\OBE_{ne_{\sigma}},leximax}_F E$. \qedhere

    \end{description}
\end{proof}

\begin{proposition}
    $\OBE_{ne_\gr,\Box}$ for $\Box \in \{min,leximin\}$ satisfies weak reinstatement.   
\end{proposition}
\begin{proof}
    Let $F= (A,R)$ be an AF and $E \subseteq A$. Suppose $a \in \mathcal{F}_{F}(E), a \notin E$ and $a \notin (E^- \cup E^+)$. Since the grounded extension is unique, we know that there can only be one extension, thus $ne_{\gr,F}(b) \in \{0,1\}$ for every $b \in A$.  
    
    If $min(ne_{\gr,F}(E))= 1$ then since $a \in \mathcal{F}_F(E)$ we know that $ne_{\gr,F}(a)= 1$, hence $min(ne_{\gr,F}(E \cup \{a\}))= 1$ and therefore $min(ne_{\gr,F}(E \cup \{a\})) = min(ne_{\gr,F}(E))$. 
    If $ne_{\gr,F}(E) = 0$, then since $ne_{\gr,F}(a)\geq 0$ we know that $min(ne_{\gr,F}(E \cup \{a\}))= 0$ and therefore $min(ne_{\gr,F}(E \cup \{a\})) = min(ne_{\gr,F}(E))$. 
    This shows that $E \cup \{a\} \equiv^{\OBE_{ne_\gr,\Box}}_ F E$ for $\Box \in \{min,leximin\}$.
\end{proof}

\begin{proposition}
     $\OBE_{ne_\sigma,\Box}$ satisfies syntax independence for $\Box \in \{sum,max,\\min, leximax, leximin\}$.
\end{proposition}
\begin{proof}
    Let $F=(A,R)$, $F'=(A',R')$ be AFs and $E, E' \subseteq A$. Assume isomorphism $\gamma: A \rightarrow A'$, then $ne_{\sigma,F}(E) = ne_{\sigma,F'}(\gamma(E))$ and therefore if  $E \sqsupseteq^{\OBE_{ne_{\sigma}},\Box}_F E'$ then $\gamma(E) \sqsupseteq^{\OBE_{ne_{\sigma}},\Box}_{F'} \gamma(E')$.
\end{proof}

\begin{proposition}
     For AF $F=(A,R)$, $E,E' \subseteq A$, and $\rho$ a gradual semantics. Then $E \sqsupseteq^{\OBE_{\rho,\Box}}_{F} E'$ iff $E \sqsupseteq^{\OBE_{sv^{\rho},\Box}}_{F} E'$ for $\Box \in \{max, min,leximax, leximin\}$.
 \end{proposition}
 \begin{proof}
  Let AF $F=(A,R)$, $E,E' \subseteq A$, $\rho$ be a gradual semantics, and  $\Box \in \{max, min,leximax, leximin\}$. 
  Let $E= \{e_1,e_2,\dots, e_n\}$ and $E' = \{e_1',e_2', \dots, e_m'\}$. Their corresponding sequences are $\rho(E)= \{\rho(e_1), \rho(e_2), \dots, \\\rho(e_n)\}$ and $\rho(E')= \{\rho(e_1'), \rho(e_2'), \dots, \rho(e_m')\}$. 
  
  Assume $E \sqsupseteq^{\OBE_{\rho,\Box}}_{F} E'$. If $\rho(e_i) \geq \rho(e_j')$ then $e_i \succeq^\rho_F e_j'$ and therefore  $sv^\rho(e_i) \geq sv^\rho(e_j')$. Thus the relationships between each pair of arguments stay the same. 
  So, for $\Box = max$ it is clear that $max(\{\rho(e_1), \rho(e_2), \dots, \rho(e_n)\}) = max(\{sv^\rho(e_1), sv^\rho(e_2), \dots, sv^\rho(e_n)\})$. The same reasoning holds for $\Box \in \{min,leximin,leximax\}$. 

  Assume $E \sqsupseteq^{\OBE_{sv^{\rho},\Box}}_{F} E'$. If $sv^\rho(e_i) \geq sv^\rho(e_j')$ then $e_i \succeq^\rho_F e_j'$, so it had to hold that $\rho(e_i) \geq \rho(e_j')$. Thus the relationships between each pair of arguments stay the same. Implying $E \sqsupseteq^{\OBE_{\rho,\Box}}_{F} E'$.
   \end{proof}

\begin{proposition}
     Let $F=(A,R)$ be an AF and $\rho$ a total gradual semantics, then $\sqsupseteq^{\OBE_{\rho,max}}_F = \sqsupseteq^{gc_{\rho}}_F$.
 \end{proposition}
 \begin{proof}
      Let $F=(A,R)$ be an AF, $E,E' \subseteq A$, and $\rho$ a total gradual semantics. Assume $E \sqsupseteq^{\OBE_{\rho,max}}_F E'$, then $max(\rho(E)) \geq max(\rho(E'))$, so for every argument $a \in E'$ there is at least one argument $b \in E$ such that $\rho(b) \geq \rho(a)$ (this only holds if $\rho$ is total), in particular $\rho(b) \geq max(\rho(E'))$ and because of transitivity of $\rho$ we have $max(\rho(E)) \geq \rho(a)$. Therefore $E \sqsupseteq^{gc_{\rho}}_F E'$.

      Assume $E \sqsupseteq^{gc_{\rho}}_F E'$, then there for all $a \in E'$ there is a $b \in E$ s.t. $\rho(b) \geq \rho(a)$. Especially there is a $b' \in E$ s.t. $\rho(b') \geq max(\rho(E'))$ and because of the transitivity of $\rho$ we have $max(\rho(E)) \geq \rho(b') \geq max(\rho(E'))$. Therefore $E \sqsupseteq^{\OBE_{\rho,max}}_F E'$.
 \end{proof}

   \begin{proposition}
     Let $F= (A,R)$ be an AF, then $A \in \maxpl_{\OBE_{\rho,\Box}}(F)$ for $\Box \in \{sum,leximax\}$.
 \end{proposition}
 \begin{proof}
     Let $F= (A,R)$ be an AF and $\rho$ a gradual semantics.
     \begin{description}
         \item[``$\Box = sum$'':]  Since for gradual semantics the strength values of every argument $a \in A$ is in $[0,1]$, we know that $\rho(a) \geq 0$ and therefore $\Sigma_{a \in A}~ \rho(A) \geq \Sigma_{b \in E}~ \rho(b)$ for every $E \subseteq A$. Hence, $A \in \maxpl_{\OBE_{\rho,sum}}(F)$.
         \item[``$\Box = leximax$'':] Let $A = \{a_1,\dots, a_n\}$ be in descending order then for any $E \subset A$ there is at least one $a_i$ missing from $E$. Let $E = \{a_1,\dots,a_{i-1},\\a_{i+1,}, \dots, a_n\}$, then $\rho_{F}(a_i) \geq \rho_{F}(a_{i+1})$, which are compared after $i$ steps. In the $i+1$ step $\rho_{F}(a_{i+1})$ and $\rho_{F}(a_{i+2})$ are compared. At some point either the previous element is strictly bigger or $\rho_{F}(a_n)$ is compared to the blank element. Therefore $A \sqsupseteq^{\OBE_{\rho,leximax}}_F E$ and therefore  $A \in \maxpl_{\OBE_{\rho, leximax}}(F)$. \qedhere
         \end{description} 
 \end{proof}

 \begin{proposition}
     $\OBE_{\rho,\Box}$ satisfies composition for $\Box \in \{sum, max, leximax,\\ min, leximin\}$ if $\rho$ satisfies Independence for gradual semantics. 
\end{proposition}
\begin{proof}
   Let $F = F_1 \cup F_2= (A_1, R_1) \cup (A_2,R_2)$ with $A_1 \cap A_2 = \emptyset$ be an AF and  $\rho$ a gradual semantics that satisfies Independence for gradual semantics.
    \begin{description}
        \item[``$\Box = sum$'':]  Assume $E,E' \subseteq A_1 \cup A_2$ with $E \cap A_1 \sqsupseteq^{\OBE_{\rho,sum}}_{F_1} E' \cap A_1$ and $E \cap A_2 \sqsupseteq^{\OBE_{\rho,sum}}_{F_2} E' \cap A_2$. If $\rho$ satisfies Independence for gradual semantics, then $\rho_{F_1}(a) = \rho_F(a)$ for all $a \in A_1$ same holds for all $b \in A_2$. We have $\Sigma_{a \in E \cap A_1}~ \rho_{F_1}(a) \geq  \Sigma_{b \in E' \cap A_1}~ \rho_{F_1}(b)$ and $\Sigma_{a \in E \cap A_2}~ \rho_{F_2}(a) \geq  \Sigma_{b \in E' \cap A_2}~ \rho_{F_2}(b)$. Then we have $\Sigma_{a \in E \cap A_1}~ \rho_{F_1}(a) + \Sigma_{a' \in E \cap A_2}~ \rho_{F_2}(a') \geq  \Sigma_{b \in E' \cap A_1} ~\rho_{F_1}(b) + \Sigma_{b' \in E' \cap A_2}~ \rho_{F_2}(b')$ and this implies $\Sigma_{a \in E}~ \rho_F(a) \geq \Sigma_{b \in E'}~ \rho_{F}(b)$, hence $E \sqsupseteq^{\OBE_{\rho,sum}}_{F} E'$.
       
        \item[``$\Box= leximax$'':] Assume $E,E' \subseteq A_1 \cup A_2$ with $E \cap A_1 \sqsupseteq^{\OBE_{\rho,leximax}}_{F_1} E' \cap A_1$ and $E \cap A_2 \sqsupseteq^{\OBE_{\rho,leximax}}_{F_2} E' \cap A_2$. If $\rho$ satisfies Independence for gradual semantics, then $\rho_{F_1}(a) = \rho_F(a)$ for all $a \in A_1$ same holds for all $b \in A_2$. So, we can use Lemma \ref{lemma:lexmax_union} with $c_1 = c_2 = 1$ to show that $E \sqsupseteq^{\OBE_{\rho,leximax}}_{F} E'$. 
      
        \item[``$\Box = min$'':] Assume $E,E' \subseteq A_1 \cup A_2$ with $E \cap A_1 \sqsupseteq^{\OBE_{\rho,min}}_{F_1} E' \cap A_1$ and $E \cap A_2 \sqsupseteq^{\OBE_{\rho,min}}_{F_2} E' \cap A_2$. If $\rho$ satisfies Independence for gradual semantics, then $\rho_{F_1}(a) = \rho_F(a)$ for all $a \in A_1$ same holds for all $b \in A_2$. We know that $min(\rho(E\cap A_1)) > min(\rho(E' \cap A_1))$. W.l.o.g. let $min(\rho(E \cap A_1)) < min(\rho(E \cap A_2))$, so $min(\rho(E \cap A_1)) = min(\rho(E))$. $min(\rho(E\cap A_1)) > min(\rho(E' \cap A_1))$ and $min(\rho(E\cap A_2)) > min(\rho(E' \cap A_2))$. If $min(\rho(E' \cap A_1)) < min(\rho(E' \cap A_2))$ then we are already done, since $min(\rho(E\cap A_1)) > min(\rho(E' \cap A_1))$. If $min(\rho(E' \cap A_1)) > min(\rho(E' \cap A_2))$ , then $min(\rho(E\cap A_1)) > min(\rho(E' \cap A_1)) > min(\rho(E' \cap A_2))$. Therefore $min(\rho(E)) > min(\rho(E'))$. Hence, $E \sqsupseteq^{\OBE_{\rho,min}}_{F} E'$.
   
        \item[``$\Box= leximin$'':] Assume $E,E' \subseteq A_1 \cup A_2$ with $E \cap A_1 \sqsupseteq^{\OBE_{\rho,leximin}}_{F_1} E' \cap A_1$ and $E \cap A_2 \sqsupseteq^{\OBE_{\rho,leximin}}_{F_2} E' \cap A_2$. If $\rho$ satisfies independence for gradual semantics, then $\rho_{F_1}(a) = \rho_F(a)$ for all $a \in A_1$ same holds for all $b \in A_2$. So, we can use Lemma \ref{lemma:lexmin_union} with $c_1 = c_2 = 1$ to show that $E \sqsupseteq^{\OBE_{\rho,leximin}}_{F} E'$. \qedhere
    \end{description}
\end{proof}

 \begin{proposition}
    $\OBE_{\rho,\Box}$ violates decomposition for $\Box \in \{sum, leximax,\\ min, leximax\}$ if $\rho$ satisfies Independence for gradual semantics, Equivalence, and Maximality. 
\end{proposition}
\begin{proof}
Let $F = F_1 \cup F_2= (A_1, R_1) \cup (A_2,R_2)$ with $A_1 \cap A_2 = \emptyset$ be an AF and $\rho$ a gradual semantics that satisfies Independence for gradual semantics, Equivalence, and Maximality. Let $a,c \in A_1$ and $b,d \in A_2$, and $a$ and $b$ are unattacked, i.e. $a^-_F = b^-_F = \emptyset$, and $c,d$ are attacked. We know that $\rho_F(a) = \rho_F(b) > \rho_F(c)$ and $\rho_F(a) = \rho_F(b) > \rho_F(d)$. Consider the two sets $\{a,d\}$ and $\{b,c\}$, it is clear, that these two sets are comparable i.e.  $\{a,d\} \sqsupseteq^{\OBE_{\rho,\Box}}_F \{b,c\}$ or  $\{b,c\} \sqsupseteq^{\OBE_{\rho,\Box}}_F \{a,d\}$ for $\Box \in \{sum, leximax, min, leximax\}$. However, we have $\rho_F(a) = \rho_{F_1}(a) > \rho_F(c) = \rho_{F_1}(c)$ and $\rho_F(b) = \rho_{F_2}(b) > \rho_F(d) = \rho_{F_2}(d)$, thus $\{a\} \sqsupseteq^{\OBE_{\rho,\Box}}_{F_1} \{c\}$ and $\{b\} \sqsupseteq^{\OBE_{\rho,\Box}}_{F_2} \{d\}$ for all $\Box \in \{sum, leximax, min, leximax\}$. So, decomposition is violated. 
\end{proof}

 \begin{proposition}
     $\OBE_{\rho,\Box}$ satisfies weak reinstatement for $\Box = sum$ and strong reinstatement for $\Box = leximax$. 
 \end{proposition}
 \begin{proof}
     Let $F=(A,R)$ be an AF and $E \subseteq A$. Suppose $a \in \mathcal{F}_{F}(E), a \notin E$ and $a \notin (E^- \cup E^+)$. 
     \begin{description}
         \item[``$\Box = sum$'':] For gradual semantics we have $\rho(b) \geq 0$ for every argument $b \in A$, hence $\Sigma_{b \in E\cup \{a\}}~ \rho(b) \geq\Sigma_{c \in E}~ \rho(c)$. Therefore $E \cup \{a\} \sqsupseteq^{\OBE_{\rho,sum}}_F E$. If $\rho(a) = 0$, then $E \cup \{a\} \equiv^{\OBE_{\rho,sum}}_F E$ and therefore strong reinstatement is violated.
   %      \item[``$\Box = max$'':] If $max(\rho(E \cup \{a\}) = \rho(a)$ then $ \rho(a) \geq \rho(E)$ and therefore  $E \cup \{a\} \sqsupseteq^{\OBE_{\rho,max}}_F E$.
    %     $max(\rho(E \cup \{a\}) \neq \rho(a)$ implies $E \cup \{a\} \equiv^{\OBE_{\rho,max}}_F E$.
         \item[``$\Box = leximax$'':] Let $\rho(E \cup \{a\}) = \{a_1, \dots, a_{i-1}, a, a_{i}, \dots, a_n\}$ and $\rho(E) = \{a_1, \dots, a_{i-1},a_i,\dots, a_n\}$. If there is $a_j$ such that $\rho(a_j) > \rho(a_{j-1})$ for $j \geq i$, then $E \cup \{a\} \sqsupset^{\OBE_{\rho,leximax}}_F E$. If there is no such $a_j$, then since $|E \cup \{a\}| > |E|$ we have in the last step $n+1$ that $a_n$ from $\rho(E \cup \{a\})$ is compared with a blank symbol of $\rho(E)$ and therefore $E \cup \{a\} \sqsupset^{\OBE_{\rho,leximax}}_F E$. \qedhere
     \end{description}
 \end{proof}

 \begin{proposition}
    If gradual semantics $\rho$ satisfies Abstraction, then $\OBE_{\rho,\Box}$ satisfies syntax independence for $\Box \in \{sum, max, leximax, min, leximin\}$.
\end{proposition}
\begin{proof}
      Let $F=(A,R)$, $F'=(A',R')$ be AFs, $E, E' \subseteq A$, and $\rho$ a gradual semantics. Assume isomorphism $\gamma: A \rightarrow A'$. If $\rho$ satisfies Abstraction, then $\rho(a) = \rho(\gamma(a))$ for every $a \in A$ and therefore syntax independence is satisfied.   
\end{proof}

\end{document}